\documentclass[11pt]{article}
\usepackage[utf8]{inputenc}
\usepackage[margin=1in]{geometry}

\RequirePackage{amsthm,amsmath,amsfonts,amssymb}
\RequirePackage[authoryear]{natbib}

\usepackage{amsmath, amssymb, outlines,parskip, mathtools}
\usepackage{hyperref}
\usepackage[capitalise]{cleveref}
\Crefname{assumption}{Assumption}{Assumptions}
\Crefname{enumerate}{condition}{conditions}
 \usepackage{color}
\usepackage{booktabs,multirow}
\usepackage{nicefrac}
\usepackage{url}
\usepackage{algcompatible}
\usepackage{algorithm}
\usepackage{enumitem}
\usepackage{array}

\usepackage{tikz}
\usetikzlibrary{positioning,arrows,decorations,decorations.markings,decorations.pathreplacing,decorations.pathmorphing,shadows,positioning,arrows.meta,matrix,fit}
\allowdisplaybreaks
\usepackage{footnote}

\newcommand*\samethanks[1][\value{footnote}]{\footnotemark[#1]}

\title{Causal Inference Under Unmeasured Confounding With Negative Controls: A Minimax Learning Approach}
\author{\shortstack{
  \large Nathan Kallus$^{1}$\thanks{Authors are listed in alphabetical order by surname.}
  \hspace{1.1em} Xiaojie Mao$^{2}$\samethanks
  \hspace{1.1em} Masatoshi Uehara$^{1}$\samethanks
  \hspace{1.1em} Zhiheng Zhang$^{3}$\samethanks\\[1ex]
  \small
  $^{1}$Cornell Tech, Cornell University;\quad
  $^{2}$School of Economics and Management, Tsinghua University;\\
  \small $^{3}$School of Statistics and Data Science, Shanghai University of Finance and Economics
}}
\date{}

\usepackage{blindtext}
\usepackage{multirow}

\newenvironment{keywords}{\par\noindent\textbf{Keywords: }}{\par}

\RequirePackage{amsmath,amsfonts,amssymb}

\usepackage{amsmath, amssymb, outlines,parskip, mathtools}
\usepackage{hyperref}
\usepackage{aliascnt}
\usepackage[capitalise]{cleveref}
\Crefname{theorem}{theorem}{theorems}
\Crefname{theorem}{Theorem}{Theorems}
\Crefname{lemma}{lemma}{lemmas}
\Crefname{assumption}{Assumption}{Assumptions}
\Crefname{enumerate}{condition}{conditions}
\Crefname{lemma}{Lemma}{Lemmas}
\Crefname{corollary}{corollary}{corollaries}
\Crefname{corollary}{Corollary}{Corollaries}
\Crefname{proposition}{proposition}{propositions}
\Crefname{proposition}{Proposition}{Propositions}

\usepackage{amsthm}
\theoremstyle{plain}

\makeatletter\let\c@theorem\relax\makeatother

\newtheorem{theorem}{Theorem}

\newaliascnt{lemma}{theorem}
\newtheorem{lemma}[lemma]{Lemma}
\aliascntresetthe{lemma}

\newaliascnt{corollary}{theorem}
\newtheorem{corollary}[corollary]{Corollary}
\aliascntresetthe{corollary}

\newaliascnt{proposition}{theorem}
\newtheorem{proposition}[proposition]{Proposition}
\aliascntresetthe{proposition}
\theoremstyle{definition}

\makeatletter\let\c@assumption\relax\makeatother

\newtheorem{definition}{Definition}

\newtheorem{remark}{Remark}

\makeatletter\let\c@example\relax\makeatother

\newtheorem{example}{Example}

\usepackage{color}
\usepackage{booktabs,multirow}
\usepackage{nicefrac}
\usepackage{url}
\usepackage{algcompatible}
\usepackage{algorithm}
\usepackage{enumitem}
\usepackage{array}
\usepackage{tikz}
\usetikzlibrary{positioning,arrows,decorations,decorations.markings,decorations.pathreplacing,decorations.pathmorphing,shadows,positioning,arrows.meta,matrix,fit}
\allowdisplaybreaks
\usepackage{footnote}

\def\obrace{\iftrue{\else}\fi}
\def\cbrace{\iffalse{\else}\fi}

\newcommand{\Ncal}{\mathcal{N}}

\newcommand{\neweedit}{}

\newcommand{\Tcal}{\mathcal{T}}
\newcommand{\ie}{\textit{i.e.}}
\newcommand{\eg}{\textit{e.g.}}
\newcommand{\RNum}[1]{\uppercase\expandafter{\romannumeral #1\relax}}

\usepackage{amsmath}

\usepackage{wrapfig}
\DeclareMathOperator*{\argmax}{arg\,max}
\DeclareMathOperator*{\argmin}{arg\,min}

\newtheorem{assumption}{Assumption}

\usepackage{graphicx}

\def\Holder{{H\"{o}lder}}

\newcommand{\bU}{\mathbf{U}}

\newcommand{\Hcal}{\mathcal{H}}

\newcommand{\ts}{\textstyle}

\newcommand{\bW}{\mathbf{W}}
\newcommand{\bZ}{\mathbf{Z}}

\newcommand{\hP}{\mathbb{P}_n}

\newcommand{\E}{\mathbb{E}}

\newcommand{\pa}{\mathrm{\pa}}

\newcommand{\rI}{\mathrm{I}}

\newcommand{\dm}{\mathrm{REG}}

\newcommand{\rd}{\mathrm{d}}

\newcommand{\prns}[1]{\left(#1\right)}
\newcommand{\braces}[1]{\left\{#1\right\}}
\newcommand{\bracks}[1]{\left[#1\right]}

\newcommand{\abs}[1]{\left|#1\right|}
\newcommand{\Rl}{\mathbb{R}}
\newcommand{\R}[1]{\mathbb{R}^{#1}}
\newcommand{\epol}{\pi}

\newcommand{\DR}{\mathrm{DR}}

\newcommand{\Star}{\mathrm{star}}

\newcommand{\Rad}{\Rcal}
\newcommand{\ipw}{\mathrm{IPW}}
\newcommand{\DM}{\mathrm{REG}}
\newcommand{\dr}{\mathrm{DR}}

\newcommand{\magd}[1]{\left\|#1\right\|}

\newcommand{\Fcal}{\mathcal{F}}

\newcommand{\Scal}{\mathcal{S}}
\newcommand{\Acal}{\mathcal{A}}

\newcommand{\Abbb}{\mathbb{A}}

\newcommand{\Wbbb}{\mathbb{W}}
\newcommand{\Ib}{\mathbb{I}}
\newcommand{\Hbbb}{\mathbb{H}}
\newcommand{\Qbbb}{\mathbb{Q}}

\newcommand{\Ycal}{\mathcal{Y}}
\newcommand{\Xcal}{\mathcal{X}}
\newcommand{\Ucal}{\mathcal{U}}
\newcommand{\Wcal}{\mathcal{W}}
\newcommand{\RR}{\mathbb{R}}
\newcommand{\Qcal}{\mathcal{Q}}

\newcommand{\Dcal}{\mathcal{D}}
\newcommand{\Gcal}{\mathcal{G}}
\newcommand{\Rcal}{\mathcal{R}}

\newcommand{\Zcal}{\mathcal{Z}}

\renewcommand{\eqref}[1]{(\ref{#1})}

\newcommand{\RN}[1]{
  \textup{\uppercase\expandafter{\romannumeral#1}}
}

\usepackage{color,graphicx,lscape,longtable}

\def\boxit#1{\vbox{\hrule\hbox{\vrule\kern6pt\vbox{\kern6pt#1\kern6pt}\kern6pt\vrule}\hrule}}

\usepackage{xcolor}
\newcount\Comments
\newcommand{\kibitz}[2]{\ifnum\Comments=1\textcolor{#1}{#2}\fi}

\newcommand{\Eb}[1]{\mathbb{E}\left[#1\right]}
\newcommand{\indic}[1]{\mathbb{I}\left[#1\right]}
\newcommand{\Prb}[1]{\mathbb{P}\left[#1\right]}

\newcommand{\obs}{\mathrm{obs}}

 \newenvironment{continuance}[2][Example]
  {\par\noindent\textbf{Example #2, Cont'd} (#1).}
  {\par}

\begin{document}

\maketitle

\begin{abstract}
We study the estimation of causal parameters when not all confounders are observed and instead negative controls are available. Recent work has shown how these can enable identification and efficient estimation via two so-called bridge functions. In this paper, we tackle the primary challenges in causal inference with negative controls: identifying and estimating these bridge functions.
Previous work has typically relied on completeness conditions for identification, imposed uniqueness assumptions for estimation, and focused on parametric models for the bridge functions. Instead, we provide an alternative identification strategy that assumes the existence of two bridge functions rather than completeness conditions. We also provide new estimators for these functions based on minimax learning formulations. These estimators accommodate general function classes such as Reproducing Kernel Hilbert Spaces and neural networks. We study finite-sample convergence results both for estimating bridge functions themselves and for the final estimation of the causal parameter under a variety of combinations of assumptions. {We avoid uniqueness conditions on the bridge functions as much as possible.}
\end{abstract}


\section{Introduction}\label{sec: intro}

Causal inference from observational data is a necessity in many fields where experimentation and randomization is limited. Even when experimentation is feasible, observational data can help support initial or supplementary investigations. Compared to experimental-intervention data, the key difficulty with observational data is confounding or endogeneity: correlations between observed actions $A$ and outcomes $Y$ that are not due to a causal relationship, as might be induced by common causes such as a healthy lifestyle leading to both selection into a pharmaceutical therapy and good health outcomes regardless of therapy.
A common identification strategy is to control for many baseline covariates and assume they fully account for all such common causes, termed unconfoundedness, ignorability, or exchangeability. However, in practice, it is dubious that all confounders are ever truly accounted for, casting doubt on any resulting conclusion.

When some confounders are unobserved,
an alternative identification strategy is to use \emph{negative controls} \citep{miao2018identifying,miao2018a,tchetgen2020introduction,CuiYifan2020Spci,deaner2021proxy,shi2020multiply}, which play a similar role as instrumental variables. Negative controls are observed covariates that have a more restricted relationship with the action and outcome: \emph{negative control actions} do not directly impact the outcome of interest and \emph{negative control outcomes} are not directly impacted by either the negative control actions or the primary action of interest. See \Cref{fig:causal_dag} for a typical causal diagram.
When these negative control variables are sufficiently informative about the unmeasured confounders, there exist the so-called \emph{bridge functions} that enable identification and estimation of causal quantities. These bridge functions are analogues to the propensity score functions and outcome regression functions one would use if all confounders were observed.

Learning these bridge functions, however, is a nontrivial task, as it no longer amounts to a regression problem as in the unconfounded case, which can be outsourced to standard machine learning methods.
{ Many previous works focus on parametric settings to simplify the estimation of bridge functions.
For example,
\cite{CuiYifan2020Spci} recently proposed a doubly robust approach that may attain the semiparametric efficiency lower bound, but they focused on parametric estimates, which may be restrictive in practice.}
Moreover, {their analysis for the estimation relied on the bridge functions' \emph{unique} existence (although causal-effect identification itself does not require uniqueness), which may be dubious in practice and even refutable in many examples}.
It also requires certain completeness assumptions to achieve identification.
\cite{shi2020multiply} studied the setting with discrete variables of finitely many possible values (see \Cref{ex: bridge-discrete} below), so bridge functions are naturally parametric. Although they briefly mention two possible approaches to tackle nonunique bridge functions, these approaches are not fully developed, and the estimation procedure ultimately focuses on unique bridge functions. Completeness and uniqueness conditions are similarly required in other prior work \citep{miao2018identifying,miao2018a,tchetgen2020introduction,deaner2021proxy}.
In this paper, we propose an alternative identification strategy without completeness conditions and introduce new minimax estimators for the bridge functions that are amenable to general function approximation beyond parametric models. We catalog a variety of settings in which the functions are learnable and estimation with negative controls is practically feasible. In all estimation settings, we avoid assuming uniqueness of bridge functions as much as possible. 

Our contributions are:

\begin{itemize}
    \item
    We introduce a novel identification strategy, which requires the existence of two different types of bridge functions, but not any completeness condition. This provides an alternative to the previous identification strategies based on the existence of at least one type of bridge functions and some completeness conditions.
    Our identification assumptions are in general different from the existing ones, but in the discrete data setting and in some linear model settings they are strictly weaker than the existing identification assumptions.
    Our identification strategy complements the approaches in the existing literature, broadening the scope of causal identification with negative controls.

    We also consider a more general setting than average-effect estimation, where we allow
    actions to be possibly stochastic and the action space to be possibly continuous.
    \item We propose new estimators for bridge functions (even if nonunique) by introducing an adversarial critic function and formulating the learning problem as a minimax game. Our minimax approach accommodates a wide range of flexible function classes, including reproducing kernel Hilbert spaces (RKHSs) and neural networks.
    Then, by plugging these bridge function estimators into different estimating equations, we derive our estimators.

    \item

    {We provide finite-sample convergence results under a variety of different assumptions, generally without assuming unique bridge functions (see \Cref{tab:summary}).}

    One important assumption is the well-specification of the hypothesis classes, which we call \emph{realizability}. Another assumption, which we call \emph{closedness}, ensures that the critic function classes in our minimax estimators are sufficiently rich.
    Depending on how much of these we are willing to assume, we obtain different convergence rates for the final estimator. One surprising result is that when we assume realizability of both bridge and critic function classes, our estimator is consistent even though the bridge functions themselves are not consistently estimated (see row \RNum{1} in  \Cref{tab:summary}). {This result generally cannot be  obtained by standard analyses in the instrumental variable literature \citep[e.g.,][]{AiChunrong2003EEoM,newey2003instrumental,DikkalaNishanth2020MEoC}.
Thus, our finite-sample analysis exploits the special structure of double negative controls rather than simply importing existing instrumental-variables analyses.    }
\end{itemize}

{The rest of the paper is organized as follows. \Cref{sec: setup} introduces our setup. \Cref{sec:identity} presents our new identification result and compares it with previous work. \Cref{sec: reformuate-moment} introduces our estimation method. In \Cref{sec: realizable}--\ref{sec: closedness}, we present  finite-sample results for the  proposed estimators. These results are summarized at the end of \Cref{sec: reformuate-moment}. In \Cref{sec: literature}, we review related literature. In \Cref{sec:experiment}, we present numerical results from simulation studies and a real-data analysis.
In \Cref{sec: comparison}, we compare our proposed identification strategy with existing ones in detail.
All proofs are deferred to the appendix.
}

\section{Setup}\label{sec: setup}
We consider an action $A \in \Acal$ that can be discrete or continuous.
We associate $\Acal$ with a base measure $\mu$, \eg, the counting measure if $\Acal$ is finite or Lebesgue measure if $\Acal$ is continuous.
Let $Y(a)$ denote {the real-valued counterfactual outcome} that would be observed if the
action were set to $a \in \Acal$ and $Y = Y(A)$ be the observed outcome corresponding to the actually observed action. Moreover, let $X \in \Xcal \subseteq \R{d}$ be a collection of observed covariates.
For a given contrast function $\epol:\Acal\times\Xcal\mapsto\mathbb R$,
we are interested in estimating the \emph{generalized average causal effect} (GACE):
\begin{align}\label{eq:estimand}
    J=\E\bracks{\int Y(a)\epol(a\mid X)\rd\mu(a)}.
\end{align}

\begin{example}[Average treatment effect]\label{ex: atomic}
Consider $\Acal = \braces{0, 1}$.
The counterfactual mean parameter $\Eb{Y(a_0)}$ for $a_0 \in \braces{0, 1}$ is an example of \Cref{eq:estimand} with $\pi\prns{a \mid x} = \indic{a = a_0}$.
If we are interested in the effect of ``treatment'' $A=1$ compared to ``control'' $A=0$, we can let $\epol(a\mid x)=2a-1$ and obtain $J=\E[Y(1)-Y(0)]$, the average treatment effect (ATE).
\end{example}

\begin{example}[Policy evaluation]
If $\epol(a\mid x)$ is a density on $\Acal$ for each $x$ with respect to (w.r.t.) $\mu$, then $J$ is the average outcome we experience when we follow the policy that assigns an action drawn from $\epol(\cdot\mid X)$ for an individual with covariates $X$ \citep[\eg, ][]{DudikMiroslav2014DRPE,Tian2008}. The measure $\mu$ can be the Lebesgue measure when the action space is continuous or the counting measure when the action space is discrete.
\end{example}

We \emph{do not} assume that the observed covariates $X$ include all confounders that affect both the action and the potential outcomes, and instead there exist some \emph{unmeasured} confounders $U \in \Ucal \subseteq \R{p_u}$  (discrete, continuous, or mixed):
\begin{align*}
    Y(a)\not\perp A \mid X, ~ \text{ but } ~ Y(a)\perp A \mid U,X.
\end{align*}
If $U$ were observed, we could identify the GACE $J$ simply by controlling for both $X,U$.
However, in this paper we assume that
confounders $U$ \emph{cannot} be observed, in which case, the GACE $J$ is generally \emph{unidentifiable} from the distribution of the observed variables $(Y, X, A)$ alone.
To overcome the challenge of unmeasured confounding, in this paper we employ the negative control framework proposed in \citet{miao2018identifying,deaner2021proxy,CuiYifan2020Spci}. This framework involves two additional types of observed variables: negative control actions $Z \in \Zcal \subseteq \R{p_z}$ and negative control outcomes $W \in \Wcal \subseteq \R{p_w}$, which can be discrete, continuous, or mixed.
These variables are called ``negative'' controls due to the assumed absence of certain causal effects: negative control actions cannot directly affect the outcome $Y$, and neither the negative control actions $Z$ nor the primary action $A$ can affect the negative control outcomes $W$.
At the same time, these variables remain useful controls because they are related to the unmeasured confounders. We can, in a sense, view them as proxies for the unmeasured confounders $U$.

\Cref{fig:causal_dag} shows a typical causal diagram for this setting. To formalize our setting and allow for more generality, however, we will employ potential outcome notation.
Let
$Y(a, z)$ and $W(a, z)$ denote the corresponding counterfactual outcomes one would observe had
the primary action and negative control action taken value $(a, z) \in \Acal\times\Zcal$. We then formalize the negative control assumptions as follows.

\begin{assumption}[Negative Controls]\label{asm:whole_assm}
\begin{enumerate}
 \item \label{whole_assm:consistency} {Consistency: $Y=Y(A,Z),W=W(A,Z)$.}
    \item \label{whole_assm:nc-action} Negative control actions: $Y(a, z) = Y(a)$, $\forall a \in \Acal$.
    \item \label{whole_assm:nc-outcome} Negative control outcomes: $W(a, z) = W$, $\forall a \in \Acal, z \in \Zcal$.
    \item \label{whole_assm:nc-unconfoundedness} Latent unconfoundedness:  $(A,Z)\perp (Y(a),W) \mid U,X, ~~ \forall a\in \Acal$.
    \item \label{whole_assm:nc-overlap} Overlap: $\abs{\epol(a|x)/f(a|x,u)}<\infty,~~\forall a\in\Acal, x\in\Xcal, u\in\Ucal$.
\end{enumerate}
\end{assumption}
Here, condition \ref{whole_assm:consistency} states that the observed outcome $Y$ is the realization of the potential outcome $Y(a)$ for units with $A=a$. This is a standard assumption in causal inference \citep{imbens2015causal}.
Conditions \ref{whole_assm:nc-action} and \ref{whole_assm:nc-outcome} express the definition of negative controls in potential-outcome notation: the negative control action $Z$ cannot affect the primary outcome $Y$, and the negative control outcome $W$ cannot be affected by either the primary action $A$ or the negative control actions $Z$.
Condition \ref{whole_assm:nc-unconfoundedness} formalizes the assumption that the unmeasured variables $U$ capture all common
causes of $(A, Z)$ and $(Y, W)$ not included in $X$.
Condition \ref{whole_assm:nc-unconfoundedness} may also be satisfied by causal diagrams other than \Cref{fig:causal_dag} (see table A.1 of \citealp{tchetgen2020introduction}).
Condition \ref{whole_assm:nc-overlap} requires sufficient overlap between the contrast function $\epol$ and the distribution of observed actions given both observed and unobserved confounders. This is a canonical assumption in causal inference and policy evaluation.

Our data consist of $n$ independent and identically distributed (iid) observations of $(Z,X,W,A,Y)$. Crucially, $U$ is \emph{unobserved}. Our aim is to estimate the GACE parameter $J$ from these data.

\begin{figure}
    \centering
\begin{tikzpicture}
\node[draw, circle, text centered, minimum size=0.75cm, line width= 1] (x) {$X$};
\node[draw, circle, above=1 of x,
text centered, minimum size=0.75cm, dashed,line width= 1] (u) {$U$};
\node[draw, circle,left=2.5 of x, minimum size=0.75cm, text centered,line width= 1] (z) {$Z$};
\node[draw, circle,right=2.5 of x, minimum size=0.75cm, text centered,line width= 1] (w) {$W$};
\node[draw, circle, below left = 1 and 1 of x, minimum size=0.75cm, text centered,line width= 1] (a) {$A$};
\node[draw, circle, below right = 1 and 1 of x, minimum size=0.75cm, text centered,line width= 1] (y) {$Y$};

\draw[-latex, line width= 1] (x) -- (a);
\draw[-latex, line width= 1] (x) -- (y);
\draw[-latex, line width= 1, dashed] (x) -- (z);
\draw[-latex, line width= 1, dashed] (x) -- (w);
\draw[-latex, line width= 1, dashed] (u) -- (x);
\draw[-latex, line width= 1] (u) -- (z);
\draw[-latex, line width= 1] (u) -- (w);
\draw[-latex, line width= 1] (u) -- (a);
\draw[-latex, line width= 1] (u) -- (y);
\draw[-latex, line width= 1] (a) -- (y);
\draw[-latex, line width= 1, dashed] (z) -- (a);
\draw[-latex, line width= 1, dashed] (w) -- (y);
\end{tikzpicture}
\caption{A typical causal diagram for negative controls. The dashed edges may be absent, and the dashed circle around $U$ indicates that $U$ is unobserved.}
    \label{fig:causal_dag}
\end{figure}

\paragraph*{Notation}
We let $\E$ denote expectations {w.r.t. $O=(Z,X,W,A,Y)$}, and $\E_n$ denote empirical average over the $n$ observations thereof.
For a function $g$ of $(z,x,w,a,y)$ (or a subset thereof) we write $g$ to mean the random variable $g(O)$.
We let $L_2(O)$ denote the space of square-integrable functions of $O$. Similarly, $L_2(W, A, X)$ and $L_2(Z, A, X)$ denote the space of square-integrable functions of just $W,A,X$ and $Z,A,X$, respectively.
For a function $g$, we let $\magd g_2$ denote the norm in these spaces. For a vector $\theta$, we let $\magd\theta$ denote the Euclidean norm. We let $\magd{\cdot}_\infty$ denote the sup norm of either a variable or function, and for a class of functions we let it denote the supremum of sup norms in the class. For subsets $\mathbb{A},\mathbb{B}$ of a field, we define $\mathbb{A+B}=\{a+b:a\in \mathbb{A},b\in \mathbb{B}\},\mathbb{AB}=\{ab:a\in \mathbb{A},b \in \mathbb{B}\}$. We call a set $S$
symmetric if $-s\in S$ for any $s\in S$, and we call $S$ star-shaped (around the origin) if $\alpha s\in S$ for any $s\in S, \alpha\in [0,1]$.  Finally, we often use $O(\cdot)$ notation to denote rates w.r.t. the sample size $n$, unless otherwise specifically indicated. For a matrix $A$, we denote its Moore Penrose inverse by  $A^{+}$.

\section{Identifying GACE via Bridge Functions}\label{sec:identity}

\subsection{The Ideal Unconfounded Setting}

If the unobserved confounders $U$ \emph{were} observed, then the GACE could be identified, that is, it can be written as a function of the distribution of $(Y, A, X, U)$. To illustrate this, define the regression function $k_0(a, u, x) = \Eb{Y \mid A = a,  U = u, X = x}$, and define the generalized propensity score $f(a \mid u, x)$ as the conditional density of the distribution $A \mid U, X$ relative to the base measure $\mu$ \citep{hirano2005the}. Based on these two functions, the following lemma shows the identification of $J$ if $U$ \emph{were} observed.
\begin{lemma}\label{lem: U-identification}
If $Y(a)\perp A \mid U,X$ and
{$\abs{\epol(A\mid X)/f(A\mid X,U)} < \infty$},
then
\begin{align*}
    &J = \E{\phi_{\ipw}(Y, A, U, X; f)} = \E{\phi_{\DM}(Y, A, U, X; k_0)} = \E{\phi_{\dr}(Y, A, U, X; k_0, f)}, \\
    &\textstyle\text{where } ~
    \phi_{\ipw}(y, a, u, x; f) = \frac{\epol(a|x)}{f(a|x,u)}y,\quad \phi_{\DM}(y, a, u, x; k_0) = \int k_0(a', u, x)\epol(a'|x)\rd \mu(a'),\\
    &\textstyle\phantom{\text{where } ~} \phi_\dr(y, a, u, x; k_0, f) = \frac{\epol(a|x)}{f(a|x,u)}\prns{y-k_0(a, u, x)} + \int k_0(a', u, x)\epol(a'|x)\rd \mu(a').
\end{align*}
\end{lemma}
\Cref{lem: U-identification} suggests estimators for $J$ if $U$ were observed: we can first estimate the nuisance functions $k_0(a, u, x)$ and/or $f(a \mid u, x)$, and then estimate $J$ by using any of the three estimating equations above with the estimated nuisance(s). The resulting three estimators are called the inverse propensity weighting (IPW) estimator, the regression-based (REG) estimator, and the doubly robust (DR) estimator, respectively \citep[\eg, ][]{robins94,DudikMiroslav2014DRPE}.

\subsection{The Negative-Control Setting}

However, in this paper, we deal with the setting where $U$ is \emph{unobserved}, so estimators above are infeasible.
In particular, neither $k_0(a, u, x)$ nor $f(a \mid u, x)$ can be identified. Instead,
{ we can use their negative control analogues called \emph{bridge functions} \citep{miao2018a,CuiYifan2020Spci}.}
\begin{definition}[Bridge functions]\label{def: bridge}
An outcome bridge function is $h_0  \in L_2(W, A, X)$ with
\begin{align}
\E[h_0(W,A,X)\mid A,U,X]&=k_0(A, U, X).\label{eq: bridge-U-h}
\end{align}
An action bridge function is $q_0$ with $\pi q_0 \in L_2(Z, A, X)$ and
\begin{align}
\E[\epol(A\mid X)q_0(Z, A, X) \mid A,U, X]&= \frac{\epol(A\mid X)}{f(A\mid U, X)}.\label{eq: bridge-U-q}
\end{align}
\end{definition}
From \Cref{def: bridge}, we can observe that an outcome bridge function $h_0$  and an action bridge function $q_0$ can play a  similar role as the regression function $k_0$ and the generalized propensity score $f$, respectively (see \Cref{lem:bridge-identification} below).
These bridge functions are not necessarily unique.
We thus define the sets of bridge functions as follows:
\begin{equation}\label{eq:bridgefunctionsets}\begin{aligned}
    \Hbbb_0
        &= \braces{h \in L_2(W, A, X): \E[Y-h(W,A,X)\mid A,U,X]=0}, \\
    \Qbbb_0 &=
      \{q:\pi q \in L_2(Z, A, X),  \\
      &\qquad\quad \E[\epol(A\mid X)\prns{q(Z, A, X)-1/f(A\mid U, X)} \mid A,U, X]=0\}.
\end{aligned}\end{equation}

The existence of bridge functions depends on the relationship between $(Y, Z, W)$ and the unmeasured confounders $U$. Generally, such bridge functions exist when the negative control proxies $Z, W$ are \emph{sufficiently informative} for the unmeasured confounding, so that they can account for the impact of the unmeasured confounders $U$ on the action $A$ and the outcome $Y$, respectively. 

\begin{example}[Discrete setting] \label{ex: bridge-discrete}
We first consider a discrete data setting as in \cite{miao2018identifying,shi2020multiply},
where the variables $W, Z, U$ are all discrete variables with  values $w_i, z_j, u_s$ for $i = 1, \dots, \abs{\Wcal}, j = 1, \dots, \abs{\Zcal}, s = 1, \dots, \abs{\Ucal}$. For any $(a,x)\in \Acal\times \Xcal$, let $P(\mathbf{W} \mid \mathbf{U}, a, x)$ denote a $\abs{\Wcal}\times\abs{\Ucal}$ matrix whose $(i, s)$th element is $\Prb{W = w_i \mid U = u_s, A = a, X = x}$, $P(\mathbf{Z} \mid \mathbf{U}, a, x)$ a $\abs{\Zcal}\times\abs{\Ucal}$ matrix whose $(j, s)$th element is $\Prb{Z = z_j \mid U = u_s, A = a, X = x}$, $\E\bracks{Y \mid \mathbf{U}, a, x}$ a $1 \times \abs{\Ucal}$ vector whose $s$th element is $\Eb{Y\mid U = u_s, A=a, X = x}$,  $F(a\mid \mathbf{U}, x)$ a $\abs{\Ucal}\times \abs{\Ucal}$ diagonal matrix whose $s$th diagonal element is $f(a \mid u_s, x)$, and $\mathbf{e}$ an all-one column vector of length $\abs{\Ucal}$.

With these notations, \Cref{eq: bridge-U-h,eq: bridge-U-q} translate into the following linear equation system:
\begin{equation}\label{eq:requirement}
\begin{aligned}
&h^\top_0(\mathbf{W}, a, x)P(\mathbf{W} \mid \mathbf{U}, a, x) = \E\bracks{Y \mid \mathbf{U}, a, x}, \\
&q_0^\top(\mathbf{Z}, a, x)P(\mathbf{Z} \mid \mathbf{U}, a, x)F(a \mid \mathbf{U}, x) = \mathbf{e}^\top.
\end{aligned}
\end{equation}
There exist solutions $h_0(\mathbf{W}, a, x)$ to the first equation if the vector $\Eb{Y \mid \mathbf{U}, a, x}$ belongs to the column space of the matrix $P(\mathbf{W} \mid \mathbf{U}, a, x)$. This requires that the proxy $W$ can capture enough information in the unobserved confounder $U$, so that the column space of matrix $P(\mathbf{W} \mid \mathbf{U}, a, x)$ is rich enough to recover $\Eb{Y \mid \mathbf{U}, a,  x}$. Similarly, there exist solutions $q_0(\mathbf{Z}, a, x)$ to the second equation if the proxy $Z$ is informative enough such that the column space of $P(\mathbf{Z} \mid \mathbf{U}, a, x)$ can recover the vector $(1/f(a \mid u_1, x), \dots, 1/f(a \mid u_{|\mathcal{U}|}, x))^\top$.

Clearly, these existence conditions would automatically hold if the matrices $P(\mathbf{W} \mid \mathbf{U}, a, x)$, $P(\mathbf{Z} \mid \mathbf{U}, a, x)$ have full column ranks (which implies that $\abs{\Wcal} \ge \abs{\Ucal}$ and $\abs{\Zcal} \ge \abs{\Ucal}$). In this case, the proxies $Z, W$ are both strongly informative for the unmeasured confounders $U$, such that they can handle arbitrary $\Eb{Y \mid \mathbf{U}, x}$ and $(1/f(a \mid u_1, x), \dots, 1/f(a \mid u_{|\mathcal{U}|}, x))^\top$. However, even in this case, the solutions are generally {nonunique} when $\abs{\Wcal} > \abs{\Ucal}$ and $\abs{\Zcal} > \abs{\Ucal}$.
If we were to further assume that $\abs{\Wcal} = \abs{\Zcal}=  \abs{\Ucal}$, $P(\mathbf{W} \mid \mathbf{U}, a, x)$ and $P(\mathbf{Z} \mid \mathbf{U}, a, x)$ are invertible square matrices, then the bridge functions would uniquely exist. 

\end{example}

\begin{example}[Linear Model]\label{ex: bridge-linear}
Suppose $(Y, W, Z, A)$ are generated as follows:
\begin{align*}
    Y &= \alpha_Y^\top U + \beta_Y^\top X + \gamma_Y^\top A + \omega_Y^\top W +  \epsilon_Y, \\
    Z &= \alpha_Z U + \beta_Z X + \gamma_Z A + \epsilon_Z, \\
    W &= \alpha_W U + \beta_W X +  \epsilon_W, \\
    A &\sim \mathrm{Unif}\prns{\underline{\alpha}_A^\top U + \underline{\beta}_A^\top X,\,\overline{\alpha}_A^\top U + \overline{\beta}_A^\top X},
\end{align*}
where $\epsilon_Y, \epsilon_Z, \epsilon_W$ are {independent mean-zero} random noises that are also independent of $\prns{A, U, X}$.

It can be shown that when there exist solutions $\theta_W, \theta_Z$ to the equations $\theta^\top_W \alpha_W = \alpha_Y^\top$ and
$\theta_Z^\top \alpha_Z = \overline{\alpha}_A^\top - \underline{\alpha}_A^\top$, then valid bridge functions $h_0, q_0$ exist
(see \Cref{lemma: bridge}):
\begin{align*}
    &h_0(W, A, X) = \prns{\theta_W + \omega_Y}^\top W + \prns{\beta_Y - \theta_W^\top \beta_W}X + \gamma_Y A, ~~ \forall \theta_W \text{ s.t. } \theta^\top_W \alpha_W = \alpha_Y^\top, \\
    &q_0(Z, A, X) = \theta_Z^\top Z + \prns{\overline{\beta}_A - \underline{\beta}_A - \theta_Z^\top \beta_Z}X -\theta_Z^\top \gamma_Z A, ~~ \forall \theta_Z \text{ s.t. } \theta_Z^\top \alpha_Z = \overline{\alpha}_A^\top - \underline{\alpha}_A^\top.
\end{align*}
The existence of solutions to equations $\theta^\top_W \alpha_W = \alpha_Y^\top$ and
$\theta_Z^\top \alpha_Z = \overline{\alpha}_A^\top - \underline{\alpha}_A^\top$  means that the proxy variables $Z, W$ capture relevant information in the unmeasured confounders $U$, such that  the column spaces of loading matrices $\alpha_Z, \alpha_W$ can cover the effects of the unmeasured confounders $U$ on $Y$ and $A$ (characterized by the vectors $\alpha_Y$ and $\overline{\alpha}_A^\top - \underline{\alpha}_A^\top$ respectively).

In particular, if $\alpha_W, \alpha_Z$ have full column rank (which implies $p_w \ge p_u$ and $p_z \ge p_u$), then the existence conditions always hold because the column spaces of $\alpha_W, \alpha_Z$ can cover any $\alpha_Y$ and $\overline{\alpha}_A^\top - \underline{\alpha}_A^\top$, no matter how the unmeasured confounders $U$ affect $Y$ and $A$. Moreover, the bridge functions are nonunique if $p_w > p_u$ and $p_z > p_u$. In these cases, we have more proxy variables than necessary to account for the unmeasured confounders, so there exist many different ways to transform the proxy variables to tackle the confounding problem.

\end{example}

\begin{example}[Nonparametric Model]\label{ex: bridge-nonparametric}
More generally, when the relevant variables can be continuous and we assume  nonparametric models for the data generating process,
we need to ensure the existence of solutions to conditional moment equations in \Cref{eq: bridge-U-h,eq: bridge-U-q}.
The existence of bridge functions again requires the proxy variables $Z, W$ to capture the relevant information in the unmeasured confounders $U$.  They require that after conditioning on variables $A, X$, there exist some transformations (i.e., bridge functions) of $Z, W$ whose information about the unmeasured confounders $U$
can recover the confounding effects of $U$ on the action $A$ and outcome $Y$ respectively.

For general nonparametric models, the conditional moment equations in \Cref{eq: bridge-U-h,eq: bridge-U-q} correspond to equations defined by potentially infinite-dimensional and ill-posed linear operators. It is therefore difficult to explicitly characterize the existence of solutions to these equations.
One set of sufficient (but not necessary) conditions that can ensure the existence of solutions to these equations is characterized by Picard's theorem in functional analysis. According to Picard's theorem, when the functions $k_0(A, U, X)$ and $1/f(A\mid U, X)$ satisfy the so-called Picard's condition (roughly meaning that they are  sufficiently smooth relative to the equations), then the solutions exist when the following completeness conditions hold (see \Cref{thm: Picard,lemma: completeness-existence-W,lemma: completeness-existence-Z} in \Cref{sec: completeness-existence}):
\begin{align}\label{eq:completeness}
    &\Eb{g(U, A, X) \mid Z, A, X} = 0 \text{ only when } g(U, A, X) = 0,  \\
    &\Eb{g(U, A, X) \mid W, A, X} = 0 \text{ only when } g(U, A, X) = 0.\label{eq:completeness2}
\end{align}
These completeness conditions roughly  mean that the negative controls $Z, W$ can  capture all variations
of the unobserved confounders $U$, after conditioning on $A, X$.

In fact, when $Z, W, U$ are all discrete variables as in  \Cref{ex: bridge-discrete}, the  completeness conditions are equivalent to the full column ranks of matrices $P(\mathbf{W} \mid \mathbf{U}, a, x)$, $P(\mathbf{Z} \mid \mathbf{U}, a, x)$. When the variables satisfy the linear models in \Cref{ex: bridge-linear}, these completeness conditions often imply that the full column ranks of the matrices  $\alpha_W, \alpha_Z$. As we discussed in \Cref{ex: bridge-discrete,ex: bridge-linear}, these conditions are sufficient but not necessary conditions for the existence of bridge functions. It is possible that even when these matrices are rank deficient so that the completeness conditions fail, the equations in \Cref{ex: bridge-discrete,ex: bridge-linear} still have solutions because their right-hand-side quantities belong to the column spaces of these matrices (see concrete examples in \Cref{sec: identification-comparison}).
In this paper, we will directly rely on the existence of bridge functions rather than the completeness conditions. 

\end{example}

{\Cref{ex: bridge-discrete,ex: bridge-linear} illustrate that bridge functions, besides existing, are nonunique if the negative control proxies carry more information than the unmeasured confounders, namely, when the negative controls have more values or dimensions than the unmeasured confounders.
As unmeasured confounders are unobserved in practice, we may tend to use as many negative control variables as possible to safeguard the existence of bridge functions. But this may also cause bridge functions to be nonunique.
{Therefore, assuming uniqueness may be too strong. Our paper avoids the uniqueness assumption as much as possible and imposes  uniqueness only when it is needed to derive stronger theoretical guarantees.}

{The lemma below shows that {even} when bridge functions are nonunique, \emph{any} one of them can identify $J$.}
We first  define an operator for the lemma below:
\begin{align*}
\Tcal: L_2(W, A, X) \rightarrow L_2(W, X),\quad(\Tcal h)(w, x) = \int h(w, a, x)\epol(a|x)\rd \mu(a).
\end{align*}
\vspace{-\baselineskip}
\begin{lemma}\label{lem:bridge-identification} Suppose that \Cref{asm:whole_assm} holds.
For any $h_0 \in \Hbbb_0$ and $q_0 \in \Qbbb_0$,
\begin{align*}
&J = \E{\tilde\phi_\ipw(O;q_0)} = \E{\tilde\phi_\DM(O;h_0)} = \E{\tilde\phi_\dr(O;h_0,q_0)}, \\
&\text{where }~  \tilde\phi_\ipw(O;q_0) =  \epol(A\mid X)q_0(Z, A, X)Y, \quad\tilde\phi_\DM(O;h_0) = (\Tcal h_0)(X,W),\\
&\phantom{\text{where }~} \tilde\phi_\dr(O;h_0,q_0) = {\epol(A\mid X)q_0(Z, A, X)\prns{Y - h_0(W, A, X)}} + {(\Tcal h_0)(W, X)}.
\end{align*}
\end{lemma}

Note that the estimating equations in \Cref{lem:bridge-identification} simply replace the regression function $k_0(A, X, U)$ and the inverse propensity score weight $1/f(A\mid X, U)$ in \Cref{lem: U-identification} by the bridge functions $h_0(W, A, X)$ and $q_0(Z, A, X)$ respectively. Since the latter only depends on observed variables, as long as we can learn \emph{any} pair of bridge functions, we can use estimating equations in \Cref{lem:bridge-identification} to  estimate $J$.

\subsection{Learning Bridge Functions from Observed Data}
\begin{figure}
    \centering\footnotesize
    \begin{tikzpicture}[>=latex']
\node[rectangle, draw, fill=gray!20, text width=2.9cm] (a1) at (0,0)
{$J=\E{\tilde\phi_\DM(O;h_0)}$\break$\forall h_0\in\Hbbb_0$ (Lemma \ref{lem:bridge-identification})\break~\vspace{-0.5em}\break Assume $\Hbbb_0\neq\emptyset$};
\node[rectangle, draw, fill=gray!20, text width=3.2cm,right=1.75cm of a1, yshift=.65cm] (a21)
{Assume $\Qbbb^{\obs}_0\neq\emptyset$};
\node[rectangle, draw, fill=gray!20, text width=3.2cm,right=1.75cm of a1, yshift=-.65cm] (a22)
{Assume {\textcolor{red!80!black}{\bf completeness}} (\Cref{eq:completeness}) so $\Hbbb^{\obs}_0=\Hbbb_0$};
\node[rectangle, draw, fill=gray!20, text width=3.1cm,right=1.75cm of a21, yshift=-.65cm] (a3)
{$J=\E{\tilde\phi_\DM(O;h_0)}$\break$\forall h_0\in\Hbbb_0^{\obs}$, where $\Hbbb_0^{\obs}\neq\emptyset$ is identifiable};
\draw[->] (a1)--(a21) node[midway,above,text width=1.1cm,align=center,sloped] {Our\break strategy};
\draw[->] (a1)--(a22) node[midway,below,text width=1.1cm,align=center,sloped] {Previous\break strategies};
\draw[->] (a21)--(a3) node[midway,above,text width=1.4cm,align=center,sloped] {\Cref{lemma: relaxed-ipw-dm}};
\draw[->] (a22)--(a3);
\node[rectangle, draw, fill=gray!20, text width=2.9cm] (b1) at (0,-2.8)
{$J=\E{\tilde\phi_\ipw(O;q_0)}$\break$\forall q_0\in\Qbbb_0$ (Lemma \ref{lem:bridge-identification})\break~\vspace{-0.5em}\break Assume $\Qbbb_0\neq\emptyset$};
\node[rectangle, draw, fill=gray!20, text width=3.2cm,right=1.75cm of b1, yshift=.65cm] (b21)
{Assume $\Hbbb^{\obs}_0\neq\emptyset$};
\node[rectangle, draw, fill=gray!20, text width=3.2cm,right=1.75cm of b1, yshift=-.65cm] (b22)
{Assume {\textcolor{red!80!black}{\bf completeness}} (\Cref{eq:completeness2}) so $\Qbbb^{\obs}_0=\Qbbb_0$};
\node[rectangle, draw, fill=gray!20, text width=3.1cm,right=1.75cm of b21, yshift=-.65cm] (b3)
{$J=\E{\tilde\phi_\ipw(O;q_0)}$\break$\forall q_0\in\Qbbb_0^{\obs}$, where $\Qbbb_0^{\obs}\neq\emptyset$ is identifiable};
\draw[->] (b1)--(b21) node[midway,above,text width=1.1cm,align=center,sloped] {Our\break strategy};
\draw[->] (b1)--(b22) node[midway,below,text width=1.1cm,align=center,sloped] {Previous\break strategies};
\draw[->] (b21)--(b3) node[midway,above,text width=1.4cm,align=center,sloped] {\Cref{lemma: relaxed-ipw-dm}};
\draw[->] (b22)--(b3);
    \end{tikzpicture}

    \caption{Different identification strategies.
    See relevant discussions below \Cref{lemma: relaxed-ipw-dm} and in \Cref{sec: comparison,sec: completeness,app: single-bridge}.}

    \label{fig:my_label}
\end{figure}

\Cref{def: bridge} defines bridge functions in terms of conditional moment equations\footnote{\Cref{eq: bridge-U-q,eq: observed-bridge-q} do not exactly fall into the usual conditional moment equation framework \citep[\eg, ][]{AiChunrong2003EEoM}, since they involve unknown density functions $f(A \mid U, X), f(A \mid W, X)$, respectively. We still call them conditional moment equations for simplicity, but estimating $q_0$ does require more care. See \Cref{sec: strategy1}.}
given the unobserved confounders $U$, so we cannot directly use it  to learn the bridge functions from the observed data.
Nevertheless, the following lemma shows that the bridge functions also satisfy analogous conditional moment equations based only on observed data. This lemma extends the characterizations of bridge functions in \cite{miao2018a,CuiYifan2020Spci}.

\begin{lemma}\label{lemma: observed-bridge}
Under \Cref{asm:whole_assm}, any $h_0 \in \Hbbb_0$ and $q_0 \in \Qbbb_0$ satisfy
\begin{align}
    \E[Y-h_0(W, A,X) \mid Z, A, X]&=0, \label{eq: observed-bridge-h}\\
    \E[\epol(A\mid X)\prns{q_0(Z, A, X) - 1/f(A \mid W, X)} \mid W, A, X] &= 0.\label{eq: observed-bridge-q}
\end{align}
\end{lemma}

The conditional moment equations in \Cref{eq: observed-bridge-h,eq: observed-bridge-q} give rise to the following alternative sets of functions that we can possibly learn from observed data, whose elements we call \emph{observed-data} bridge functions:
\begin{align*}
    &\Hbbb_0^{\obs} = \braces{h \in L_2(W, A, X): \E[Y-h(W, A,X) \mid Z, A, X]=0}, \\
    &\Qbbb_0^{\obs} =
    \braces{q: \pi q \in L_2(Z, A, X),\E[\epol(A\mid X)\prns{q(Z, A, X) - 1/f(A \mid W, X)} \mid W, A, X] = 0}.
\end{align*}
 According to \Cref{lemma: observed-bridge}, $\Hbbb_0 \subseteq \Hbbb_0^{\obs}$ and $\Qbbb_0 \subseteq \Qbbb_0^{\obs}$.
So, when $\Hbbb_0 \ne \emptyset, \Qbbb_0 \ne \emptyset$, we also have $\Hbbb_0^\text{obs} \neq \emptyset,\Qbbb_0^\text{obs} \neq \emptyset$.
Actually, it can be shown that $\Hbbb_0 \subseteq \Hbbb_0^{\obs}$ and $\Qbbb_0 \subseteq \Qbbb_0^{\obs}$ when the completeness conditions in \Cref{eq:completeness,eq:completeness2} hold respectively (see \Cref{lemma: valid-bridge} in \Cref{sec: comparison}). But without completeness conditions, the sets of observed-data bridge functions $\Hbbb_0^{\obs}$ or $\Qbbb_0^{\obs}$  can be strictly bigger than the sets of true bridge functions $\Hbbb_0,\Qbbb_0$.
{In the following lemma, we further show that observed-data bridge functions in $\Hbbb_0^{\obs}$ or $\Qbbb_0^{\obs}$ can be directly used to identify $J$, even if they are not true bridge functions in $\Hbbb_0,\Qbbb_0$.}
\begin{theorem}\label{lemma: relaxed-ipw-dm}
Suppose \Cref{asm:whole_assm} holds.
\begin{enumerate}
\item \label{lemma: relaxed-ipw-dm-1} If $\Hbbb_0\neq \emptyset$ and $\Qbbb^{\obs}_0\neq \emptyset$, then $J = \E{\tilde\phi_\DM(O;h_0)}$ for any $h_0 \in \Hbbb_0^{\obs}$.
\item \label{lemma: relaxed-ipw-dm-2} If $\Qbbb_0\neq \emptyset$ and $\Hbbb^{\obs}_0\neq \emptyset$, then $J = \E{\tilde\phi_\ipw(O;q_0)}$ for any $q_0 \in \Qbbb_0^{\obs}$.
\item \label{lemma: relaxed-ipw-dm-3} If the conditions in either statement \ref{lemma: relaxed-ipw-dm-1} or statement \ref{lemma: relaxed-ipw-dm-2}  hold, then $J = \E{\tilde\phi_\dr(O;h_0, q_0)}$ for any $h_0 \in \Hbbb_0^{\obs}$ and $q_0 \in \Qbbb_0^{\obs}$.
\end{enumerate}
\end{theorem}

\Cref{lemma: relaxed-ipw-dm}  suggests a straightforward way to estimate $J$: first estimate $h_0 \in \Hbbb_0^{\obs}$ and $q_0 \in \Qbbb_0^{\obs}$ by solving \Cref{eq: observed-bridge-h,eq: observed-bridge-q}, and then estimate the GACE $J$ by using any of the estimating equations in Lemma \ref{lem:bridge-identification}.

{Note that by Lemma \ref{lemma: observed-bridge}, existence of both true bridge functions, $\Hbbb_0 \ne \emptyset, \Qbbb_0 \ne \emptyset$, is a sufficient condition to guarantee the assumptions of each of the three statements of \Cref{lemma: relaxed-ipw-dm}.}
{In \Cref{app: single-bridge}, we show that the existence of just one bridge function, \ie, $\Hbbb_0 \ne \emptyset$ or $\Qbbb_0 \ne \emptyset$, is insufficient for identification. Again by \Cref{lemma: observed-bridge}, this also means the existence of just one observed data bridge function, \ie, $\Hbbb^{\obs}_0 \ne \emptyset$ or $\Qbbb^{\obs}_0 \ne \emptyset$, is similarly insufficient. In \Cref{app: single-bridge}, we  prove that even the simultaneous existence of two observed data bridge functions, \ie, $\Hbbb^{\obs}_0 \ne \emptyset$ and $\Qbbb^{\obs}_0 \ne \emptyset$, is also insufficient for identification. }

{Interestingly, the identification formulae in \Cref{lemma: relaxed-ipw-dm} hold for any bridge functions  $h_0 \in \Hbbb_0^{\obs}$ and $q_0 \in \Qbbb_0^{\obs}$, even if they violate the conditional moment equations in \Cref{eq: bridge-U-h,eq: bridge-U-q}, \ie, it holds for $h_0 \in \Hbbb_0^{\obs}\backslash\Hbbb_0$ and $q_0 \in\Qbbb^{\obs}\backslash \Qbbb_0$.
In other words,  even when the sets of bridge functions are unidentifiable {(\ie, $\Hbbb_0 \subsetneq \Hbbb_0^\text{obs}$ or $\Qbbb_0 \subsetneq \Qbbb_0^\text{obs}$)}, the GACE parameter $J$ can still be identifiable, provided that conditions in \Cref{lemma: relaxed-ipw-dm} hold.}

Our identification strategies in \Cref{lemma: relaxed-ipw-dm} are different from previous identification results based on observed-data bridge functions \citep{CuiYifan2020Spci,deaner2021proxy,miao2018identifying}.
This previous literature assumes the existence of one observed-data bridge function (\ie, either $\Hcal_0^{\obs} \ne \emptyset$ or $\Qcal_0^{\obs} \ne \emptyset$), and impose the \emph{completeness conditions}  in \Cref{eq:completeness} or \Cref{eq:completeness2} to ensure {$\Hbbb^{\obs}_0=\Hbbb_0$ or $\Qbbb^{\obs}_0=\Qbbb_0$} (see \Cref{lemma: valid-bridge}).
Take the identification via $\tilde\phi_\DM$ as an example.
Previous literature assume the completeness condition in \Cref{eq:completeness} and assume $\Hbbb^{\obs}_0 \ne \emptyset$ (or equivalently $\Hbbb_0 \ne \emptyset$ under the completeness condition).
Then any \emph{observed-data} bridge function $h_0 \in \Hbbb^{\obs}_0$ must also be a true bridge function in $\Hbbb_0$, and thus can identify GACE according to \Cref{lem:bridge-identification}.
In contrast, our identification result in \Cref{lemma: relaxed-ipw-dm}
 statement \ref{lemma: relaxed-ipw-dm-1} assumes
$\Hbbb_0\neq \emptyset$ and $\Qbbb^{\obs}_0\neq \emptyset$, but not any completeness condition. In other words, our identification strategy replaces a completeness condition by the condition $\Qbbb^{\obs}_0\neq \emptyset$.
Our result is based on a completely \emph{new} proof that allows us to directly identify GACE via any $h_0 \in \Hbbb^{\obs}_0$ without requiring  $h_0$ to also belong to $\Hbbb_0$.
In this way, we can achieve identification without assuming any completeness conditions, and allow for $\Hbbb_0 \subsetneq \Hbbb_0^\text{obs}$. See \Cref{sec: old-identification} for a detailed review of the identification strategies in \cite{CuiYifan2020Spci} and \Cref{fig:my_label}  for an illustration for the difference in identification strategies.

One may hope to compare our new identification assumptions with those in the previous literature. Intuitively, our conditions of the existence of bridge functions require the proxy variables $Z$ and $W$ to capture the \emph{relevant} variations of the unobserved confounders $U$, such that the effects of $U$ on some bridge function transformations of these proxy variables can recover the effects of $U$ on the action $A$ and outcome $Y$ respectively. In contrast, the completeness conditions in the existing literature require the proxy variables $Z$ and $W$ to capture \emph{all} variations of $U$, so that no matter how $U$ affect the action $A$ and outcome $Y$, the proxy variables can always account for all unmeasured confounding. Therefore, our identification strategy does not rely on completeness (even implicitly), but instead replaces it with a different, more targeted requirement.

The differences between the two different types of conditions are particularly clear in the discrete data setting (\Cref{ex: bridge-discrete}) and linear model setting (\Cref{ex: bridge-linear}). As  we discussed in these two examples and \Cref{ex: bridge-nonparametric}, our assumptions about the existence of bridge functions require the existence of solutions to some linear equation systems, while the completeness conditions usually require the coefficient matrices of these linear equations to have full rank. The full rank conditions  are sufficient but not necessary for the existence of solutions.  There exist many instances where the completeness conditions fail but our proposed conditions still hold.
Therefore, in the discrete data setting and linear model setting, our proposed conditions can be \emph{strictly} weaker than the completeness conditions imposed in the existing literature.
We refer to \Cref{sec: identification-comparison} for more concrete discussions.

The comparisons between the two different types of conditions are more difficult in general nonparametric models (\Cref{ex: bridge-nonparametric}). In general models, the existence of bridge functions amounts to the existence of solutions to some linear equations defined by some  potentially infinite-dimensional and ill-posed linear operators. In contrast, the completeness conditions require the adjoint of these linear operators to be injective mappings. For infinite-dimensional operators and equations, these two different types of conditions are closely related but are  not directly comparable\footnote{It is possible that one type of conditions hold while the other type of conditions fail, due to challenging properties of infinite-dimensional function spaces that differ from finite-dimensional spaces.} without further regularity conditions.  One approach to relate them is through  Picard's theorem in functional analysis, as we discussed in \Cref{ex: bridge-nonparametric}. Under the regularity conditions therein, one can again prove that the completeness conditions are sufficient for our conditions. See \Cref{sec: completeness-existence} for more details. Another setting with clear comparisons is to consider our combined identification conditions $\Hbbb_0 \ne \emptyset, \Qbbb_0 \ne \emptyset$ that enable all three identification formulas in  \Cref{lemma: relaxed-ipw-dm}. To achieve the same identification goal, the existing literature would require $\Hbbb_0^{\obs} \ne \emptyset, \Qbbb_0^{\obs} \ne \emptyset$ and also the two completeness conditions in \Cref{eq:completeness,eq:completeness2}, which in turn imply our combined identification condition $\Hbbb_0 \ne \emptyset, \Qbbb_0 \ne \emptyset$.

In summary, our proposed identification conditions and theory complement the identification results in the previous literature, thereby broadening the scope of causal identification with negative controls. But we note that these different identification strategies lead to the same identification formula (see \Cref{lemma: relaxed-ipw-dm}), so their differences do not impact the statistical estimation and inference.
In the rest of this paper, we  study the estimation and inference of the GACE parameter in \Cref{eq:estimand} based on the unified identification formula. 

\section{Minimax Estimation of GACE}\label{sec: reformuate-moment}

\subsection{Estimation of GACE}

Once we obtain bridge function estimators $\hat h, \hat q$ as explained in the next section,
we can plug them into the estimating equations in \Cref{lemma: relaxed-ipw-dm} to construct the following  estimators for the GACE parameter $J$:
\begin{align*}
    \hat J_{\ipw}
        &= \E_n{\bracks{\tilde\phi_\ipw(O;\hat q)}} = \E_n{\bracks{\epol(A\mid X)\hat q(Z, A, X)Y}}, \\
    \hat J_{\DM}
        &= \E_n{\bracks{\tilde\phi_\DM(O;\hat h)}}= \E_n{\bracks{(\Tcal \hat h)(X,W)}}, \\
    \hat J_{\DR}
        &= \E_n{\bracks{\tilde\phi_\dr(O;\hat h, \hat q)}} = \E_n{\bracks{{\epol(A\mid X)\hat q(Z, A, X)\prns{Y - \hat h(W, A, X)}} + {(\Tcal \hat h)(X,W)}}}.
\end{align*}
{To use these estimators, we need that the corresponding equations identify GACE to begin with. Thus, throughout the rest of the paper, we assume \Cref{asm:whole_assm} and $\Hbbb_0 \ne \emptyset, \Qbbb_0 \ne \emptyset$ (see discussions below \Cref{lemma: relaxed-ipw-dm}).
}

In this section, we discuss how to construct $\hat h, \hat q$, presenting two types of minimax estimators. In subsequent sections, we discuss resulting guarantees for GACE estimation under various estimators and assumptions, an
overview of which we present in \Cref{sec: overview}.

\subsection{Minimax Estimators of Bridge functions}

Estimating bridge functions based on \Cref{eq: observed-bridge-h,eq: observed-bridge-q} requires solving conditional moment equations, which is generally a difficult estimation problem.
Estimation methods in the previous literature on negative controls mainly focus on parametric methods, sieve methods or Reproducing Kernel Hilbert Space methods (see \Cref{sec: literature} for a review).
{ In this paper, we propose to use minimax approaches to estimate the bridge functions,
 which accommodate not only the hypothesis classes used in the previous literature, but also more flexible ones such as neural networks.}

In this section, we introduce two minimax reformulations of the conditional moment equations in \Cref{eq: observed-bridge-h,eq: observed-bridge-q}. Each reformulation motivates an estimation strategy for bridge functions. Similar reformulations have also been used by previous literature for instrument variable (IV) estimation (see  \Cref{sec: literature}). { To introduce these reformulations, we consider a generic conditional estimating equation problem,
\begin{align}\label{eq: generic-crm}
    \Eb{\rho\prns{g_0(O_1), O_1} \mid O_2} = 0,
\end{align}
where $g_0$ is a function that we wish to solve for   and $O_1, O_2$ are two sets of random variables.
}

\subsubsection{Strategy I: Minimax Estimators without Stabilizers}\label{sec: strategy1}
We first note that \Cref{eq: generic-crm} has the following reformulation:
\begin{align}
\notag
    \Eb{\rho\prns{g_0(O_1), O_1} \mid O_2} = 0 &\iff \Eb{g'(O_2)\rho\prns{g_0(O_1), O_1}} = 0, ~~ \forall g' \in L_2(O_2), \\
    &\iff\textstyle\sup_{g'\in L_2(O_2)}\prns{\Eb{g'(O_2)\rho\prns{g_0(O_1), O_1}}}^2 = 0.\label{eq:strategy1}
\end{align}
This motivates the following estimators for bridge functions
\begin{align}
  \hat h &\in \textstyle\argmin_{h\in \Hbbb}\max_{q\in \Qbbb'}~~\prns{\E_n{[q(Z, A, X)\prns{ h(W, A, X)-Y}}]}^2, \label{eq: est1-1}\\
  \hat q &\in  \textstyle\argmin_{q\in \Qbbb}\max_{h\in \Hbbb'}~~\prns{\E_n{[\epol(A|X)q(Z, A, X)h(W,A,X) - (\Tcal h)(W, X)}]}^2,  \label{eq: est1-2}
\end{align}

 Our estimators $\hat h, \hat q$ can be viewed as solutions to minimax games, where an adversarial player picks elements from function classes $\Qbbb',\Hbbb'$ to form the most difficult marginal moments while our estimators minimize the violations of such moments.
 { \Cref{eq: est1-1,eq: est1-2} involve two types of function classes: we call $\Qbbb, \Hbbb$ the  \emph{bridge classes} and $\Qbbb',\Hbbb'$ the \emph{critic classes}, and we call elements of the latter \emph{critic functions}.}
 Note that throughout, $\Qbbb, \Hbbb, \Qbbb',\Hbbb'$ can change with $n$. We review some examples in \Cref{sec: example} below.

Note that although \Cref{eq: observed-bridge-q} involves the generalized propensity score $f(A\mid W, X)$, it does not appear in \Cref{eq: est1-2} at all. In this sense, our estimation method for $\hat q$ is different from a na\"ive application of \Cref{eq:strategy1} to \Cref{eq: observed-bridge-q}, wherein we would first get a preliminary generalized propensity score estimator $\hat f\prns{A \mid W, X}$ and then solve
$$\textstyle
  \argmin_{q\in \Qbbb}\max_{h\in \Hbbb'}  (\E_n[h(W,A,X)\epol(A|X)\{q_0-1/\hat f(A|W,X)\}])^2.
$$
Instead, our estimator in \Cref{eq: est1-2} exploits the fact that $$\Eb{h(W, A, X)\epol\prns{A\mid X}/f(A\mid W, X)} = \Eb{\prns{\Tcal h}(X, W)}.$$ Thus, it obviates the need to estimate the  generalized propensity score before estimating the bridge functions.
This fact also characterizes the difference in the estimation of $h_0$ and $q_0$. The estimation of $h_0$ is analogous to the nonparametric IV regression problem \citep{newey2003instrumental,darolles2010nonparametric}, and the estimator \Cref{eq: est1-1} is analogous to minimax approaches therein. In contrast, the estimation of $q_0$ requires additional considerations.
\begin{continuance}[Average treatment effect]{\ref{ex: atomic}}
Consider binary action $A\in\braces{0, 1}$ in \Cref{ex: atomic}.
In this case, the conditional moment equation for the action bridge function $q_0$ is equivalent to
\begin{align}\label{eq: ate-binary-q}
    \Eb{\indic{A = a}q_0(Z, a, X) - 1\mid W, X}= 0, ~~ a \in \braces{0, 1}.
\end{align}
{Apparently, this equation does not explicitly depend on the propensity score either. }Note that when $Z = W = \emptyset$, the action bridge function given by this conditional moment equation is exactly the inverse propensity score weight $1/\Prb{A = a \mid X}$.
\end{continuance}

\subsubsection{Strategy II: Minimax Estimators with Stabilizers}\label{sec: strategy2}
Again considering the generic estimating equation in \Cref{eq: generic-crm}, note
that for any constant $\lambda > 0$,
\begin{align}\label{eq:strategy2}\begin{aligned}
&\Eb{\rho\prns{g_0(O_1), O_1} \mid O_2} = 0 \iff
\\&0=\frac{1}{4\lambda}\E{\prns{\Eb{\rho\prns{g_0(O_1), O_1} \mid O_2}}^2}=\sup_{g'\in L_2(O_2)} \Eb{g'(O_2)\rho\prns{g_0(O_1), O_1}} - \lambda \|g'\|^2_2.
\end{aligned}\end{align}

This motivates the following estimators for bridge functions
\begin{align}
  \hat h \in\textstyle \argmin_{h\in \Hbbb}\max_{q\in \Qbbb'}&~\E_n[q\prns{Z, A, X}\prns{ h(W, A, X)-Y}]-\lambda \E_n[q^2\prns{Z, A, X}], \label{eq: est2-1}\\
  \label{eq: est2-2}
  \hat q \in\textstyle  \argmin_{q\in \Qbbb}\max_{h\in \Hbbb'}&~\E_n[\epol(A \mid X)q(Z, A, X)h(W, A, X) - (\Tcal h)(W, X)]\\&~-\lambda\E_n[h^2\prns{W, A, X}]. \notag
\end{align}

We call the terms $\lambda \E_n[q^2],\lambda\E_n[h^2]$ \emph{stabilizers}.
Stabilizers are different from regularizers. Regularizers typically introduce estimation bias
\citep[\eg, ][]{knight2000asymptotics,CarrascoMarine2007C7LI}, { so we generally let them vanish when the sample size grows}.
Stabilizers, on the other hand, do not introduce bias and are merely a way to reformulate the conditional moment equations.
In particular, the objective function of the stabilizer-based reformulation  involves only marginal expectations that can be easily estimated by sample averages.
{ In \Cref{sec: closedness}, we show that  they generally should not vanish. One exception is when critic classes $\Qbbb'$ and $\Hbbb'$ are symmetric: in this case, the objective functions in \Cref{eq: est2-1,eq: est2-2} with $\lambda = 0$ are equivalent to their counterparts  in \Cref{eq: est1-1,eq: est1-2}.}

\subsection{Examples of Bridge Function Estimators}\label{sec: example}
In this part, we give examples of bridge function estimators based on three different critic function classes: linear class, RKHS, and neural networks. In particular, for the linear class and RKHS, the inner maximization problems in \Cref{eq: est1-1,eq: est1-2,eq: est2-1,eq: est2-2} have closed-form solutions, so that the minimax problems can be solved by standard optimization techniques such as stochastic gradient descent.

{ In this subsection, we focus on minimax estimators without stabilizers given in \Cref{eq: est1-1,eq: est1-2}. We can similarly compute minimax estimators with stabilizers given in \Cref{eq: est2-1,eq: est2-2}, and defer the details to \Cref{sec: stabilizers}.}

\subsubsection{Linear Classes}\label{sec: linear}
Given $\phi:\Zcal\times \Acal \times \Xcal \to \R{d_1},\psi: \Wcal\times \Acal \times \Xcal \to \R{d_2}$, set
\begin{align}
  \Qbbb' =\{ (z,a,x)\mapsto \alpha_1^\top\phi(z,a,x):\alpha_1 \in \mathbb{R}^{d_1},\|\alpha_1\|\leq c_1\}, \label{eq: lin-crit-q}\\
  \Hbbb'=\{ (w,a,x)\mapsto \alpha_2^\top \psi(w, a,x):\alpha_2 \in \mathbb{R}^{d_2},\|\alpha_2\|\leq c_2\}. \label{eq: lin-crit-h}
\end{align}
Typical classical examples of basis functions include splines, polynomials, and wavelets \citep{ChenXiaohong2007C7LS}. Another example is random feature expansions for positive definite kernels \citep{JMLR:v18:15-178}. When $d_1,d_2$ grow with $n$, these function classes are also called \emph{linear sieves}. See also discussions in \Cref{sec: sieve}.

It is easy to show that with linear critic classes, the inner maximum objectives in \Cref{eq: est1-1,eq: est1-2} have closed-form expressions, and the resulting bridge function estimators are:
\begin{align}\label{eq: linear-ls}\begin{aligned}
    &\textstyle\hat h \in \argmin_{h \in \Hbbb}\prns{\E_n[(Y-h ) \phi  ]}^{\top}\prns{\E_n[(Y-h)   \phi  ]},\\
    &\textstyle\hat q \in \argmin_{q\in \Qbbb}\prns{\E_n[q\epol \psi -\Tcal \psi]}^{\top}\prns{\E_n[q\epol \psi -\Tcal \psi]}.
\end{aligned}\end{align}
Given basis functions $\tilde \phi:\Zcal\times \Acal \times \Xcal \to \R{\tilde d_1}$ and $\tilde \psi: \Wcal\times \Acal \times \Xcal \to \R{\tilde d_2}$, if further $\Qbbb$ and  $\Hbbb$ are also linear classes (without norm constraints for simplicity), namely,
\begin{align}\label{eq: H-linear}\begin{aligned}
  &\epol\Qbbb=\{ (z,a,x)\mapsto \alpha_1^\top\tilde \phi(z,a,x):\alpha_1 \in \mathbb{R}^{\tilde{d}_1}\}.\\
  &\Hbbb=\{ (w,a,x)\mapsto \alpha_2^\top \tilde \psi(w,a,x):\alpha_2 \in \mathbb{R}^{\tilde{d}_2}\},
\end{aligned}\end{align}
Then,
{the corresponding IPW and REG estimators have s:
\begin{align*}
    \hat J_{\DM} &= \E_n[ \Tcal \tilde \psi]^{\top}\{\E_n[\tilde \psi \phi^{\top}]\E_n[\phi\tilde \psi^{\top} ]\}^{+}\E_n[ \tilde \psi \phi^{\top}]\E_n[ Y \phi],\\
    \hat J_{\ipw} &= {\E_n[\Tcal \psi]}^\top \E_n[\psi \tilde \phi^{\top}] \{\E_n[\tilde \phi \psi^{\top}]\E_n[ \psi \tilde \phi^{\top} ] \}^{+}\E_n[Y\tilde \phi].
\end{align*}
When $\tilde \phi=\phi,\tilde \psi=\psi$, we have $\hat J_{\DM}=  \hat J_{\ipw} = \hat J_{\dr}$. This extends a similar equivalence result for the unconfounded setting \citep{Singh2020,UeharaMasatoshi2019MWaQ,kallus2020generalized}.
}
\begin{lemma}\label{lem:equivalence}
Suppose $\tilde \phi=\phi$ and $\tilde \psi=\psi$. Then, we have
\begin{align}\label{eq:convinient}
    \hat J_{\ipw}=   \hat J_{\DM}=  \hat J_{\DR}=\E_n[\Tcal \psi ]^{\top}\E_n[\phi \psi^{\top}]^{+}\E_n[Y\phi].
\end{align}
\end{lemma}
\begin{continuance}[Average treatment effect]{\ref{ex: atomic}}
Consider $\Acal = \braces{0,1}$ and the parameter $\Eb{Y(a_0)}$ for $a_0\in\braces{0, 1}$ as in \Cref{ex: atomic}. Given  $\phi_b: \Zcal \times \Xcal \to \mathbb{R}^{d_1},\,\psi_b: \Wcal \times \Xcal \to \mathbb{R}^{d_2}$, we set the basis functions $\phi=\tilde \phi=(\indic{a=0}\phi^{\top}_b(z, x),\indic{a=1}\phi^{\top}_b(z, x))^{\top}$ and $\psi=\tilde \psi=(\indic{a=0}\psi^{\top}_b(w, x),\indic{a=1}\psi^{\top}_b(w, x))^{\top}$. Then,
\begin{align*}
 & \hat J_{\ipw}=   \hat J_{\DM}=  \hat J_{\DR}=  \hat \xi(a_0), \\
     & \hat \xi(a_0)\coloneqq  \E_n[\Ib(A=a_0)  \psi_b(W, X)]^{\top}\E_n[\Ib(A=a_0)\phi_b(Z, X) \psi^{\top}_b(W, X)]^{+}\E_n[\Ib(A=a_0)Y\phi_b(Z, X)].
\end{align*}
We can then estimate the average treatment effect by $\hat \xi(1) - \hat \xi(0)$.
\end{continuance}

{\neweedit
\begin{continuance}[Discrete setting]{\ref{ex: bridge-discrete}}
In \Cref{ex: bridge-discrete}, we let $W, Z, U$ be discrete. Now further assume that $X, A, Y$ are also discrete, and let $\phi, \psi$ be atomic basis functions corresponding to all possible values of $\prns{Z, A, X}$ and $\prns{W, A, X}$ respectively. Then  \Cref{eq:convinient} reduces to
\begin{align*}
    \textstyle
   \hat J_{\ipw}=   \hat J_{\DM}=  \hat J_{\DR} =\sum_{y \in \Ycal,a\in\Acal,x\in\Xcal}y\pi(a\mid x)\hat P(y, \bZ,a,x)\hat P(\bW, \bZ,a,x)^{+}\hat P(\bW,x),
\end{align*}
where $\hat P(y,\bZ,a,x), \hat P(\bW,\bZ,a,x), \hat P(\bW,x)$ are vectors and matrices consisting of sample frequency estimates of the corresponding probabilities. For example, $\hat P(y, \bZ,a,x)\in\R{\abs{Z}}$ is the column vector whose $j$th element is $\frac1n\sum_{i=1}^n \mathbb I(Y_i = y, Z_i = z_j, X_i = x, A_i = a)$.

\end{continuance}
}

\subsubsection{RKHS}

Consider two positive-semidefinite kernels $k_z:(\Zcal,\Acal,\Xcal)\times (\Zcal,\Acal,\Xcal)\to \mathbb{R}$ and  $k_w:(\Wcal,\Acal,\Xcal)\times (\Wcal,\Acal,\Xcal)\to \mathbb{R}$, and denote the induced RKHSs by $\mathcal L_z$ and $\mathcal{L}_w$ with RKHS norms $\|\cdot\|_{\mathcal{L}_z}$ and $\|\cdot\|_{\mathcal{L}_w}$, respectively. We consider the following critic classes:
\begin{align}
\Qbbb'  =\{q : q\in \mathcal{L}_z, \|q\|_{\mathcal{L}_z}\leq c_1\}, ~~ \Hbbb'  =\{h : h\in \mathcal{L}_w,\|h\|_{\mathcal{L}_w}\leq c_2\}. \label{eq: crit-rkhs}
\end{align}

\begin{lemma}\label{lemma: RKHS-1}
For $\Qbbb', \Hbbb'$ given in \Cref{eq: crit-rkhs}, the estimators in \Cref{eq: est1-1,eq: est1-2} are given by
\begin{align}
    \hat h
        &\in \textstyle\argmin_{h\in \Hbbb}~\prns{\psi_n\prns{h}}^{\top} K_{z,n}\psi_n\prns{h},\label{eq: est-rkhs-1-h}    \\
    \hat q
        &\in \textstyle\argmin_{q\in \Qbbb}~\prns{\phi_n\prns{q}}^{\top}K_{w1,n}\phi_n\prns{q}-2\prns{\phi_n\prns{q}}^{\top}K_{w2,n}\mathbf{1}_n,\label{eq: est-rkhs-1-q}
\end{align}
where $K_{z,n}$, $K_{w1,n}$, $K_{w2, n}$ are $n\times n$ Gram matrices whose $(i, j)$th entry is $k_z((Z_i,A_i,X_i),(Z_j,A_j,X_j))$, $k_w((W_i,A_i,X_i),(W_j,A_j,X_j)$, $\E_{\epol(A_j|X_j)}[k_w((W_i,A_i,X_i),(W_j,A_j,X_j))]$, respectively, and
$\psi_n\prns{h}\in\R{n}$, $\phi_n\prns{q} \in \R{n}$ are column vectors whose $i$th elements are  $(Y_i-h(X_i,A_i,Z_i))$, ${q(X_i,A_i,Z_i)\epol(A_i|X_i)}$, respectively, and $\mathbf{1}_n \in \mathbb R^{n}$ is an all-ones vector.
\end{lemma}
In \Cref{eq: est-rkhs-1-h,eq: est-rkhs-1-q}, computing the final bridge estimators only involves minimization problems whose objectives are convex in $h$ or $q$ and only depend on $h$ or $q$ via their evaluation at data points. If $\Hbbb, \Qbbb$ are linear hypothesis classes like those in \Cref{eq: H-linear}, then $\hat h$ and $\hat q$ also have closed-form solutions. If $\Hbbb$ is an RKHS hypothesis class with kernel $\tilde k$  with either norm constraints or norm regularizers, then the optimal solution to \Cref{eq: est-rkhs-1-h} will have the form $\hat h(z,a,x)=\sum_{i=1}^n\alpha_i \tilde k((Z_i,A_i,X_i),(z,a,x))$, leading to a convex quadratic program in $\alpha$ with a closed-form solution. The same holds for \Cref{eq: est-rkhs-1-q} if $\Qbbb$ is an RKHS hypothesis class. If the hypothesis classes $\Hbbb,\Qbbb$ are some more complex classes such as neural networks, then we can use stochastic gradient descent methods to solve for $\hat h, \hat q$, which have been shown to be highly successful in many nonconvex applications \citep{JainPrateek2017NOfM}.

 {

\begin{continuance}[Average treatment effect]{\ref{ex: atomic}}
Again, consider binary action $\Acal = \braces{0, 1}$ as in \Cref{ex: atomic} and the estimation of $\E[Y(a_0)]$ for $a_0\in\{0,1\}$.
Given kernels $\bar k_z:(\Zcal \times \Xcal)\times (\Zcal \times \Xcal)\to \RR$ and $\bar k_w:(\Wcal, \Xcal)\times (\Wcal, \Xcal)\to \RR$, we set
\begin{align*}
    &k_z((Z_i,A_i,X_i),(Z_j,A_j,X_j))=\indic{A_i=A_j}\bar k_z((Z_i,X_i),(Z_j,X_j)), \\
    &k_w((W_i,A_i,X_i),(W_j,A_j,X_j))=\indic{A_i=A_j}\bar k_w((W_i,X_i),(W_j,X_j)).
\end{align*}
Then we can solve \Cref{eq: est-rkhs-1-h,eq: est-rkhs-1-q} to get estimators $\hat h\prns{\cdot, a_0, \cdot}, \hat q\prns{\cdot, a_0, \cdot}$, using only observations corresponding to the action $a_0$. For the detailed calculation, see \Cref{subsec:rkhs}. Then we can use $\hat J_{\DM} = \E_n[\hat h(W,a_0,X)]$, $\hat J_{\ipw} = \E_n[\indic{A=a_0}\hat q\prns{Z, A, X}Y]$ or $\hat J_{\dr} = \E_n[\hat h(W,a_0,X)+\indic{A=a_0}\hat q\prns{Z, A, X}({Y - \hat h(W,A,X)})]$.

\end{continuance}
}

\subsubsection{Neural Networks}
An $L$-layer neural network with input $x$ can be generically written in the following form:
\begin{align*}
     \sigma_L(W_L(\sigma_{L-1}  (\ldots \sigma_{2}(W_{1}x + b_{1})\ldots)+b_{L-1})+ b_L )
\end{align*}
where $W_l \in \mathbb{R}^{d_{l, 1}\times d_{l, 2} },b_l \in \mathbb{R}^{d_{l,2}},\sigma_l$ for $l = 1, 2, \dots, L$ are called as weights, biases and activation functions, respectively. One standard choice for the activation functions is the ReLU function defined by $\sigma(x)=[ \max(x_1,0,),\ldots, \max(x_L,0)]^{\top}$. {When $\Qbbb,\Hbbb,\Qbbb',\Hbbb'$ are neural network classes, we need to solve non-convex minimax optimization problems to compute the minimax estimators. These can be solved by several types of simultaneous stochastic gradient descent methods. One method is the (simultaneous version of) Adam \citep{KingmaDiederikP2014AAMf}, which is a variant of gradient descent with momentum and per-parameter adaptive learning rates. An improved approach is the Optimistic Adam \citep{DaskalakisConstantinos2017TGwO}, which is an adaptation of Adam with additional negative momentum.}

\subsection{An Overview of the Estimation Theory for GACE} \label{sec: overview}

In \Cref{sec: realizable,sec: closedness,sec: unstab-closed}, we will show that GACE estimators based on different estimating equations and different minimax bridge functions estimators, either without stabilizers (\Cref{sec: strategy1}) or with stabilizers (\Cref{sec: strategy2}), have different theoretical properties.
Each type of estimator has its merits depending on the assumptions.

In the rest of the paper, we will \emph{always} assume that functions in classes $\Qbbb, \Hbbb, \Qbbb', \Hbbb'$ are square integrable. We further define  linear operators $P_z: L_2\prns{W, A, X} \rightarrow L_2\prns{Z, A, X}$ and $P_w: L_2\prns{Z, A, X} \rightarrow  L_2\prns{W, A, X}$ as follows:
\begin{align}
    P_z \prns{h}= \Eb{h\prns{W, A, X} \mid Z, A, X}, ~~ P_w \prns{q}=  \Eb{q\prns{Z, A, X} \mid W, A, X}. \label{eq: operator}
\end{align}
We will show that two additional types of assumptions will play an important role in theoretical guarantees for different estimators.
\begin{enumerate}
    \item \emph{Realizability}, which characterizes whether the classes $\Hbbb,\Qbbb,\Hbbb',\Qbbb'$ contain some observed bridge functions,

    \ie, $\Hbbb\cap\Hbbb^{\obs}_0 \ne \emptyset,\,\Qbbb \cap \Qbbb^{\obs}_0 \ne \emptyset$, or $\Hbbb'\cap\Hbbb^{\obs}_0 \ne \emptyset,\,\Qbbb' \cap \pi\Qbbb^{\obs}_0 \ne \emptyset$.

    \item \emph{Closedness}, which characterizes whether the critic classes $\Hbbb', \Qbbb'$ are rich enough and $P_z,P_w$ are smooth enough,

    requiring $P_z\prns{\Hbbb - \Hbbb^{\obs}_0} \subseteq  \Qbbb'$ and/or $\epol P_w\prns{\Qbbb - \Qbbb^{\obs}_0}\subseteq \Hbbb'$.
\end{enumerate}
We remark that both assumptions are model specification assumptions. While it is difficult to validate these assumptions from data, specifying richer function classes generally make them more plausible. In this paper, we aim to develop estimation methods and theory that accommodate general function classes, thereby addressing the limitations of existing approaches based on strong parametric restrictions. Moreover, our theory is general and can also characterize the potential errors due to partial violation of the realizability and closedness conditions.  

The theoretical results for our GACE estimators are organized as follows:

    \begin{itemize}
        \item In \Cref{sec: realizable}, we  derive finite-sample error bounds for GACE estimators based on minimax bridge function estimators \emph{without stabilizers}  (\Cref{sec: strategy1}), under \emph{double} realizability assumptions for \emph{both} the bridge hypothesis class and the critic class\footnote{{In special cases such as linear models, we only require realizability assumptions for the bridge hypothesis class \emph{or} the critic class. See discussions in \Cref{sec: single-realizability}}.}. Note that the bounds are derived for the GACE estimators directly, regardless of the convergence of the bridge function estimators.
        \item In \Cref{sec: unstab-closed}, we   assume a realizability assumption on the bridge hypothesis class and a closedness assumption on the critic class, under which we
        derive error bounds for the minimax bridge function estimators \emph{without stabilizers} in terms of a projected mean-squared-error (MSE) metric. Then we translate these into the error bounds for the downstream plug-in GACE estimators.
        \item   In \Cref{sec: closedness}, we  derive \emph{faster-rate} projected MSE error bounds for minimax bridge function estimators \emph{with stabilizers}  under the same assumption as \Cref{sec: unstab-closed}. Then we also translate these error bounds into the error bounds for the downstream plug-in GACE estimators. Finally, we show that when we additionally assume that bridge functions are unique and that the conditional moment equations in \Cref{eq: observed-bridge-h,eq: observed-bridge-q} are not too ill-posed, we have that
        the resulting doubly robust estimator $\hat J_{\DR}$ is asymptotically normal with asymptotic variance equal to the semiparametric efficiency bound.
    \end{itemize}

In \Cref{tab:summary}, we summarize our results for different estimators under the different assumptions. We remark that although the results in \Cref{tab:summary} are based on  covering numbers, our theory in \Cref{sec: realizable,sec: unstab-closed,sec: closedness} provides  meta theorems that can accommodate a wide range of function classes (e.g., VC subgraph classes, nonparametric smooth classes, neural networks). These meta theorems are instantiated to different function classes to obtain concrete error bounds.  

\begin{table}[t]
\caption{\label{tab:summary} Convergence rates of different GACE estimators under different assumptions.
The identification assumptions in \Cref{lemma: relaxed-ipw-dm} are always assumed.
The sets $\Hbbb^{\obs}_0,\Qbbb^{\obs}_0$ denote the sets of observed bridge functions (see \Cref{lemma: observed-bridge}). The function classes $\Hbbb,\Qbbb'$ are used to construct the estimator $\hat h$
and $\Qbbb,\Hbbb'$ to construct the estimator $\hat q$.
{We here summarize our conclusions (ignoring polylogs) when all relevant function classes are nonparametric with log covering number at radius $\epsilon$ scaling as $\epsilon^{-\beta}$.}
``Est'' refers to which estimating equation is used for the final estimator (see \Cref{sec: overview}).
``Sta" refers to whether the bridge function estimators use stabilizers (see \Cref{sec: reformuate-moment}).
``Uni'' refers to whether we assume bridge functions are unique. The projection operators $P_z,P_w,P_u$ are defined in \Cref{eq: operator,eq:Pu}.
}
{\newcommand{\notecolsize}{0.205}\newcommand{\ratecolsize}{0.105}\small\begin{tabular}{m{0.015\textwidth}m{0.455\textwidth}m{0.04\textwidth}m{\ratecolsize\textwidth}m{0.0125\textwidth}m{0.0125\textwidth}m{\notecolsize\textwidth}}\toprule
         & Main Assumptions  & {Est}

         & Convergence Rate &  Sta &  Uni & ~Notes   \\ \midrule
        \multirow{4}{*}{\RNum{1}}  &  {$\Qbbb^{\obs}_0 \cap \Qbbb \neq \emptyset ,\Hbbb^{\obs}_0 \cap \Hbbb' \neq \emptyset$}  & IPW  &   \multirow{4}{*}{{$n^{-\min(\frac{1}{2},
        \beta
        )}$}} &  \multirow{4}{*}{N}  &  \multirow{4}{*}{N}  &  \multirow{4}{\notecolsize\textwidth}{$\hat h, \hat q$ need not converge to any point in $\Hbbb^{\obs}_0,\Qbbb^{\obs}_0$}   \\  \cline{2-3}
        &  {$\Hbbb^{\obs}_0 \cap \Hbbb \neq \emptyset,\epol\Qbbb^{\obs}_0 \cap \Qbbb' \neq \emptyset$}  & REG  \\  \cline{2-3}
        &  {$\Qbbb^{\obs}_0\cap \Qbbb\neq \emptyset,\Hbbb^{\obs}_0\cap \{h: h-\Hbbb \subseteq \Hbbb'\}\neq \emptyset$} &  DR  \\\cline{2-3}
        & {$\Hbbb^{\obs}_0\cap \Hbbb\neq \emptyset,\Qbbb^{\obs}_0\cap \{q: \epol(q-\Qbbb)\subseteq \Qbbb'\}\neq \emptyset$} & DR \\
        \hline

        \multirow{2}{*}{\RNum{2}} &   \multirow{2}{*}{$\Qbbb^{\obs}_0\cap \Qbbb\neq \emptyset,\epol P_w(\Qbbb-\Qbbb^{\obs}_{0})\subseteq\Hbbb'$} &  \multirow{2}{*}{\rotatebox{90}{IPW~}\,\,\rotatebox{90}{/\,DR~}}  &  $n^
        {-\frac{\beta}{2\beta+1}}
        $  &   Y  &  \multirow{2}{*}{N} &    \multirow{2}{\notecolsize\textwidth}{$\|\epol P_w(\hat q-q_0)\|_2 \to 0$ for any $q_0\in \Qbbb^{\obs}_0$}    \\  \cline{4-5} & & &  $n^{-\min(\frac{1}{4},
        \frac{\beta}{2}
        )}$& N  \\ \hline

      \multirow{2}{*}{\RNum{3}} &   \multirow{2}{*}{$\Hbbb^{\obs}_0\cap \Hbbb\neq \emptyset,P_z(\Hbbb-\Hbbb^{\obs}_{0})\subseteq \Qbbb'$} &  \multirow{2}{*}{\rotatebox{90}{REG~}\,\,\rotatebox{90}{/\,DR~}} & $
     n^
     {-\frac{\beta}{2\beta+1}}
     $   &  Y & \multirow{2}{*}{N} &  \multirow{2}{\notecolsize\textwidth}{$\|P_z(\hat h-h_0)\|_2 \to 0$ for any $h_0\in \Hbbb^{\obs}_0$
     }  \\   \cline{4-5}
      &  &  &$n^{-\min(\frac{1}{4},
      \frac{\beta}{2}
      )}$ & N
     \\  \hline

      \multirow{6}{*}{\RNum{4}} &  {Assumptions in rows \RNum{2} and \RNum{3},} & \multirow{3}{*}{DR} & \multirow{3}{\ratecolsize\textwidth}{{$\max(n^{-\frac{1}{2}},\break\phantom.~~\tau^2_1 n^{-
      \frac{2\beta}{2\beta+1}
      })$}} & \multirow{3}{*}{Y} & \multirow{3}{*}{N} & \multirow{3}{\notecolsize\textwidth}{Faster than row \RNum{1} if $\tau^{2}_{1}=o(n^{
      \frac{\beta\prns{1-2\beta}}{2\beta+1}
      })$} \\
      & {$\|P_u(\hat q\epol-q_0\epol)\|_2\leq \tau_1\|P_w(\hat q\epol-q_0\epol)\|_2$,} \\
      & {$\|P_u(\hat h-h_0)\|_2\leq \tau_1\|P_z(\hat h-h_0)\|_2$} \\
      \cline{2-7}
      & {Assumptions in rows \RNum{2} and \RNum{3},} & \multirow{3}{*}{DR} & \multirow{3}{\ratecolsize\textwidth}{{$\max(n^{-\frac{1}{2}},\break\phantom.~~\tau_2 n^{-
      \frac{2\beta}{2\beta+1}
      })$}}  &  \multirow{3}{*}{Y} & \multirow{3}{*}{Y}  & \multirow{3}{\notecolsize\textwidth}{Achieves efficiency if
     $\tau_2=o(n^{
     \frac{2\beta - 1}{4\beta+2}
     })$} \\
     & {$\|\hat q\epol-q_0\epol\|_2\leq \tau_2\|P_w(\hat q\epol-q_0\epol)\|_2$,} \\
     & {$\|\hat h-h_0\|_2\leq \tau_2\|P_z(\hat h-h_0)\|_2$}\\
     \bottomrule
    \end{tabular}}
\end{table}

\section{Finite-sample Analysis of Estimators without Stabilizers under Realizability}\label{sec: realizable}
In this section, we analyze different GACE estimators based on minimax estimation of bridge functions \emph{without} stabilizers, \ie, minimax estimators $\hat h,\hat q$ given in \Cref{eq: est1-1,eq: est1-2}. Throughout this section, we  assume \emph{only} realizability but \emph{not} closedness.

\subsection{Error Bounds of GACE estimators}\label{seq: error-gace-nostab}

First, we bound the errors of the IPW and REG estimators without stabilizers. Then, we consider analogous results for DR.

\begin{theorem}[Analysis of IPW and REG estimators]\label{thm:ipw_reg_no_stab}
For any $h_0\in \Hbbb^{\obs}_0,q_0\in \Qbbb^{\obs}_0$,
\begin{align}
&\begin{aligned}
   | \hat J_{\ipw}-J| &\leq \sup_{q\in \Qbbb}|(\E_n - \E)[q\epol Y]|+2\sup_{q\in \Qbbb,h\in \Hbbb'}|(\E_n - \E) [-q\epol h+\Tcal h]|+ \\[-0.6em]
      &\phantom{\leq}+\inf_{h\in \Hbbb',} \sup_{q\in \Qbbb}|\E[(q-q_0)\epol(h_0-h)]|+\inf_{q\in \Qbbb} \sup_{h\in \Hbbb'}|\E[(q_0-q)\epol h ]|.
      \end{aligned}\label{eq:ipw_no_stab}
\\&\begin{aligned}
    |\hat J_{\DM}-J|&\leq  \sup_{h\in \Hbbb}|(\E_n - \E)[\Tcal h ]|+2\sup_{q\in \Qbbb',h\in \Hbbb}|(\E_n - \E)[q(Y-h) ]  |\\[-0.6em]
    &\phantom{\leq}+\inf_{h\in \Hbbb}\sup_{q\in \Qbbb'}|\E[\{h_0-h\}q]|+\inf_{q\in \Qbbb'}\sup_{h\in \Hbbb}|\E[(q_0\epol-q) (h-h_0)|.
    \end{aligned}\label{eq:reg_no_stab}
\end{align}
\end{theorem}
The first two empirical process terms on the right-hand sides of each of \Cref{eq:ipw_no_stab,eq:reg_no_stab} account for the ``variance'' between the empirical and population estimating equations.
In \Cref{sec: rate_analysis}, we will show that these terms converge to $0$ for some choices of $\Hbbb, \Hbbb', \Qbbb, \Qbbb'$.
The last two terms on the right-hand sides of each of \Cref{eq:ipw_no_stab,eq:reg_no_stab} are ``bias'' terms accounting for how well our bridge and critic classes satisfy the realizability conditions. If the realizability conditions are violated up to a certain amount, then these two terms can be bounded accordingly. More generally, we can grow the function classes to make ``bias'' term vanish eventually and even to balance the ``bias'' and ``variance'' terms (see \Cref{exa:sieve_neural} below for sieve classes and neural network classes). If the realizability conditions hold exactly, then the two bias terms become zero, leading to the simplified bounds given below. 

\begin{corollary}[IPW and REG under realizability]\label{cor: ipw_reg_no_stab}~

    If $\Qbbb\cap\Qbbb^{\obs}_0\neq \emptyset, \Hbbb' \cap \Hbbb^{\obs}_0\neq \emptyset$, then

   \begin{align*}
   |\hat J_{\ipw}-J| &\textstyle\leq \sup_{q\in \Qbbb}\abs{(\E_n-\E)\bracks{q\epol Y}} +2\sup_{q\in \Qbbb,h\in \Hbbb'}\abs{(\E_n-\E) \bracks{-q\epol h+\Tcal h }}.
    \end{align*}

    If $\Hbbb \cap \Hbbb^{\obs}_0  \neq \emptyset ,\Qbbb' \cap  \epol\Qbbb^{\obs}_0 \neq \emptyset$, then

    \begin{align*}
    |\hat J_{\mathrm{REG}}-J|&\textstyle\leq  \sup_{h\in \Hbbb} \abs{(\E_n - \E)\bracks{\Tcal h }}+2\sup_{q\in \Qbbb',h\in \Hbbb}\abs{(\E_n -\E)\bracks{q(Y-h)}}.
\end{align*}
\end{corollary}

Interestingly, and in stark contrast to the unconfounded setting, these bounds show that the GACE estimators $\hat J_{\mathrm{REG}},\hat J_{\ipw}$ may converge to the true GACE even when the bridge function estimators $\hat h$, $\hat q$ do not converge to any valid observed bridge functions in $\Hbbb^{\obs}_0$, $\Qbbb^{\obs}_0$, respectively.
This is possible because the loss functions for the minimax estimators $\hat h, \hat q$ without stabilizers are related to the errors of estimators $\hat J_{\ipw}$ and $\hat J_{\mathrm{REG}}$. In our proofs for \Cref{thm:ipw_reg_no_stab} and \Cref{cor: ipw_reg_no_stab}, we  decompose the errors of $\hat J_{\ipw}$ and $\hat J_{\mathrm{REG}}$ into some vanishing empirical process terms and some terms directly related to the minimax objectives without stabilizers. Then we only use the fact that $\hat h, \hat q$ minimize those objectives to handle those terms, without requiring $\hat h, \hat q$ to converge to any valid bridge functions. This is how we manage to prove the consistency of the IPW and REG estimators without consistency of the bridge function estimators. Below we provide a simple example for such phenomenon. 

\begin{example}\label{ex: inconsist-h}
Suppose $\Hbbb^{\obs}_0=\{h_0\},\,\Hbbb=\{a_1+a_2h_0:a_1\in \Rl,a_2\in \Rl\},\,\Qbbb^{\obs}_0=\{q_0\},\,\Qbbb'=\{\epol q_0\}$\footnote{The classes $\Hbbb$ and $\Qbbb'$ are designed for analytic tractability. In practice, one may hardly specify such classes because they involve the ground truth $h_0, q_0$. But this simple example conveniently illustrates the key phenomenon.}.
Then we can solve the minimax optimization problem in \Cref{eq: est1-1} to get the outcome bridge function estimator $\hat h$. The bridge function estimator can be then used to form the estimator $\hat J_{\mathrm{REG}}$. To understand the behavior of $\hat h$, we can analyze the population analogue of the minimax optimization problem. It is easy to verify that the minimizers for  the population minimax objective
are given by
\begin{align*}
  \textstyle  \argmin_{h\in \Hbbb}~\sup_{q \in \Qbbb'} ~\prns{\Eb{ q \prns{h-Y}}}^2
        = \braces{a_1+a_2h_0: a_1=(1-a_2)\E[ \epol q_0h_0]/\E[ \epol q_0], a_2 \in \mathbb R}.
\end{align*}
This means that the truth $h_0$ is only one of the possible population minimizers. Therefore, the estimator $\hat h$ that minimizes the empirical analog of the minimax objective  generally would not converge to $h_0$. Nevertheless, it is easy to show that the double realizability conditions $\Hbbb^{\obs}_0\cap \Hbbb\neq \emptyset, \epol\Qbbb^{\obs}_0\cap \Qbbb'\neq \emptyset$  are satisfied. Then \Cref{cor: ipw_reg_no_stab} guarantees  that $|\hat J_{\mathrm{REG}}-J|$ converges to $0$ as $n$ goes to infinity. This provides an example where the GACE estimator is consistent but the bridge function estimator is not. 

\end{example}

\begin{theorem}[Analysis of DR estimator]\label{thm:dr_no_stab}
For any $h_0\in \Hbbb^{\obs}_0,q_0\in \Qbbb^{\obs}_0$,
\begin{align*}
     | \hat J_{\DR}-J| &\leq \sup_{q\in \Qbbb,h\in \Hbbb}|(\E_n - \E)[q\epol Y-q\epol h+\Tcal h]|+2\sup_{q\in \Qbbb,h\in \Hbbb'}|(\E_n - \E) [-q\epol h+\Tcal h]|+\\[-0.6em]
      &\phantom{\leq}+\inf_{q\in \Qbbb} \sup_{h\in \Hbbb'}|\E[(q_0-q)\epol h ]|+\inf_{h\in \Hbbb'}\sup_{h'\in \Hbbb} \sup_{q\in \Qbbb}|\E[(q-q_0)\epol(h_0-h'-h)]|,\\
    |\hat J_{\DR}-J|&\leq  \sup_{q\in \Qbbb,h\in \Hbbb}|(\E_n - \E)[q\epol Y-q\epol h+\Tcal h ]|+2\sup_{q\in \Qbbb',h\in \Hbbb}|(\E_n - \E)[q(Y-h) ]  |\\[-0.6em]
    &\phantom{\leq}+\inf_{h\in \Hbbb}\sup_{q\in \Qbbb'}|\E[\{h_0-h\}q]|+\inf_{q\in \Qbbb'}\sup_{q'\in \Qbbb}\sup_{h\in \Hbbb}|\E[(q_0\epol-q'\epol
    -q) (h-h_0)|.
\end{align*}
\end{theorem}
In each bound, the first two terms are ``variance'' terms and the last two terms are ``bias'' terms analogous to those in \Cref{thm:ipw_reg_no_stab}.
Under realizability, these bias terms again become zero.

\begin{corollary}[DR under realizability]\label{cor: dr_no_stab}
If $\Qbbb^{\obs}_0\cap \Qbbb\neq \emptyset, \Hbbb^{\obs}_0 \cap \{h:h-\Hbbb \subseteq \Hbbb'\}\neq \emptyset$,
\begin{align*}
         | \hat J_{\DR}-J|&\textstyle\leq \sup_{q\in \Qbbb,h\in \Hbbb}|(\E_n-\E)[q\epol Y-q\epol h+\Tcal h]|+4\sup_{q\in \Qbbb,h\in \Hbbb'}|(\E_n-\E) [-q\epol h+\Tcal h]|.
\end{align*}
If $\Hbbb^{\obs}_0\cap \Hbbb\neq \emptyset, \Qbbb^{\obs}_0 \cap \{q: \epol(q-\Qbbb)\subseteq \Qbbb'\}\neq \emptyset$, then
\begin{align*}
         | \hat J_{\DR}-J|&\textstyle\leq \sup_{q\in \Qbbb,h\in \Hbbb}|(\E_n-\E)[q\epol Y-q\epol h+\Tcal h]|+4\sup_{q\in \Qbbb',h\in \Hbbb}|(\E_n-\E)[q(Y-h) ]  |.
\end{align*}
\end{corollary}

\Cref{cor: dr_no_stab} suggests that if \emph{either} $\Qbbb$ \emph{or} $\Hbbb$ is well-specified,
and the
associated $\Hbbb'$ or $\Qbbb'$ is rich enough relative to $\Hbbb$ or $\Qbbb$, then $\hat J_{\DR}$ is consistent, provided these function classes have limited complexity to ensure convergence of the empirical process terms above. In particular, $\hat J_{\DR}$ is consistent  when  either empirical process conditions on $\Qbbb, \Hbbb'$ or on $\Hbbb, \Qbbb'$ hold.
Moreover, when $\Hbbb = \braces{0},\Qbbb'=\braces{0}$ so that  $\hat J_{\DR}$ reduces to $\hat J_{\ipw}$,
 the first equation exactly recovers the result for $\hat J_{\ipw}$ in  \Cref{cor: ipw_reg_no_stab}.
Analogously, when $\Qbbb = \braces{0},\Hbbb'=\braces{0}$,  the second equation recovers the result for $\hat J_{\DM}$ in \Cref{cor: ipw_reg_no_stab}.

\subsection{Convergence Rates of GACE Estimators for Common Function Classes}\label{sec: rate_analysis}

We next analyze the convergence rates of $\hat J_{\mathrm{REG}}$ and $\hat J_{\ipw}$  by further bounding the ``variance'' terms in \Cref{thm:ipw_reg_no_stab,cor: ipw_reg_no_stab} in terms of the complexity of some common function classes.
For simplicity and brevity, we set  $\Hbbb'=\Hbbb$ and $\Qbbb'=\epol \Qbbb$.
The results can easily be specialized to each of $\hat J_{\mathrm{REG}}$ or $\hat J_{\ipw}$ { without these restrictions} when only the corresponding pair of bridge and critic classes are realizable.
Each result we present also implies the same convergence rate for $\hat J_{\mathrm{DR}}$ if we simply modify the critic
class realizability condition to be as in \Cref{thm:dr_no_stab,cor: dr_no_stab}. We omit these for brevity.

\begin{example}[VC-subgraph classes]
\label{exa: vc_subgraph}
VC-subgraph classes are function classes whose subgraph sets have bounded VC dimension \citep[Chapter 19]{VaartA.W.vander1998As}.
{For example, $\{\theta^{\top}\phi(\cdot):\|\theta\|_2\leq 1,\theta \in \mathbb{R}^d \}$ has VC-subgraph dimension at most $d+1$.}

\begin{corollary}\label{cor: vc_class}
Let $\Hbbb, \Qbbb$ have finite VC-subgraph dimensions $V(\Hbbb), V(\Qbbb)$, respectively. Assume  $\Hbbb\cap \Hbbb^{\obs}_0\neq \emptyset$, $\Qbbb \cap \Qbbb^{\obs}_0\neq \emptyset$, and $\|\Hbbb\|_{\infty},\|\Qbbb\|_{\infty}, \|\pi\prns{A\mid X}\|_{\infty}, \|Y\|_{\infty}<\infty$.
Then, letting $O(\cdot)$ be the order w.r.t. $n,V(\Hbbb)$, $V(\Qbbb)$, and $\delta$, with probability $1-\delta$, we have
 \begin{align*}
          \max\{| \hat J_{\DM}-J|, ~| \hat J_{\ipw}-J|\}=O(\sqrt{(V(\Hbbb)+V(\Qbbb)+1+\log(1/\delta))/n}).
 \end{align*}
\end{corollary}

\end{example}

\begin{example}[Nonparametric classes characterized by metric entropy]\label{exa: nonpara}

Many common nonparametric classes cannot be characterized by VC-subgraph dimensions. Instead, their complexity is characterized by their metric entropies \citep{WainwrightMartinJ2019HS:A}.
For example, a {\Holder} ball $\Wbbb$ with smoothness level $\alpha$ and an input dimension $d$ has metric entropy under infinity norm $\log \mathcal{N}(\varepsilon,\Wbbb,\|\cdot\|_{\infty})=O(\varepsilon^{-d/\alpha})$. { Nonparametric function classes in RKHS such as an RKHS with Mat\' ern kernels and an RKHS with Gaussian kernels also have metric entropy characterizations \citep{kuhn2011covering}. }

\begin{corollary}\label{cor: nonparametric}
Suppose  $\max\braces{\log \Ncal(\varepsilon,\Hbbb,\|\cdot\|_{\infty}), \log \Ncal(\varepsilon,\Qbbb,\|\cdot\|_{\infty})} \le c_0\varepsilon^{-\beta}$ for $\beta>0$ and $c_0 > 0$. Further assume  $\Hbbb\cap \Hbbb^{\obs}_0\neq \emptyset$, $\Qbbb\cap \Qbbb^{\obs}_0\neq \emptyset$, and $\|\Hbbb\|_{\infty},\|\Qbbb\|_{\infty}, \|\pi\prns{A\mid X}\|_{\infty}, \|Y\|_{\infty}<\infty$. Then, letting $O(\cdot)$ be the order w.r.t. $n$ and $\delta$, with probability $1-\delta$, we have
\begin{align*}
             \max\braces{| \hat J_{\DM}-J|, ~| \hat J_{\ipw}-J|}
              =\begin{cases}  O(n^{-1/2}+\sqrt{\log(1/\delta)/n})&\quad \beta<2 \\  O(n^{-1/2}\log(n)+\sqrt{\log(1/\delta)/n})&\quad \beta=2 \\ O(n^{-1/\beta}+\sqrt{\log(1/\delta)/n})&\quad \beta>2  \end{cases}
\end{align*}
\end{corollary}

\Cref{cor: nonparametric} states that the convergence rates of $\hat J_{\mathrm{REG}}$ and $\hat J_{\ipw}$ are determined by  the worse of metric entropies of $\Hbbb $ and $\Qbbb$. For example, when we use {\Holder} balls for both, the convergence rate is $O(n^{-\min(1/2,\alpha/d)})$. This implies that when these two function classes are Donsker classes \citep{VaartA.W.vander1998As}, \ie, $\beta<2$, both estimators have parametric convergence rates. But if either function class is non-Donsker, \ie, $\beta\geq 2$, then the estimators typically have slower convergence rates.

\end{example}

\begin{example}[Neural networks and linear sieves]\label{exa:sieve_neural}
{ Many nonparametric function classes (\eg, {\Holder} balls as in \Cref{exa: nonpara}) are not amenable to direct optimization. A practical solution to this is to use sieves, which approximate the target nonparametric classes by classes that grow as the sample size $n$ grows. For example, splines and polynomial sieves have been widely used in sieve estimators \citep{ChenXiaohong2007C7LS}. Neural networks can be also regarded as sieves \citep{yarotsky2017error,gribonval2021approximation,suzuki2018adaptivity}. When $\Hbbb,\Qbbb,\Hbbb',\Qbbb'$ are sieve classes they depend on $n$ and this dependence is implicit in our notation.}
{ Since sieve classes approach target nonparametric classes only as sample size grows infinitely, it is not suitable to impose realizability on these sieve classes for any finite $n$.
Instead, we must consider the convergence rate of the ``bias'' terms in \Cref{thm:ipw_reg_no_stab,thm:dr_no_stab}. In \Cref{sec: sieve-nn}, we show that
GACE estimation error bounds scale as $O(n^{-\alpha/(2\alpha+d)})$ when we use linear sieves to approximate {\Holder} balls or when we use neural networks to approximate Sobolev balls, each with smoothness $\alpha$ in $d$ input dimensions and when each limiting nonparametric class satisfies realizability.}

\end{example}

\section{Finite-sample Analysis of Estimators without Stabilizers under Realizability and Closedness}\label{sec: unstab-closed}
In the previous section, we studied the convergence of our GACE estimators based on minimax bridge function estimators by assuming the realizability of $\Hbbb,\Qbbb'$ (or of $\Hbbb',\Qbbb$). { In particular, our convergence analysis does not require the convergence of bridge function estimators. }
In this section, we show that under a different set of assumptions, realizability, and closedness, bridge function estimators do converge in a suitable notion, which then implies convergence of GACE estimators.

\subsection{Convergence of Bridge function estimators}
For any observed bridge functions $h_0 \in \Hbbb^{\obs}_0$ and $q_0 \in \Qbbb^{\obs}_0$, we quantify the estimation errors of bridge function estimators $\hat h$ and $\hat q$ by  $\|P_z(\hat h-h_0) \|_2^2$ and $ \|\epol P_w(\hat q-q_0) \|_2$, which we call  projected mean-squared errors (projected MSE).
Projected MSE measures how much $\hat h$ and $\hat q$ violate
\Cref{eq: observed-bridge-h,eq: observed-bridge-q}:
\begin{align}
     \|P_z(\hat h-h_0) \|_2^2 &=    \E[{({\E[Y-\hat h(W, A,X) \mid Z, A, X]})^2}],\, \label{eq: error-measure-h} \\
   \|\epol P_w(\hat q-q_0) \|_2^2  &= \E[{({\E[\epol(A\mid X)\prns{\hat q(Z, A, X) - 1/f(A \mid W, X)} \mid W, A, X]})^2}]. \label{eq: error-measure-q}
\end{align}
Obviously, these estimation errors are invariant to the choice of $h_0$ and $q_0$ so they are particularly relevant when bridge functions are nonunique. Note that even when $\|P_z(\hat h-h_0) \|_2 \to 0$ and $\|\epol P_w(\hat q-q_0) \|_2 \to 0$, $\hat h$ and $\epol \hat q$ may not necessarily converge to any fixed limits in terms of the $\|\cdot\|_2$ norm since $\Hbbb^{\obs}_0$ and $\Qbbb^{\obs}_0$ need not be singletons.

\begin{theorem}\label{thm:slow_rates}
Consider the bridge function estimators without stabilizers, $\hat h, \hat q$ in \Cref{eq: est1-1,eq: est1-2}.
\begin{enumerate}
    \item \label{thm:slow_rates-1} Suppose $\Hbbb\cap \Hbbb^{\obs}_0 \neq \emptyset$ and take some $h_0\in \Hbbb\cap \Hbbb^{\obs}_0$. If $P_z(\Hbbb-h_0)\subseteq \Qbbb'$
    , then
\begin{align*}
        \|P_z(\hat h-h_0)\|^2_2 \leq \textstyle 2\sup_{q\in \Qbbb',h\in \Hbbb} |(\E_n-\E)[\{y-h\}q] |.
\end{align*}
    \item \label{thm:slow_rates-2} Suppose $\Qbbb\cap \Qbbb^{\obs}_0 \neq \emptyset$ and take some $q_0\in \Qbbb\cap \Qbbb^{\obs}_0$. If $\pi P_w(\Qbbb-q_0)\subseteq \Hbbb'$,

    then
\begin{align*}
     \|\pi P_w(\hat q-q_0)\|^2_2 \leq \textstyle 2\sup_{q\in \Qbbb,h\in \Hbbb'} |(\E_n-\E)[-q\epol h+\Tcal h] |.
\end{align*}
\end{enumerate}
\end{theorem}

In \Cref{thm:slow_rates} statement \ref{thm:slow_rates-1}, we assume the realizability condition for $\Hbbb$, and the closedness condition $P_z(\Hbbb-h_0)\subset \Qbbb'$. The latter closedness condition is invariant to choice of $h_0 \in \Hbbb^{\obs}_0 \cap \Hbbb$, by following \Cref{eq: observed-bridge-h}. This closedness condition intuitively indicates that the critic class $\Qbbb'$ is rich enough relative to the hypothesis class $\Hbbb$ and the operator $P_z$ is smooth enough. A similar observation is made in \Cref{thm:slow_rates} statement \ref{thm:slow_rates-2}.

{ We can further bound the empirical process terms in \Cref{thm:slow_rates} for specific bridge classes and critic classes by  following the same calculations as in \Cref{sec: rate_analysis}, which then leads to the projected MSE convergence rates of bridge function estimators {without} stabilizers. For example, when we use VC-subgraph classes for all function classes, the rate is $O(n^{-1/4})$, and when we use {\Holder} balls for all function classes, the rate is $O(\max(n^{-1/4},n^{-\alpha/(2d)}))$.}

\subsection{Convergence of GACE Estimators}\label{sec: closed-policy-analysis}

In this part, we bound the errors of GACE estimators based on \Cref{thm:slow_rates}. First, we first introduce Rademacher complexity, a widely used function class complexity measure \citep{WainwrightMartinJ2019HS:A}.
\begin{definition}[Rademacher Complexity]
Given a class $\mathcal{G}$ of functions of variables $O$, let
$$\ts\Rad(\Gcal)=\frac1{2^n}\sum_{\epsilon\in\{-1,+1\}^n}{\sup_{g\in \Gcal}\frac{1}{n}\sum_{i=1}^n \epsilon_i g(O_i)}.$$
Note that $\Rad(\Gcal)$ is random as it depends on our sample $O_1,\dots,O_n$.
\end{definition}
The following gives the convergence rates of $\hat J_{\DM}$ and $\hat J_{\ipw}$ in terms of the Rademacher complexity of some function classes, as well as the projected MSEs of bridge function estimators $\hat h,\hat q$. We then derive the convergence rates of the doubly robust estimator $\hat J_{\DR}$.

\begin{theorem}
{
\label{thm:ipw-dm}
Take arbitrary $h_0 \in \Hbbb_0^\obs$ and $q_0 \in \Qbbb_0^\obs$ and assume $\|\pi q_0\|_\infty < \infty$, $\|\pi\Qbbb\|_2 < \infty$, $\|h_0\|_\infty < \infty$, $\|\Hbbb\|_2 < \infty$, $\|Y\|_{\infty} < \infty$. Then, for some universal constants $c_1$ and $c_2$, with probability at least $1-\delta$,
\begin{align}
    |\hat J_{\mathrm{REG}}-J| \leq  \Rad(\Tcal\Hbbb)+\|\epol q_0\|_2\|P_z( \hat h - h_0)\|_2 +c_1\sqrt{(\log(c_2/\delta))/n}, \label{eq:dm} \\
    |\hat J_{\mathrm{IPW}}-J| \leq    \Rad(\epol\Qbbb)+\|h_0\|_2\|\epol P_w(\hat q - q_0)\|_2+ c_2\sqrt{(\log(c_2/\delta))/n}.  \label{eq:ipw}
\end{align}
}
\end{theorem}

In the right-hand sides of \Cref{eq:dm,eq:ipw}, the  Rademacher complexity terms bound the stochastic equicontinuity term due to nuisance plug-in \citep{newey94}.
The projected MSE terms correspond to the biases due to plug-in nuisance estimates. These terms can be bounded according to \Cref{thm:slow_rates}.
Typically, these bias terms dominate in \Cref{eq:dm,eq:ipw}. Therefore, the convergence rates of $\hat J_{\mathrm{REG}},\hat J_{\mathrm{IPW}}$ are typically the same as the projected MSE convergence rates of bridge function estimators. Notably, these convergence rates are often slower than the convergence rates in \Cref{cor: ipw_reg_no_stab}. For example, for nonparametric classes with a metric entropy exponent $\beta$ (see \Cref{exa: nonpara}), the convergence rates in \Cref{thm:ipw-dm} become $O\prns{\max(n^{-1/4},n^{-1/2\beta} )}$, while the convergence rates in \Cref{cor: ipw_reg_no_stab} become $O\prns{\max(n^{-1/2},n^{-1/\beta} )}$ according to \Cref{cor: nonparametric}.
{ This shows that even when using the same minimax estimators without stabilizers, different assumptions (realizability in \Cref{cor: ipw_reg_no_stab} and realizability plus closedness in \Cref{thm:ipw-dm}) can lead to different convergence rates. }

\begin{theorem}
\label{thm:drr}
{
Take arbitrary $h_0 \in \Hbbb_0^\obs$ and $q_0 \in \Qbbb_0^\obs$ and assume conditions in \Cref{thm:ipw-dm}.
Then, for some universal constants $c_1$ and $c_2$, with probability at least $1-\delta$,
\begin{align*}
    |\hat J_{\DR}-J|&\textstyle\leq \Rad(\epol \Qbbb\{y-\Hbbb\}+\Tcal \Hbbb)+ \sup_{q\in \Qbbb}\|\epol(q_0-q)\|_2\|P_z( \hat h - h_0)\|_2 +c_1 \sqrt{\log(c_2 /\delta)/n} \\
    |\hat J_{\DR}-J|&\textstyle\leq  \Rad(\epol \Qbbb\{y-\Hbbb\}+\Tcal \Hbbb)+\sup_{h\in \Hbbb}\|h_0-h\|_2\|\epol P_w(\hat q - q_0)\|_2  +c_1\sqrt{ \log(c_2/\delta)/n}.
\end{align*}
}
\end{theorem}
The error bounds above are typically dominated by the bias terms. Thus, \Cref{thm:drr} suggests that $\hat J_{\DR}$ is consistent when the projected MSE of either $\hat h$ \emph{or} $\hat q$ vanishes to $0$, which
 is reminiscent of double robustness in the unconfounded setting \citep[\eg,~][]{bang05}.

\section{Finite-sample Analysis of Estimators with Stabilizers} \label{sec: closedness}
{
In the previous sections, we studied the convergence of our GACE estimators based on minimax bridge function estimators \emph{without} stabilizers (\Cref{sec: strategy1}). In this section, we derive projected MSE bounds of the estimators \emph{with} stabilizers (\Cref{sec: strategy2}), \emph{also} under realizability and closedness.
Again, according to \Cref{thm:ipw-dm,thm:drr}, these projected MSE bounds translate into convergence rates of resulting GACE estimators.
We further show that the DR estimator is asymptotically normal and semiparametrically efficient when additionally assuming unique bridge functions and
limiting the ill-posedness of associated inverse problems.
}

We first derive the error bounds for the minimax bridge function estimators $\hat h,\,\hat q$ \emph{with} stabilizers (see \Cref{eq: est2-1,eq: est2-2}). Our analysis is   adapted\footnote{Our  \Cref{thm:w_function,thm:q_function} are largely adapted from theory in \cite{DikkalaNishanth2020MEoC}. There are mainly two differences: first, we consider unpenalized estimators for simplicity while \cite{DikkalaNishanth2020MEoC} considers penalized estimators; second, our problems for the action  bridge function estimation do not fit into the standard conditional moment equation problem (as explained in footnote 1) so we need to adapt accordingly.
} from the general theory in \cite{DikkalaNishanth2020MEoC}. The analyses use the  function class complexity measure called the \emph{critical radius}, which is widely used in statistical learning theory \citep[e.g.,][]{bartlett2005,WainwrightMartinJ2019HS:A}.

\begin{definition}[Critical Radius]\label{def: crit-radius}
{Given a function class $\mathcal{G}$, its empirical critical radius is the smallest $\eta>0$ that satisfies $\Rad(\Gcal^{|\eta})\leq \eta^2/\|\Gcal\|_{\infty}$ where $\Gcal^{|\eta} = \braces{g \in \Gcal: \E_n[g^2]\leq \eta^2}$.}
\end{definition}

\begin{theorem}[Convergence rate of $\hat h$ with stabilizer]\label{thm:w_function}
Assume $\|Y\|_{\infty} < \infty$, $\|\Hbbb\|_{\infty} < \infty$. Let $\hat h$ be the minimax estimator with stabilizer in \Cref{eq: est2-1}. Fix any $h_0\in \Hbbb^{\obs}_0, h'\in \Hbbb$. Assume:
\begin{enumerate}
\item $\Qbbb'$ is symmetric and star-shaped.
\item $\eta_{h}$ upper bounds the critical radii of $\Qbbb'$ and the star hull of the following class $\Gcal_h$:
    \begin{align*}
    \Gcal_h \coloneqq \braces{(w,z, a,x)\mapsto  \prns{h(w, a, x)-h'(w,a,x)}q(z, a, x): h\in \Hbbb, q\in \Qbbb' }.
\end{align*}
\end{enumerate}
Then, letting $O(\cdot)$ be the order w.r.t $\lambda,\delta,n,\eta_{h}$, with probability $1-\delta$, we have
\begin{equation*}\begin{aligned}
    \|P_z(\hat h-h_0)\|_2  = O\bigl(&(1+\lambda + \lambda^{-1})(\eta_{h}+\sqrt{1+\log(1/\delta)/n}) + \textstyle\sup_{h\in\Hbbb}\inf_{q\in \Qbbb'}\|q-P_z(h-h')\|_2 \\&+ (\lambda(\eta_{h}+\sqrt{1+\log(1/\delta)/n}))^{-1}\|P_z(h'- h_0)\|^2_2 +\|P_z(h'- h_0)\|_2.
\end{aligned}
\end{equation*}
\end{theorem}
The first term in the bound above is a statistical ``variance'' term. The second to fourth terms are ``bias'' terms due to approximate closedness or approximate realizability. When the realizability condition and closedness condition are violated up to a certain amount, those terms can be bounded accordingly. For example, in \Cref{exa:sieve_neural} Cont'd, we consider linear sieve and neural network classes that do not exactly satisfy these conditions but the errors due to violation vanish in the limit. When the realizability and closedness conditions hold exactly, the error bound can be simplified to the bound given below.   

\begin{corollary}[$\hat h$ with stabilizer under realizability and closedness]\label{thm:w_function_easy}
Assume the conditions in \Cref{thm:w_function}, $h'=h_0\in\Hbbb^{\obs}_0\cap \Hbbb\neq\emptyset$, and $P_z(\Hbbb-h_0)\subseteq \Qbbb'$.
With probability $1-\delta$,
\begin{align*}
   \|P_z(\hat h-h_0)\|_2  = O((1+\lambda + \lambda^{-1})(\eta_{h} + \sqrt{(1+\log(1/\delta))/n})).
\end{align*}
\end{corollary}

\Cref{thm:w_function_easy} states that under realizability and closedness, the convergence rate of $\hat h$ is determined by the critical radii $\eta_{h}$ of the critic class $\Qbbb'$ and the class $\mathcal G_h$ induced by both $\Hbbb$ and $\Qbbb'$. In \Cref{sec: stab-error-bound}, we will further bound $\eta_{h}$ for common function classes.

We can analogously bound the estimation error of the action bridge function estimator.
\begin{theorem}[Convergence rate of $\hat q$ with stabilizer]\label{thm:q_function}
Assume $\|\epol(A \mid X)/f(A \mid X,W)\|_2<\infty$, $\|\epol\Qbbb\|_{\infty}$. Let $\hat q$ be the minimax estimator with stabilizer in \Cref{eq: est2-2}. Fix any $q_0\in \Qbbb^{\obs}_0, q'\in \Qbbb$. Assume:
\begin{enumerate}
\item $\Hbbb'$ is symmetric and star-shaped.
\item $\eta_{q}$ upper bounds the critical radii of $\Hbbb'$ and the star hull of the following class $\Gcal_q$:
    \begin{align*}
    \Gcal_q\coloneqq \{(w, z, a, x)\mapsto \epol(a\mid x)\prns{q(z, a, x)-q'(z,a,x)}h(w, a, x): q\in \Qbbb,h\in \Hbbb' \}.
    \end{align*}
\end{enumerate}
Then, letting $O(\cdot)$ be the order w.r.t. $\lambda,\delta,n,\eta_{q}$, with probability $1-\delta$, we have
\begin{equation*}\begin{aligned}
    \|\epol P_w(\hat q-q_0)\|_2  = O\bigl(&(1+\lambda + \lambda^{-1})(\eta_{q}+\sqrt{1+\log(1/\delta)/n}) + \textstyle\sup_{q\in \Qbbb}\inf_{h\in\Hbbb'} \|h-\epol P_w{\prns{q - q'}}\|_2 \\&+ (\lambda(\eta_{q}+\sqrt{1+\log(1/\delta)/n}))^{-1}\|\epol P_w{\prns{q'-q_0}}\|_2^2 + \|\epol P_w{\prns{q'-q_0}}\|_2.
\end{aligned}
\end{equation*}
\end{theorem}
\begin{corollary}[$\hat q$ with stabilizer under realizability and closedness]\label{thm:q_function_easy}
Assume the conditions in \Cref{thm:q_function}, $q'=q_0\in\Qbbb^{\obs}_0\cap \Qbbb\neq\emptyset$, and $\pi P_w(\Qbbb-q_0)\subseteq \Hbbb'$. With probability $1-\delta$,
\begin{align*}
   \|\epol P_w(\hat q-q_0)\|_2 = O((1+\lambda + \lambda^{-1})(\eta_{q} + \sqrt{\{1+\log(1/\delta)\}/n})).
\end{align*}
\end{corollary}

The bounds in \Cref{thm:w_function_easy,thm:q_function_easy} suggest we can choose $\lambda$ constant.
{ In the rest of this section, we set $\lambda = 1$ for simplicity.}

{ In \Cref{sec: stab-error-bound}, we discuss specific rates for some standard function classes. In  \Cref{sec:with_stabilizers}, we use \Cref{thm:ipw-dm,thm:drr} to translate projected MSE bounds in \Cref{thm:w_function_easy,thm:q_function_easy} into convergence rates of IPW, REG, and DR estimators of GACE.
In \Cref{sec: both}, we give a tighter characterization of the doubly robust estimator under some additional assumptions.
}

\subsection{Convergence Rates of Bridge Function Estimators for Common Function Classes}\label{sec: stab-error-bound}

\Cref{thm:w_function_easy,thm:q_function_easy} show that the convergence rates of minimax estimators $\hat h, \hat q$ with stabilizers (\Cref{sec: strategy2}) are determined by critical radii of relevant function classes. We next provide concrete rates for these function classes.  We focus on analyzing $\|P_z(\hat h-h_0)\|_2$ as an example. The projected MSE $\|\epol P_w(\hat q-q_0)\|_2$ can be analyzed analogously.

\begin{continuance}[VC-subgraph classes]{\ref{exa: vc_subgraph}} We first consider VC-subgraph classes.
\begin{corollary}\label{cor: vc_nonpara}
Assume that the conditions in \Cref{thm:w_function_easy} hold for some $h_0 \in \Hbbb \cap \Hbbb^{\obs}_0$. Suppose $\Hbbb$ and $\Qbbb'$ are VC-subgraph classes with VC-subgraph dimensions $V(\Hbbb), V(\Qbbb')$, respectively. Then, letting $O(\cdot)$ be the order w.r.t. $n,\delta,V(\Hbbb),V(\Qbbb')$, with probability $1- \delta$,
    \begin{align*}
    \|P_z(\hat h-h_0)\|_2 =O(\sqrt{(V(\Qbbb')+V(\Hbbb))\log n +1+ \log(1/\delta)}/\sqrt{n}).
    \end{align*}
\end{corollary}
\end{continuance}

\begin{continuance}[Nonparametric classes characterized by metric entropy]{\ref{exa: nonpara}}
We next analyze classes with limited metric entropy (\eg, {\Holder} balls, Sobolev balls, and RKHSs).

\begin{corollary}\label{cor: nonpara}
Assume that the conditions in \Cref{thm:w_function_easy} hold for some $h_0 \in \Hbbb \cap \Hbbb^{\obs}_0$.
Suppose that
$
    \max(\log \Ncal(\varepsilon,\Hbbb,\|\cdot\|_{\infty}), \ \log \Ncal(\varepsilon,\Qbbb',\|\cdot\|_{\infty})) \leq c_0(1/\varepsilon)^{\beta}$.
Then, letting $O(\cdot)$ be the order with respect to  $n,\beta,\delta$, with probability $1- \delta$,
    \begin{align*}
    \|P_z(\hat h-h_0)\|_2=\begin{cases}  O(n^{-1/(2+\beta)}+\sqrt{\log(1/\delta)/n})&\quad\beta<2 \\  O(n^{-1/4}\log n+\sqrt{\log(1/\delta)/n})&\quad\beta=2 \\  O(n^{-1/(2\beta)}+\sqrt{\log(1/\delta)/n})&\quad\beta>2 \end{cases}
\end{align*}
\end{corollary}
\end{continuance}

\begin{continuance}[Neural networks and linear sieves]{\ref{exa:sieve_neural}}
If bridge and critic classes grow with sample size and approach nonparametric classes only in the limit, it is inappropriate to directly impose realizability and closedness on the bridge and critic classes.
Instead, we must consider the ``bias'' terms due to approximate realizability and closedness.
In \Cref{sec: sieve-nn}, we show projected MSE rates of $\tilde O(n^{-\alpha/(2\alpha+d)})$ when we use linear sieves to approximate {\Holder} balls or when we use neural networks to approximate Sobolev balls, each with smoothness $\alpha$ in $d$ input dimensions and when each limiting nonparametric class satisfies both realizability and closedness.
\end{continuance}

We note that the projected MSE error bounds derived here for the bridge function estimators with stabilizers typically vanish at faster rates than the error bounds in \Cref{sec: unstab-closed} for estimators without stabilizers, when both assuming the same realizability and closedness conditions. For example, the estimators with stabilizers achieve projected MSE error rates $O(n^{-1/2})$ for VC-subgraph classes and $O(n^{-\alpha/(2\alpha+d)})$ for {\Holder} balls (corresponding to $\beta = d/\alpha < 2$ in \Cref{cor: nonpara}). In contrast, the estimators without stabilizers can only achieve $O(n^{-1/4})$ and $O(\max(n^{-1/4},n^{-\alpha/2d}))$ error rates, respectively. 

Given the projected MSE error bounds for minimax estimators, we can also derive the   convergence rates for the downstream GACE estimators by plugging them into the bounds for plug-in GACE estimators (\Cref{thm:ipw-dm,thm:drr}). The convergence rates of GACE estimators are typically  dominated by the corresponding projected MSE convergence rates. Therefore, the faster projected MSE error rates for the minimax estimators with stabilizers would also translate into faster error rates for the corresponding GACE estimators. 

\subsection{Tighter Analysis of the Doubly Robust Estimator}
\label{sec: both}
In \Cref{thm:drr}, we showed that the convergence rate of the doubly robust estimator $\hat J_{\DR}$ is determined by the projected MSE convergence rate of a single bridge function estimator, when  conditions in either \Cref{thm:w_function_easy} \emph{or} \Cref{thm:q_function_easy} are satisfied.
In this section, we show that if the conditions in \emph{both} theorems  are satisfied, and the related inverse problems are not too ill-posed, then the doubly robust estimator $\hat J_{\DR}$ can converge at a faster rate than  the bridge function estimators.

We first introduce measures of  the ill-posedness of inverse problems defined  by the conditional moment equations in \Cref{eq: observed-bridge-h,eq: observed-bridge-q}, relative to the bridge classes $\Hbbb, \Qbbb$.
\begin{definition}[Measures of Ill-posedness]\label{def:ill-posed}  Fix  $h_0 \in \Hbbb^{\obs}_0\cap\Hbbb,q_0 \in \Qbbb^{\obs}_0\cap\Qbbb$. Define
\begin{align}
&P_u \prns{h}= \Eb{h\prns{W, Z, A, X} \mid U, A, X},\label{eq:Pu}\\
    &\tau^\Hbbb_{1} = \sup_{h \in \Hbbb}\frac{\|P_u(h -h_0)\|_2}{\|P_z(h -h_0)\|_2}, ~~ \tau^\Qbbb_{1} = \sup_{q \in \Qbbb} \frac{\|P_u(\epol q-\epol q_0)\|_2}{\|P_w(\epol q-\epol q_0)\|_2}. \label{eq:dr_property1} \\
    &\tau^\Hbbb_{2} = \sup_{h \in \Hbbb} \frac{\|h-h_0\|_2}{\|P_z(h -h_0)\|_2},~~ \tau^\Qbbb_{2} = \sup_{q \in \Qbbb} \frac{\|\epol q-\epol q_0\|_2}{\|P_w(\epol q-\epol q_0)\|_2}. \label{eq:dr_property2}
\end{align}
In this definition, we follow the convention $\frac{0}{0} = 0$.
\end{definition}
To interpret the ill-posedness measures above, let us focus on $\tau^\Hbbb_{1}$ and $\tau^\Hbbb_{2}$ as examples and note that  $\Eb{h - h_0\mid Z, A, X} =\Eb{ \Eb{h-h_0\mid A, U, X} \mid Z, A, X}$ according to \Cref{asm:whole_assm} condition \eqref{whole_assm:nc-unconfoundedness}. This means that we have two different levels of inverse problems: the first involves inverting $\Eb{h-h_0\mid A, U, X}$ from $\Eb{h - h_0\mid Z, A, X}$, and the second involves inverting $h-h_0$ from $\Eb{h - h_0\mid Z, A, X}$ for $h \in \Hbbb$. The degrees of ill-posedness of these two levels  relative to the hypothesis class $\Hbbb$ are quantified by $\tau^\Hbbb_{1}$ and $\tau^\Hbbb_{2}$, respectively. And, we have $1 \le \tau^\Hbbb_{1} \le \tau^\Hbbb_{2}$.
Note that if the completeness condition in \Cref{eq:completeness} holds\footnote{Under this completeness condition, we have $\Hbbb_0^{\obs} = \Hbbb_0$ according to \Cref{sec: comparison} \Cref{lemma: valid-bridge}. Thus any $h_0 \in \Hbbb_0^{\obs} \cap \Hbbb$ must satisfy both  \Cref{eq: bridge-U-h,eq: observed-bridge-h}.},
$\tau^\Hbbb_{1}$ is invariant to the choice of $h_0 \in \Hbbb_0^{\obs}\cap\Hbbb$, since  $\Eb{h_0 \mid Z, A, X} = \Eb{Y \mid Z, A, X}$ and $\Eb{h_0 \mid U, A, X} = \Eb{Y \mid U, A, X}$ according to \Cref{eq: bridge-U-h,eq: observed-bridge-h}. So $\tau^\Hbbb_{1}$ may be finite even when the bridge function is not uniquely identified. In contrast,  $\tau^\Hbbb_{2}$ is $+\infty$ when $ \Hbbb^{\obs}_0\cap\Hbbb$ is not a singleton.
Moreover, note that the definition of $\tau^\Hbbb_{1}$ and  $\tau^\Hbbb_{2}$ depend on the choice of hypothesis class $\Hbbb$ (which can change with $n$). When specialized to linear sieves (\Cref{sec: sieve}), these measures are related to the sieve ill-posedness measures introduced in \citet{blundell2003semi,chen2012estimation}.

The ill-posedness measures in \Cref{def:ill-posed} directly influence the convergence behavior of minimax bridge function estimators in different error measures. For example,  if $\tau^\Hbbb_{1} < +\infty$ and $\tau^\Hbbb_{2} < +\infty$ for all $n$, then for the minimax estimator $\hat h$, the error bound on $\|P_z(\hat h - h_0)\|$ derived in \Cref{thm:w_function_easy}  easily translates into bounds on
$\|\E[\hat h -h_0 \mid A,U,X]\|_2$
and $\|\hat h -h_0\|_2$, with additional inflation factors of $\tau^\Hbbb_{1}$ and $\tau^\Hbbb_{2}$ respectively.
So larger $\tau^\Hbbb_{1}$ and $\tau^\Hbbb_{2}$, \ie, more ill-posed inverse problems, correspond to larger error in terms of $\|\E[\hat h -h_0 \mid A,U,X]\|_2$ and $\|\hat h -h_0\|_2$.
The ill-posedness measures depend both on the data generating process and on the choice of hypothesis classes $\Hbbb, \Qbbb$. As is discussed in \citet{blundell2003semi} for linear sieves, larger hypothesis classes typically correspond to bigger ill-posedness measures.

Next, we slightly revise our doubly robust estimator by cross-fitting the bridge function estimators, a technique that has been widely used to remove restrictive Donsker conditions on hypothesis classes \citep[\eg,~][]{ZhengWenjing2011CTME,ChernozhukovVictor2018Dmlf}. For simplicity, we focus on two-fold cross-fitting; it is straightforward to extend to more folds.
\begin{definition}[Cross-fitted DR Estimator]
Randomly split the whole sample into two halves denoted as $\Dcal_0,\Dcal_1$. Then for $j = 0, 1$, fit $\hat h^{(j)}, \hat q^{(j)}$ based on $\Dcal_j$ according to \Cref{eq: est2-1,eq: est2-2}, respectively. Finally, denoting $I_i = \indic{i \in \Dcal_0}$,  redefine the DR estimator by
\begin{align*}
    \hat J_{\DR} = \frac{1}{n}\sum_{i=1}^{n}\epol(A_i|X_i)\hat q^{(I_i)}(Z_i, A_i, X_i)(Y_i-\hat h^{(I_i)}(W_i, A_i, X_i))+(\Tcal \hat h^{(I_i)})(W_i, X_i).
\end{align*}
\end{definition}
Now we derive the property of the cross-fitted DR estimator in terms of the ill-posedness measures $\tau^\Hbbb_{1}, \tau^\Qbbb_{1}$ in \Cref{eq:dr_property1}, \emph{without} assuming unique bridge functions.

\begin{theorem}[DR estimator without bridge function uniqueness]\label{thm:dr2}
Suppose the conditions in \Cref{thm:w_function_easy,thm:q_function_easy} hold.
Then, letting $O(\cdot)$ be the order w.r.t. $n$, $\tau^\Hbbb_{1}$, $\tau^\Qbbb_{1}$, $\eta_{h}$, $\eta_{q}$, and $\delta$,
with probability at least $1- \delta$, we have
\begin{align}
    |\hat  J_{\DR}-J| =
    O\prns{\tau^\Hbbb_{1}\tau^\Qbbb_{1}\eta_{h}\eta_{q} + \sqrt{\log(1/\delta)/n}}.
    \label{eq: rate-dr-stabilizer}
\end{align}
\end{theorem}
{ Note that the convergence rate in \Cref{thm:dr2} can be faster than the rate in \Cref{cor: dr_no_stab}. For example, suppose all relevant function classes are {\Holder} balls of order $\alpha$ in $d$ dimensions. In this case, the convergence rate of $\hat  J_{\DR}$ based on bridge function estimators \emph{with} stabilizers given in \Cref{thm:dr2} reduces to $\max(\tau^\Hbbb_{1}\tau^\Qbbb_{1} n^{-2\alpha/(2\alpha+d)}, n^{-1/2})$, while the convergence rate in \Cref{cor: dr_no_stab} for the estimator  \emph{without} stabilizers is {$\max( n^{-\alpha/d}, n^{-1/2})$.} The former converges faster than the latter if the ill-posedness measures satisfy $\tau^\Hbbb_{1}\tau^\Qbbb_{1} = o\prns{n^{\alpha(d-2\alpha)/(2\alpha d+d^2)}}$. This faster convergence is possible primarily because \Cref{thm:dr2} assumes both realizability and closedness, under which the bridge function estimators with stabilizers converge to certain valid observed bridge functions $h_0 \in \Hbbb_0^{\obs}, q_0 \in \Qbbb_0^{\obs}$ in terms of errors $\| \E[{\hat h-h_0} \mid A, U, X]\|_2,\,\|\Eb{\epol\prns{\hat q-q_0} \mid A, U, X}\|_2 $. In contrast, \Cref{cor: dr_no_stab} assumes only realizability, under which the bridge function estimators may not converge to any true bridge function at all. Actually, under realizability and closedness, we can also tailor \Cref{thm:dr2} to estimators without stabilizers by leveraging the projected MSE bounds in \Cref{thm:slow_rates} and the ill-posedness measures $\tau^\Hbbb_{1}, \tau^\Qbbb_{1}$ defined in \Cref{eq:dr_property1}.
But since the projected MSE rates of bridge function estimators without stabilizers are typically slower than their counterparts with stabilizers, the  convergence rate of the resulting DR estimator will again be slower than the rate in \Cref{thm:dr2}.}

{ Next, we show $\hat  J_{\DR}$ is asymptotically normal when additionally assuming unique bridge functions.}
Note that with nonunique bridge functions, the minimax estimators $\hat h, \hat q$ may not converge to any fixed limit, even when $\| \E[{\hat h-h_0} \mid A, U, X]\|_2$ and $\|\Eb{\epol\prns{\hat q-q_0} \mid A, U, X}\|_2 $ vanish to $0$ for all $h_0\in\Hbbb_0, q_0\in\Qbbb_0$. In this nonunique case, even if $\tau^\Hbbb_{1}\tau^\Qbbb_{1}{\eta_{h}\eta_{q}  } = o(n^{-1/2})$ so that $\hat  J_{\DR}$ is $\sqrt{n}$-consistent, certain challenging stochastic equicontinuity term (\ie, \Cref{eq:stochastic-equicont} in \Cref{sec: proof-sec-stabilizer-dr}) is non-negligible and may distort the asymptotic normality. This intractable term  makes  it very difficult to derive the asymptotic distribution of $\hat  J_{\DR}$. In the following lemma, we further assume that the bridge classes single out \emph{unique} bridge functions, and show $\hat J_{\DR}$ has an asymptotically normal distribution with a  closed-form variance.

\begin{theorem}[DR estimator with unique bridge functions]\label{thm:dr}
Suppose that conditions in \Cref{thm:w_function_easy,thm:q_function_easy} hold,
$\Hbbb\cap \Hbbb^{\obs}_0=\{h_0\}$ and $\epol\Qbbb \cap \epol\Qbbb^{\obs}_0=\{\epol q_0\}$.

If { $\tau^\Hbbb_{2}\eta_{h} = o(1)$, $\tau^\Qbbb_{2}\eta_{q} = o(1)$, and $\min(\tau^\Hbbb_{2},\tau^\Qbbb_{2})\eta_{q}\eta_{h} = o(n^{-1/2})$},  then
\[
\sqrt{n}(\hat  J_{\DR}-J) \overset{d}{\to} \mathcal N(0, V_{\text{eff}}), \quad V_{\text{eff}} = \E{\tilde\phi^2_\DR(O;h_0,q_0)}.
\]
\end{theorem}

For simplicity, we focus here  on the case of exact realizability and closedness of the bridge and critic classes. If, following \Cref{exa:sieve_neural}, we use classes that grow with $n$ and slowly approach nonparametric classes that satisfy realizability and closedness, then we can state the unique bridge function condition to be in terms of the limiting classes.
Note that the condition $\min(\tau^\Hbbb_{2},\tau^\Qbbb_{2})\eta_{h}\eta_{q} =o(n^{-1/2})$ in \Cref{thm:dr} is generally incomparable to $\tau^\Hbbb_{1}\tau^\Qbbb_{1}{\eta_{h}\eta_{q}} = o(n^{-1/2})$ in \Cref{thm:dr2}, while both conditions are implied by $\tau^\Hbbb_{2}\tau^\Qbbb_{2}\eta_{h}\eta_{q} =o(n^{-1/2})$.

Finally, we remark that the asymptotic variance $V_{\textit{eff}}$ in \Cref{thm:dr} coincides with the semiparametric efficiency bound derived in \Cref{sec: efficiency-bound} under the assumption that $\Hbbb_0^{\obs} = \braces{h_0}$ and some additional regularity conditions.
These efficiency results extend the asymptotic efficiency results in \cite{CuiYifan2020Spci}. See \Cref{sec: efficiency-bound} for details.

 \subsection{Comparing Estimators with and without Stabilizers}\label{sec:with_stabilizers}

So far we have analyzed theoretical guarantees for the minimax bridge function estimators with and without stabilizers and the corresponding GACE estimators. The theory is based on  different  assumptions and analyses: \Cref{sec: realizable} assumes double realizability for both the bridge function class and the critic function class and directly derives the error bounds for the GACE estimators with no stabilizers, while \Cref{sec: unstab-closed,sec: closedness} assume realizability for the bridge function class and closedness for the critic function class and analyze the error bounds for the minimax bridge function estimators and plug-in GACE estimators.

Under the realizability and closedness assumptions, our analyses in \Cref{sec: unstab-closed,sec: closedness} show that the minimax bridge function estimators with stabilizers have clear advantages because their projected MSE error bounds have faster convergence rates than their counterparts without stabilizers. The faster rates for bridge function estimators can also translate into faster error rates for the downstream plug-in GACE estimators. Moreover, the faster rates of the bridge function estimators with stabilizers also mean that they are more likely to satisfy the product rate conditions in \Cref{thm:dr}, so that the resulting double robust estimator is more likely to achieve asymptotic normality.

 However, this does not mean that using stabilizers always leads to better convergence rates. In fact, the GACE estimators $\hat J_{\mathrm{REG}},\hat J_{\mathrm{IPW}}$ based on bridge function estimators without stabilizers can sometimes achieve better convergence rates under the double realizability conditions. For example, when all relevant function classes are  {\Holder} balls, the error convergence rates for $\hat J_{\mathrm{REG}},\hat J_{\mathrm{IPW}}$ without stabilizers are $O\prns{\max(n^{-1/2},n^{-\alpha/d})}$ according to \Cref{cor: ipw_reg_no_stab,cor: nonparametric}. In contrast, the error bounds for estimators $\hat J_{\mathrm{REG}},\hat J_{\mathrm{IPW}}$ with stabilizers require the realizability and closedness conditions, and when instantiated to  {\Holder} balls, the error bounds are the slower $\tilde O(n^{-\alpha/(2\alpha+d)})$ rates. This shows the potential advantages of minimax bridge function estimators \emph{without} stabilizers in GACE estimation. However, we remark that this comparison is also subtle because the two different convergence rates are based on quite different assumptions and analyses.

 Therefore, there is no universal  dominance of one approach over the other in terms of theoretical error rates. In our empirical study in \Cref{sec:experiment}, we find that these two different approaches tend to produce similar estimation results in nearly all simulation and real data experiments, suggesting that both are viable in practice.
 In practice, the practitioners can try both but if they want to start with a particular type of estimators, they can choose the estimators with stabilizers for the fast convergence rate guarantees and statistical inference guarantees under the corresponding realizability and closedness conditions.

\section{Related Literature}\label{sec: literature}
{
\subsection{Negative Controls}
\cite{miao2018identifying} first develop sufficient conditions for identifying counterfactual distributions with negative controls.
\cite{shi2020multiply} focus on multiply robust estimation of the average treatment effect in a discrete setting.
\cite{miao2018a} propose to estimate the average treatment effect by first estimating the outcome bridge function using standard Generalized
Method of Moments (GMM) methods \citep{Hansen82}.
\cite{CuiYifan2020Spci} derive the semiparametric efficiency bound for average treatment effect based on both the outcome bridge function and action bridge function.
{Most previous work focuses on the average treatment effect with discrete treatment, assumes completeness conditions, requires unique bridge functions for estimation\footnote{\cite{shi2020multiply} studies the discrete data setting in \Cref{ex: bridge-discrete} and considers negative controls that can take more values than the unobserved confounders. In this case, the bridge functions are nonunique. They briefly describe two possible approaches to tackle nonunique bridge functions: one by considering the pseudo-inverse of certain non-invertible conditional probability matrix, and the other by coarsening the values of negative controls according to the number of values of the unobserved confounders. But they do not develop the two approaches in full detail. Instead, they still focus on unique bridge functions in estimation.}, and restricts attention to parametric bridge-function models.}. In contrast, our paper studies a generalized average causal effect with general actions (continuous or discrete), develops new identification results without imposing completeness conditions, allows for {nonunique bridge functions} in estimation,  and proposes  flexible minimax learning approaches to accommodate both nonparametric and parametric estimation of bridge functions. More details on comparing identification strategies is given in \Cref{sec: comparison}.
}

\cite{deaner2021proxy} studies the identification of the counterfactual mean on the treated, based on the outcome bridge function or the action bridge function. They also require completeness conditions and their estimation is restricted to two-stage sieve estimators for the outcome bridge function \citep{chen2012estimation,chen2015sieve}.
{ Later, \cite{SinghRahul2020KMfU,mastouri2021proximal} propose to use a two-stage kernel estimator for the outcome bridge function, and \cite{xu2021deep} propose a neural network extension with adaptive features, all of which require a unique outcome bridge function and assume completeness conditions.
In contrast, our paper considers minimax estimation approaches accommodating a wider array of nonparametric machine learning methods, and assumes weaker assumptions. }

{ A few recent papers also propose minimax estimators for the bridge functions. \citet{GhassamiAmirEmad2021MKML}
studies doubly robust estimation of average treatment effect, and \cite{QiZhengling2021PLfI} studies evaluating and learning optimal treatment regimes.
These two papers focus on discrete treatments and also require completeness conditions.
Moreover, they study minimax estimators \emph{with} stabilizers, whose theoretical guarantees are in line with some of our results in \Cref{sec: closedness}.
Under both completeness and uniqueness conditions, \cite{mastouri2021proximal} proposes a Maximum Moment Restriction approach, which is a special example of our minimax estimation \emph{without} stabilizers.
Both \citet{GhassamiAmirEmad2021MKML} and \cite{mastouri2021proximal} focus on RKHS hypothesis classes.
In contrast, our paper assumes weaker assumptions, analyzes both types of minimax estimators, and considers a wider variety of hypothesis classes.
Like \citet{GhassamiAmirEmad2021MKML,QiZhengling2021PLfI}, our analysis for minimax estimators \emph{with} stabilizers also builds on \cite{DikkalaNishanth2020MEoC}, but our analysis for the minimax estimators \emph{without} stabilizers is completely different from other literature.
Moreover, our definitions of ill-posedness measures $\tau^\Hbbb_{1}, \tau^\Qbbb_{1}$ that allow for nonunique bridge functions are also new.
}

\subsection{Instrumental Variables and Minimax Estimation}
Estimation of the outcome bridge function $h_0$ is closely related to the nonparametric instrumental variable (IV) regression problem \citep[e.g.,][]{WhitneyK.Newey2013NIVE}. Our paper nonetheless differs significantly from the previous literature.
First, the conditional moment equation for the action bridge function $q_0$ in \Cref{eq: observed-bridge-q} is  distinct from the IV conditional moment equation since it includes an unknown density.
See the discussions in \Cref{sec: strategy1}. Second, our target estimand is the  GACE
rather than bridge functions, so most of our analysis is different from previous literature on the estimation of the IV regression function. For example, our new analysis shows that our GACE estimator with bridge functions without stabilizers
remains consistent even when the bridge function estimators are inconsistent (\Cref{ex: inconsist-h}). Despite these differences, as we mentioned, the estimation of $h_0$ itself is similar to the IV regression problem, so we next summarize the relevant literature on IV regression.

The nonparametric IV regression problem is typically cast into conditional moment equations \citep[\eg,][]{AiChunrong2003EEoM}. One classical approach to this problem is a nonparametric analogue of the two-stage least squares (2SLS) method based on sieve estimators \citep{newey2003instrumental} or kernel density estimators \citep{darolles2010nonparametric,Hall05IV}. Another classical approach is to use sieves to convert the conditional moments into unconditional moments of increasing dimension \citep[\eg,][]{ChenXiaohong2007C7LS,chenreview2016}, and then combine all unconditional moments via standard GMM method \citep{Hansen82}.
Later, \cite{pmlr-v70-hartford17a,NEURIPS2019_17b3c706} extend the two-stage approach by employing neural network density estimator or  conditional mean embedding in RKHS respectively in the first stage.
It remains unclear how to incorporate more general hypothesis classes while still providing rigorous theoretical guarantees for nonparametric estimators.

Recently, there has been intense interest in minimax approaches that reformulate the nonparametric IV regression problem  via  \Cref{eq:strategy1}, \Cref{eq:strategy2}, or other related variants \citep{LewisGreg2018AGMo,ZhangRui2020MMRf,DikkalaNishanth2020MEoC,ChernozhukovVictor2020AEoR,LiaoLuofeng2020PENE,NIPS2019_8615,bennett2020variational,MuandetKrikamol2019DIVR}. The resulting minimax formulation is more analogous to the empirical risk minimization framework predominant in machine learning, which  permits more general function classes.
Notably, \citet{DikkalaNishanth2020MEoC} provide a thorough analysis of the convergence rates of their minimax estimators.
{Our theoretical analysis for the minimax estimators with stabilizers builds on their analysis  (see \Cref{thm:w_function_easy,thm:q_function_easy}).}
Minimax estimators have also been  employed in some other contexts, such as causal-effect estimation \citep{ChernozhukovVictor2020AEoR,HirshbergDavid2019AMLE} and reinforcement learning policy evaluation  \citep{FengYihao2019AKLf,YangMengjiao2020OEvt,UeharaMasatoshi2021FSAo}, when all confounders are measured. Unlike this literature, our paper focuses on the more challenging setting with unmeasured confounders and addresses the confounding problem by using negative controls.

\section{Numerical Experiments}\label{sec:experiment}

{

 In this section, we illustrate our estimators with and without stabilizers in both numerical simulations and a real-data analysis. In particular, we feature the performance of minimax bridge function estimators using neural networks.
 The code for all experiments can be found in \url{https://github.com/ZHzhang01/Minimax_JMLR}.

\subsection{Simulation Study}\label{sec: simulation}

\begin{table}[t]\footnotesize
    \caption{     \label{tab:d2} $g(t)=\sin(t)$, $n=400$. }
      \begin{tabular}{c|ccccccccc}
   \hline & IPW & IPW(sta) & REG & REG(sta) & DR & DR(sta) & Linear & Plain REG & DR (no W,Z) \\
   \hline MSE &0.0399  &0.0353 & 0.0798  &0.0892 &0.0771 &0.0844  &  0.1348 &0.7718 &0.7771\\
   Squared Bias &0.0005  & 0.0322&  0.0693 &0.0794 &0.0735 &  0.0800&0.1289 & 0.7660&0.7732\\
   Variance & 0.0394  & 0.0031 &0.0105  &0.0098 &0.0037 &  0.0044 &  0.0059 & 0.0059&0.0039\\
   \hline
   \end{tabular}

   \vspace{\baselineskip}

    \caption{     \label{tab:d3} $g(t)=\sin(t)$, $n=1200$.}
      \begin{tabular}{c|ccccccccc}
   \hline & IPW & IPW(sta) & REG & REG(sta) & DR & DR(sta) & Linear & Plain REG & DR (no W,Z) \\
   \hline MSE & 0.0324  &0.0288 &0.0724  &0.0757 &0.0691 & 0.0739 &0.1654 &1.0011 &1.0015\\
   Squared Bias & 0.0063  &0.0276 &0.0658  & 0.0697&0.0679 & 0.0721 &0.1639 & 0.9984&0.9999\\
   Variance & 0.0262  &0.0012 &0.0066  & 0.0060&0.0012 & 0.0018 &0.0015 &0.0027 &0.0015\\
   \hline
   \end{tabular}

   \vspace{\baselineskip}

    \caption{\label{tab:d4} $g(t) = t^3$, $n=400$.}
      \begin{tabular}{c|ccccccccc}
   \hline & IPW & IPW(sta) & REG & REG(sta) & DR & DR(sta) & Linear & Plain REG & DR (no W,Z) \\
   \hline MSE &  0.0202 &0.0469 &0.0037  &0.0033 &0.0029 &0.0025  &0.0655 &0.5037 &0.5142\\
   Squared Bias & 0.0124  &0.0430 &0.0004  &0.0002 &0.0006 & 0.0003 &0.0514 &0.4966 &0.5092\\
   Variance &  0.0077 & 0.0038&0.0033  & 0.0031 & 0.0023 & 0.0022 &0.0141 &0.0072 &0.0050\\
   \hline
   \end{tabular}

   \vspace{\baselineskip}

    \caption{     \label{tab:d5} $g(t) = t^3$, $n=1200$.}

      \begin{tabular}{c|ccccccccc}
   \hline & IPW & IPW(sta) & REG & REG(sta) & DR & DR(sta) & Linear & Plain REG & DR (no W,Z) \\
   \hline MSE & 0.0122  &0.0431 &0.0026  & 0.0018 &0.0025 &0.0012  &0.0915 &0.8313 &0.8352\\
   Squared Bias & 0.0059  &0.0417 &0.0008  &0.0001 &0.0016 & 0.0002 &0.0894 &0.8265 &0.8327\\
   Variance & 0.0064  &0.0014 &0.0018  & 0.0017&0.0009 & 0.0009 &0.0021 & 0.0049&0.0025\\
   \hline
   \end{tabular}
   \end{table}

   \begin{table}\footnotesize
        
        \caption{ \label{abla_i}    \label{tab:d2-abla} $g(t)=\sin(t)$, $n=400$, ablation study-I. }

      \begin{tabular}{c|ccccccccc}
   \hline & IPW & IPW(sta) & REG & REG(sta) & DR & DR(sta) & Linear & Plain REG & DR (no W,Z) \\
   \hline MSE &0.0118  & 0.0153& 0.0148  &0.0139 & 0.0083&0.0081  &0.0447 &0.0183 &0.0151\\
   Squared Bias & 0.0024  &0.0012 & 0.0006 & 0.0006& 0.0005& 0.0005 & 0.0000& 0.0000&0.0000\\
   Variance & 0.0094 & 0.0141& 0.0142  & 0.0133& 0.0077& 0.0077 & 0.0447&0.0183 &0.0151\\
   \hline
   \end{tabular}

    \caption{ \label{abla_ii}  \label{tab:d4-abla} $g(t)=t^3$, $n=400$, ablation study-II. }
      \begin{tabular}{c|ccccccccc}
   \hline & IPW & IPW(sta) & REG & REG(sta) & DR & DR(sta) & Linear & Plain REG & DR (no W,Z) \\
   \hline MSE & 0.0330  & 0.0211& 0.0170 & 0.0155&0.0168 &0.0132  &0.0594 & 0.0199&0.0173\\
   Squared Bias & 0.0235  & 0.0061 &0.0054  & 0.0040& 0.0067& 0.0042 & 0.0035&0.0000 &0.0000\\
   Variance & 0.0095  & 0.0150&0.0116 &0.0115 &0.0101 & 0.0090 &0.0558 &0.0199 &0.0173\\
   \hline
   \end{tabular}
   \end{table}

In our simulations, we adjust the data generating process (DGP) used in \citet{CuiYifan2020Spci}  to allow for multi-dimensional variables and nonlinear bridge functions.
Concretely, we first start with the DGP in \citet{CuiYifan2020Spci}, which uses one-dimensional variables, to generate multi-dimensional variables $U, X', Z', W' \in \R{d}$ with $d = 60$ (and $A \in \braces{0, 1}$). This DGP ensures the existence of  bridge functions $h_0\prns{W', A, X'}$ and $q_0\prns{Z', A, X'}$ that are linear in $\prns{W', A, X'}$ and $\prns{Z', A, X'}$ respectively.
To introduce nonlinearity, we transform $\prns{W', Z', X'}$ into $\prns{W, Z, X}$ via $X = g\prns{GX'}$, $Z = g\prns{GZ'}$, $W = g\prns{GW'}$ where $G \in \R{d \times d}$ is an invertible matrix and $g\prns{\cdot}$ is a nonlinear function applied elementwise to $GX', GZ', GW'$ respectively.
In the final data, we only observe $\prns{W, Z, X}$ but not $\prns{W', Z', X'}$. The corresponding bridge functions $h_0\prns{W, A, X}$ and $q_0\prns{Z, A, X}$ exist (when $g()$ is invertible) and are nonlinear. See \Cref{ape:parameters} for details on parameter specifications.

Our goal is to estimate the counterfactual mean parameter $J = \Eb{Y(1)}$, which is an example of the GACE parameter with $\pi\prns{a \mid x} = \mathbb{I}(a=1)$ (see \Cref{ex: atomic}). We compare the performance of four different types of estimators. \textit{(i)} The first type of estimators (\texttt{IPW, IPW(sta), REG, REG(sta), DR, DR(sta)}) are our proposed IPW, REG, and DR estimators based on \emph{nonlinear} minimax bridge function estimators without and with stabilizers. To construct the outcome bridge function estimator $\hat h$ (both with and without stabilizers), we use three-layer neural networks as the bridge class $\Hbbb$, and an RKHS with an RBF kernel as the critic class $\Qbbb'$.
To construct the action bridge function estimator $\hat q$ (both with and without stabilizers), we again use three-layer neural networks as the bridge class $\Qbbb$ and an RKHS as the critic class $\Hbbb'$.
Throughout we set the stabilizer parameter as $\lambda = 1$.
See \Cref{ape:parameters} for more details of model specifications and hyperparameter choices.
\textit{(ii)} The second type of estimator (\texttt{Linear}) is the DR estimator based on \emph{linear} minimax bridge function estimators without stabilizers, where all bridge and critic classes are simple linear classes described in \Cref{sec: linear} with basis functions
$\phi\prns{z, a, x}=\tilde \phi\prns{z, a, x}=(z^{\top},a,x^{\top})^{\top}$ and $\psi\prns{w, a, x}=\tilde \psi\prns{w, a, x}=\prns{w^{\top},a,x^{\top}}^{\top}$.
This estimator has a closed-form given in Lemma~\ref{lem:equivalence}. \textit{(iii)} The third baseline is a regression-adjustment estimator (\texttt{Plain REG}) that fits a neural network to approximate the conditional outcome regression function $m(a, x)=\mathbb{E}[Y \mid A=a, X=x]$ using only the treatment $A$ and covariates $X$, without accounting for the unmeasured confounding. In particular, we train the same multilayer network architecture used in our bridge-function models, but we restrict its inputs to $(A, X)$ and optimize it by minimizing mean squared error to predict $Y$; negative-control variables $(W, Z)$ are deliberately excluded from both the inputs and the loss. \textit{(iv)} The fourth type of estimator (\texttt{DR (no W,Z)}) is the regular doubly robust estimator in the unconfounded setting. Like the \texttt{Plain REG} estimator, this estimator also ignores the unmeasured confounding and does not use the negative controls.

In \Cref{tab:d2,tab:d3,tab:d4,tab:d5}, we show the performance of different estimators over $1000$ replications of experiments, with sample sizes $n = 400$ and $1200$, and variable transformation maps $g\prns{t} = \sin(t)$ and $g\prns{t} = t^3$, respectively.
We report the normalized MSE, normalized squared bias, and normalized variance, each scaled by the squared true estimand, over the $1,000$ replications.

Across all settings, the benchmark estimators \texttt{Plain REG} and \texttt{DR (no W,Z)} use flexible neural network nuisance estimators but they ignore the unmeasured confounding, resulting in high bias and MSE. The benchmark \texttt{Linear} does account for unmeasured confounding but uses restrictive linear bridge function estimators, which also lead to high bias and MSE because of significant misspecification.  In contrast, our proposed minimax \texttt{IPW}, \texttt{REG}, and \texttt{DR} estimators learn nonlinear bridge functions from flexible neural networks together with RKHS critics without and with stabilizers (\texttt{sta}), and explicitly use the negative controls to adjust for unmeasured confounding. These six estimators achieve substantially lower errors. In particular, \texttt{DR} and \texttt{DR(sta)} achieve robust performance: they typically attain the lowest or near-lowest MSE, combining bias reduction with improved stability, especially in the case $g(t) = t^3$. Moreover, the REG and DR estimators perform broadly similarly with and without stabilizers, whereas stabilization has a more noticeable effect on IPW under $g(t) = t^3$.

To isolate the role of unmeasured confounding, we further run ablation studies designed to remove unmeasured confounding. Concretely, in this setting, the treatment $A$ is drawn from a logistic model that depends only on the pre-transformation covariates $X^{\prime}$, and all other variables $(W, Z, U)$ and the outcome $Y$ are generated afterward given $\left(A, X^{\prime}\right)$. This construction ensures that $A \perp(Y(1), Y(0)) \mid X^{\prime}$. For the observed $X$, ignorability $A \perp(Y(1), Y(0)) \mid X^{}$ holds exactly under the cubic transformation; under the sine transformation, global non-invertibility induces only negligible residual dependence at the reported precision. We defer the details to Appendix~\ref{ape:parameters} and report the estimation results in Tables~\ref{tab:d2-abla} and~\ref{tab:d4-abla}. Under these settings, Plain REG and DR (no $W, Z)$ exhibit essentially zero squared bias, consistent with exact or near-exact ignorability, but their variances keep their MSEs above that of DR(sta), which attains the lowest MSE in both Tables~\ref{tab:d2-abla} and~\ref{tab:d4-abla} ($0.0081$ and $0.0132$ , respectively). The proposed DR and DR(sta) estimators therefore remain highly competitive and, in these two settings, outperform the unconfounded baselines in terms of MSE.

\subsection{Real Data Analysis}

We also apply our methods to the right heart catheterization (RHC) dataset from the Study to Understand Prognoses and Preferences for Outcomes and Risks for Treatments (SUPPORT). The treatment variable $A$ indicates  whether  a patient received an RHC within $24$ hours of admission or not. The binary outcome $Y$ indicates whether the patient survives for 30 days after admission. This dataset includes $5735$ participants with $72$ covariates. See \citet{hirano03} Table 2 for more details.

A number of papers have analyzed the RHC dataset assuming no unmeasured confounding \citep{hirano03,TanZhiqiang2006ADAf,VermeulenKarel2015BDRE}.
In contrast, \citet{CuiYifan2020Spci} and \citet{tchetgen2020introduction} further address  unmeasured confounding that may arise as the patients self-selected or were selected into the treatment. They view some measurements in a blood test prior to the treatment as noisy snapshots of the unobserved underlying physiological state. In particular, they propose a practical strategy to determine some plausible proxies, ending up with the negative control actions $Z=(\texttt{pafi1}, \texttt{paco21})$ and negative control outcome proxies $W=(\texttt{ph1}, \texttt{hema1})$ (see \Cref{ape:parameters} for detailed descriptions). We follow their practice and use the same proxies in our analysis. As in those papers, the validity and informativeness of these proxies remain substantive identifying assumptions and cannot be established from the observed data alone.

\Cref{tab:real_result} shows point estimates and $95\%$ confidence intervals for the average treatment effect. The first two columns correspond to our proposed \texttt{DR} and \texttt{DR (sta)}, respectively.
The bridge classes are also specified as three-layer neural networks, and the critic classes are specified analogously to those in \Cref{sec: simulation} (albeit with different covariate and negative control dimensions). The third column \texttt{Linear(closed-form)} corresponds to the linear case with closed-form as in~\eqref{eq:convinient}.
The fourth column \texttt{Linear(moment-based)} uses the linear moment estimator for the outcome bridge function used in \cite{CuiYifan2020Spci}.
Moreover, in \citet{VermeulenKarel2015BDRE}, they assume no unmeasured confounding and report an estimate of $-0.0612$ with standard error $0.0141$. All estimates in \Cref{tab:real_result} suggest that applying RHC causes higher $30$-day mortality (or lower survival probability) than not applying RHC.
We can observe that when using simple linear bridge function estimators (the third and fourth columns), the estimates are similar to the results previously reported in \citet{VermeulenKarel2015BDRE}. However, when using more flexible models for the bridge functions (the first two columns), the results suggest that applying RHC might not be as harmful as previous methods suggest.

Because the ground-truth effect is unknown in this real-data study,  it is impossible to conclusively determine whether the estimates from our proposed methods based on flexible minimax bridge function estimation are closer to the true effects than simpler linear parametric estimators. Instead, we conduct some additional sensitivity analyses to check the robustness of our proposed method and the linear parametric estimators with respect to the model specifications. In Table \ref{tab:real_result_abla1}, we assess the robustness of the minimax estimators with neural network bridge functions by modifying the neural-network architecture. Specifically, we separately halve the width and the depth of the networks. The estimates remain quite stable across these specifications. In Table \ref{tab:real_result_abla2}, we further change the specifications of the linear (moment-based) bridge function estimators, by adding interaction terms, quadratic terms, or both, using the negative controls and a subset of the covariates. These seemingly mild changes in parametric form lead to large shifts in the  treatment effect estimates and confidence intervals. This illustrates that the linear specification is sensitive to how the functional form is hand-crafted, whereas the neural-network-based minimax DR estimator is more stable.

    We further conduct two complementary robustness analyses specific to the RHC
application. First, the allocation
$Z=(\texttt{pafi1},\texttt{paco21})$ and
$W=(\texttt{ph1},\texttt{hema1})$ is only one of six possible assignments of
the four candidate proxy variables. Across all six assignments, the proposed
neural-network minimax DR estimators remain negative and range only from
$-0.0630$ to $-0.0429$. In contrast, the simple linear closed-form estimator
is positive in three of the six allocations (ranging from $0.0299$ to
$1.0435$), and the linear moment-based estimator is positive in two; see
\Cref{ape:proxy-allocation}.

    Second, we keep the proxy allocation fixed and systematically omit blocks of
observed baseline covariates along two pre-specified rankings, yielding 28
increasingly severe omission scenarios. DR and DR(sta) retain negative
estimates with 95\% confidence intervals below zero in every constructed
scenario. The linear estimators are much less stable: along the
treatment--outcome predictive omission path, removing only the top four
covariate blocks changes the linear moment-based estimate from $-0.0828$ in
the full-$X$ analysis to $0.0828$, and the linear closed-form estimate from
$-0.0860$ to $0.3822$; both corresponding confidence intervals include zero.
More elaborate linear basis specifications can also produce very wide
intervals under more severe omissions. Details are provided in
\Cref{ape:covariate-omission}. Together, these exercises provide empirical
evidence that the flexible minimax estimators are more stable under the
specified proxy-allocation and covariate-omission perturbations. We emphasize that these analyses are empirical stress tests rather than formal
tests of the proxy identification assumptions. Withholding observed covariates
can change the effective unmeasured-confounding structure and need not preserve
the proxy assumptions. Therefore, the observed stability does not establish
proxy validity or robustness to arbitrary unmeasured confounding. Rather, it
provides practically useful evidence that, under the specified proxy-allocation
and covariate-omission perturbations, the flexible minimax estimators are less
sensitive than the linear alternatives.

\begin{table}
\caption{\label{tab:real_result} Treatment effect estimates (standard errors) and $95\%$ confidence intervals of the average treatment effect.}
\scalebox{0.9}{

    \begin{tabular}{lllll}
    \toprule
      &  DR  & DR(sta) &
      Linear(closed-form)
      & Linear(moment-based)

      \\
         \midrule
     Treatment effects (SEs) &-0.0528 (0.01194)  &  -0.0535 (0.01064)  & -0.0860 (0.02408)  & -0.0828 (0.01977) \\

     $95\%$ CIs  & [-0.0762, -0.0294]    & [-0.0743, -0.0326]    &  [-0.1332, -0.0388] &   [-0.1216, -0.0441] \\

     \bottomrule
    \end{tabular}
    }
\end{table}

\begin{table}[t]
\caption{\label{tab:real_result_abla1} Treatment effect estimates (standard errors) and $95\%$ confidence intervals for our proposed DR minimax estimators based on neural network bridge function estimation,  across two different specifications of the neural networks. The ``Half Width'' setting reduces the number of hidden neurons in the neural network bridge function estimators by a half, while the ``Half Depth'' setting reduces the depth of the neural networks by a half. }
\scalebox{0.9}{

        \begin{tabular}{ccccc}
    \toprule
      & \multicolumn{2}{c}{Half Width} & \multicolumn{2}{c}{Half Depth} \\

       &  DR  & DR(sta) &
      DR
      &
      DR(sta)
      \\
         \midrule
     Treatment effects (SEs) & -0.0595 (0.00982)  &  -0.0524 (0.01052)  &-0.0436 (0.01009)  & -0.0605 (0.01017)   \\
     $95\%$ CIs  & [-0.0787, -0.0403]   & [-0.0730, -0.0318]    &  [-0.0634, -0.0238] &   [-0.0804, -0.0406] \\
     \bottomrule
    \end{tabular}
}

     \caption{\label{tab:real_result_abla2} Treatment effect estimates (standard errors) and $95\%$ confidence intervals of the ATE estimator with linear bridge functions, across three different specifications of the linear (moment-based) bridge functions. The ``Interaction'' specification additionally incorporates product interaction terms of variables ($WX_{\text{sub}}, ZX_{\text{sub}}$) in the linear bridge function estimation (where $X_{\text{sub}}$ is a subset of covariates $X$ that consist of the two mostly predictive covariates for the outcome $Y$), ``Quadratic'' specification incorporates quadratic terms ($W^2, Z^2$), while ``Both'' incorporates both interaction and quadratic terms.}
        \scalebox{0.9}
        {
        \begin{tabular}{ll lll}
    \toprule
      & Linear(moment-based)&  Linear(Interaction)  & Linear(Quadratic) & Linear(Both)
      \\
         \midrule
     Treatment effects (SEs) & -0.0828 (0.01977) & -0.0372 (0.07466) & -0.0859 (0.02427) & +0.1053 (0.97969)\\
     $95\%$ CIs &[-0.1216, -0.0441]  & [-0.1835, +0.1091]    &   [-0.1334, -0.0383] & [-1.8149, +2.0255] \\
     \bottomrule
    \end{tabular} }
\end{table}

\section{Conclusions and Future Work}
In this paper, we tackled a central challenge in doing causal inference using negative controls: estimating the bridge functions.

We developed an alternative identification strategy that does not rely on the completeness conditions widely assumed in the previous literature.
We proposed new minimax estimators for the bridge functions that were amenable to general function approximation. We studied the behavior of these bridge function estimators and the resulting GACE estimators under a range of different assumptions.

Our work can be extended to tackle complex estimation in other settings. For example, both \citet{Tennenholtz_Shalit_Mannor_2020,Lee2021} study complex settings where a causal estimand is identified using a proxy or negative-control approach -- the first in a partially observable reinforcement-learning setting, and the second in a more general directed-acyclic-graph setting -- but both focus on the setting of \emph{discrete} data distributions for simplicity.
By leveraging our work, which allows flexible hypothesis classes, we may be able to tackle these more complicated settings on continuous spaces.

\newpage
\appendix

\section{Some Weaker Conditions that Fail to Ensure Identification}\label{app: single-bridge}

In this section, we prove three statements: (1) the existence of action bridge functions alone ($\mathbb{Q}_0\neq \emptyset$) is not enough to identify GACE; (2) the existence of outcome bridge functions alone ($\mathbb{H}_0 \neq \emptyset$) is not enough to identify GACE; (3) the existence of two observed data bridge functions ($\mathbb{H}^{\text{obs}}_0 \ne   \emptyset$ and $\mathbb{Q}^{\text{obs}}_0 \ne \emptyset$) is not enough for identification.

\subsection{The Existence of Action Bridge Functions Alone Is Not Enough }

Suppose $A,Z,U,Y$ are binary variables and $X,W$ are empty (or  constant) variables.
We want to construct  two different instances that satisfy the following conditions:
\begin{itemize}
    \item The joint distributions of $(Z,A,Y)$ are the same.
    \item Both instances satisfy $Z \perp Y \mid A,U$.
    \item The $2\times 2$ matrices corresponding to $P(\mathbf{Z}\mid \mathbf{U},A=0)$ and $P(\mathbf{Z}\mid \mathbf{U},A=1)$ are both invertible,  which ensures the existence of action bridge functions.

    \item The values of GACE corresponding to the contrast function $\pi(a \mid x)=I(a=0)$, namely the counterfactual mean parameter $\E[Y(0)]$, are different in the two instances. This target parameter can be written as
    \begin{align*}
        \E[Y(0)]=\E[\E[Y\mid A=0,U]].
    \end{align*}
\end{itemize}

Table \ref{tab:two_different_second} describes two different instances satisfying the above conditions where the parameter $\E[Y(0)]$ is equal to $0.5532$ in one and to $0.543896$ in the other.

\subsection{The Existence of Outcome Bridge Functions Alone is Not Enough}

Suppose $A,W,U,Y$ are binary variables and $X,Z$ are empty (or constant) variables.
We want to construct  two different instances that satisfy the following conditions:
\begin{itemize}
    \item The joint distributions of $(W,A,Y)$ are the same.
    \item Both instances satisfy $W \perp A \mid U$.
    \item The $2 \times 2$ matrices corresponding to $P(\mathbf{W} \mid \mathbf{U},A=0)$ and $P(\mathbf{W} \mid \mathbf{U},A=1)$  are  both invertible, which ensures the existence of outcome bridge functions.
    \item The GACE parameter $\E[Y(0)]$ takes different values in two instances.
\end{itemize}

Table \ref{tab:two_different} describes two different instances satisfying the above conditions where the parameter $\E[Y(0)]$ is equal to $1/2$ in one and to $72/175 \approx 0.411429$ in the other.

\begin{table}[t!]
\centering
\caption{Two instances where the action bridge functions exist and the distributions of the observables $(Z, A, Y)$ are identical but the parameter $\Eb{Y(0)}$ has different values.
 }
\label{tab:two_different_second}
\begin{tabular}{ccccccc}
\toprule
&  U  &  A & Z & Y  &  $ P$ (instance 1) & $ P$ (instance 2)   \\ \midrule
&  0  &  0 &  0  &  0 &  0.075  & 0.025\\
&  0  &  0 &  0  &  1 &  0.05625 & 0.04375\\
&  0  &  0 &  1  &  0 &  0.025 & 0.075\\
&  0  &  0 &  1  &  1 &  0.01875 & 0.13125\\
&  0  &  1 &  0  &  0 &  0.0375 & 0.09\\
&  0  &  1 &  0  &  1 &  0.01875 &0.12\\
&  0  &  1 &  1  &  0 &  0.1125 & 0.03\\
&  0  &  1 &  1  &  1 &  0.05625 & 0.04\\
&  1  &  0 &  0  &  0 &  0.025 & 0.075\\
&  1  &  0 &  0  &  1 &  0.04375 & 0.05625\\
&  1  &  0 &  1  &  0 &  0.075 & 0.025\\
&  1  &  0 &  1  &  1 &  0.13125 & 0.01875\\
&  1  &  1 &  0  &  0 &  0.1125 & 0.06\\
&  1  &  1 &  0  &  1 &  0.13125 & 0.03\\
&  1  &  1 &  1  &  0 &  0.0375 & 0.12\\
&  1  &  1 &  1  &  1 &  0.04375 & 0.06 \\ \bottomrule
\end{tabular}
\end{table}

\begin{table}[t!]
\centering
\caption{Two instances where the outcome bridge functions exist and the distributions of the observables $(W, A, Y)$ are identical but the parameter $\Eb{Y(0)}$ has different values.
}
\label{tab:two_different}
\begin{tabular}{ccccccc}
    \toprule
 &  U  &  Y & A & W  &  $ P$ (instance 1) & $ P$ (instance 2)\\ \midrule
 &  0  &  0 &  0  &  0 &  0.256  & 0.241\\
 &  0  &  0 &  0  &  1 &  0.064 & 0.029\\
 &  0  &  0 &  1  &  0 &  0.056 & 0.023\\
 &  0  &  0 &  1  &  1 &  0.014 & 0.003\\
 &  0  &  1 &  0  &  0 &  0.064 & 0.060\\
 &  0  &  1 &  0  &  1 &  0.016 & 0.020\\
 &  0  &  1 &  1  &  0 &  0.024 & 0.020\\
 &  0  &  1 &  1  &  1 &  0.006 & 0.004\\
 &  1  &  0 &  0  &  0 &  0.004 & 0.019\\
 &  1  &  0 &  0  &  1 &  0.016 & 0.051\\
 &  1  &  0 &  1  &  0 &  0.024 & 0.057\\
 &  1  &  0 &  1  &  1 &  0.096 & 0.107\\
 &  1  &  1 &  0  &  0 &  0.016 & 0.020\\
 &  1  &  1 &  0  &  1 &  0.064 & 0.060\\
 &  1  &  1 &  1  &  0 &  0.056 & 0.060\\
 &  1  &  1 &  1  &  1 &  0.224 & 0.226  \\
 \bottomrule
\end{tabular}
\end{table}

For the two instances in \Cref{tab:two_different}, direct summation over $U$ gives the same distribution of $(W,A,Y)$. Moreover, for $a=0,1$,
\[
P_1(\mathbf W\mid \mathbf U,A=a)=
\begin{pmatrix}0.8&0.2\\0.2&0.8\end{pmatrix},
\qquad
P_2(\mathbf W\mid \mathbf U,A=a)=
\begin{pmatrix}0.86&0.26\\0.14&0.74\end{pmatrix}.
\]
Both matrices have determinant $0.6$, and neither depends on $a$, so $W\perp A\mid U$ and outcome bridge functions exist in both instances.

\subsection{The Existence of Observed Data Bridge Functions is Not Enough}\label{sec: obs-data-bridge-not-enough}
In \Cref{lemma: relaxed-ipw-dm}, we prove that the target GACE parameter is identifiable if either $\Hbbb_0 \ne \emptyset$ and $\Qbbb^{\text{obs}}_0 \ne \emptyset$ or  $\Qbbb_0 \ne \emptyset$ and $\Hbbb^{\text{obs}}_0 \ne \emptyset$. In particular, the identification conditions are satisfied if $\Hbbb_0 \ne \emptyset$ and $\Qbbb_0 \ne \emptyset$. However, one may wonder whether the weaker condition $\Hbbb^{\text{obs}}_0 \ne \emptyset$ and $\Qbbb^{\text{obs}}_0 \ne \emptyset$ would suffice for the identification. In this part, we show this is not enough for identification.

Specifically, we ignore the covariates for simplicity and construct two instances of joint distributions of $(Z, W, U, A, Y)$, where all variables are zero-one binary random variables. These two instances are observationally equivalent, in the sense that their distributions of $(Z, W, A, Y)$ are identical. Moreover, they both satisfy the conditions $\Hbbb^{\text{obs}}_0 \ne \emptyset$ and $\Qbbb^{\text{obs}}_0 \ne \emptyset$  of interest, but do not satisfy our proposed conditions (i.e., they have $\Hbbb_0 = \emptyset$ and  $\Qbbb_0 = \emptyset$). We will show that their corresponding average treatment effect parameters are different, so the causal identification fails.

In both instances, we impose the restrictions $(Z, W) \perp (A, Y, U)$, and for simplicity we let $P(Z= 1) = P(Z = 0) = P(W = 1) = P(W = 0) = 0.5$. Then we only need to specify the distributions of $(A, Y, U)$. Note that as long as our specifications satisfy $A \not \perp U$ and $Y \not \perp U \mid A$,
we have $\Hbbb_0 = \emptyset$ and $\Qbbb_0 =  \emptyset$ but $\Hbbb^{\text{obs}}_0 \ne \emptyset$ and $\Qbbb^{\text{obs}}_0 \ne \emptyset$ always hold.
\begin{itemize}
    \item $\Hbbb_0 = \emptyset$ and $\Qbbb_0 =  \emptyset$. Indeed, if $\Hbbb_0 \ne \emptyset$, then there exists $h(A, W)$ such that $\Eb{h(A, W) \mid A, U} = \Eb{Y \mid A, U}$. Because $W \perp (A, U)$, this requires
\begin{align*}
    \sum_{w \in \{0, 1\}} h(a, w)P(W = w) = P(Y = 1\mid A = a, U = u), ~~ u = 0, 1.
\end{align*}
Notably, there does not exist such a $h$ function, because the right hand side of the equation above varies with $u \in \{0, 1\}$ as long as $Y \not\perp U \mid A$ but the left hand side does not. Similarly, for $\Qbbb_0 \ne \emptyset$, we need $q(A, Z)$ such that $\Eb{q(A, Z) \mid A, U} = 1/P(A \mid U)$. Because $Z \perp (A, U)$, this requires
\begin{align*}
    \sum_{z\in\{0, 1\}}q(a, z)P(Z = z) = 1/P(A= a \mid U = u), ~~ u = 0, 1.
\end{align*}
Again, there does not exist such $q$ function because $A \not \perp U$. Therefore, we have  $\Hbbb_0 = \emptyset$ and $\Qbbb_0 = \emptyset$ as long as $A \not\perp U, Y \not\perp U \mid A$.

\item $\Hbbb^{\text{obs}}_0 \ne \emptyset$ and $\Qbbb^{\text{obs}}_0 \ne \emptyset$. Specifically, we need to find $h(A, W)$ such that
$\Eb{Y \mid A, Z} = \Eb{h(A, W) \mid A, Z}$. By $(W, Z) \perp (A, Y, U)$, this is equivalent to
\begin{align*}
   \Eb{Y \mid A} = \Eb{\Eb{h(A, W) \mid A, U} \mid A, Z} = \Eb{\Eb{h(A, W) \mid A, U} \mid A} = \Eb{h(A, W) \mid A}
\end{align*}
Namely, we need to find a function $h(a, w)$ such that $\Eb{Y \mid A = a} = \Eb{h(a, W)}$ for $a \in \{0, 1\}$, which obviously always exists so that $\Hbbb^{\text{obs}}_0 \ne \emptyset$. Similarly, for $\Qbbb^{\text{obs}}_0 \ne \emptyset$, we only need to find $q(a, z)$ such that
\begin{align*}
    \Eb{q(a, Z)} = \frac{1}{P(A= a)}, ~~ a \in \{0, 1\}.
\end{align*}
  Again, such a $q$ always exists as long as $P(A = 1) >0, P(A = 0 )> 0$.
\end{itemize}

Therefore, we only need to construct two instances of joint distributions of $(A, Y, U)$ such that: (1) $A \not \perp U$ and $Y \not \perp U \mid A$; (2) their corresponding  distributions of $(A, Y)$ are identical; (3) their corresponding average treatment effects are different. Below we present two possible instances.

\paragraph{Instance One. } The following give the distributions for the first instance:
\begin{align*}
    P(U = 1) = 0.3725465,\\
    P(A = 1 \mid U = 0) = 0.7306441,~~ P(A = 1 \mid U = 1) = 0.4271815,  \\
    P(Y = 1 \mid A = 0, U = 0) = 0.8064139, ~~P(Y = 1 \mid A = 0, U = 1) = 0.8523738, \\
    P(Y = 1 \mid A = 1, U = 0) = 0.1364452, ~~ P(Y \mid A = 1, U = 1) = 0.5224844.
\end{align*}
These imply that the observed variables $(A, Y)$ have the following distributions:
\begin{align*}
    P(A = 1) = 0.6175902, ~~P(Y = 1, A = 1) = 0.1457034, ~~P(Y = 1, A = 0) =  0.3181885.
\end{align*}
The average treatment effect under this instance is
\begin{align*}
    \Eb{\Eb{Y \mid A = 1, U} - \Eb{Y \mid A = 0, U}} = -0.5432734.
\end{align*}

\paragraph{Instance Two.}
The following give the distributions for the second instance:
\begin{align*}
    P(U = 1) = 0.7354514,\\
    P(A = 1 \mid U = 0) = 0.541148,~~ P(A = 1 \mid U = 1) = 0.6450871,  \\
    P(Y = 1 \mid A = 0, U = 0) = 0.8654667, ~~P(Y = 1 \mid A = 0, U = 1) = 0.8165264, \\
    P(Y = 1 \mid A = 1, U = 0) = 0.4652918, ~~ P(Y \mid A = 1, U = 1) = 0.16671.
\end{align*}
These imply that the observed variables $(A, Y)$ have the same distributions:
\begin{align*}
    P(A = 1) = 0.6175902, ~~P(Y = 1, A = 1) = 0.1457034, ~~P(Y = 1, A = 0) =  0.3181885.
\end{align*}
The average treatment effect under this instance is
\begin{align*}
    \Eb{\Eb{Y \mid A = 1, U} - \Eb{Y \mid A = 0, U}} = -0.5837741.
\end{align*}

\section{Comparing Identification Strategies}\label{sec: comparison}
One notable difference between our paper and previous literature is that our paper does not assume any completeness condition but previous papers do. {In this section, we take the identification strategy in \citet{CuiYifan2020Spci} as an example to illustrate the  role of completeness conditions in the previous literature and compare the previous identification strategies with  our proposed identification strategy. }

\subsection{Conditions in \citet{CuiYifan2020Spci}  }\label{sec: old-identification}

We first recall the sets of bridge functions given by the conditional moment equations in  \Cref{eq: bridge-U-h,eq: bridge-U-q}:
\begin{align*}
    &\Hbbb_0
        = \braces{h \in L_2(W, A, X): \E[Y-h(W,A,X)\mid A,U,X]=0}, \\
&\Qbbb_0 =
      \braces{q : \pi q \in L_2(Z, A, X), \E[\epol(A\mid X)\prns{q(Z, A, X)-1/f(A\mid U, X)} \mid A,U, X]=0},
\end{align*}
and the sets of bridge functions given by the observed data conditional moment equations in  \Cref{eq: observed-bridge-h,eq: observed-bridge-q}:
\begin{align*}
    &\Hbbb_0^{\obs}
        = \braces{h \in  L_2(W, A, X): \E[Y-h(W,A,X)\mid Z,A,X]=0}, \\
    &\Qbbb_0^{\obs} =
      \braces{q :  \pi q\in  L_2(Z, A, X), \E[\epol(A\mid X)\prns{q(Z, A, X)-1/f(A\mid W, X)} \mid W, A, X]=0}.
\end{align*}
In \Cref{lem:bridge-identification}, we already show that any bridge functions in $\Hbbb_0$ and $\Qbbb_0$ can identify the causal parameter. However, we cannot directly learn functions in $\Hbbb_0$ and $\Qbbb_0$ from the observed data, because they depend on the unmeasured confounders $U$. Instead, we can at most learn functions in $\Hbbb_0^{\obs}$ and $\Qbbb_0^{\obs}$ from the observed data.
Although in \Cref{lemma: observed-bridge}, we already show that $\Hbbb_0 \subseteq \Hbbb_0^{\obs}$ and $\Qbbb_0 \subseteq \Qbbb_0^{\obs}$,
the converse may not be true (e.g., \Cref{sec: obs-data-bridge-not-enough} provides some examples). In other words, functions in $\Hbbb_0^{\obs}$ (or $\Qbbb_0^{\obs}$) that we can possibly learn from observed data may not necessarily belong to $\Hbbb_0$ (or $\Qbbb_0$). Thus without further conditions, we cannot use  \Cref{lem:bridge-identification} for identification, since this lemma only applies to  functions in $\Hbbb_0$ and $\Qbbb_0$.

\cite{CuiYifan2020Spci} handles this problem by assuming  the following  completeness conditions.

\begin{assumption}\label{assump: full-completeness}
Consider the following conditions:
\begin{enumerate}
    \item  \label{assump: completeness-1} For any $g(A, U, X) \in L_2(A, U, X)$, $\Eb{g(A, U, X) \mid W, A, X} = 0$ only when $g(A, U, X) = 0$.
    \item \label{assump: completeness-2} For any $g(A, U, X) \in L_2(A, U, X)$, $\Eb{g(A, U, X) \mid Z, A, X} = 0$ only when $g(A, U, X) = 0$.
\end{enumerate}
\end{assumption}
With these completeness conditions, we can show that $\Qbbb_0^{\text{obs}} = \Qbbb_0$ and $\Hbbb_0^{\text{obs}} = \Hbbb_0$, namely, any bridge functions that we can learn from the observed data (\ie, functions in $\Qbbb_0^{\text{obs}}$ and $\Hbbb_0^{\text{obs}}$) are indeed bridge functions defined by unmeasured confounders (\ie, functions in $\Qbbb_0$ and $\Hbbb_0$). These conclusions have appeared as intermediate steps in the identification theorems of \cite{miao2018identifying,CuiYifan2020Spci}.
\begin{lemma}\label{lemma: valid-bridge}
Suppose \Cref{asm:whole_assm} holds.
\begin{enumerate}
    \item If \Cref{assump: full-completeness} condition \ref{assump: completeness-1} further holds, then
    $\Qbbb_0^{\text{obs}} = \Qbbb_0$.
    \item If \Cref{assump: full-completeness} condition \ref{assump: completeness-2} further holds, then $\Hbbb_0^{\text{obs}} = \Hbbb_0$.
\end{enumerate}
\end{lemma}

Therefore, by assuming completeness conditions in \Cref{assump: full-completeness}, previous literature can use \Cref{lem:bridge-identification} to identify the causal parameter by any $q_0 \in \Qbbb_0^{\text{obs}} = \Qbbb_0$ and/or $h_0 \in \Hbbb_0^{\text{obs}} = \Hbbb_0$.

\begin{remark}
 Although \cite{shi2020multiply,miao2018a} do not exactly assume the completeness conditions in \Cref{assump: full-completeness}, their identification strategy is the same as those when assuming \Cref{assump: full-completeness}: namely,
 they impose conditions that ensure  $\Hbbb_0^{\text{obs}} = \Hbbb_0$, so that they can still use \Cref{lem:bridge-identification} to achieve identification. \cite{miao2018a} assumes that $\Hbbb_0 \ne \emptyset$ and that $\Hbbb_0^{\text{obs}}$ is a singleton. Since $\Hbbb_0 \subseteq \Hbbb_0^{\text{obs}}$ (\Cref{lem:bridge-identification}), we must have $\Hbbb_0=\Hbbb_0^{\text{obs}}$ is equal to the singleton.
\cite{shi2020multiply} studies discrete negative controls and unmeasured confounders (see \Cref{ex: bridge-discrete}), and focuses on the setting where these variables have the same number of values, \ie, $|\Wcal| = |\Zcal| = |\Ucal|$. In this case, they assume that the matrix $P(\mathbf{W} \mid \mathbf{Z}, a, x)$ is invertible for any $a \in \Acal, x \in \Xcal$ (see \Cref{assump: obs-completeness} condition 1 below). It is easy to show that this condition  implies that $P(\mathbf{U} \mid \mathbf{Z}, a, x)$ is also invertible, namely, \Cref{assump: full-completeness} condition \ref{assump: completeness-2} holds. Thus, \cite{shi2020multiply} implicitly requires  $\Hbbb_0^{\text{obs}} = \Hbbb_0$ as well in their main theory\footnote{\cite{shi2020multiply} also discuss the case with $|\Wcal| > |\Ucal|$ and $|\Zcal| > |\Ucal|$ so that bridge functions are nonunique and the matrix $P(\mathbf{W} \mid \mathbf{Z}, a, x)$ is not invertible. They remark that we can address the n'o'nu'ni'queness of bridge functions by considering the pseudo-inverse of the non-invertible matrix $P(\mathbf{W} \mid \mathbf{Z}, a, x)$ or by coarsening the negative controls $W$ and $Z$ so that their numbers of values are identical to that of $U$. But these approaches are not fully developed, and they still primarily focus on the setting with $|\Wcal| = |\Ucal|, |\Zcal| = |\Ucal|$ in their paper. In this main setting, their identification again requires $\Hbbb_0^{\text{obs}} = \Hbbb_0$ and $\Qbbb_0^{\text{obs}} = \Qbbb_0$.}.
\end{remark}

{ As a concrete example, we tailor the identification results in \cite{CuiYifan2020Spci} to the GACE parameter $J$ in the following proposition.
\begin{proposition}[\cite{CuiYifan2020Spci}]\label{prop: cui2020}
Suppose that \Cref{asm:whole_assm} holds.
\begin{enumerate}
\item \label{prop: cui2020-1} If $\Hbbb^{\obs}_0\neq \emptyset$ (or, if $\Hbbb_0\neq \emptyset$) and \Cref{assump: full-completeness} condition \ref{assump: completeness-2} holds, then $J = \E{\tilde\phi_\DM(O;h_0)}$ for any $h_0 \in \Hbbb_0^{\obs}$.
\item \label{prop: cui2020-2} If $\Qbbb^{\obs}_0\neq \emptyset$ (or, if $\Qbbb_0\neq \emptyset$) and \Cref{assump: full-completeness} condition \ref{assump: completeness-1} holds, then $J = \E{\tilde\phi_\ipw(O;q_0)}$ for any $q_0 \in \Qbbb_0^{\obs}$.
\item If conditions in either statement \ref{prop: cui2020-1} or statement \ref{prop: cui2020-2} hold, then $J = \E{\tilde\phi_\dr(O; h_0, q_0)}$ for any $h_0 \in \Hbbb_0^{\obs}$ and any $q_0 \in \Qbbb_0^{\obs}$.
\end{enumerate}
\end{proposition}
The proof for \Cref{prop: cui2020} goes as follows: given the conditions in statement \ref{prop: cui2020-1}, we must have $\Hbbb_0^{\text{obs}} = \Hbbb_0$ according to \Cref{lemma: valid-bridge}, so any $h_0 \in \Hbbb_0^{\obs}$ must satisfy $h_0 \in \Hbbb_0$ and thus can be used to identify $J$ according to \Cref{lem:bridge-identification}. Statement \ref{prop: cui2020-2} can be proved analogously by noting $\Qbbb_0^{\text{obs}} = \Qbbb_0$ given the conditions therein. The completeness conditions in \Cref{prop: cui2020} are crucial since they ensure $\Hbbb_0^{\text{obs}} = \Hbbb_0$ and $\Qbbb_0^{\text{obs}} = \Qbbb_0$.
Note that as \cite{CuiYifan2020Spci} remarked, \Cref{prop: cui2020} can also start with $\Hbbb_0 \ne \emptyset$ and $\Qbbb_0 \ne \emptyset$. These two conditions are equivalent to $\Hbbb_0^{\obs} \ne \emptyset$  and $\Qbbb_0^{\obs} \ne \emptyset$, respectively, under the completeness conditions in \Cref{assump: full-completeness}.
}

\subsection{Comparing Our Conditions and Conditions in \cite{CuiYifan2020Spci}}\label{sec: identification-comparison}
{
In the following proposition, we consider the discrete setting in \Cref{ex: bridge-discrete}.
We show that our identification strategies in
\Cref{lemma: relaxed-ipw-dm} require \emph{strictly} weaker conditions than the counterparts in  \Cref{prop: cui2020}. In particular, we show that conditions in \Cref{prop: cui2020} statement \ref{prop: cui2020-1} and statement \ref{prop: cui2020-2} imply conditions in \Cref{lemma: relaxed-ipw-dm}  statement \ref{lemma: relaxed-ipw-dm-1} and  statement \ref{lemma: relaxed-ipw-dm-2}, respectively. However, the converse is not true: there exist instances where conditions in \Cref{lemma: relaxed-ipw-dm} statement \ref{lemma: relaxed-ipw-dm-1} and statement \ref{lemma: relaxed-ipw-dm-2} hold but conditions in \Cref{prop: cui2020} statement \ref{prop: cui2020-1} and statement \ref{prop: cui2020-2}  are violated.

\begin{proposition}\label{prop: weaker-identify}
Consider the discrete setting in \Cref{ex: bridge-discrete} or the linear model setting in \Cref{ex: bridge-linear}. Suppose that \Cref{asm:whole_assm} holds.
\begin{enumerate}
\item \label{prop: weaker-identify-1} If
 $\Hbbb^{\obs}_0\neq \emptyset$ and the completeness condition in \Cref{assump: full-completeness} condition \ref{assump: completeness-2} holds, then $\Hbbb_0\neq \emptyset$ and $\Qbbb^{\obs}_0\neq \emptyset$. However, there exist instances such that $\Hbbb_0\neq \emptyset$ and $\Qbbb^{\obs}_0\neq \emptyset$ but \Cref{assump: full-completeness} condition \ref{assump: completeness-2} does not hold and $\Hbbb_0 \subsetneq \Hbbb_0^{\obs}$.
\item \label{prop: weaker-identify-2} If $\Qbbb^{\obs}_0\neq \emptyset$ and the completeness condition in  \Cref{assump: full-completeness} condition \ref{assump: completeness-1} holds, then $\Qbbb_0\neq \emptyset$ and $\Hbbb^{\obs}_0\neq \emptyset$. However, there exist instances such that $\Qbbb_0\neq \emptyset$ and $\Hbbb^{\obs}_0\neq \emptyset$ but \Cref{assump: full-completeness} condition \ref{assump: completeness-1} does not hold and $\Qbbb_0 \subsetneq \Qbbb_0^{\obs}$.
\end{enumerate}
\end{proposition}
}
\begin{proof}
{ We prove \Cref{prop: weaker-identify} statement \ref{prop: weaker-identify-1} here. Statement \ref{prop: weaker-identify-2} can be proved analogously.

\paragraph{Discrete Data Setting.} We first consider the discrete setting in \Cref{ex: bridge-discrete} where $W, Z, U$ are all discrete variables. For simplicity, we omit the covariates. 

First, according to \Cref{lemma: valid-bridge}, under \Cref{assump: full-completeness} condition \ref{assump: completeness-2}, $\Hbbb_0 = \Hbbb_0^{\obs}$ so
$\Hbbb_0^{\obs} \ne \emptyset$ implies $\Hbbb_0 \ne \emptyset$. Moreover, in the discrete setting given in \Cref{ex: bridge-discrete}, the completeness condition in \Cref{assump: full-completeness} condition \ref{assump: completeness-2} amounts to $P(\bU\mid \bZ,a)$ having full row rank for any $a$. By Bayes rule, this implies that $P(\bZ\mid \bU,a)$ has full column rank for any $a$. It then follows  from \Cref{eq:requirement} that $\Qbbb_0 \ne \emptyset$. Since $\Qbbb_0\subseteq \Qbbb_0^{\obs}$ according to \Cref{lemma: observed-bridge}, we must have $\Qbbb_0^{\obs} \ne \emptyset$ as well. This proves the first part of statement \ref{prop: weaker-identify-1}.

Now we consider the second part of statement \ref{prop: weaker-identify-1}.
  We construct instances such that $\Hbbb_0 \ne \emptyset$ and $\Qbbb_0^{\obs} \ne \emptyset$  but the completeness condition in \Cref{assump: full-completeness} condition \ref{assump: completeness-2} does not hold.
The idea is as follows: according to \Cref{eq:requirement}, $\Hbbb_0 \ne \emptyset, \Qbbb_0 \ne \emptyset$ if there exist solutions $h_0\prns{\bW, a}$ and $q_0\prns{\bZ, a}$ to the linear equation systems below for any $a \in \braces{0, 1}$:
\begin{align*}
  h^\top_0(\mathbf{W}, a)P(\mathbf{W} \mid \mathbf{U}, a) = \E\bracks{Y \mid \mathbf{U}, a},\,q_0^\top(\mathbf{Z}, a)P(\mathbf{Z} \mid \mathbf{U}, a)F(a \mid \mathbf{U}) = \mathbf{e}^\top.
\end{align*}
Thus
once we specify $P(\mathbf{W} \mid \mathbf{U}, a), \E\bracks{Y \mid \mathbf{U}, a}, P(\mathbf{Z} \mid \mathbf{U}, a), F(a \mid \mathbf{U})$ in such a way that solutions to the linear equations above exist, we have
$\Hbbb_0\neq \emptyset$ and $\Qbbb^{\obs}_0\neq \emptyset$ (because $\Qbbb_0 \subseteq \Qbbb_0^{\obs}$ according to \Cref{lemma: observed-bridge}). At the same time, we can specify them in such a way that $P(\mathbf{Z} \mid \mathbf{U}, a)$ has  deficient column rank, which in turn implies that $P(\mathbf{U} \mid \mathbf{Z}, a)$ has deficient row rank so the completeness condition given in \Cref{assump: full-completeness} condition \ref{assump: completeness-2} does not hold.
This is not difficult as long as $F(a \mid \mathbf{U})$ is chosen appropriately such that $\mathbf{e}^\top F^{-1}(a \mid \mathbf{U})$ belongs to the row space of $P(\mathbf{Z} \mid \mathbf{U}, a)$ and thus $q_0\prns{\bZ, a}$ exists. In other words, we construct an instance where the proxy $Z$ does not capture all variations of $U$ so the row space of $P(\mathbf{Z} \mid \mathbf{U}, a)$ is deficient. Nevertheless, $Z$ still captures the relevant information in $U$, such that the row space of $P(\mathbf{Z} \mid \mathbf{U}, a)$ can still cover the vector $\mathbf{e}^\top F^{-1}(a \mid \mathbf{U})$ that characterizes the unobserved confounding in the action $A$.  

Let us consider a simplified setting where $Y \in \braces{0, 1}$, $U, W, Z \in \braces{0, 1, 2}$, $A \in \braces{0, 1}$.
We provide a concrete instance for $a = 0$ that violates \Cref{assump: full-completeness} condition \ref{assump: completeness-2} but satisfies $\Hbbb_0\neq \emptyset$ and $\Qbbb^{\obs}_0\neq \emptyset$.
It is easy to construct an instance for $a = 1$ analogously so we omit it here. We specify the following: for $a = 0$,
\begin{align*}
P\prns{\bU \mid \bZ, a} =
\begin{bmatrix}
\frac{1}{4} & \frac{1}{4} & \frac{1}{4} \\
\frac{1}{4} & \frac{1}{4} & \frac{1}{4} \\
\frac{1}{2} & \frac{1}{2} & \frac{1}{2}
\end{bmatrix},
~~
P\prns{\bW \mid \bU, a} =
\begin{bmatrix}
\frac{1}{3} & \frac{1}{6} & \frac{1}{4} \\
\frac{1}{3} & \frac{1}{2} & \frac{5}{12} \\
\frac{1}{3} & \frac{1}{3} & \frac{1}{3}
\end{bmatrix}, \\
P\prns{\bZ \mid a} = \begin{bmatrix}
\frac{1}{2} \\ \frac{1}{3} \\ \frac{1}{6}
\end{bmatrix},
~~
P\prns{Y = 1 \mid \bU, a} =
\begin{bmatrix}
\frac{1}{5}  \\
\frac{2}{5} \\
\frac{3}{10}
\end{bmatrix},
~~
F\prns{a \mid \bU} =
\begin{bmatrix}
\frac{3}{8} & 0 & 0 \\
0 & \frac{3}{8} & 0 \\
0 & 0 & \frac{3}{8}
\end{bmatrix}.
\end{align*}
In this instance, the rank of $P\prns{\bU \mid \bZ, a}$ is $1$ so \Cref{assump: full-completeness} condition \ref{assump: completeness-2} does not hold. However, we can easily check that
\begin{align*}
q_0\prns{\bZ, a} =
\begin{bmatrix}
2 \\
3 \\
4
\end{bmatrix} \in \Qbbb_0 \subseteq \Qbbb_0^{\obs},
~~
h_0\prns{\bW, a}
=
\begin{bmatrix}
-\frac{3}{10} \\
\frac{9}{10} \\
0
\end{bmatrix}
\in \Hbbb_0.
\end{align*}
Thus $\Hbbb_0\neq \emptyset$ and $\Qbbb^{\obs}_0\neq \emptyset$ but \Cref{assump: full-completeness} condition \ref{assump: completeness-2} does not hold. In this instance, actually $\Hbbb_0 \subsetneq \Hbbb_0^{\obs}$. For example:
\begin{align*}
h_0\prns{\bW, a} =
\begin{bmatrix}
\frac{6}{5} \\ 0 \\ 0
\end{bmatrix}
\in \Hbbb_0^{\obs} \setminus \Hbbb_0.
\end{align*}
In contrast, if the condition in \Cref{prop: cui2020} statement \ref{prop: cui2020-1} were true, we must have $\Hbbb_0 = \Hbbb_0^{\obs}$. This again confirms that conditions in our \Cref{lemma: relaxed-ipw-dm} statement \ref{lemma: relaxed-ipw-dm-1} are strictly weaker than the counterparts in \Cref{prop: cui2020} statement \ref{prop: cui2020-1}.
}

\paragraph{Linear Model Setting.} We now consider the linear model setting in \Cref{ex: bridge-linear}. In that example, we assume the following data generating process:
\begin{align*}
    Y &= \alpha_Y^\top U + \beta_Y^\top X + \gamma_Y^\top A + \omega_Y^\top W +  \epsilon_Y, \\
    Z &= \alpha_Z U + \beta_Z X + \gamma_Z A + \epsilon_Z, \\
    W &= \alpha_W U + \beta_W X +  \epsilon_W, \\
    A &\sim \mathrm{Unif}\prns{\underline{\alpha}_A^\top U + \underline{\beta}_A^\top X,\,\overline{\alpha}_A^\top U + \overline{\beta}_A^\top X},
\end{align*}
We first prove that if the completeness condition in \Cref{assump: full-completeness} condition \ref{assump: completeness-2} holds, then the  matrix $\alpha_Z$ must have  full column rank. We argue this  by contraposition.

Let $d_U = \dim(U)$. Suppose $\mathrm{rank}(\alpha_Z) = r < d_U$. Then there exist matrices $C \in \mathbb{R}^{r \times d_U}$ and $D \in \mathbb{R}^{(d_U - r)\times d_U}$ such that the map $u \mapsto (Cu, Du)$ is bijective, and the row space of $C$ coincides with that of $\alpha_Z$.
Define $T = C U, S = D U$.
Then $U \leftrightarrow (T,S)$ bijectively, and there exists a matrix $M$ such that $\alpha_Z U = M T$. Hence
\[
Z = M T + \beta_Z X + \gamma_Z A + \epsilon_Z, \text{ and } Z \perp S \mid (A,T,X).
\]

Define a nonzero $V = v^\top S \in L^2$ for some $v \in \R{d_U - r}$ such that $\mathrm{Var}(V \mid A,T,X) > 0$ on a set of positive probability. Further define $g(A,U,X) = V - \mathbb{E}[V \mid A,T,X]$.
Then $g(A, U, X) \neq 0$ with positive probability, and $\mathbb{E}[g(A, U, X) \mid A,T,X] = 0$.
Using $Z \perp S \mid (A,T,X)$, we obtain
\[
\mathbb{E}[g(A, U, X) \mid A,Z,X,T]  = \mathbb{E}[g(A, S, T, X) \mid A,Z,X,T]  = \mathbb{E}[g(A, U, X) \mid A,T,X] = 0,
\]
and hence, even though $g(A, U, X)$ is nonzero with positive probability, we still have
\[
\mathbb{E}[g(A, U, X) \mid A,Z,X] = \Eb{\mathbb{E}[g(A, U, X) \mid A,Z,X,T] \mid A, Z, X} = 0.
\]
This contradicts completeness. Therefore $\mathrm{rank}(\alpha_Z) = d_U$. This means that completeness condition implies the full column rank of matrix $\alpha_Z$.

Then it follows from the completeness condition that there exists solution $\theta_Z$ to the linear equation system $\theta_Z^\top \alpha_Z = \overline{\alpha}_A^\top - \underline{\alpha}_A^\top$, which in turn implies the existence of bridge function $q_0 \in \Qbbb_0$ according to the discussions in \Cref{ex: bridge-linear}: $q_0(Z, A, X) = \theta_Z^\top Z + \prns{\overline{\beta}_A - \underline{\beta}_A - \theta_Z^\top \beta_Z}X -\theta_Z^\top \gamma_Z A \in \Qbbb_0$. This proves the first part of statement 1.

To prove the second part of statement 1, we only need to construct instances such that $\Hbbb_0 \ne \emptyset$ and $\Qbbb_0^{\obs} \ne \emptyset$  but the completeness condition in \Cref{assump: full-completeness} condition \ref{assump: completeness-2} does not hold. To this end, we construct $\alpha_W, \alpha_Z, \alpha_Y, \overline{\alpha}_A, \underline{\alpha}_A$ such that $\alpha_W, \alpha_Z$ both have deficient column rank but
$\alpha_Y$ is in the column space of $\alpha_W$ and $\overline{\alpha}_A -  \underline{\alpha}_A$ is in the column space of $\alpha_Z$. Then there exist solutions $\theta_W, \theta_Z$ to the equations $\theta^\top_W \alpha_W = \alpha_Y^\top$ and $\theta_Z^\top \alpha_Z = \overline{\alpha}_A^\top - \underline{\alpha}_A^\top$ respectively. According to the discussions in \Cref{ex: bridge-linear}, we have
\begin{align*}
    &h_0(W, A, X) = \prns{\theta_W + \omega_Y}^\top W + \prns{\beta_Y - \theta_W^\top \beta_W}X + \gamma_Y A \in \Hbbb_0, \\
    &q_0(Z, A, X) = \theta_Z^\top Z + \prns{\overline{\beta}_A - \underline{\beta}_A - \theta_Z^\top \beta_Z}X -\theta_Z^\top \gamma_Z A \in \Qbbb_0.
\end{align*}
So for any such instance, we have $\Hbbb_0 \ne \emptyset$ and $\Qbbb_0 \ne \emptyset$. This again implies that $\Hbbb_0 \ne \emptyset$ and $\Qbbb^{\text{obs}}_0 \ne \emptyset$. Moreover, because $\alpha_Z$ has deficient column rank, the completeness condition in \Cref{assump: full-completeness} condition \ref{assump: completeness-2} does not hold. In summary, as long as $\alpha_Z$ does not have full column rank but $\alpha_Y, \overline{\alpha}_A, \underline{\alpha}_A$ are specified such that $\alpha_Y$ is in the column space of $\alpha_W$ and $\overline{\alpha}_A -  \underline{\alpha}_A$ is in the column space of $\alpha_Z$, then $\Hbbb_0 \ne \emptyset$ and $\Qbbb^{\text{obs}}_0 \ne \emptyset$ but the completeness condition in \Cref{assump: full-completeness} condition \ref{assump: completeness-2} fails. Again, this is because the proxy variables $Z$ do not capture variations of $U$ along all directions so that $\alpha_Z$ has deficient column rank and the completeness condition fails. However, it captures the variation of $U$ along the relevant direction, such that the column space of $\alpha_Z$ can still cover $\overline{\alpha}_A -  \underline{\alpha}_A$  that characterizes the unobserved confounding in action $A$. 

As a concrete example,  consider the instance with the following parameter specifications:
\begin{align*}
    \alpha_W = \begin{bmatrix} 1 & 0 \end{bmatrix}, \alpha_Y = \begin{bmatrix} 2 \\ 0 \end{bmatrix}
\end{align*}
and
\begin{align*}
    \alpha_Z = \begin{bmatrix} 0 & 1 \end{bmatrix}, \quad \overline{\alpha}_A = \begin{bmatrix} 1 \\ 3 \end{bmatrix}, \quad \underline{\alpha}_A = \begin{bmatrix} 1 \\ 2 \end{bmatrix}.
\end{align*}
Note in this instance, the dimension of proxy variables  $p_w = p_z = 1$ is smaller than the dimension of the unobserved confounders $p_u = 2$. Under this case our proposed identification conditions still hold but the previous identification conditions fail.

\end{proof}

\begin{remark}
In \Cref{prop: weaker-identify}, we focus on the discrete data setting  in \Cref{ex: bridge-discrete} and the linear model setting in \Cref{ex: bridge-linear} to compare the IPW identification assumptions and REG identification assumptions in \Cref{lemma: relaxed-ipw-dm} and \Cref{prop: cui2020}, separately.
In the discrete data setting and linear model setting, the completeness conditions in \Cref{assump: full-completeness} are sufficient but not necessary for the existence of bridge functions. So our identification strategies in \Cref{lemma: relaxed-ipw-dm} that only require existence of bridge functions (and observed bridge functions) are strictly weaker than those based on completeness conditions.
For more general settings, the existence of bridge functions involve equations defined by potentially infinite-dimensional and ill-posed linear operators. Because of technical challenges with infinite-dimensional function spaces, the existence of bridge functions and the corresponding completeness condition are closely related but not directly comparable without further regularity conditions. In \Cref{sec: completeness-existence}, we show that the completeness condition can imply the existence of bridge functions if additional regularity conditions are assumed.
\end{remark}

\begin{remark}
Often one may hope to assume that identification assumptions in both statements \ref{lemma: relaxed-ipw-dm-1} and \ref{lemma: relaxed-ipw-dm-2} of \Cref{lemma: relaxed-ipw-dm} (or assumptions in both statements \ref{prop: cui2020-1} and \ref{prop: cui2020-2} of \Cref{prop: cui2020}) hold simultaneously, rather than separately. Then one can apply all of the three identification formulae and the corresponding estimators, or consider semiparametrically efficient estimation (see \Cref{sec: efficiency-bound}).
In this case, our combined identification assumptions are $\Hbbb_0 \ne \emptyset, \Qbbb_0 \ne \emptyset$, while the combined identification assumptions in \Cref{prop: cui2020} are both $\Hbbb_0^{\obs} \ne \emptyset, \Qbbb_0^{\obs} \ne \emptyset$ and the two completeness conditions in \Cref{assump: full-completeness}.
The latter are sufficient conditions for the former, since \Cref{assump: full-completeness} implies that $\Hbbb_0^{\obs} = \Hbbb_0, \Qbbb_0^{\obs} = \Qbbb_0$.
This means that when considering combined identification assumptions, our identification assumptions are never stronger than existing ones.
\end{remark}

\section{Completeness Conditions}\label{sec: completeness}
In \Cref{sec: comparison}, we show that completeness conditions in \Cref{assump: full-completeness} play an important role in the identification strategy in previous literature. In this section, we review some related completeness conditions, discuss their relations, and further describe how our assumptions in \Cref{sec: setup}  differ from those in previous literature.
{For completeness conditions in other settings, such as nonparametric instrumental variable models, see review and discussions in
\cite{hu2011nonparametric,haultfoeuille2011on,darolles2010nonparametric,newey2003instrumental}.}
Throughout this section, we always assume \Cref{asm:whole_assm} so we suppress this assumption in all statements.

\subsection{Different Versions of Completeness Conditions}
In the following assumptions, we list two other completeness conditions that also involve the unobserved confounders $U$. Although these conditions are not directly assumed in previous literature, we will show in \Cref{lemma:completeness-relation} that they are implied by some other conditions that appear in the previous literature.
\begin{assumption}\label{assump: full-completeness-2}
Consider the following conditions:
\begin{enumerate}

    \item \label{full-completeness-1} For any $g(W, A, X) \in L_2(W, A, X)$, $\Eb{g(W, A, X) \mid A, U, X} = 0$ only when $g(W, A, X) = 0$.
    \item \label{full-completeness-2} For any $g(Z, A, X) \in L_2(Z, A, X)$, $\Eb{g(Z, A, X) \mid A, U, X} = 0$ only when $g(Z, A, X) = 0$.
\end{enumerate}
\end{assumption}
In the following lemma, we further show that completeness conditions in \Cref{assump: full-completeness-2} can ensure the uniqueness of bridge functions.
\begin{lemma}\label{lemma: unique-bridge}
{If completeness condition \eqref{full-completeness-1} in \Cref{assump: full-completeness-2}  holds, then $\Hbbb_0$ is either empty or a singleton.
If completeness condition \eqref{full-completeness-2} in \Cref{assump: full-completeness-2} holds, then $\Qbbb_0$ is either empty or a singleton.}
\end{lemma}

Next, we introduce two completeness conditions that involve only observed variables.
\begin{assumption}\label{assump: obs-completeness}
Consider the following conditions:
\begin{enumerate}
     \item \label{obs-completeness-1} For any $g(W, A, X) \in L_2(W, A, X)$, $\Eb{g(W, A, X) \mid Z, A, X} = 0$ only when $g(W, A, X) = 0$.
    \item \label{obs-completeness-2} For any $g(Z, A, X) \in L_2(Z, A, X)$, $\Eb{g(Z, A, X) \mid W, A, X} = 0$ only when $g(Z, A, X) = 0$.
\end{enumerate}
\end{assumption}

In the following lemma, we further show the relationship among    \Cref{assump: full-completeness,assump: full-completeness-2,assump: obs-completeness}.
\begin{lemma}\label{lemma:completeness-relation}
Assume \Cref{asm:whole_assm} holds.
\begin{enumerate}
    \item \label{completeness-relation-1} If \Cref{assump: full-completeness,assump: full-completeness-2} hold, then \Cref{assump: obs-completeness}  holds.
    \item \label{completeness-relation-2} If \Cref{assump: obs-completeness} holds, then \Cref{assump: full-completeness-2} holds.
\end{enumerate}
\end{lemma}

{\Cref{lemma:completeness-relation} shows that
\Cref{assump: full-completeness,assump: full-completeness-2} are sufficient for \Cref{assump: obs-completeness}, and
\Cref{assump: obs-completeness} is sufficient for
\Cref{assump: full-completeness-2}.
Since \Cref{assump: full-completeness-2} ensure unique bridge functions according to \Cref{lemma: unique-bridge}, assuming \Cref{assump: obs-completeness}  implicitly requires the bridge functions to be unique.
For example, \cite{CuiYifan2020Spci} shows that \Cref{assump: obs-completeness} can be used to  justify the existence of observed bridge functions, \ie, $\Hbbb_0^{\text{obs}} \ne \emptyset$ and $\Qbbb_0^{\text{obs}} \ne \emptyset$. Our discussion shows that this assumption does not only have implications for the existence of bridge functions, but also has indirect implications for the uniqueness of bridge functions.
Our paper avoids assuming completeness conditions so we do not risk implicitly imposing uniqueness of the bridge functions.}

\subsection{Completeness Conditions and Existence of Bridge Functions}\label{sec: completeness-existence}
In this section, we show the existence of bridge functions in \Cref{eq: bridge-U-h,eq: bridge-U-q} under the  completeness conditions in
\Cref{assump: full-completeness}
and some additional regularity conditions, using the singular value decomposition approach in \cite{kress1989linear,miao2018identifying,CarrascoMarine2007C7LI}.

\subsubsection{Characterizing Linear Operators}
Let $K_{W \mid a, x}: L_2(W \mid A = a, X = x) \to L_2(U \mid A = a, X = x)$ and $K_{Z \mid a, x}: L_2(Z \mid A = a, X = x) \to L_2(U \mid  A = a,  X = x)$ be the linear operators defined as follows:
\begin{align*}
  &[K_{W \mid a, x}h](a, u, x) = \Eb{h(W, a, x) \mid  A = a, U = u, X = x} = \int K(w, a, u, x)h(w, a, x)f(w \mid a, x)\rd \mu(w),  \\
  &[K_{Z \mid a, x}q](a, u, x) = \Eb{q(Z, a, x) \mid  A = a, U = u, X = x} = \int K'(z, a, u, x)h(z, a, x)f(z \mid a, x)\rd \mu(w),
\end{align*}
where $K(w, a, u, x)$ and $K'(z, a, u, x) $ are the corresponding kernel functions defined as follows:
\[
K(w, a, u, x) = \frac{f(w, u \mid a, x)}{f(u\mid a, x)f(w \mid a, x)}, ~~ K'(z, a, u, x) = \frac{f(z, u \mid a, x)}{f(u\mid a, x)f(z \mid a, x)}.
\]
Their adjoint operators $K^*_{W \mid a, x}: L_2(U \mid A = a, X = x) \to L_2(W \mid A = a, X = x)$ and $K^*_{Z \mid a, x}: L_2(U \mid A = a, X = x) \to L_2(Z \mid A = a, X = x)$ are given as follows:
\begin{align*}
  &[K^*_{W \mid a, x}g](w, a, x) = \int K(w, a, u, x)g(u, a, x)f(u \mid a, x)\rd \mu(w) = \Eb{g(U, a, x) \mid  W = w, A = a, X = x},  \\
  &[K_{Z \mid a, x}g](z, a, x)  = \int K'(z, a, u, x)g( a,u, x)f(u \mid a, x)\rd \mu(w) = \Eb{g(U, a, x) \mid  Z = z, A = a,  X = x}.
\end{align*}

The existence of bridge functions is equivalent to existence of solutions to the following equations of the first kind:
\begin{align*}
    [K_{W \mid a, x}h](a, u, x) = k_0(a, u,  x), ~~ [K_{Z\mid a, x}q]( a, u, x) = 1/f(a \mid u, x), ~~ \text{a.e. } u, a, x \text{ w.r.t } \mathbb{P}.
\end{align*}

Moreover, the completeness conditions in \Cref{assump: full-completeness} is equivalent to the injectivity of the adjoint operators $K^*_{Z \mid a, x}$ and $K^*_{W \mid a, x}$.
\begin{proposition}\label{prop: adjoint-inject}
    When the completeness condition in \Cref{assump: full-completeness} condition 1 holds, the adjoint operator $K^*_{W \mid a, x}$ is injective and its null space is $\mathcal{N}(K^*_{W \mid a, x}) = \{g: K^*_{W \mid a, x}g = 0\} = \{0\}$ for almost all $a, x$. When the completeness condition in \Cref{assump: full-completeness} condition 2 holds, the adjoint operator $K^*_{Z \mid a, x}$ is injective and its null space is $\mathcal{N}(K^*_{Z \mid a, x}) = \{g: K^*_{Z \mid a, x}g = 0\} = \{0\}$ for almost all $a, x$.
\end{proposition}

We further assume the following regularity conditions.

\begin{assumption}\label{assump: operator}
Assume that for almost every $a, x$:
\begin{enumerate}
    \item \label{compactness-1} $\iint f(w \mid u, a, x)f(u \mid w, a, x)\rd \mu(w)\rd \mu(u) < \infty$.
    \item \label{compactness-2} $\iint f(z \mid u, a, x)f(u \mid z, a, x)\rd \mu(z)\rd \mu(u) < \infty$.
\end{enumerate}
\end{assumption}
According to Example 2.3 in \citet[P5659]{CarrascoMarine2007C7LI}, \Cref{assump: operator}  ensures that both $K_{W \mid a, x}$ and $K_{Z \mid a, x}$ are compact operators. Then by Theorem 2.41 in \citet[P5660]{CarrascoMarine2007C7LI}, both $K_{W \mid a, x}$ and $K_{Z \mid a, x}$ admit singular value decomposition:  there exist
\[
\prns{\lambda^j_{W \mid a, x}, \varphi^j_{W \mid a, x}, \psi^j_{W \mid a, x}}_{j=1}^{+\infty}, ~~ \prns{\lambda^j_{Z \mid a, x}, \varphi^j_{Z \mid a, x}, \psi^j_{Z \mid a, x}}_{j=1}^{+\infty}
\]
with
orthonormal sequences  $$\braces{\varphi^j_{W \mid a, x} \in L_2(W \mid a, x)}, \braces{\varphi^j_{Z \mid a, x} \in L_2(Z \mid a, x)}, \braces{\psi^j_{W \mid a, x} \in L_2(U \mid a, x)}, \braces{\psi^j_{Z \mid a, x} \in L_2(U \mid a, x)}$$ such that
\begin{align*}
    K_{W \mid a, x} \varphi^j_{W\mid a, x} = \lambda^j_{W\mid a, x}\psi^j_{W\mid a, x}, ~~ K_{Z \mid a, x} \varphi^j_{Z\mid a, x} = \lambda^j_{Z\mid a, x}\psi^j_{Z\mid a, x}.
\end{align*}
\subsubsection{Existence of Bridge Functions}
Following \cite{miao2018identifying}, we use Picard's Theorem \citep[Theorem 15.18]{kress1989linear} to characterize the existence of solutions to equations of the first kind by the singular value decomposition of the associated operators.
\begin{lemma}[Picard's Theorem]\label{thm: Picard}
Let $K: H_{1} \to  H_{2}$ be a compact operator with singular system $\left(\lambda_{j}, \varphi_{j}, \psi_{j}\right)_{j=1}^{+\infty}$, and $\phi$ be a given function in $H_{2}$.
Then the equation of the first kind $K h=\phi$ has solutions if and only if
\begin{enumerate}
    \item $\phi \in \mathcal{N}\left(K^{*}\right)^{\perp}$, where $\mathcal{N}\left(K^{*}\right)=\left\{h: K^{*} h=0\right\}$ is the null space of the adjoint operator $K^{*}$.
    \item Picard's condition $\sum_{n=1}^{+\infty} \lambda_{n}^{-2}\left|\left\langle\phi, \psi_{n}\right\rangle\right|^{2}<+\infty$ is satisfied.
\end{enumerate}
\end{lemma}

In the following two lemmas,  we apply Picard's theorem to show that the completeness conditions can imply the existence of bridge functions under regularity conditions. According to \Cref{prop: adjoint-inject}, the completeness conditions ensure that the adjoint operators $K^*_{W \mid a, x}$ and $K^*_{Z \mid a, x}$ have null spaces $\mathcal{N}(K^*_{W \mid a, x}) = \emptyset, \mathcal{N}(K^*_{Z \mid a, x}) = \emptyset$. Then the first condition in Lemma \ref{thm: Picard} holds automatically. Hence we mainly need to additionally impose  Picard's condition in Lemma \ref{thm: Picard}. Under  Picard's condition and other regularity conditions, the completeness conditions imply the existence of bridge functions.

\begin{lemma}\label{lemma: completeness-existence-W}
Assume the operator $K^*_{W \mid a, x}$ satisfies \Cref{assump: operator} condition \ref{compactness-1} and  Picard's condition, i.e., $\sum_{j = 1}^{+\infty}\prns{\lambda_{W\mid a, x}^{j}}^{-2}
    \braces{\int k_0(a, u, x)\psi^j_{W\mid a, x}(u, a, x) \rd \mu(u)}^2
    < \infty$ for almost all $a, x$. If the completeness condition in \Cref{assump: full-completeness} condition \ref{assump: completeness-1} holds, then
 there exists function $h_0 \in L_2(W|A=a,X=x)$ such that for almost all $a,x$
\[
\Eb{Y - h_0(W, a, x) \mid A = a, U, X = x} = 0.
\]
\end{lemma}

\begin{lemma}\label{lemma: completeness-existence-Z}

Assume the operator $K^*_{Z \mid a, x}$ satisfies \Cref{assump: operator} condition \ref{compactness-2} and  Picard's condition, i.e., $\sum_{j = 1}^{+\infty}\prns{\lambda_{Z \mid a, x}^{j}}^{-2}
    \braces{\int \frac{\epol(a \mid x)}{f(a \mid u, x)}\psi^j_{Z\mid a, x}(u, a, x) \rd \mu(u)}^2
    < \infty$. If the completeness condition in \Cref{assump: full-completeness} condition \ref{assump: completeness-2} holds, then there exists function $q_0$ such that $\epol(a \mid x) q_0(Z,a,x)\in L_2(Z \mid A=a,X=x)$ for almost all $a,x$ and
\[
\Eb{\epol\prns{a \mid x}q_0(Z, a, x) \mid a, U, x} = \frac{\epol\prns{a\mid x}}{f(a \mid U, x)}.
\]

\end{lemma}

\section{Semiparametric Efficiency Bound} \label{sec: efficiency-bound}
{In this section, we derive the semiparametric efficiency bound for $J=\E\bracks{\int Y(a)\epol(a|X)\rd(a)}$  under the nonparametric model $\mathcal M_{np}$ that does not restrict the data distributions other than require\footnote{These correspond to the identification assumptions in \Cref{lemma: relaxed-ipw-dm} statement 1. Alternatively, we can also follow \cite{CuiYifan2020Spci} and require the model to satisfy conditions in \Cref{prop: cui2020} statement \ref{prop: cui2020-1}, namely $\Hbbb_0^{\obs} \ne \emptyset$ and the completeness conditions in \Cref{assump: full-completeness}  condition \ref{full-completeness-2}.} $\Hbbb_0 \ne \emptyset$ and $\Qbbb_0^{\obs} \ne \emptyset$.
Our efficiency analysis generalizes Theorem 3.1 in  \cite{CuiYifan2020Spci} for  average treatment effect with discrete treatments. }

\begin{theorem}\label{thm: efficiency}
    Let $\mathbb{P}$ be a data generating distribution such that $\Hbbb^{\obs}_0 = \braces{h_0}, \Qbbb^{\obs}_0 = \braces{q_0}$, and the corresponding conditional expectation operator $P_z$ defined in \Cref{eq: operator} is bijective.
  Under
  \Cref{asm:whole_assm}, the efficient influence function of $J$ under the model $\mathcal{M}_{n p}$
locally at the data generating distribution $\mathbb{P}$ is given as follows:
$$
\mathrm{EIF}(J)=\epol(A\mid X) q_0(Z, A, X)[Y-h_0(W, A, X)]+\mathcal Th_0(W, X) - J.
$$
The corresponding semiparametric efficiency bound of $J$  is  $V_{\text{eff}} = \mathbb{E}\left[\mathrm{EIF}^{2}(J)\right]$.
\end{theorem}

In \Cref{thm:dr}, we show that our GACE estimator proposed in \Cref{sec: both} can attain the efficiency bound in \Cref{thm: efficiency} when the bridge functions are unique.

\section{Additional Examples of Minimax Estimators}\label{sec: minimax-stab}
\subsection{RKHS Minimax Estimators without Stabilizers in the Discrete Action Setting}\label{subsec:rkhs}

{
Consider \Cref{ex: atomic} with a binary action $A\in\braces{0, 1}$. We aim to estimate the counterfactual mean parameter $\Eb{Y(a)}$ for $a \in \braces{0, 1}$, which is a special example of the GACE parameter corresponding to $\pi=\Ib(A=a)$.

We first consider estimating the outcome bridge function $h_0\prns{W, a, X}$ based on a bridge class $\Hbbb$ over $\Wcal \times \Acal \times \Xcal$. For the critic class $\Qbbb': \Zcal \times \Acal \times \Xcal \mapsto \R{}$, we can first specify a
kernel function $\bar k_z:(\Zcal \times \Xcal)\times (\Zcal \times \Xcal)\to \RR$ and then use the RKHS induced by the following product kernel:
\begin{align*}
k_z((Z_i,A_i,X_i),(Z_j,A_j,X_j))=\indic{A_i=A_j}\bar k_z\prns{(Z_i,X_i),(Z_j,X_j)}.
\end{align*}
Then we can apply \Cref{eq: est-rkhs-1-h} in \Cref{lemma: RKHS-1} to construct the estimator $\hat h$ as follows:
\begin{align*}
\argmin_{h\in \Hbbb}\prns{\psi_n\prns{h}}^{\top} \bar K_{z,n}\psi_n\prns{h},
\end{align*}
where
 \begin{align*}
  \{\bar K_{z,n}\}_{i,j}=\Ib(A_i=a)\Ib(A_j=a)\bar k_z((Z_i,X_i),(Z_j,X_j)),\{\psi_n(h)\}_i=\Ib(A_i=a)(Y_i-h(W_i,a,X_i)).
 \end{align*}

Similarly, we can also consider estimating the action bridge function $q_0\prns{Z, a, X}$ based on a bridge class $\Qbbb$ over $\Zcal \times \Acal \times  \Xcal$. For the critic class $\Hbbb' : \Wcal \times \Acal \times  \Xcal \mapsto \R{}$, we can first specify a kernel function $\bar k_w:(\Wcal, \Xcal)\times (\Wcal, \Xcal)\to \RR$ and then use the RKHS induced by the following product kernel:
\begin{align*}
k_w((W_i,A_i,X_i),(W_j,A_j,X_j))=\indic{A_i=A_j}\bar k_w\prns{(W_i,X_i),(W_j,X_j)}.
\end{align*}
Again, we can apply \Cref{eq: est-rkhs-1-q} in \Cref{lemma: RKHS-1} to construct the estimator $\hat q$ with
\begin{align*}
     \{K_{w1,n}\}_{i,j}=\{K_{w2,n}\}_{i,j}=\Ib(A_i=a)\rI(A_j=a)\bar k_w((W_i,X_i),(W_j,X_j)),\quad \{\phi_n(q)\}_i=\Ib(A_i=a)q(X_i,a,Z_i).
 \end{align*}
The resulting estimator can be equivalently written as
 \begin{align*}
     \argmin_{q\in \Qbbb}\prns{\{\phi_n\prns{q}-\mathbf{1}_n \}}^{\top}K_{w1,n}\{\phi_n\prns{q}-\mathbf{1}_n \}.
 \end{align*}
 }

\subsection{Minimax Estimators with Stabilizers} \label{sec: stabilizers}

{  In \Cref{sec: example}, we discuss the explicit form of minimax estimators without stabilizers when specialized to some common function classes.
Here, we do the same for minimax estimators with stabilizers given in \Cref{eq: est2-1,eq: est2-2}.

\paragraph*{\textbf{Linear Classes}}  We can again consider the linear critic classes $\Qbbb', \Hbbb'$ in \Cref{eq: lin-crit-q,eq: lin-crit-q}. However, with the norm constraints in $\Qbbb', \Hbbb'$, the inner maximization in \Cref{eq: est2-1,eq: est2-2} may no longer admit closed-form solutions.
To circumvent this issue, we relax the hard norm constraints by setting $c_1=c_2=\infty$ in \Cref{eq: lin-crit-q,eq: lin-crit-q}, and regularizing coefficient norms in the inner maximization objectives.
}

{\neweedit
\begin{lemma}\label{lemma: linear-est-2}
Fix $c_1 = c_2 = +\infty$ in the critic classes $\Qbbb', \Hbbb'$ in \Cref{eq: lin-crit-q,eq: lin-crit-h}. Consider the following estimators adapted from \Cref{eq: est2-1,eq: est2-2}:
\begin{align*}
  \hat h &= \argmin_{h\in \Hbbb}\max_{q \in \Qbbb'}2\E_n[ \{h(W, A, X)-Y \}q\prns{Z, A, X}]-\lambda \E_n[q^2\prns{Z, A, X}] - \gamma_1\|\alpha_1\|^2, \\
  \hat q &=  \argmin_{q\in \Qbbb}\max_{h \in \Hbbb'}2\E_n[q(Z, A, X)\epol(A \mid X)h(W, A, X) - \Tcal h(W, X)]-\lambda\E_n[h^2\prns{W, A, X}] - \gamma_2\|\alpha_2\|^2
\end{align*}
Then we have
\begin{align}
   & \hat h=\argmin_{h\in \Hbbb}\E_n[\{Y-h \} \phi  ]^{\top}\{\gamma_1 I+\lambda \E_n[\phi\phi^{\top}]\}^{-1}\E_n[\{Y-h\}   \phi  ], \label{eq:est_sta1_form} \\
   &\hat q = \argmin_{q\in \Qbbb}\E_n[q\epol \psi -\Tcal \psi]^{\top}\{\gamma_2 I+\lambda \E_n[\psi\psi^{\top}]\}^{-1}\E_n[\{q\epol \psi -\Tcal \psi\}].\label{eq:est_sta2_form}
\end{align}
\end{lemma}

If further $\Qbbb$ and $\Hbbb$ are linear classes as in \Cref{eq: H-linear}, by adding Tikhonov regularization on the coefficients of $\Hbbb,\Qbbb$ as
\begin{align*}
   & \hat h=\argmin_{h\in \Hbbb}\E_n[\{Y-h \} \phi  ]^{\top}\{\gamma_1 I+\lambda \E_n[\phi\phi^{\top}]\}^{-1}\E_n[\{Y-h\}   \phi  ] +\gamma'_1 \|\alpha_2\|^2, \\
   &\hat q = \argmin_{q\in \Qbbb}\E_n[q\epol \psi -\Tcal \psi]^{\top}\{\gamma_2 I+\lambda \E_n[\psi\psi^{\top}]\}^{-1}\E_n[\{q\epol \psi -\Tcal \psi\} ] +\gamma'_2 \|\alpha_1\|^2,
\end{align*}
we can obtain the closed-form solution:
\begin{align*}
    \hat h &=\tilde \psi^{\top} \hat \alpha_2,\,\hat \alpha_2=\{\E_n[\tilde \psi \phi^{\top}]\{\gamma_1 I+\lambda \E_n[\phi\phi^{\top}]\}^{-1}\E_n[\phi\tilde \psi^{\top} ] + \gamma'_1 I \}^{-1}\E_n[ \tilde \psi \phi^{\top}]\{\gamma_1 I+\lambda \E_n[\phi\phi^{\top}]\}^{-1}\E_n[ Y \phi],\\
    \pi \hat q &=\tilde \phi^{\top}\hat \alpha_1,\,\hat \alpha_1=\{\E_n[\tilde \phi \psi^{\top} ]\{\gamma_2 I+\lambda \E_n[\psi\psi^{\top}]\}^{-1}\E_n[\psi \tilde \phi^{\top} ] + \gamma'_2 I\}^{-1}\E_n[ \tilde \phi \psi^{\top}]\{\gamma_2 I+\lambda \E_n[\psi\psi^{\top}]\}^{-1}\E_n[\Tcal \psi].
\end{align*}
Finally,
\begin{align*}
    \hat J_{\DM} &= \E_n[\Tcal \psi]^{\top}\{\E_n[\tilde \psi \phi^{\top}]\{\gamma_1 I+\lambda \E_n[\phi\phi^{\top}]\}^{-1}\E_n[\phi\tilde \psi^{\top} ] + \gamma'_1 I \}^{-1}\E_n[ \tilde \psi \phi^{\top}]\{\gamma_1 I+\lambda \E_n[\phi\phi^{\top}]\}^{-1}\E_n[ Y \phi], \\
    \hat J_{\ipw} &= \E_n[ Y \phi]^{\top}\{\E_n[\tilde \phi \psi^{\top} ]\{\gamma_2 I+\lambda \E_n[\psi\psi^{\top}]\}^{-1}\E_n[\psi \tilde \phi^{\top} ] + \gamma'_2 I\}^{-1}\E_n[ \tilde \phi \psi^{\top}]\{\gamma_2 I+\lambda \E_n[\psi\psi^{\top}]\}^{-1}\E_n[\Tcal \psi].
\end{align*}
}

{ \paragraph*{\textbf{\text{RKHS}}} We can also consider the RKHS critic classes $\Qbbb', \Hbbb'$ in \Cref{eq: crit-rkhs}. But just like the linear classes above, the inner maximization in \Cref{eq: est2-1,eq: est2-2} may no longer have closed-form solutions with the RKHS norm constraints in \Cref{eq: crit-rkhs} either. We can again relax these norm constraints by setting $c_1=c_2=\infty$ in \Cref{eq: crit-rkhs}, and instead regularize the inner maximization objectives.
}

\begin{lemma}\label{lem:kernel3}
 Consider the following estimators adapted from \Cref{eq: est2-1,eq: est2-2}:
\begin{align}
  \hat h &= \argmin_{h\in \Hbbb}\max_{q\in \mathcal{L}_w,}2\E_n[ \{h(W, A, X)-Y \}q\prns{Z, A, X}]-\lambda \E_n[q^2\prns{Z, A, X}]-\gamma_1 \|q\|_{\mathcal{L}_w},
  \label{eq: est-rkhs-2-h-minimax}\\
  \hat q &=  \argmin_{q\in \Qbbb}\max_{h\in \mathcal{L}_z}2\E_n[q(Z, A, X)\epol(A \mid X)h(W, A, X) - \Tcal h(W, X)]-\lambda\E_n[h^2\prns{W, A, X}]-\gamma_2 \|h\|_{\mathcal{L}_w}.\label{eq: est-rkhs-2-q-minimax}
\end{align}
Then we have
\begin{align}
    \hat h  &= \argmin_{h\in \Hbbb} {\psi_n\prns{h}}^{\top}K^{1/2}_{z,n}\{\gamma_1 I+\lambda K_{z,n} \}^{-1}K^{1/2}_{z,n}{\psi_n\prns{h}},
    \label{eq: est-rkhs-2-h}  \\
    \hat q  &=\argmin_{q\in \Qbbb} {\phi_n\prns{q}}^{\top} K^{1/2}_{w1,n}(\lambda K_{w1,n}+\gamma_2 I)^{-1}K^{1/2}_{w1,n}{\phi_n\prns{q}}-2\{{\phi_n\prns{q}}^{\top}(\lambda K_{w1,n}+\gamma_2 I)^{-1} K_{w2,n}\mathbf{1}_n \}. \label{eq: est-rkhs-2-q}
\end{align}
\end{lemma}
We remark that the computational time for RKHS estimators with stabilizers may be larger than the counterparts without stabilizers. Indeed, for RKHS estimators with stabilizers, evaluating the objective functions in \Cref{eq: est-rkhs-2-h,eq: est-rkhs-2-q} requires matrix inverse that generally takes $O(n^3)$ time but evaluating the counterparts in \Cref{eq: est-rkhs-1-h,eq: est-rkhs-1-q} only takes $O(n^2)$ time.

{\neweedit
If further $\Qbbb,\Hbbb$ are RKHS, by solving
\begin{align*}
    \hat h  &= \argmin_{h\in  \mathcal{L}_z} {\psi_n\prns{h}}^{\top}K^{1/2}_{z,n}\{\gamma_1 I+\lambda K_{z,n} \}^{-1}K^{1/2}_{z,n}{\psi_n\prns{h}} + \gamma'_1 \|h\|_{\mathcal{L}_z},   \\
    \hat q  &=\argmin_{q\in \mathcal{L}_w} {\phi_n\prns{q}}^{\top} K^{1/2}_{w1,n}(\lambda K_{w1,n}+\gamma_2 I)^{-1}K^{1/2}_{w1,n}{\phi_n\prns{q}}-2\{{\phi_n\prns{q}}^{\top}(\lambda K_{w1,n}+\gamma_2 I)^{-1} K_{w2,n}\mathbf{1}_n \}+ \gamma'_2 \|h\|_{\mathcal{L}_w},
\end{align*}
we obtain
\begin{align*}
  \hat h(\cdot)  &=\sum_i \hat \alpha_i k((W_i,A_i,X_i),\cdot),\quad \pi \hat q(\cdot) = \sum_i \hat \beta_i k((Z_i,A_i,X_i),\cdot)
 \end{align*}
 where
\begin{align*}
  \hat \alpha  &= \{ K_{w1,n} \{\gamma_1 I+\lambda K_{z,n} \}^{-1} K_{w1,n}  +\gamma'_1 I \}^{-1} K_{w1,n}  K^{1/2}_{z,n}\{\gamma_1 I+\lambda K_{z,n} \}^{-1}K^{1/2}_{z,n}\bar Y,\\
  \hat \beta  &= \{ K_{z,n} (\lambda K_{w1,n}+\gamma_2 I)^{-1}  K_{z,n} + \gamma'_2 I   \}^{-1}  K_{z,n} K^{1/2}_{w1,n}(\lambda K_{w1,n}+\gamma_2 I)^{-1} K^{1/2}_{w1,n}K_{w2,n} \mathbf{1}_n .
\end{align*}
where $\bar Y=(Y_1,\cdots,Y_n)^{\top}$. Finally,
\begin{align*}
     \hat J_{\DM} &= \{K_{w2,n} \mathbf{1}_n \}^{\top} \{ K_{w1,n} \{\gamma_1 I+\lambda K_{z,n} \}^{-1} K_{w1,n}  +\gamma'_1 I \}^{-1} K_{w1,n}  K^{1/2}_{z,n}\{\gamma_1 I+\lambda K_{z,n} \}^{-1}K^{1/2}_{z,n}\bar Y ,  \\
     \hat J_{\ipw} &= \bar Y^{\top} \{ K_{z,n} (\lambda K_{w1,n}+\gamma_2 I)^{-1}  K_{z,n} + \gamma'_2 I   \}^{-1}  K_{z,n} K^{1/2}_{w1,n}(\lambda K_{w1,n}+\gamma_2 I)^{-1} K^{1/2}_{w1,n}K_{w2,n} \mathbf{1}_n .
\end{align*}
}

{\section{Error Bounds for Linear Models under Single Realizability}\label{sec: single-realizability}

In \Cref{cor: ipw_reg_no_stab}, we establish consistency of estimators when realizability $\Qbbb\cap\Qbbb^{\obs}_0\neq \emptyset, \Hbbb' \cap \Hbbb^{\obs}_0\neq \emptyset$ or $\Hbbb \cap \Hbbb^{\obs}_0  \neq \emptyset ,\Qbbb' \cap  \epol\Qbbb^{\obs}_0 \neq \emptyset$ holds. These conditions essentially require realizability for both bridge classes and critic classes. In this section, we show that when using linear models in \Cref{sec: linear}, under mild matrix invertibility conditions, we need realizability condition on only a single but not both function classes.

\begin{corollary}[IPW and REG under single realizability]\label{cor: ipw_reg_no_stab2}
Let $\Qbbb, \Qbbb', \Hbbb, \Hbbb'$ be the linear classes given in \Cref{sec: linear}.
\begin{enumerate}
\item If $\Qbbb\cap\Qbbb^{\obs}_0\neq \emptyset$ and $\E[\tilde \psi(Z,A,X) \psi^{\top}(W,A,X)]$ has full column rank,
then  \begin{align*}
   |\hat J_{\ipw}-J| &\textstyle\leq \sup_{q\in \Qbbb}\abs{(\E_n-\E)\bracks{q\epol Y}} +2\sup_{q\in \Qbbb,h\in \Hbbb'}\abs{(\E_n-\E) \bracks{-q\epol h+\Tcal h }}.
    \end{align*}
This statement also holds if instead $\Hbbb \cap \Hbbb^{\obs}_0\neq \emptyset $ and $\E[\tilde \psi(Z,A,X) \psi^{\top}(W,A,X)]$ is full column rank.
\item If $\Hbbb \cap \Hbbb^{\obs}_0  \neq \emptyset$ and  $\E[ \psi(Z,A,X) \tilde \psi^{\top}(W,A,X)]$ has full row rank, then

    \begin{align*}
    |\hat J_{\mathrm{REG}}-J|&\textstyle\leq  \sup_{h\in \Hbbb} \abs{(\E_n - \E)\bracks{\Tcal h }}+2\sup_{q\in \Qbbb',h\in \Hbbb}\abs{(\E_n -\E)\bracks{q(Y-h)}}.
\end{align*}
This statement also holds if instead $\Qbbb' \cap  \epol\Qbbb^{\obs}_0 \neq \emptyset$ and $\E[\psi(Z,A,X) \tilde \psi^{\top}(W,A,X)]$ has full column rank.
\end{enumerate}
\end{corollary}

\Cref{cor: ipw_reg_no_stab2} implies that the IPW and REG estimators are consistent as long as a single realizability condition holds and certain matrices are full column or row rank. Compared to \Cref{cor: ipw_reg_no_stab}, we do not need realizability for both the bridge function class and critic function class.

\Cref{cor: ipw_reg_no_stab2} is immediately obtained from \Cref{thm:ipw_reg_no_stab} by showing bias terms are $0$. For instance, suppose $\Qbbb\cap\Qbbb^{\obs}_0\neq \emptyset$ and $\E[\tilde \psi(Z,A,X) \psi^{\top}(W,A,X)]$ is full row rank. Then,  the second term \eqref{eq:ipw_no_stab} is $0$ since $\Qbbb\cap\Qbbb^{\obs}_0\neq \emptyset$. The first term is also $0$ since we can take a linear projection of $h_0$ in $\Hbbb'$ even if we cannot choose $h=h_0$.
}

\section{Linear Sieves and Neural Networks}\label{sec: sieve-nn}
{
In \Cref{exa: nonpara}, we consider nonparametric classes such as \Holder\  balls and Sobolev balls. However, it is difficult to optimize over these function classes since they are infinite-dimensional.
Instead, we may use a sequence of more tractable function classes with increasing complexity to approximate those nonparametric classes.
}

\subsection{Linear Sieves}\label{sec: sieve}
{
We first consider using linear sieve classes   for $\Hbbb$ and $\Qbbb$ that consist of linear combinations of some basis functions such as splines, polynomials, wavelets, and so on \citep{ChenXiaohong2007C7LS}.

To give an example, we consider a simple setting where  $ \mathcal W \times \mathcal Z \times \mathcal A \times \mathcal X = [0, 1]^{d_W + d_Z+ d_A+ d_X}$ with $d_W = d_Z$. We consider generic  basis functions $\braces{\psi_j(w,a,x)}_{j=1}^{k_n}$ and $\braces{\phi_j(z,a,x)}_{j=1}^{k_n}$ (which for simplicity have the same size $k_n$) where $\psi_j: [0, 1]^d \mapsto \R{}, \phi_j: [0, 1]^d \mapsto \R{}$ for $d = d_W +  d_A + d_X =  d_Z +  d_A + d_X$. These basis functions induce the following linear  sieve classes
\begin{equation}\label{eq: sieve}\begin{aligned}
    \mathcal S_{1, n} &= \braces{\prns{w, a, x} \mapsto \sum_{j = 1}^{k_n} \omega_j \psi_j(w,a,x):\omega\in\R{k_n}},\quad \mathcal S_{2, n} &= \braces{\prns{z, a, x} \mapsto \sum_{j = 1}^{k_n} \omega_j \phi_j(z,a,x):\omega\in\R{k_n}}.
\end{aligned}\end{equation}
We assume these two classes can approximate \Holder\ balls over $[0,1]^d$ whose \Holder\  smoothness parameters and radius are assumed to be the same for simplicity. We denote the common smoothness level as $\alpha$ and denote the corresponding
\Holder\ balls as $\Lambda^{\alpha}([0,1]^d)$.
Following \cite{ChenXiaohong2007C7LS}, we assume standard approximation errors as follows:
\begin{align}
    &\forall h \in \Lambda^{\alpha}([0,1]^d), \exists h_n \in S_{1, n}, \text{ s.t. } \|h - h_n\| = O\prns{k_n^{-\alpha/{d}}},\label{eq: sieve-error1} \\
    &\forall q \in \Lambda^{\alpha}([0,1]^d), \exists q_n \in S_{2, n}, \text{ s.t. } \|q - q_n\| = O\prns{k_n^{-\alpha/{d}}}.\label{eq: sieve-error2}
\end{align}

\paragraph*{\textbf{Estimators without Stabilizers}} In the following \Cref{cor:linear_sieves1}, we first bound the errors of IPW and REG estimators based on sieve minimax bridge function estimators \emph{without} stabilizers.
\begin{corollary}\label{cor:linear_sieves1}
Assume sieve classes $\mathcal S_{1, n}, \mathcal S_{2, n}$ satisfy the approximation error conditions in \Cref{eq: sieve-error1,eq: sieve-error2} with $k_n = O(n^{\frac{d}{d + 2\alpha}})$ and $\|\mathcal S_{1, n}\|_{\infty} < \infty,\| \mathcal S_{2, n}\|_{\infty} < \infty$. Also suppose
$ \|\epol(A \mid X)\|_{\infty} < \infty, \|Y\|_{\infty} < \infty, \|\epol(A \mid X)/f(A \mid X, W)\|_{2}<\infty$, and $\Hbbb^{\obs}_0\cap \Lambda^{\alpha}([0,1]^d)\neq \emptyset $, $\epol\Qbbb^{\obs}_0\cap \Lambda^{\alpha}([0,1]^d)\neq \emptyset$. Let $ \hat J_{\DM}$ be the REG estimator based on the  $\hat h$ given in \Cref{eq: est1-1} with $\Hbbb = \mathcal S_{1, n}$ and $\Qbbb' = \mathcal S_{2, n}$, and $ \hat J_{\ipw}$ be the IPW estimator based on $\hat q$ given in \Cref{eq: est1-2} with $\Qbbb = \mathcal S_{2, n}$ and $\Hbbb' = \mathcal S_{1, n}$. Then with probability $1-\delta$, we have
\begin{align*}
                  \max\braces{| \hat J_{\DM}-J|, ~| \hat J_{\ipw}-J|}        = O\prns{n^{-\alpha/(2\alpha+d)}+\sqrt{\log(1/\delta)/n}}.
\end{align*}
\end{corollary}
Compared to \Cref{cor: vc_class,cor: nonparametric}, \Cref{cor:linear_sieves1} assumes realizability on $\Lambda^{\alpha}([0,1]^d)$ rather than on $\Hbbb, \Qbbb, \Hbbb', \Qbbb'$. Therefore, we cannot prove \Cref{cor:linear_sieves1} by  directly applying \Cref{cor: ipw_reg_no_stab}, as the latter requires realizability assumptions for $\Hbbb, \Qbbb, \Hbbb', \Qbbb'$. Instead, we need to use the more general \Cref{thm:ipw_reg_no_stab}, which do not require such realizability assumptions and characterize the bias due to violations of these assumptions.
These bias terms can be upper-bounded by sieve approximation errors in \Cref{eq: sieve-error1,eq: sieve-error2}, and the variance terms can be bounded in terms of VC dimensions that depend on $k_n$. Choosing the sieve size $k_n = O(n^{\frac{d}{d + 2\alpha}})$ to balance the bias upper bounds and the variance upper bounds in \Cref{thm:ipw_reg_no_stab}, we can obtain the convergence rate in \Cref{cor:linear_sieves1}.

\paragraph*{\textbf{Estimators with Stabilizers}} Below we also bound the projected MSEs of sieve minimax bridge function estimators with stabilizers. These bounds can easily translate into error bounds of the corresponding GACE estimators according to \Cref{thm:ipw-dm,thm:drr}. Alternatively, they can also translate into more refined error bounds for the DR estimator by leveraging the ill-posedness measures in \Cref{def:ill-posed} and following \Cref{thm:dr2,thm:dr}.
\begin{corollary}\label{cor:linear_sieves2}
Assume sieve classes $\mathcal S_{1, n}, \mathcal S_{2, n}$ satisfy the approximation error conditions in \Cref{eq: sieve-error1,eq: sieve-error2} with $k_n = O(n^{\frac{d}{d + 2\alpha}})$ and $\|\mathcal S_{1, n}\|_{\infty} < \infty,\| \mathcal S_{2, n}\|_{\infty} < \infty$. Also suppose
$ \|\epol(A \mid X)\|_{\infty} < \infty, \|Y\|_{\infty} < \infty, \|\epol(A \mid X)/f(A \mid X, W)\|_{2}<\infty$.
\begin{enumerate}
 \item Let $\hat h$ be the outcome bridge function estimator given in \Cref{eq: est2-1} with $\Hbbb = \mathcal S_{1, n}, \Qbbb' = \mathcal S_{2, n}$ and $\lambda = 1$.
 If $\Hbbb^{\obs}_0\cap \Lambda^{\alpha}([0,1]^d)\neq \emptyset $ and there exists $h_0\in \Hbbb^{\obs}_0\cap \Lambda^{\alpha}([0,1]^d)$ such that $P_z(\Lambda^{\alpha}([0,1]^d)-h_0 )\subset  \Lambda^{\alpha}([0,1]^d)$, then with probability at least $1 - \delta$,
 \begin{align*}
    \|P_z(\hat h-h_0)\|_2=\tilde O(n^{-\alpha/(2\alpha+d)}+\sqrt{\log(1/\delta)/n}).
\end{align*}
\item Let $\hat q$ be the action bridge function estimator given in \Cref{eq: est2-2} with $\Qbbb = \mathcal S_{2, n}, \Hbbb' = \mathcal S_{1, n}$ and $\lambda = 1$.
If $\Qbbb^{\obs}_0\cap \Lambda^{\alpha}([0,1]^d)\neq \emptyset $ and there exists $q_0\in \Qbbb^{\obs}_0\cap \Lambda^{\alpha}([0,1]^d)$ such that $\epol P_w(\Lambda^{\alpha}([0,1]^d)-q_0 )\subset  \Lambda^{\alpha}([0,1]^d)$, then with probability at least $1 - \delta$,
 \begin{align*}
    \|P_w(\hat q-q_0)\|_2=\tilde O(n^{-\alpha/(2\alpha+d)}+\sqrt{\log(1/\delta)/n}).
\end{align*}
 \end{enumerate}
\end{corollary}
Similarly, \Cref{cor:linear_sieves2}
assumes realizability on $\Lambda^{\alpha}(([0,1]^d)$ rather than on $\Hbbb, \Qbbb, \Hbbb', \Qbbb'$.
So we cannot prove \Cref{cor:linear_sieves2} by directly applying \Cref{thm:w_function_easy,thm:q_function_easy}, as the latter requires realizability assumptions for $\Hbbb, \Qbbb, \Hbbb', \Qbbb'$. Instead, we need to use the more general \Cref{thm:w_function,thm:q_function} which do not require such realizability assumptions and characterize the bias due to violations of these assumptions.
By appropriately balancing the bias and variance terms, we can obtain the projected MSE bounds above.
}

\subsection{Neural Networks}\label{sec: NN}
{
Alternatively, we can consider neural networks in the minimax estimation of bridge functions.
For example, we again consider  $ \mathcal W \times \mathcal Z \times \mathcal A \times \mathcal X = [0, 1]^{d_W + d_Z+ d_A+ d_X}$ with $d_W = d_Z$ and denote
$d = d_W +  d_A + d_X =  d_Z +  d_A + d_X$
.
We consider $\mathcal{F}_{1, n}$ and $\mathcal{F}_{2, n}$ as classes of neural networks over $\mathcal{W} \times \mathcal A \times \mathcal X$ and $\mathcal Z \times \mathcal A \times \mathcal X$ respectively, both with
$L$ layers, $\Omega$ parameters and $1-$Lipschitz continuous activation function (\eg, ReLU activation functions).

It is known from the universal approximation property of neural networks \citep{yarotsky2017error} that these neural network classes can approximate Sobolev balls over $[0, 1]^d$. For simplicity, we assume that the two limiting Sobolev balls have the same  smoothness parameter $\alpha$, which we denote as $\Scal^{\alpha}([0,1]^d)$.
\begin{lemma}[Theorem 1 in \cite{yarotsky2017error}] \label{lem:yaro}
    Suppose $\mathcal{F}_{1, n}, \mathcal{F}_{2, n}$ have at most $L=O(\log \Omega)$ layers and $\Omega$ parameters. Then for any $h, q \in \Scal^{\alpha}([0,1]^d)$, there exists $h_n \in \mathcal{F}_{1, n}$ and $q_n \in \mathcal{F}_{2, n}$ such that
    \begin{align*}
        \|h - h_n\|_\infty =\tilde{O}( \Omega^{-\alpha/ d}), ~~ \|q - q_n\|_\infty =\tilde{O}( \Omega^{-\alpha/ d}).
    \end{align*}
\end{lemma}

\paragraph*{\textbf{Estimators without Stabilizers}}
We first regard neural networks as sieve estimators which approximate Sobolev balls. Thus similar to the analysis for \Cref{cor:linear_sieves2}, we can also bound the errors of IPW and REG estimators based on neural network minimax bridge function estimators \emph{without} stabilizers.

\begin{corollary}\label{cor:neural_realiza}
Assume neural network classes $\mathcal{F}_{1, n}, \mathcal{F}_{2, n}$ both have layers $L=\Theta(\log(n))$ and  $\Omega=\Theta(n^{d/(2\alpha+d)})$ weight parameters bounded in $[-B, B]$ for a positive constant $B$ and $\|\mathcal{F}_{1, n}\|_{\infty},\|\mathcal{F}_{2, n}\|_{\infty} < \infty$. Also suppose
$ \|\epol(A \mid X)\|_{\infty} < \infty, \|Y\|_{\infty} < \infty, \|\epol(A \mid X)/f(A \mid X, W)\|_{2}<\infty$,
and $\Hbbb^{\obs}_0\cap \mathcal{S}^{\alpha}([0,1]^d)\neq \emptyset $, $\epol\Qbbb^{\obs}_0\cap \mathcal{S}^{\alpha}([0,1]^d)\neq \emptyset$. Let $ \hat J_{\DM}$ be the REG estimator based on the  $\hat h$ given in \Cref{eq: est1-1} with $\Hbbb = \mathcal F_{1, n}$ and $\Qbbb' = \mathcal F_{2, n}$, and $ \hat J_{\ipw}$ be the IPW estimator based on $\hat q$ given in \Cref{eq: est1-2} with $\Qbbb = \mathcal F_{2, n}$ and $\Hbbb' = \mathcal F_{1, n}$. Then with probability $1-\delta$, we have
\begin{align}\label{eq: nn-err-bound-1}
                  \max\braces{| \hat J_{\DM}-J|, ~| \hat J_{\ipw}-J|}        = O\prns{n^{-\alpha/(2\alpha+d)}+\sqrt{\log(1/\delta)/n}}.
\end{align}
\end{corollary}
Note that in \Cref{cor:neural_realiza}, we assume realizability on the Sobolev balls $\mathcal{S}^{\alpha}([0,1]^d)$ rather than  on $\Hbbb, \Qbbb, \Hbbb', \Qbbb'$.
To prove \Cref{cor:neural_realiza}, we bound the bias terms in \Cref{thm:ipw_reg_no_stab} by the neural network approximation errors given in \citet{yarotsky2017error}.
The variance terms therein can be upper bounded by calculating Rademacher complexities of some function classes based on neural networks (see \Cref{lem:covering_neural,lem:dudley}).
By choosing the neural network size appropriately to balance these terms, we can arrive at the final convergence rate in  \Cref{eq: nn-err-bound-1}.

Note that the analysis in \Cref{cor:neural_realiza} requires the number of parameters in neural networks to be smaller than the sample size $n$. In practice, neural networks are sometimes over-parameterized to achieve better performance so that the number of parameters can indeed exceed the sample size. In the over-parameterized case, we may leverage the results in \cite{golowich2018size} to compute Rademacher complexities of neural networks and obtain similar conclusions.

\paragraph*{\textbf{Estimators with Stabilizers}} Below we also bound the projected MSEs of neural network minimax bridge function estimators with stabilizers. These bounds again can be easily translated into error bounds of the corresponding GACE estimators according to \Cref{thm:ipw-dm,thm:drr,thm:dr2,thm:dr}.
\begin{corollary}\label{cor:linear_sieves3}
Assume neural network classes $\mathcal{F}_{1, n}, \mathcal{F}_{2, n}$ both have layers $L=\Theta(\log(n))$ and  $\Omega=\Theta(n^{d/(2\alpha+d)})$ weight parameters bounded in $[-B, B]$ for a positive constant $B$ and $\|\mathcal{F}_{1, n}\|_{\infty},\|\mathcal{F}_{2, n}\|_{\infty} < \infty$. Also suppose
$ \|\epol(A \mid X)\|_{\infty} < \infty, \|Y\|_{\infty} < \infty, \|\epol(A \mid X)/f(A \mid X, W)\|_{2}<\infty$.
\begin{enumerate}
 \item Let $\hat h$ be the outcome bridge function estimator given in \Cref{eq: est2-1} with $\Hbbb = \mathcal F_{1, n}$ and $\Qbbb' = \mathcal F_{2, n}$.
 If $\Hbbb^{\obs}_0\cap \mathcal{S}^{\alpha}([0,1]^d)\neq \emptyset $ and there exists $h_0\in \Hbbb^{\obs}_0\cap \mathcal{S}^{\alpha}([0,1]^d)$ such that $P_z(\mathcal{S}^{\alpha}([0,1]^d)-h_0 )\subset  \mathcal{S}^{\alpha}([0,1]^d)$, then with probability at least $1 - \delta$,
 \begin{align*}
    \|P_z(\hat h-h_0)\|_2=\tilde O(n^{-\alpha/(2\alpha+d)}+\sqrt{\log(1/\delta)/n}).
\end{align*}
\item Let $\hat q$ be the action bridge function estimator given in \Cref{eq: est2-2} with $\Qbbb = \mathcal F_{2, n}$ and $\Hbbb' = \mathcal F_{1, n}$.
If $\Qbbb^{\obs}_0\cap \mathcal{S}^{\alpha}([0,1]^d)\neq \emptyset $ and there exists $q_0\in \Qbbb^{\obs}_0\cap \mathcal{S}^{\alpha}([0,1]^d)$ such that $\epol P_w(\mathcal{S}^{\alpha}([0,1]^d)-q_0 )\subset  \mathcal{S}^{\alpha}([0,1]^d)$, then with probability at least $1 - \delta$,
 \begin{align*}
    \|P_w(\hat q-q_0)\|_2=\tilde O(n^{-\alpha/(2\alpha+d)}+\sqrt{\log(1/\delta)/n}).
\end{align*}
 \end{enumerate}
\end{corollary}
Again, we assume realizability on the Sobolev balls $\mathcal{S}^{\alpha}([0,1]^d)$ rather than  on $\Hbbb, \Qbbb, \Hbbb', \Qbbb'$. So we use the more general \Cref{thm:w_function,thm:q_function} and choose the size of neural networks appropriately to get the bounds above.
}

\section{Additional Experiment Details}\label{ape:parameters}

This appendix provides implementation details for the supplementary numerical
and real-data analyses. We first describe the data-generating process and
estimation settings for the simulation studies. We then report two complementary
robustness analyses for the RHC study: sensitivity to the allocation of candidate
proxy variables to $Z$ and $W$, and sensitivity to the omission of observed
baseline covariates. These analyses probe distinct sources of specification
uncertainty in the real-data application.

\subsection{Simulation Study}
{
Tailoring the data generating process in \cite{CuiYifan2020Spci} to multi-dimensional variables, we generate $U, X', Z', W' \in \R{d}$ with $d = 60$ and $A \in \braces{0, 1}$ as follows:
\begin{enumerate}
\item $X' \sim  \Ncal(0,0.5\mathbb{I}_d)$ where $\mathbb{I}_d$ is a $d\times d$ identity matrix.
\item $A \mid X' \sim \mathrm{Ber}\prns{p(X')}$ where
\begin{align*}
p\prns{X'} = \frac{1}{1 + \exp(0.5 - 0.05{\mathbf{e}_d}^\top X')},
\end{align*}
where $\mathbf{e}_d\in\R{d}$ is an all-one vector.
\item Draw $W',Z', U$ from
 \begin{align*}
     W',Z', U\mid A,X' \sim \Ncal \prns{ \begin{bmatrix}  \mu_0+\mu_a A+\mu_x X' \\  \alpha_0+\alpha_a A+\alpha_x X' \\  \kappa_0+\kappa_a A+\kappa_x X'\end{bmatrix}   , \begin{bmatrix} \sigma^2_w, \sigma^2_{wz},\sigma^2_{wu} \\ \sigma^2_{wz}, \sigma^2_{z},\sigma^2_{zu} \\ \sigma^2_{wu}, \sigma^2_{zu},\sigma^2_{u}  \end{bmatrix}}.
 \end{align*}
Here we set the parameters above as $\mu_0=\alpha_0=\kappa_0=0.2\mathbf{e}_d$, $\alpha_a=\kappa_a=\mu_x=\alpha_x=\kappa_x=\Ib_d$, $\sigma^2_z=\sigma^2_u=\sigma^2_w=0.1(\Ib_d+\mathbf{e}_d\mathbf{e}_d^\top),\sigma^2_{wu}=\sigma^2_{zu}=0.1\mathbf{e}_d\mathbf{e}_d^\top$.
Finally, we choose $\sigma^2_{wz}$ and $\mu_a$ to ensure that $W'\perp \prns{A,Z'} \mid U,X'$, which corresponds to condition \ref{whole_assm:nc-unconfoundedness} in \Cref{asm:whole_assm}. To achieve this, note that
\begin{align}\label{eq: specification}
   \E[W' \mid U,X',A,Z']=\mu_0 + \mu_a A+ \mu_x X'+\Sigma_{w(z,u)}\Sigma^{-1}_{z,u}\begin{bmatrix}  Z'-\alpha_0-\alpha_aA-\alpha_x X'\\
   U-\kappa_0-\kappa_a A-\kappa_x X'
   \end{bmatrix}
\end{align}
where
\begin{align*}
   \Sigma_{w (z,u)}=(\sigma^2_{wz},\sigma^2_{wu}),\quad \Sigma_{z,u}=\begin{bmatrix} \sigma^2_z, \sigma^2_{zu} \\ \sigma^2_{zu} , \sigma^2_u \end{bmatrix}.
\end{align*}
We simply select $\sigma^2_{wz}$ and $\mu_a$ so that \Cref{eq: specification} does not depend on $A$ and $Z'$.
\item Draw $Y$ from
\begin{align*}
Y \mid X', U, W' \sim \mathcal{N}\prns{A + \mathbf{e}_d^\top X' + \mathbf{e}_d^\top U + \mathbf{e}_d^\top W', 1}.
\end{align*}
\item Transform $\prns{W', Z', X'}$ into $\prns{W, Z, X}$ via $X = g\prns{GX'}$, $Z = g\prns{GZ'}$, $W = g\prns{GW'}$ where $G \in \R{d \times d}$ is an invertible matrix and $g\prns{\cdot}$ is a nonlinear function applied elementwise to $GX', GZ', GW'$ respectively.
\item Make only $\prns{W, Z, X, A, Y}$ observable in the final data.
\end{enumerate}
Following \citet{CuiYifan2020Spci}, we can show that there exist a linear function $\tilde h(W', a, X')$ and a generalized linear function $\tilde q(Z', a, X')$ satisfying
\begin{align*}
    \E[Y\mid U,a,X']=\E[\tilde h(W',a,X')\mid U, A= a, X'],\quad  1/\mathrm{Pr}(A=a\mid U,X')=\E[\tilde q(Z',a,X')\mid U,A = a,X'], \, a=0,1.
\end{align*}
These induce bridge functions (when $g()$ is invertible)
\begin{align*}
h_0\prns{W, A, X} = \tilde h\prns{G^{-1}g^{-1}\prns{W}, a, G^{-1}g^{-1}\prns{X}}, ~~ q_0\prns{Z, A, X} = \tilde q\prns{G^{-1}g^{-1}\prns{Z}, a, G^{-1}g^{-1}\prns{X}}.
\end{align*}

For the experiments in \Cref{sec: simulation}, we use three-layer neural networks  of the following form as the bridge class $\Hbbb$:
\begin{align}
    x_1   &\leftarrow \mathrm{ReLU}(\Omega_1x_0+B_1),\Omega_1 \in \mathbb{R}^{2d\times 2d},B_1 \in \mathbb{R}^{2d},\, x_2 \leftarrow \mathrm{ReLU}(\Omega_2x_1+B_2),\Omega_2 \in \mathbb{R}^{d\times d},B_2 \in \mathbb{R}^{d} \\
        x_3 & \leftarrow  (\Omega_3x_2+B_3),\Omega_3 \in \mathbb{R}^{1 \times d},B_3 \in \mathbb{R}. \nonumber
\end{align}
The bridge class for $\Qbbb$ is the same neural network class except that the last layer is replaced by $\mathrm{Softplus} (\Omega_3x_2+B_3)$.
For the critic class $\Qbbb'$, we use an RKHS with an RBF kernel $k(a,b)=\exp(-\|a-b\|^2/(2\iota))$ where the variance parameter $\iota$ is set to the median of pairwise distances in the data.
When we use stabilizers, we set the parameters as $\gamma_1 = 5$ and $\lambda=1$. Recall these hyperparameters appear in \Cref{eq:est_sta1_form} and \Cref{eq:est_sta2_form}.
For the critic class $\Hbbb'$, we use an RKHS with a linear kernel. Since the critic classes are both RKHS, computing the bridge function estimators only requires optimizing over $\Hbbb, \Qbbb$ as we show in \Cref{lemma: RKHS-1}.
We typically implement this using RMSprop \citep{Tieleman2012} with a learning rate $0.0002$, momentum $0.95$ and batch size $n/40$.

To evaluate the estimation performance, we run  $1000$ replications of each experiment. Thus for each estimator, we obtain $1000$ estimated values $\hat J_k$ for $k = 1, \dots, 1000$. We report the normalized mean squared errors (MSE) $\frac{1}{1000}\sum_{k=1}^{1000}M_k$ with $M_k=\frac{|\hat J_k-J|^2}{J^2}$, normalized squared bias ${\abs{\frac{1}{1000}\sum_{k=1}^{1000}\hat J_k - J}^2/J^2}$, and the variance of $\braces{\hat J_k: k = 1, \dots, 1000}$ normalized by the squared true estimand $J^2$.

\paragraph{Ablation Study (No Unmeasured Confounding)}

The goal of the ablation in Tables~\ref{abla_i}--\ref{abla_ii} is to remove unmeasured confounding. We keep a similar high-dimensional structure to that used above. The main change---and the reason this ablation has \emph{no} unmeasured confounding---is how we decouple $A$ and the latent noise $U$. In the main simulation, even after conditioning on $X'$, the latent $U$ still depends on $A$ (its conditional mean includes an $A$-shift), and $U$ also enters $Y$. Thus $A \not\!\perp U \mid X'$, so $A \not\!\perp (Y(1),Y(0)) \mid X$, and standard adjustment on $X$ alone is biased.

In the ablation, we choose the conditional law so that, given $X'$, the latent $U$ no longer carries any $A$-specific shift (i.e., $\kappa_a = 0$). Equivalently, in the ablation we enforce $A \perp U \mid X'$ while keeping the rest of the model (the dimension, the nonlinear transformations $g$, the correlated Gaussian structure of $(W',Z',U)$, and the outcome equation for $Y$) unchanged.

For \(g(t)=\sin(t)\), the global non-invertibility of the transformation induces a small residual dependence after conditioning on the transformed covariates \(X\). In the present design, this effect is negligible at the reported precision: both Plain REG and DR (no \(W,Z\)) have normalized squared bias reported as \(0.0000\) in Table~\ref{tab:d2-abla}, while the cubic transformation in Table~\ref{tab:d4-abla} is one-to-one and satisfies the stated ignorability exactly.

\subsection{Real Data Analysis}

Because the true treatment effect is unavailable in the RHC study, we assess
robustness along two distinct dimensions. The first analysis varies the roles
assigned to the candidate proxy variables, whereas the second progressively
removes baseline covariates to assess sensitivity to unmeasured confounding.

\subsubsection{Robustness to Proxy-Variable Allocation}
\label{ape:proxy-allocation}

\begin{table}[t]\footnotesize
\centering
     \caption{The performance of different estimators across $\binom{4}{2}=6$ permutations of the proxy variables $Z$ and $W$. The linear moment-based baseline estimator is the same as the one in Table~\ref{tab:real_result}.
     Permutation I: Z=[``pafi1'', ``paco21''],W=[``ph1'', ``hema1'']; Permutation II: Z = [``pafi1'', ``ph1''], W = [``paco21'', ``hema1''] ; Permutation III: Z = [``pafi1'', ``hema1''], W = [``paco21'', ``ph1'']; Permutation IV: Z = [``paco21'', ``ph1''], W = [``pafi1'', ``hema1'']; Permutation V: Z = [``paco21'', ``hema1''], W = [``pafi1'', ``ph1'']; Permutation VI: Z = [``ph1'', ``hema1''], W = [``pafi1'', ``paco21'']. \label{tab:real_result_permute} }

   \scalebox{1}{
    \begin{tabular}{ll lll}
\toprule
Permutation  & &DR&DR(sta) & Linear(moment-based)  \\
\midrule
\multirow{2}{*}{I}
& Treatment effects (SEs)  &-0.0528(0.01194)&-0.0535 (0.01064)
& -0.0828 (0.01977)
  \\
& 95\% CIs  &[-0.0762, -0.0294] &[-0.0743, -0.0326]
& [-0.1216, -0.0441]
 \\
\hline
\multirow{2}{*}{II}
& Treatment effects (SEs) &-0.0515(0.01175) & -0.0619 (0.01065)
& -0.0805 (0.02092)
 \\
& 95\% CIs  &[-0.0745, -0.0284] &[-0.0828, -0.0411]
& [-0.1215, -0.0395]
  \\
\hline
\multirow{2}{*}{III}
& Treatment effects (SEs) &-0.0499 (0.01156) &-0.0630 (0.01037)
& -0.0710 (0.02148)
 \\
& 95\% CIs  &[-0.0725, -0.0272]&[-0.0833, -0.0427]
& [-0.1132, -0.0289]
  \\
\hline
\multirow{2}{*}{IV}
& Treatment effects (SEs) &-0.0505 (0.01232) &-0.0473 (0.01080)
& -0.1048 (0.05429)
  \\
& 95\% CIs  &[-0.0747, -0.0264]&[-0.0685, -0.0261]
& [-0.2112, +0.0016]
  \\
\hline
\multirow{2}{*}{V}
& Treatment effects (SEs) &-0.0454 (0.01176)&-0.0474 (0.01150)
& +0.0435 (0.10965)
  \\
& 95\% CIs  &[-0.0684, -0.0223]&[-0.0699, -0.0248]
& [-0.1714, +0.2584]
  \\
\hline
\multirow{2}{*}{VI}
& Treatment effects (SEs) &-0.0429 (0.01106)&-0.0590 (0.01084)
& +0.0113 (0.07680)
  \\
& 95\% CIs  &[-0.0646, -0.0213] &[-0.0802, -0.0377]
& [-0.1393, +0.1618]
  \\
\bottomrule
\end{tabular} }

\scalebox{1}{ \begin{tabular}{lll ll}
\toprule
Permutation  & &Linear(Interaction) & Linear(Quadratic) & Linear(Both) \\
\midrule
\multirow{2}{*}{I}
& Treatment effects (SEs)
&-0.0372 (0.07466)
& -0.0859 (0.02427)
& +0.1053 (0.97969) \\
& 95\% CIs
&[-0.1835, +0.1091]
& [-0.1334, -0.0383]
& [-1.8149, +2.0255] \\
\hline
\multirow{2}{*}{II}

& Treatment effects (SEs)
& -0.0931 (0.11916)
& -0.0672 (0.02117)
& -0.0488 (0.04447) \\
& 95\% CIs
& [-0.3267, +0.1404]
& [-0.1087, -0.0258]
& [-0.1359, +0.0384] \\
\hline
\multirow{2}{*}{III}

& Treatment effects (SEs)
&-0.0776 (0.01704)
& +0.0325 (0.13767)
& +0.1296 (0.53953) \\
& 95\% CIs
& [-0.1110, -0.0442]
& [-0.2373, +0.3023]
& [-0.9279, +1.1870] \\
\hline
\multirow{2}{*}{IV}

& Treatment effects (SEs)
& -0.0716 (0.03930)
& -0.0885 (0.03415)
& -0.0897 (0.05779) \\
& 95\% CIs
&[-0.1486, +0.0054]
& [-0.1554, -0.0216]
& [-0.2030, +0.0235] \\
\hline
\multirow{2}{*}{V}

& Treatment effects (SEs)
& +0.1496 (0.72227)
& +0.1067 (0.09530)
& +0.0828 (0.20450) \\
& 95\% CIs
& [-1.2660, +1.5653]
& [-0.0801, +0.2935]
& [-0.3180, +0.4836] \\
\hline
\multirow{2}{*}{VI}

& Treatment effects (SEs)
& -0.1554 (0.24149)
& +0.0964 (0.09538)
& +0.6260 (3.24644) \\
& 95\% CIs
& [-0.6287, +0.3179]
& [-0.0905, +0.2833]
& [-5.7370, +6.9890] \\
\hline
\hline
\end{tabular}}

\end{table}

\begin{table}[]
\caption{The performance of different estimators across $\binom{4}{2}=6$ permutations of the proxy variables $Z$ and $W$ as described in Table~\ref{tab:real_result_permute}. The linear closed-form baseline estimator is the same as the one in Table~\ref{tab:real_result}.}\label{tab:2sls}
    \centering
     \scalebox{0.9}{
        \begin{tabular}{llll l}
    \toprule
Permutation & &   DR & DR(sta) &Linear(closed-form)
      \\
         \midrule
         \multirow{2}{*}{I} &  Treatment effects (SEs) & -0.0528 (0.01194) &  -0.0535 (0.01064) &  -0.0860 (0.02408) \\
 &95\% CIs &[-0.0762, -0.0294] & [-0.0743, -0.0326]& [-0.1332, -0.0388] \\
 \hline
    \multirow{2}{*}{II} &  Treatment effects (SEs) & -0.0515 (0.01175) &  -0.0619 (0.01065) &  -0.0808 (0.02439) \\
 &95\% CIs &[-0.0745, -0.0284] & [-0.0828, -0.0411]& [-0.1286, -0.0330]
      \\
      \hline
        \multirow{2}{*}{III} &  Treatment effects (SEs) &  -0.0499 (0.01156) &  -0.0630 (0.01037) &  -0.1588 (1.95893) \\
 &95\% CIs &[-0.0725, -0.0272] &  [-0.0833, -0.0427]&  [-3.9983, +3.6807]\\
      \hline
        \multirow{2}{*}{IV}  &  Treatment effects (SEs) &  -0.0505 (0.01232) &  -0.0473 (0.01080) &  {+1.0435 (25.86638)} \\
 &95\% CIs & [-0.0747, -0.0264] &   [-0.0685, -0.0261]&   {[-49.6546, +51.7416]}\\
  \hline
          \multirow{2}{*}{V} &Treatment effects (SEs) &  -0.0454 (0.01176) &  -0.0474 (0.01150) & {+0.1193 (0.36014)} \\
 &95\% CIs &  [-0.0684, -0.0223] &    [-0.0699, -0.0248]&    {[-0.5866, +0.8252]}\\
  \hline
          \multirow{2}{*}{VI} &Treatment effects (SEs) &  -0.0429 (0.01106) &   -0.0590 (0.01084) & {+0.0299 (0.08592)} \\
 &95\% CIs &   [-0.0646, -0.0213] &     [-0.0802, -0.0377]&   {[-0.1385, +0.1983]}  \\
     \bottomrule
    \end{tabular} }
\\
      \scalebox{0.9}{
\begin{tabular}{lllll}
\toprule
Permutation &  & Linear(Interaction) & Linear(Quadratic) & Linear(Both) \\
\midrule
\multirow{2}{*}{I} &  Treatment effects(SEs) & +0.0017(5.32580) &  +0.0301 (0.28169) &  -0.2281(2.21813) \\
 &95\% CIs &[-10.4369, +10.4402] & [-0.5220, +0.5822]& [-4.5756, +4.1194] \\
 \hline
\multirow{2}{*}{II}
& Treatment effects (SEs)
& -0.9171 (36.43500)
& -0.2175 (0.56159)
& -0.0269 (0.14245) \\
& 95\% CIs
& [-72.3297, +70.4955]
& [-1.3182, +0.8832]
& [-0.3061, +0.2523] \\
\hline

\multirow{2}{*}{III}
& Treatment effects (SEs)
& -0.0926 (0.07764)
& -0.8109 (17.23161)
& -0.4625 (2.08727) \\
& 95\% CIs
& [-0.2448, +0.0596]
& [-34.5849, +32.9630]
& [-4.5535, +3.6286] \\
\hline

\multirow{2}{*}{IV}
& Treatment effects (SEs)
& -0.0820 (0.18875)
& -0.2891 (4.99796)
& -0.0683 (0.32093) \\
& 95\% CIs
& [-0.4520, +0.2879]
& [-10.0851, +9.5069]
& [-0.6973, +0.5607] \\
\hline

\multirow{2}{*}{V}
& Treatment effects (SEs)
& +0.5472 (30.68336)
& +0.1404 (0.31630)
& +0.2647 (0.74140) \\
& 95\% CIs
& [-59.5922, +60.6866]
& [-0.4795, +0.7604]
& [-1.1884, +1.7179] \\
\hline

\multirow{2}{*}{VI}
& Treatment effects (SEs)
& -4.4027 (408.41541)
& +0.0538 (0.22008)
& -1.4556 (20.80650) \\
& 95\% CIs
& [-804.8970, +796.0915]
& [-0.3776, +0.4851]
& [-42.2364, +39.3251] \\
\bottomrule
\end{tabular}
}
\end{table}

In the primary RHC analysis in the main text, we follow \cite{CuiYifan2020Spci} and use $Z=(\texttt{pafi1}, \texttt{paco21}),W=(\texttt{ph1}, \texttt{hema1})$ as the proxy variables. These proxy variables are in fact selected according to the
procedure proposed  in~\citet{tchetgen2020introduction}.
Among the ten physiological measurements collected within the first 24 hours of ICU admission,
\citet{tchetgen2020introduction} first identify a reduced set of candidate proxy variables.
Specifically, each of the ten variables (\texttt{serum sodium, serum potassium, serum creatinine,
bilirubin, albumin, PaO2/(.01*FiO2) ratio, PaCO2, serum PH (arterial), white blood
cell count, and hematocrit}) is evaluated based on its association with the treatment
$A$ and the outcome $Y$: treatment associations are assessed via logistic regression of $A$ on
the candidate variable and baseline covariates, while outcome associations are assessed via
linear regression of $Y$ on $A$, the candidate variable, and baseline covariates.
Based on the strength of these associations, \citet{tchetgen2020introduction} select the four variables (\texttt{pafi1, ph1, paco21, hema1}) that are strongly related
to both $A$ and $Y$ as candidate proxies, while treating the remaining variables as covariates $X$. \citet{tchetgen2020introduction} state that they ``{hypothesize that the four selected variables are a priori equally likely to
serve as treatment-inducing proxies or outcome-inducing proxies}''. They propose to decide the allocation of the four variables to the negative control actions and negative control outcome proxies by ranking the associations of these variables with the outcome and treatment, respectively. Specifically, they ``select proxies in decreasing order of strength of association, first selecting the proxy with strongest association with the outcome as outcome-inducing proxy and likewise for the treatment''. They break the ties in a particular way during the process and end up with the proxy variable combination
$Z=(\texttt{pafi1}, \texttt{paco21}),W=(\texttt{ph1}, \texttt{hema1})$.

In this part, we assess the robustness of the estimation results to different allocations of the four proxy variables.
From the four proxy variables, we consider all possible ways of allocating two variables
to the negative control treatment set $Z$ and the remaining two to the negative control outcome
set $W$, yielding $\binom{4}{2}=6$ distinct $(Z,W)$ settings.
For each setting, we apply the same proximal estimation methods and report the results in Table~\ref{tab:real_result_permute} and Table~\ref{tab:2sls}. We can observe that our proposed DR estimators based on flexible minimax bridge function estimation are overall stable across different permutations of the proxy variables. In contrast, the linear moment-based baseline estimator and the linear closed-form baseline estimator are  much more sensitive. They can easily give estimates with opposite signs to the original estimates in Table~\ref{tab:real_result}. Moreover, they are still sensitive to the choice of model specification (e.g., incorporating interaction, quadratic or both terms or not), across all six permutations.

\subsubsection{Covariate-Omission Stress Test}
\label{ape:covariate-omission}

This stress test examines the empirical sensitivity of the RHC conclusions to
losing baseline covariates. We keep the proxy-variable allocation fixed at
$Z=(\texttt{pafi1},\texttt{paco21})$ and
$W=(\texttt{ph1},\texttt{hema1})$, but deliberately withhold selected observed
covariates from $X$ and refit the complete analysis. Thus, the exercise does
not constitute a formal sensitivity analysis for arbitrary unmeasured
confounding; rather, it assesses whether the conclusions remain stable under
structured and increasingly severe omissions of information that is observed
in the original data.

The original $X$ contains 49 clinically meaningful raw-variable blocks,
represented by 68 columns after one-hot encoding. We always remove all encoded
columns belonging to a raw variable together, so that each omission scenario
has a clear clinical interpretation. Together with the four proxy variables,
these variables account for the 72 covariates used in the main RHC analysis.

We construct two deterministic omission paths to avoid relying on an
arbitrarily chosen set of covariates. The first, denoted the \emph{A/Y-predictive}
path, ranks each covariate block by its joint predictive strength for treatment
and outcome: we fit $\ell_2$-regularized regressions for $A$ given $X$ and for
$Y$ given $(A,X)$, and rank block $j$ by
$(s_{A,j}s_{Y,j})^{1/2}$, where $s_{A,j}$ and $s_{Y,j}$ are standardized
coefficient-block norms. The second, denoted the \emph{proxy-predictive} path,
ranks blocks by their standardized coefficient-block norms from a multivariate
ridge regression of the four proxies $(Z,W)$ on $X$. These rankings are used
only to define omission scenarios and are not used to tune any
treatment-effect estimator.

Along each path, we remove the top
$k\in\{1,2,3,4,5,10,15,20,25,30,35,40,45\}$ blocks. We additionally report
the full-$X$ analysis and a final union scenario that removes the union of the
two top-45 sets. Because the two top-45 sets overlap in 41 blocks, their union
contains all 49 blocks and leaves no observed $X$ covariates. Hence, the
analysis comprises $1+2\times13+1=28$ scenarios. For every scenario, we refit
all nuisance functions and effect estimators from scratch using the same
preprocessing, two-fold cross-fitting, training epochs, random seeds, and
tuning parameters as in the main RHC analysis.

\begin{figure}[tbp]
    \centering
    \includegraphics[width=\textwidth]{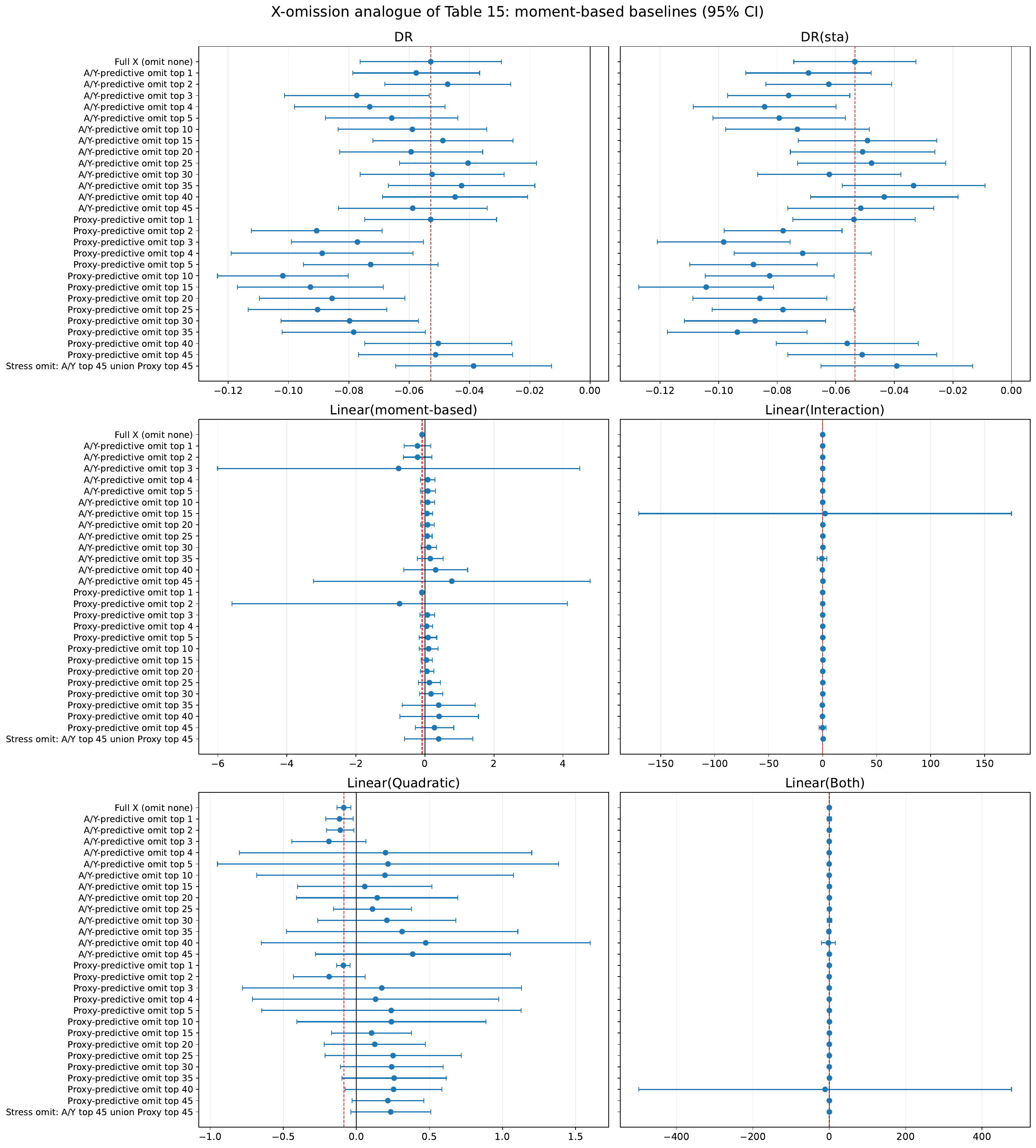}
    \caption{Covariate-omission stress test for DR, DR(sta), and moment-based
    linear specifications. Points and horizontal bars show treatment-effect
    estimates and 95\% confidence intervals. ``A/Y top $k$'' and ``Proxy top $k$''
    remove the $k$ highest-ranked covariate blocks under the treatment--outcome
    and proxy-predictive rankings, respectively. The solid vertical line denotes
    zero; the dashed red line denotes the full-$X$ estimate.}
    \label{fig:x-omission-moment}
\end{figure}

\begin{figure}[tbp]
    \centering
    \includegraphics[width=\textwidth]{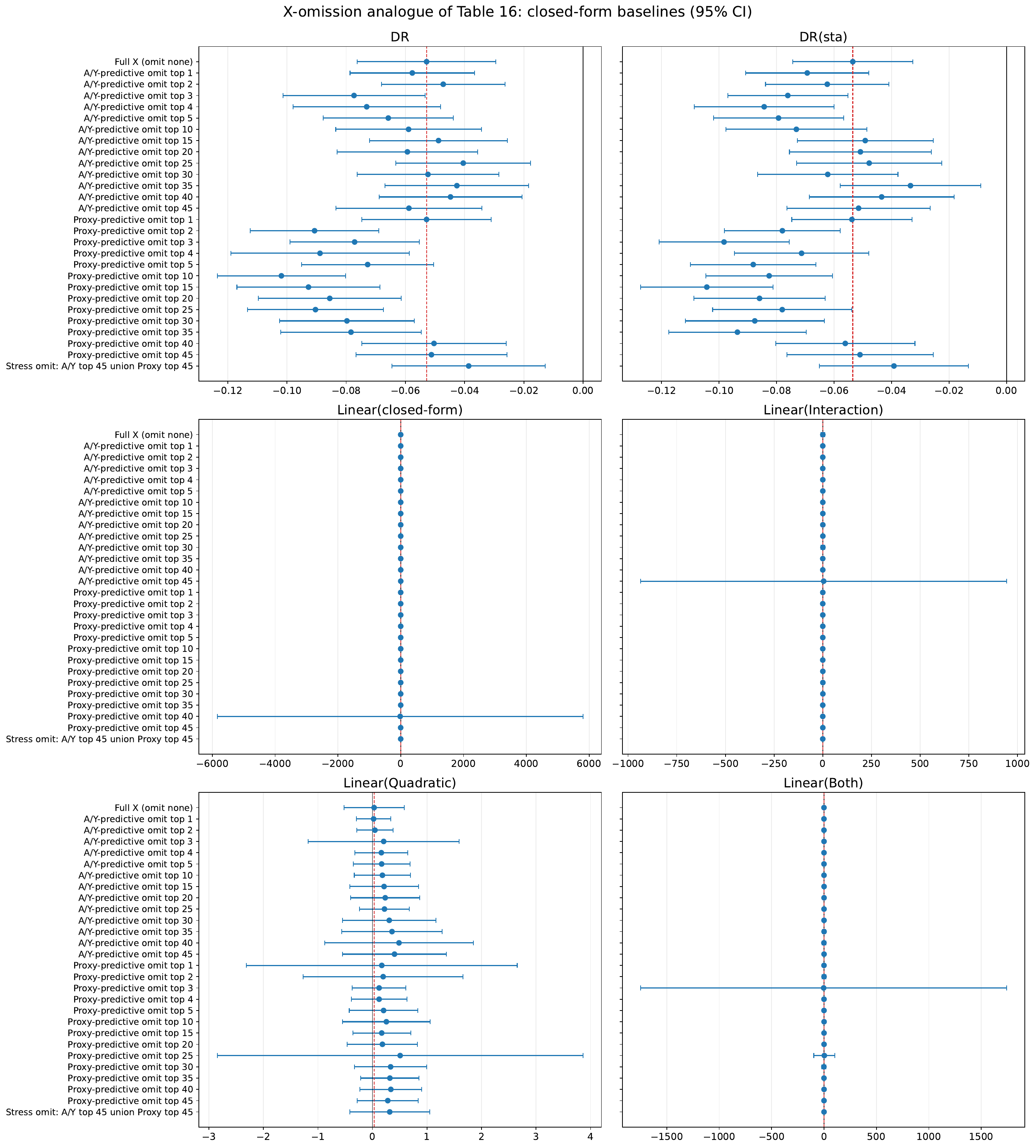}
    \caption{Covariate-omission stress test for DR, DR(sta), and closed-form
    linear specifications. Points and horizontal bars show treatment-effect
    estimates and 95\% confidence intervals. ``A/Y top $k$'' and ``Proxy top $k$''
    remove the $k$ highest-ranked covariate blocks under the treatment--outcome
    and proxy-predictive rankings, respectively. The solid vertical line denotes
    zero; the dashed red line denotes the full-$X$ estimate.}
    \label{fig:x-omission-closed}
\end{figure}

Figures~\ref{fig:x-omission-moment} and~\ref{fig:x-omission-closed} display
point estimates and 95\% confidence intervals along the two omission paths;
the solid vertical line denotes zero and the dashed red line denotes the
corresponding full-$X$ estimate. Tables~\ref{tab:x-omission-stress-moment}
and~\ref{tab:x-omission-stress-closed} provide the numerical results for the
moment-based and closed-form linear baselines, respectively.

\begin{table}[p]
\centering
\caption{Covariate-omission analogue.  The first line in each cell is the effect estimate (standard error), and the second line is its 95\% confidence interval; a bold cell has a 95\% confidence interval whose upper endpoint is below zero.  Linear interaction, quadratic, and both denote the corresponding moment-based basis specifications.}
\label{tab:x-omission-stress-moment}
\scriptsize
\setlength{\tabcolsep}{2.2pt}
\renewcommand{\arraystretch}{1.05}
\resizebox{\textwidth}{!}{
\begin{tabular}{@{}lcccccc@{}}
\hline
$X$ omission & DR & DR(sta) & \shortstack{Linear\\(moment-based)} & \shortstack{Linear\\(Interaction)} & \shortstack{Linear\\(Quadratic)} & \shortstack{Linear\\(Both)} \\
\hline
Full $X$ & \shortstack{\boldmath $-0.0528\;(0.01194)$\\$[-0.0762,-0.0294]$} & \shortstack{\boldmath $-0.0535\;(0.01064)$\\$[-0.0743,-0.0326]$} & \shortstack{\boldmath $-0.0828\;(0.01977)$\\$[-0.1216,-0.0441]$} & \shortstack{$-0.0372\;(0.07466)$\\$[-0.1835,+0.1091]$} & \shortstack{\boldmath $-0.0859\;(0.02427)$\\$[-0.1334,-0.0383]$} & \shortstack{$+0.1053\;(0.97969)$\\$[-1.8149,+2.0255]$} \\
A/Y top $1$ & \shortstack{\boldmath $-0.0577\;(0.01073)$\\$[-0.0787,-0.0366]$} & \shortstack{\boldmath $-0.0693\;(0.01091)$\\$[-0.0907,-0.0479]$} & \shortstack{$-0.2160\;(0.19689)$\\$[-0.6019,+0.1699]$} & \shortstack{$-0.0237\;(0.06570)$\\$[-0.1524,+0.1051]$} & \shortstack{\boldmath $-0.1146\;(0.04802)$\\$[-0.2088,-0.0205]$} & \shortstack{$+0.1316\;(2.64882)$\\$[-5.0601,+5.3233]$} \\
A/Y top $2$ & \shortstack{\boldmath $-0.0472\;(0.01064)$\\$[-0.0681,-0.0264]$} & \shortstack{\boldmath $-0.0624\;(0.01096)$\\$[-0.0838,-0.0409]$} & \shortstack{$-0.2124\;(0.20980)$\\$[-0.6236,+0.1989]$} & \shortstack{$-0.0122\;(0.06684)$\\$[-0.1432,+0.1188]$} & \shortstack{\boldmath $-0.1096\;(0.04773)$\\$[-0.2032,-0.0161]$} & \shortstack{$+0.0787\;(2.06838)$\\$[-3.9753,+4.1328]$} \\
A/Y top $3$ & \shortstack{\boldmath $-0.0774\;(0.01223)$\\$[-0.1013,-0.0534]$} & \shortstack{\boldmath $-0.0761\;(0.01065)$\\$[-0.0970,-0.0552]$} & \shortstack{$-0.7625\;(2.67704)$\\$[-6.0095,+4.4845]$} & \shortstack{$-0.0007\;(0.10310)$\\$[-0.2028,+0.2014]$} & \shortstack{$-0.1872\;(0.12910)$\\$[-0.4403,+0.0658]$} & \shortstack{$-0.0118\;(0.44009)$\\$[-0.8743,+0.8508]$} \\
A/Y top $4$ & \shortstack{\boldmath $-0.0731\;(0.01272)$\\$[-0.0980,-0.0482]$} & \shortstack{\boldmath $-0.0843\;(0.01243)$\\$[-0.1087,-0.0599]$} & \shortstack{$+0.0828\;(0.10833)$\\$[-0.1296,+0.2951]$} & \shortstack{$-0.0271\;(0.04441)$\\$[-0.1142,+0.0599]$} & \shortstack{$+0.1997\;(0.51015)$\\$[-0.8002,+1.1996]$} & \shortstack{$-0.1419\;(0.14219)$\\$[-0.4206,+0.1368]$} \\
A/Y top $5$ & \shortstack{\boldmath $-0.0658\;(0.01120)$\\$[-0.0877,-0.0438]$} & \shortstack{\boldmath $-0.0793\;(0.01155)$\\$[-0.1019,-0.0567]$} & \shortstack{$+0.0833\;(0.10994)$\\$[-0.1322,+0.2988]$} & \shortstack{$-0.0270\;(0.04419)$\\$[-0.1137,+0.0596]$} & \shortstack{$+0.2166\;(0.59490)$\\$[-0.9494,+1.3826]$} & \shortstack{$-0.1459\;(0.14819)$\\$[-0.4363,+0.1446]$} \\
A/Y top $10$ & \shortstack{\boldmath $-0.0589\;(0.01255)$\\$[-0.0835,-0.0343]$} & \shortstack{\boldmath $-0.0731\;(0.01250)$\\$[-0.0976,-0.0486]$} & \shortstack{$+0.0793\;(0.10028)$\\$[-0.1172,+0.2759]$} & \shortstack{$-0.0054\;(0.05192)$\\$[-0.1072,+0.0963]$} & \shortstack{$+0.1958\;(0.44810)$\\$[-0.6825,+1.0741]$} & \shortstack{$-0.0941\;(0.14009)$\\$[-0.3687,+0.1804]$} \\
A/Y top $15$ & \shortstack{\boldmath $-0.0488\;(0.01187)$\\$[-0.0721,-0.0256]$} & \shortstack{\boldmath $-0.0492\;(0.01205)$\\$[-0.0728,-0.0256]$} & \shortstack{$+0.0614\;(0.08117)$\\$[-0.0976,+0.2205]$} & \shortstack{$+2.2670\;(88.00725)$\\$[-170.2272,+174.7612]$} & \shortstack{$+0.0577\;(0.23454)$\\$[-0.4020,+0.5174]$} & \shortstack{$+0.3397\;(2.16806)$\\$[-3.9097,+4.5891]$} \\
A/Y top $20$ & \shortstack{\boldmath $-0.0593\;(0.01210)$\\$[-0.0830,-0.0356]$} & \shortstack{\boldmath $-0.0508\;(0.01260)$\\$[-0.0755,-0.0261]$} & \shortstack{$+0.0760\;(0.10037)$\\$[-0.1207,+0.2727]$} & \shortstack{$+0.0982\;(0.22882)$\\$[-0.3503,+0.5467]$} & \shortstack{$+0.1429\;(0.28161)$\\$[-0.4090,+0.6949]$} & \shortstack{$+0.2442\;(1.15575)$\\$[-2.0211,+2.5094]$} \\
A/Y top $25$ & \shortstack{\boldmath $-0.0405\;(0.01159)$\\$[-0.0632,-0.0177]$} & \shortstack{\boldmath $-0.0478\;(0.01289)$\\$[-0.0730,-0.0225]$} & \shortstack{$+0.0690\;(0.07118)$\\$[-0.0705,+0.2085]$} & \shortstack{$+0.1427\;(0.44011)$\\$[-0.7200,+1.0053]$} & \shortstack{$+0.1106\;(0.13600)$\\$[-0.1560,+0.3771]$} & \shortstack{$+0.2393\;(0.84064)$\\$[-1.4084,+1.8869]$} \\
A/Y top $30$ & \shortstack{\boldmath $-0.0524\;(0.01220)$\\$[-0.0763,-0.0285]$} & \shortstack{\boldmath $-0.0622\;(0.01246)$\\$[-0.0866,-0.0378]$} & \shortstack{$+0.1103\;(0.11384)$\\$[-0.1129,+0.3334]$} & \shortstack{$+0.1804\;(0.24856)$\\$[-0.3068,+0.6676]$} & \shortstack{$+0.2089\;(0.24092)$\\$[-0.2633,+0.6811]$} & \shortstack{$+0.7489\;(2.93033)$\\$[-4.9945,+6.4924]$} \\
A/Y top $35$ & \shortstack{\boldmath $-0.0426\;(0.01239)$\\$[-0.0669,-0.0183]$} & \shortstack{\boldmath $-0.0334\;(0.01245)$\\$[-0.0578,-0.0090]$} & \shortstack{$+0.1559\;(0.19132)$\\$[-0.2191,+0.5309]$} & \shortstack{$-0.7394\;(2.31807)$\\$[-5.2828,+3.8040]$} & \shortstack{$+0.3133\;(0.40369)$\\$[-0.4779,+1.1046]$} & \shortstack{$-0.9834\;(2.50975)$\\$[-5.9025,+3.9357]$} \\
A/Y top $40$ & \shortstack{\boldmath $-0.0448\;(0.01230)$\\$[-0.0689,-0.0207]$} & \shortstack{\boldmath $-0.0435\;(0.01285)$\\$[-0.0686,-0.0183]$} & \shortstack{$+0.3104\;(0.47300)$\\$[-0.6167,+1.2375]$} & \shortstack{$-0.3323\;(0.52529)$\\$[-1.3619,+0.6973]$} & \shortstack{$+0.4748\;(0.57366)$\\$[-0.6496,+1.5992]$} & \shortstack{$-2.1529\;(9.35104)$\\$[-20.4809,+16.1751]$} \\
A/Y top $45$ & \shortstack{\boldmath $-0.0588\;(0.01258)$\\$[-0.0835,-0.0342]$} & \shortstack{\boldmath $-0.0515\;(0.01272)$\\$[-0.0764,-0.0265]$} & \shortstack{$+0.7789\;(2.04441)$\\$[-3.2282,+4.7859]$} & \shortstack{$+0.1732\;(0.35572)$\\$[-0.5240,+0.8704]$} & \shortstack{$+0.3857\;(0.33988)$\\$[-0.2804,+1.0519]$} & \shortstack{$+0.1987\;(0.30957)$\\$[-0.4080,+0.8055]$} \\
Proxy top $1$ & \shortstack{\boldmath $-0.0529\;(0.01114)$\\$[-0.0747,-0.0310]$} & \shortstack{\boldmath $-0.0538\;(0.01067)$\\$[-0.0747,-0.0329]$} & \shortstack{\boldmath $-0.0882\;(0.02086)$\\$[-0.1291,-0.0474]$} & \shortstack{$-0.0338\;(0.13887)$\\$[-0.3060,+0.2384]$} & \shortstack{\boldmath $-0.0886\;(0.02385)$\\$[-0.1353,-0.0418]$} & \shortstack{$+0.1565\;(1.64385)$\\$[-3.0655,+3.3784]$} \\
Proxy top $2$ & \shortstack{\boldmath $-0.0907\;(0.01108)$\\$[-0.1124,-0.0690]$} & \shortstack{\boldmath $-0.0780\;(0.01026)$\\$[-0.0981,-0.0579]$} & \shortstack{$-0.7343\;(2.47874)$\\$[-5.5926,+4.1241]$} & \shortstack{$+0.0032\;(0.11144)$\\$[-0.2152,+0.2216]$} & \shortstack{$-0.1860\;(0.12546)$\\$[-0.4319,+0.0599]$} & \shortstack{$-0.2244\;(1.72596)$\\$[-3.6073,+3.1585]$} \\
Proxy top $3$ & \shortstack{\boldmath $-0.0772\;(0.01115)$\\$[-0.0990,-0.0553]$} & \shortstack{\boldmath $-0.0983\;(0.01155)$\\$[-0.1209,-0.0756]$} & \shortstack{$+0.0708\;(0.10719)$\\$[-0.1393,+0.2809]$} & \shortstack{$-0.0184\;(0.04342)$\\$[-0.1035,+0.0667]$} & \shortstack{$+0.1742\;(0.48693)$\\$[-0.7802,+1.1286]$} & \shortstack{$-0.1019\;(0.11032)$\\$[-0.3182,+0.1143]$} \\
Proxy top $4$ & \shortstack{\boldmath $-0.0888\;(0.01536)$\\$[-0.1190,-0.0587]$} & \shortstack{\boldmath $-0.0713\;(0.01191)$\\$[-0.0947,-0.0480]$} & \shortstack{$+0.0485\;(0.09224)$\\$[-0.1323,+0.2293]$} & \shortstack{$+0.0465\;(0.08656)$\\$[-0.1232,+0.2162]$} & \shortstack{$+0.1320\;(0.42988)$\\$[-0.7105,+0.9746]$} & \shortstack{$+0.0899\;(0.47028)$\\$[-0.8318,+1.0117]$} \\
Proxy top $5$ & \shortstack{\boldmath $-0.0728\;(0.01138)$\\$[-0.0951,-0.0504]$} & \shortstack{\boldmath $-0.0881\;(0.01112)$\\$[-0.1099,-0.0663]$} & \shortstack{$+0.0902\;(0.12862)$\\$[-0.1619,+0.3423]$} & \shortstack{$+0.0832\;(0.11481)$\\$[-0.1418,+0.3083]$} & \shortstack{$+0.2390\;(0.45256)$\\$[-0.6481,+1.1260]$} & \shortstack{$+0.2179\;(0.46996)$\\$[-0.7032,+1.1391]$} \\
Proxy top $10$ & \shortstack{\boldmath $-0.1019\;(0.01105)$\\$[-0.1235,-0.0802]$} & \shortstack{\boldmath $-0.0826\;(0.01125)$\\$[-0.1046,-0.0605]$} & \shortstack{$+0.1062\;(0.14018)$\\$[-0.1685,+0.3810]$} & \shortstack{$+0.1977\;(0.35576)$\\$[-0.4996,+0.8950]$} & \shortstack{$+0.2400\;(0.32968)$\\$[-0.4062,+0.8862]$} & \shortstack{$+0.4553\;(1.24366)$\\$[-1.9823,+2.8928]$} \\
Proxy top $15$ & \shortstack{\boldmath $-0.0927\;(0.01233)$\\$[-0.1169,-0.0686]$} & \shortstack{\boldmath $-0.1043\;(0.01172)$\\$[-0.1272,-0.0813]$} & \shortstack{$+0.0470\;(0.08288)$\\$[-0.1155,+0.2094]$} & \shortstack{$+0.1633\;(0.33528)$\\$[-0.4938,+0.8205]$} & \shortstack{$+0.1044\;(0.13940)$\\$[-0.1688,+0.3777]$} & \shortstack{$+0.3236\;(0.85749)$\\$[-1.3570,+2.0043]$} \\
Proxy top $20$ & \shortstack{\boldmath $-0.0856\;(0.01232)$\\$[-0.1097,-0.0614]$} & \shortstack{\boldmath $-0.0859\;(0.01166)$\\$[-0.1088,-0.0631]$} & \shortstack{$+0.0594\;(0.10014)$\\$[-0.1369,+0.2556]$} & \shortstack{$+0.0192\;(0.23968)$\\$[-0.4506,+0.4889]$} & \shortstack{$+0.1264\;(0.17701)$\\$[-0.2206,+0.4733]$} & \shortstack{$+0.0684\;(0.38524)$\\$[-0.6867,+0.8234]$} \\
Proxy top $25$ & \shortstack{\boldmath $-0.0904\;(0.01172)$\\$[-0.1133,-0.0674]$} & \shortstack{\boldmath $-0.0780\;(0.01235)$\\$[-0.1022,-0.0538]$} & \shortstack{$+0.1295\;(0.16461)$\\$[-0.1931,+0.4522]$} & \shortstack{$+0.1000\;(0.22317)$\\$[-0.3374,+0.5375]$} & \shortstack{$+0.2514\;(0.23815)$\\$[-0.2154,+0.7182]$} & \shortstack{$+0.1475\;(0.47312)$\\$[-0.7798,+1.0748]$} \\
Proxy top $30$ & \shortstack{\boldmath $-0.0798\;(0.01163)$\\$[-0.1026,-0.0570]$} & \shortstack{\boldmath $-0.0876\;(0.01232)$\\$[-0.1117,-0.0634]$} & \shortstack{$+0.1757\;(0.17122)$\\$[-0.1599,+0.5113]$} & \shortstack{$+0.0635\;(0.33287)$\\$[-0.5889,+0.7159]$} & \shortstack{$+0.2415\;(0.17917)$\\$[-0.1096,+0.5927]$} & \shortstack{$+0.0673\;(0.60075)$\\$[-1.1102,+1.2447]$} \\
Proxy top $35$ & \shortstack{\boldmath $-0.0784\;(0.01212)$\\$[-0.1021,-0.0546]$} & \shortstack{\boldmath $-0.0936\;(0.01219)$\\$[-0.1175,-0.0697]$} & \shortstack{$+0.3989\;(0.53878)$\\$[-0.6572,+1.4549]$} & \shortstack{$-0.3444\;(0.39895)$\\$[-1.1263,+0.4375]$} & \shortstack{$+0.2585\;(0.18191)$\\$[-0.0981,+0.6150]$} & \shortstack{$+0.4921\;(0.59737)$\\$[-0.6788,+1.6629]$} \\
Proxy top $40$ & \shortstack{\boldmath $-0.0503\;(0.01246)$\\$[-0.0748,-0.0259]$} & \shortstack{\boldmath $-0.0561\;(0.01236)$\\$[-0.0803,-0.0319]$} & \shortstack{$+0.4115\;(0.58097)$\\$[-0.7272,+1.5502]$} & \shortstack{$-0.2305\;(0.13512)$\\$[-0.4954,+0.0343]$} & \shortstack{$+0.2549\;(0.16939)$\\$[-0.0771,+0.5869]$} & \shortstack{$-10.6770\;(248.79711)$\\$[-498.3193,+476.9653]$} \\
Proxy top $45$ & \shortstack{\boldmath $-0.0512\;(0.01301)$\\$[-0.0767,-0.0257]$} & \shortstack{\boldmath $-0.0510\;(0.01299)$\\$[-0.0764,-0.0255]$} & \shortstack{$+0.2756\;(0.28496)$\\$[-0.2829,+0.8341]$} & \shortstack{$-0.1269\;(1.65824)$\\$[-3.3770,+3.1233]$} & \shortstack{$+0.2157\;(0.12537)$\\$[-0.0301,+0.4614]$} & \shortstack{$+0.2299\;(0.20345)$\\$[-0.1689,+0.6286]$} \\
Top-45 union (no $X$) & \shortstack{\boldmath $-0.0386\;(0.01322)$\\$[-0.0645,-0.0127]$} & \shortstack{\boldmath $-0.0392\;(0.01322)$\\$[-0.0651,-0.0133]$} & \shortstack{$+0.3976\;(0.50583)$\\$[-0.5938,+1.3890]$} & \shortstack{$+0.3976\;(0.50583)$\\$[-0.5938,+1.3890]$} & \shortstack{$+0.2355\;(0.13931)$\\$[-0.0376,+0.5085]$} & \shortstack{$+0.2355\;(0.13931)$\\$[-0.0376,+0.5085]$} \\
\hline
\end{tabular}
}
\end{table}

\begin{table}[p]
\centering
\caption{Covariate-omission analogue.  The first line in each cell is the effect estimate (standard error), and the second line is its 95\% confidence interval; a bold cell has a 95\% confidence interval whose upper endpoint is below zero.  Linear interaction, quadratic, and both denote the corresponding closed-form basis specifications.}
\label{tab:x-omission-stress-closed}
\scriptsize
\setlength{\tabcolsep}{2.2pt}
\renewcommand{\arraystretch}{1.05}
\resizebox{\textwidth}{!}{
\begin{tabular}{@{}lcccccc@{}}
\hline
$X$ omission & DR & DR(sta) & \shortstack{Linear\\(closed-form)} & \shortstack{Linear\\(Interaction)} & \shortstack{Linear\\(Quadratic)} & \shortstack{Linear\\(Both)} \\
\hline
Full $X$ & \shortstack{\boldmath $-0.0528\;(0.01194)$\\$[-0.0762,-0.0294]$} & \shortstack{\boldmath $-0.0535\;(0.01064)$\\$[-0.0743,-0.0326]$} & \shortstack{\boldmath $-0.0860\;(0.02408)$\\$[-0.1332,-0.0388]$} & \shortstack{$+0.0017\;(5.32580)$\\$[-10.4369,+10.4402]$} & \shortstack{$+0.0301\;(0.28169)$\\$[-0.5220,+0.5822]$} & \shortstack{$-0.2281\;(2.21813)$\\$[-4.5756,+4.1194]$} \\
A/Y top $1$ & \shortstack{\boldmath $-0.0577\;(0.01073)$\\$[-0.0787,-0.0366]$} & \shortstack{\boldmath $-0.0693\;(0.01091)$\\$[-0.0907,-0.0479]$} & \shortstack{$-0.0455\;(0.21887)$\\$[-0.4745,+0.3835]$} & \shortstack{\boldmath $-0.1369\;(0.06844)$\\$[-0.2710,-0.0027]$} & \shortstack{$+0.0207\;(0.16058)$\\$[-0.2941,+0.3354]$} & \shortstack{$-0.3788\;(3.96309)$\\$[-8.1464,+7.3889]$} \\
A/Y top $2$ & \shortstack{\boldmath $-0.0472\;(0.01064)$\\$[-0.0681,-0.0264]$} & \shortstack{\boldmath $-0.0624\;(0.01096)$\\$[-0.0838,-0.0409]$} & \shortstack{$-0.0313\;(0.20944)$\\$[-0.4418,+0.3792]$} & \shortstack{$-0.1212\;(0.06817)$\\$[-0.2548,+0.0124]$} & \shortstack{$+0.0450\;(0.16952)$\\$[-0.2873,+0.3773]$} & \shortstack{$-0.2390\;(1.93839)$\\$[-4.0382,+3.5602]$} \\
A/Y top $3$ & \shortstack{\boldmath $-0.0774\;(0.01223)$\\$[-0.1013,-0.0534]$} & \shortstack{\boldmath $-0.0761\;(0.01065)$\\$[-0.0970,-0.0552]$} & \shortstack{$-0.1187\;(0.16266)$\\$[-0.4375,+0.2001]$} & \shortstack{$-0.2382\;(0.29158)$\\$[-0.8097,+0.3333]$} & \shortstack{$+0.2054\;(0.70717)$\\$[-1.1806,+1.5915]$} & \shortstack{$+0.2710\;(2.09452)$\\$[-3.8342,+4.3763]$} \\
A/Y top $4$ & \shortstack{\boldmath $-0.0731\;(0.01272)$\\$[-0.0980,-0.0482]$} & \shortstack{\boldmath $-0.0843\;(0.01243)$\\$[-0.1087,-0.0599]$} & \shortstack{$+0.3822\;(0.82374)$\\$[-1.2323,+1.9967]$} & \shortstack{$+0.0320\;(0.05937)$\\$[-0.0843,+0.1484]$} & \shortstack{$+0.1626\;(0.24639)$\\$[-0.3203,+0.6455]$} & \shortstack{$+0.0048\;(0.09888)$\\$[-0.1890,+0.1986]$} \\
A/Y top $5$ & \shortstack{\boldmath $-0.0658\;(0.01120)$\\$[-0.0877,-0.0438]$} & \shortstack{\boldmath $-0.0793\;(0.01155)$\\$[-0.1019,-0.0567]$} & \shortstack{$+0.3925\;(0.87208)$\\$[-1.3167,+2.1018]$} & \shortstack{$+0.0325\;(0.05943)$\\$[-0.0840,+0.1490]$} & \shortstack{$+0.1681\;(0.26488)$\\$[-0.3511,+0.6873]$} & \shortstack{$+0.0040\;(0.10115)$\\$[-0.1942,+0.2023]$} \\
A/Y top $10$ & \shortstack{\boldmath $-0.0589\;(0.01255)$\\$[-0.0835,-0.0343]$} & \shortstack{\boldmath $-0.0731\;(0.01250)$\\$[-0.0976,-0.0486]$} & \shortstack{$+0.5083\;(1.31866)$\\$[-2.0763,+3.0928]$} & \shortstack{$+0.0465\;(0.05949)$\\$[-0.0701,+0.1631]$} & \shortstack{$+0.1822\;(0.26321)$\\$[-0.3337,+0.6981]$} & \shortstack{$+0.0063\;(0.10715)$\\$[-0.2038,+0.2163]$} \\
A/Y top $15$ & \shortstack{\boldmath $-0.0488\;(0.01187)$\\$[-0.0721,-0.0256]$} & \shortstack{\boldmath $-0.0492\;(0.01205)$\\$[-0.0728,-0.0256]$} & \shortstack{$+0.4872\;(1.27862)$\\$[-2.0189,+2.9933]$} & \shortstack{$+0.1881\;(0.62168)$\\$[-1.0304,+1.4066]$} & \shortstack{$+0.2125\;(0.32162)$\\$[-0.4179,+0.8428]$} & \shortstack{$+0.7678\;(7.80289)$\\$[-14.5259,+16.0615]$} \\
A/Y top $20$ & \shortstack{\boldmath $-0.0593\;(0.01210)$\\$[-0.0830,-0.0356]$} & \shortstack{\boldmath $-0.0508\;(0.01260)$\\$[-0.0755,-0.0261]$} & \shortstack{$+0.4307\;(1.03634)$\\$[-1.6005,+2.4619]$} & \shortstack{$+0.1980\;(0.55698)$\\$[-0.8936,+1.2897]$} & \shortstack{$+0.2337\;(0.32290)$\\$[-0.3992,+0.8666]$} & \shortstack{$+0.5704\;(1.50678)$\\$[-2.3828,+3.5237]$} \\
A/Y top $25$ & \shortstack{\boldmath $-0.0405\;(0.01159)$\\$[-0.0632,-0.0177]$} & \shortstack{\boldmath $-0.0478\;(0.01289)$\\$[-0.0730,-0.0225]$} & \shortstack{$+0.2415\;(0.30545)$\\$[-0.3572,+0.8402]$} & \shortstack{$+0.3289\;(0.61259)$\\$[-0.8718,+1.5296]$} & \shortstack{$+0.2201\;(0.23303)$\\$[-0.2367,+0.6768]$} & \shortstack{$+0.3053\;(0.56427)$\\$[-0.8006,+1.4113]$} \\
A/Y top $30$ & \shortstack{\boldmath $-0.0524\;(0.01220)$\\$[-0.0763,-0.0285]$} & \shortstack{\boldmath $-0.0622\;(0.01246)$\\$[-0.0866,-0.0378]$} & \shortstack{$+0.3082\;(0.47931)$\\$[-0.6313,+1.2476]$} & \shortstack{$+0.5785\;(4.67175)$\\$[-8.5781,+9.7351]$} & \shortstack{$+0.3078\;(0.43772)$\\$[-0.5501,+1.1657]$} & \shortstack{$+0.3648\;(0.78970)$\\$[-1.1830,+1.9126]$} \\
A/Y top $35$ & \shortstack{\boldmath $-0.0426\;(0.01239)$\\$[-0.0669,-0.0183]$} & \shortstack{\boldmath $-0.0334\;(0.01245)$\\$[-0.0578,-0.0090]$} & \shortstack{$+0.3763\;(0.71757)$\\$[-1.0301,+1.7828]$} & \shortstack{$-0.2261\;(0.43047)$\\$[-1.0698,+0.6176]$} & \shortstack{$+0.3580\;(0.47028)$\\$[-0.5638,+1.2797]$} & \shortstack{$+0.7547\;(8.04350)$\\$[-15.0105,+16.5200]$} \\
A/Y top $40$ & \shortstack{\boldmath $-0.0448\;(0.01230)$\\$[-0.0689,-0.0207]$} & \shortstack{\boldmath $-0.0435\;(0.01285)$\\$[-0.0686,-0.0183]$} & \shortstack{$+0.7328\;(2.59698)$\\$[-4.3573,+5.8228]$} & \shortstack{$-0.0929\;(0.17329)$\\$[-0.4326,+0.2468]$} & \shortstack{$+0.4864\;(0.69580)$\\$[-0.8773,+1.8502]$} & \shortstack{$-0.9167\;(8.91968)$\\$[-18.3993,+16.5659]$} \\
A/Y top $45$ & \shortstack{\boldmath $-0.0588\;(0.01258)$\\$[-0.0835,-0.0342]$} & \shortstack{\boldmath $-0.0515\;(0.01272)$\\$[-0.0764,-0.0265]$} & \shortstack{$+1.2266\;(9.11972)$\\$[-16.6481,+19.1012]$} & \shortstack{$+4.4303\;(479.22105)$\\$[-934.8430,+943.7035]$} & \shortstack{$+0.4038\;(0.48580)$\\$[-0.5484,+1.3559]$} & \shortstack{$+0.5914\;(1.74717)$\\$[-2.8331,+4.0159]$} \\
Proxy top $1$ & \shortstack{\boldmath $-0.0529\;(0.01114)$\\$[-0.0747,-0.0310]$} & \shortstack{\boldmath $-0.0538\;(0.01067)$\\$[-0.0747,-0.0329]$} & \shortstack{\boldmath $-0.1026\;(0.02854)$\\$[-0.1585,-0.0466]$} & \shortstack{$-0.0281\;(0.07174)$\\$[-0.1687,+0.1125]$} & \shortstack{$+0.1708\;(1.26729)$\\$[-2.3131,+2.6546]$} & \shortstack{$-0.1116\;(0.88010)$\\$[-1.8366,+1.6134]$} \\
Proxy top $2$ & \shortstack{\boldmath $-0.0907\;(0.01108)$\\$[-0.1124,-0.0690]$} & \shortstack{\boldmath $-0.0780\;(0.01026)$\\$[-0.0981,-0.0579]$} & \shortstack{$-0.1310\;(0.15201)$\\$[-0.4289,+0.1670]$} & \shortstack{$-0.2657\;(0.22036)$\\$[-0.6976,+0.1662]$} & \shortstack{$+0.1959\;(0.74761)$\\$[-1.2694,+1.6612]$} & \shortstack{$+0.4347\;(9.99059)$\\$[-19.1468,+20.0163]$} \\
Proxy top $3$ & \shortstack{\boldmath $-0.0772\;(0.01115)$\\$[-0.0990,-0.0553]$} & \shortstack{\boldmath $-0.0983\;(0.01155)$\\$[-0.1209,-0.0756]$} & \shortstack{$+0.3739\;(0.87271)$\\$[-1.3366,+2.0844]$} & \shortstack{$-0.0823\;(0.33441)$\\$[-0.7378,+0.5731]$} & \shortstack{$+0.1220\;(0.25107)$\\$[-0.3701,+0.6141]$} & \shortstack{$-3.2327\;(890.53077)$\\$[-1748.6731,+1742.2076]$} \\
Proxy top $4$ & \shortstack{\boldmath $-0.0888\;(0.01536)$\\$[-0.1190,-0.0587]$} & \shortstack{\boldmath $-0.0713\;(0.01191)$\\$[-0.0947,-0.0480]$} & \shortstack{$+0.3076\;(0.75082)$\\$[-1.1640,+1.7792]$} & \shortstack{$-0.0687\;(0.51351)$\\$[-1.0751,+0.9378]$} & \shortstack{$+0.1220\;(0.25938)$\\$[-0.3864,+0.6304]$} & \shortstack{$+0.0528\;(1.32697)$\\$[-2.5480,+2.6537]$} \\
Proxy top $5$ & \shortstack{\boldmath $-0.0728\;(0.01138)$\\$[-0.0951,-0.0504]$} & \shortstack{\boldmath $-0.0881\;(0.01112)$\\$[-0.1099,-0.0663]$} & \shortstack{$+0.2866\;(0.55437)$\\$[-0.8000,+1.3732]$} & \shortstack{$-0.0094\;(0.72909)$\\$[-1.4384,+1.4196]$} & \shortstack{$+0.2036\;(0.32108)$\\$[-0.4257,+0.8329]$} & \shortstack{$+0.0415\;(5.00662)$\\$[-9.7715,+9.8545]$} \\
Proxy top $10$ & \shortstack{\boldmath $-0.1019\;(0.01105)$\\$[-0.1235,-0.0802]$} & \shortstack{\boldmath $-0.0826\;(0.01125)$\\$[-0.1046,-0.0605]$} & \shortstack{$+0.2916\;(0.51606)$\\$[-0.7198,+1.3031]$} & \shortstack{$-0.1830\;(0.65134)$\\$[-1.4596,+1.0937]$} & \shortstack{$+0.2547\;(0.41040)$\\$[-0.5496,+1.0591]$} & \shortstack{$+0.7833\;(3.97540)$\\$[-7.0085,+8.5751]$} \\
Proxy top $15$ & \shortstack{\boldmath $-0.0927\;(0.01233)$\\$[-0.1169,-0.0686]$} & \shortstack{\boldmath $-0.1043\;(0.01172)$\\$[-0.1272,-0.0813]$} & \shortstack{$+0.2738\;(0.54692)$\\$[-0.7981,+1.3458]$} & \shortstack{$-0.2750\;(0.71043)$\\$[-1.6675,+1.1174]$} & \shortstack{$+0.1694\;(0.27099)$\\$[-0.3618,+0.7005]$} & \shortstack{$+0.2775\;(0.74544)$\\$[-1.1836,+1.7385]$} \\
Proxy top $20$ & \shortstack{\boldmath $-0.0856\;(0.01232)$\\$[-0.1097,-0.0614]$} & \shortstack{\boldmath $-0.0859\;(0.01166)$\\$[-0.1088,-0.0631]$} & \shortstack{$+0.2803\;(0.61124)$\\$[-0.9177,+1.4784]$} & \shortstack{$+0.1843\;(0.57232)$\\$[-0.9375,+1.3060]$} & \shortstack{$+0.1824\;(0.32865)$\\$[-0.4618,+0.8265]$} & \shortstack{$+0.5201\;(1.42461)$\\$[-2.2721,+3.3123]$} \\
Proxy top $25$ & \shortstack{\boldmath $-0.0904\;(0.01172)$\\$[-0.1133,-0.0674]$} & \shortstack{\boldmath $-0.0780\;(0.01235)$\\$[-0.1022,-0.0538]$} & \shortstack{$+0.1919\;(0.24210)$\\$[-0.2826,+0.6665]$} & \shortstack{$+0.3316\;(2.52698)$\\$[-4.6213,+5.2845]$} & \shortstack{$+0.5078\;(1.71145)$\\$[-2.8466,+3.8623]$} & \shortstack{$+3.2718\;(51.40849)$\\$[-97.4888,+104.0325]$} \\
Proxy top $30$ & \shortstack{\boldmath $-0.0798\;(0.01163)$\\$[-0.1026,-0.0570]$} & \shortstack{\boldmath $-0.0876\;(0.01232)$\\$[-0.1117,-0.0634]$} & \shortstack{$+0.2485\;(0.26836)$\\$[-0.2775,+0.7745]$} & \shortstack{$+0.3808\;(3.13343)$\\$[-5.7607,+6.5223]$} & \shortstack{$+0.3328\;(0.33683)$\\$[-0.3274,+0.9930]$} & \shortstack{$-1.2052\;(9.89994)$\\$[-20.6091,+18.1987]$} \\
Proxy top $35$ & \shortstack{\boldmath $-0.0784\;(0.01212)$\\$[-0.1021,-0.0546]$} & \shortstack{\boldmath $-0.0936\;(0.01219)$\\$[-0.1175,-0.0697]$} & \shortstack{$+0.5998\;(1.77453)$\\$[-2.8783,+4.0779]$} & \shortstack{$-0.2672\;(0.23840)$\\$[-0.7344,+0.2001]$} & \shortstack{$+0.3193\;(0.27267)$\\$[-0.2151,+0.8537]$} & \shortstack{$-0.2187\;(1.20054)$\\$[-2.5717,+2.1344]$} \\
Proxy top $40$ & \shortstack{\boldmath $-0.0503\;(0.01246)$\\$[-0.0748,-0.0259]$} & \shortstack{\boldmath $-0.0561\;(0.01236)$\\$[-0.0803,-0.0319]$} & \shortstack{$-16.6860\;(2971.51473)$\\$[-5840.8549,+5807.4828]$} & \shortstack{$-0.1922\;(0.21057)$\\$[-0.6049,+0.2205]$} & \shortstack{$+0.3367\;(0.28999)$\\$[-0.2317,+0.9051]$} & \shortstack{$+0.5845\;(2.25366)$\\$[-3.8327,+5.0017]$} \\
Proxy top $45$ & \shortstack{\boldmath $-0.0512\;(0.01301)$\\$[-0.0767,-0.0257]$} & \shortstack{\boldmath $-0.0510\;(0.01299)$\\$[-0.0764,-0.0255]$} & \shortstack{$-1.1029\;(14.56344)$\\$[-29.6473,+27.4414]$} & \shortstack{$+0.1222\;(0.11813)$\\$[-0.1093,+0.3538]$} & \shortstack{$+0.2804\;(0.28593)$\\$[-0.2801,+0.8408]$} & \shortstack{$+0.1626\;(0.15495)$\\$[-0.1411,+0.4663]$} \\
Top-45 union (no $X$) & \shortstack{\boldmath $-0.0386\;(0.01322)$\\$[-0.0645,-0.0127]$} & \shortstack{\boldmath $-0.0392\;(0.01322)$\\$[-0.0651,-0.0133]$} & \shortstack{$-0.2408\;(0.87537)$\\$[-1.9565,+1.4749]$} & \shortstack{$-0.2408\;(0.87537)$\\$[-1.9565,+1.4749]$} & \shortstack{$+0.3161\;(0.37408)$\\$[-0.4170,+1.0493]$} & \shortstack{$+0.3161\;(0.37408)$\\$[-0.4170,+1.0493]$} \\
\hline
\end{tabular}
}
\end{table}

    Across the 28 constructed scenarios, both DR and DR(sta) retain negative point
estimates and 95\% confidence intervals below zero. Their estimates remain
close to the full-$X$ reference even under the most severe omissions. In
contrast, the linear estimators show substantially greater sensitivity to both
covariate omission and basis specification, with several scenarios yielding
sign changes or very wide confidence intervals. Thus, within the structured
omission designs considered here, the flexible minimax estimators exhibit
markedly greater empirical stability than the parametric linear alternatives.

\section{Proofs}
In all of the proofs, when we assume $\|Y\|_{\infty},\|\Hbbb\|_{\infty},\|\Qbbb\|_{\infty}, \|\Hbbb'\|_{\infty},\|\Qbbb'\|_{\infty},
\|\epol\|_{\infty},
\|\epol/f(a|x,w)\|_2$ are finite, we denote their upper bounds as $C_Y,C_{\Hbbb},C_{\Qbbb},C_{\Hbbb'},C_{\Qbbb'}, C_\pi, C_w$ respectively.

\subsection{Supporting Lemmas}
\begin{lemma}[Dudley integral]\label{lem:dudley}
\begin{align*}
    \Rad(\Fcal)\lesssim \inf_{\tau\geq 0}\braces{\tau+\int_{\tau}^{\sup_{f\in \Fcal}\sqrt{\E_n[f^2]}}\sqrt{\frac{\log \Ncal(\beta ,\Fcal,\|\cdot\|_n)}{n}}d  \beta }.
\end{align*}
\end{lemma}

Note $\sup_{f\in \Fcal}\sqrt{\hP[f^2]}$ is upper bounded by the envelope $\|\Fcal\|_{\infty}$.

\begin{lemma}{Covering number of VC-subgraph classes \citep[Lemma 19.15]{VaartA.W.vander1998As} }\label{lem: vc}
For a VC class of functions $\Fcal$ with measurable envelope function $F$ and $r\geq 1$, one has for any probability measure $Q$ with $\|F\|_{Q,r}>0$,
\begin{align*}
    \Ncal(\epsilon \|F\|_{Q,r},\Fcal,L_r(Q))\lesssim V(\Fcal)(4e)^{V(\Fcal)}\prns{\frac{2}{\epsilon}}^{rV(\Fcal)}.
\end{align*}
\end{lemma}

\begin{lemma}[Theorem 1 in \cite{golowich2018size}] \label{lem:golowich}
    Let $\mathcal{F}_{NN}$ be a class of neural networks with $L$ layers and activation functions that are $1$-Lipschitz continuous and positive-homogeneous, that is, $\sigma(az) = a \sigma(z)$ holds with $a > 0$.
    Further, let $A_\ell$ be a weight matrix of an $\ell$-th layers for $\ell = 1,...,L$ and assume that $\|A_\ell\|_{F} \leq M_F(\ell)$ with a bound $M_F(\ell)$.
    Then, we have $\mathcal{R}_n(\infty; \mathcal{F}_{NN}) \leq \max_x\|x\|_2 (2\sqrt{\log L} + 1) \prod_{\ell=1}^L M_F(\ell)/\sqrt{n}$.
\end{lemma}

\begin{lemma}[Corollary 14.3 in \citet{WainwrightMartinJ2019HS:A}]\label{lem:critical_basic}
        Let $\Ncal(\tau;\Gcal^{\mid \eta},\|\cdot\|_n)$ denote the $\tau$-covering number of $\Gcal^{\mid \eta}= \{g \in \Gcal \mid \|g\|_n\leq \delta\}$ under the $\|\cdot\|_n$ norm, where $\Gcal$ is a uniformly bounded function class. Then, the empirical version of critical inequality $\mathcal{R}_n\prns{\Gcal^{\mid \eta}} \le \eta^2/\|\Gcal\|_\infty$ is satisfied for any $\eta>0$ such that
        \begin{align*}
            \frac{1}{\sqrt{n}}\int^{\eta}_{\eta^2/(2\|\Gcal\|_{\infty})}\sqrt{\log  \Ncal(t,\Gcal^{\mid \eta},\|\cdot\|_n)}\rd t\leq \frac{\eta^2}{\|\Gcal\|_{\infty}}.
        \end{align*}
\end{lemma}

\begin{lemma}[Theorem 14.1 in \citet{WainwrightMartinJ2019HS:A}]\label{lem:support1}
Given a star-shaped and $b$-uniformly bounded function class $\Gcal$, let $\eta_n$ be any positive solution of the inequality $\mathcal{R}_n(\Gcal^{\mid \eta})\leq \eta^2/b$. Then, for any $t\geq \eta_n$, we have
\begin{align*}
    \abs{\|g\|^2_n-\|g\|^2_2} \leq \frac{1}{2}\|g\|^2_2+\frac{1}{2} t^2, ~~ \forall g\in \Gcal,
\end{align*}
with probability at least $1 - c_1 \exp\prns{-c_2\frac{nt^2}{b^2}}$.
\end{lemma}

\begin{lemma}[Lemma 7 in  \citep{FosterDylanJ.2019OSL}]\label{lem:support2}
Consider a uniformly bounded function class $\Fcal$ and pick an arbitrary $f^{*}\in \Fcal$. Let $\eta_n$ be solution to
\begin{align*}
   \Rad(\mathrm{star}^{\mid \eta}(\Fcal-f^*))\leq \frac{\eta^2}{\|\Fcal\|_\infty},  \text{ where } \mathrm{star}^{\mid \eta}(\Fcal-f^*) \coloneqq  \braces{\alpha\prns{f-f^*}: f\in\Fcal, \alpha\in[0, 1], \|\alpha\prns{f-f^*}\|_n \le \eta}.
\end{align*}
Moreover, assume that the loss $l(\cdot,\cdot)$ is $L$-Lipschitz in the first argument. Then, for $\tilde \eta_n=\eta_n+\sqrt{c_0\log(c_1/\delta)/n}$ with some universal constants $c_0,c_1$, the following holds with probability at least $1-\delta$,
\begin{align*}
     |(\E_n[l(f(x),z)]-\E_n[l(f^{*}(x),z)])-(\E[l(f(x),z)]-\E[l(f^{*}(x),z)])|\lesssim L \tilde\eta_n (\|f-f^{*}\|_2+\tilde\eta_n).
\end{align*}
\end{lemma}

\begin{lemma}[Neural network; Lemma 21 in \cite{JMLR:v21:20-002}] \label{lem:covering_neural}
    Let $\Fcal_{NN}$ be a set of functions by a neural network with $L$ layers, $\Omega$ weights in $[-B,B]$, and $1-$Lipschitz continuous activation function.
    Then, for $\tau \in (0,1]$, we have
    \begin{align*}
        \log \Ncal(\tau, \Fcal_{NN}, \|\cdot\|_\infty) \leq \Omega \log \left( \frac{ 2LB^L (\Omega+1)^L}{\tau} \right).
    \end{align*}
\end{lemma}

\begin{lemma}\label{lemma: identification-3}
Suppose that \Cref{asm:whole_assm} holds.
\begin{enumerate}
\item If $\Qbbb_0 \ne \emptyset$ and $\Hbbb_0^{\obs} \ne \emptyset$, then for any $h_0 \in \Hbbb^{\obs}_0$ and $q \in L_2(Z, A, X)$,
\begin{align*}
\Eb{\tilde\phi_{\ipw}\prns{O; q}} - J = \Eb{h_0(W, A, X)\Eb{\epol(A\mid X)\prns{q(Z, A, X) - 1/{f(A \mid W, X)}} \mid W, A, X}}.
\end{align*}
\item If $\Hbbb_0 \ne \emptyset$ and $\Qbbb_0^{\obs} \ne \emptyset$, then for any $q_0 \in \Qbbb^{\obs}_0$ and $h \in L_2(W, A, X)$,
\begin{align*}
\Eb{\tilde\phi_{\DM}\prns{O; h}} - J = \Eb{\epol(A\mid X)q_0(Z, A, X)\Eb{h(W, A, X) - Y \mid Z, A, X}}.
\end{align*}
\end{enumerate}
\end{lemma}

\begin{proof}[Proof for \Cref{lemma: identification-3}]
Before proving the conclusions, note that for any $h \in L_2(W, A, X)$ and $q_0 \in \Qbbb_0^{\obs}$,
\begin{align}\label{eq: T-operator}
    \E\bracks{q_0(Z,A,X)\epol(A \mid X)h(W, A,X)}
        &=\E\bracks{\E\bracks{\Eb{q_0(Z,A,X)\epol(A|X)\mid W, A, X}h(W, A, X)}} \nonumber \\
        &= \Eb{\frac{\epol(A\mid X)}{f(A \mid W, X)}h(W, A, X)} \nonumber \\
        &= \Eb{\Eb{\frac{\epol(A\mid X)}{f(A \mid W, X)}h(W, A, X) \mid W, X}} \nonumber \\
        &= \Eb{\int \frac{\epol(a\mid X)}{f(a \mid W, X)}h(W, a, X)f(a \mid W, X) \rd\mu(a)} \nonumber \\
        &= \Eb{(\mathcal T h)(W, X)}.
\end{align}
Here the second equality follows because $q_0$ satisfies \Cref{eq: observed-bridge-q}.

\textbf{IPW.} First, by taking some element $h_0$ in $\Hbbb^{\obs}_0$, we have
\begin{align*}
    \Eb{\tilde\phi_\ipw(O;q)}
        &= \Eb{\epol(A\mid X)q(Z, A, X)Y}  \\
        &= \Eb{\epol(A\mid X)q(Z, A, X)\Eb{Y\mid Z,A,X}} \\
        &= \Eb{\epol(A\mid X)q(Z, A, X)h_0(W, A, X)} \\
        &= \Eb{\Eb{\epol(A\mid X)q(Z, A, X)\mid W, A, X}h_0(W, A, X)}.
\end{align*}
Moreover, by taking some element $q_0'\in \Qbbb_0$,
\begin{align*}
    J   &=  \Eb{\epol(A\mid X)q_0'(Z, A, X)Y }  \tag{\Cref{lem:bridge-identification}}\\
     &=  \Eb{\epol(A\mid X)q_0'(Z, A, X)\E[Y|A,X,W] }  \tag{$Y \perp Z \mid A, X,W$}\\
        &= \Eb{\Eb{\epol(A\mid X)q_0'(Z, A, X)\mid W, A, X}h_0(W, A, X)} \\
        &= \Eb{\frac{\epol(A\mid X)}{f(A \mid W, X)}h_0(W, A, X)}.
\end{align*}

It follows that
\begin{align*}
    \Eb{\tilde\phi_\ipw(O;q)} - J =\Eb{\{\Eb{\epol(A\mid X)q(Z, A, X)\mid W, A, X} - \frac{\epol(A\mid X)}{f(A \mid W, X)}\}h_0(W, A, X)}.
\end{align*}

\textbf{REG.}
For any $q_0\in \Qbbb^{\obs}_0$, note that
\begin{align*}
    \Eb{\tilde\phi_\DM(O;h)}
        &= \Eb{(\mathcal T h)(W, X)} \\
        &= \Eb{\epol(A\mid X)q_0(Z, A, X)h(W, A, X)} \tag{\Cref{eq: T-operator}} \\
        &= \Eb{\epol(A\mid X)q_0(Z, A, X)\Eb{h(W, A, X) \mid Z, A, X}}.
\end{align*}
Moreover, by taking some element $h_0'\in \Hbbb_0$,
\begin{align*}
    J
        &= \Eb{(\mathcal T h_0')(W, X)} \tag{\Cref{lem:bridge-identification}} \\
        &= \Eb{\epol(A\mid X)q_0(Z, A, X)h_0'(W, A, X)} \tag{\Cref{eq: T-operator}}\\
        &= \Eb{\epol(A\mid X)q_0(Z, A, X)\Eb{h_0'(W, A, X) \mid Z, A, X}} \\
        &= \Eb{\epol(A\mid X)q_0(Z, A, X)\Eb{Y \mid Z, A, X}}.
\end{align*}
It follows that
\[
\Eb{\tilde\phi_{\DM}\prns{O; h}} - J = \Eb{\epol(A\mid X)q_0(Z, A, X)\Eb{h(W, A, X) - Y \mid Z, A, X}}.
\]
\end{proof}

\subsection{Proofs for \Cref{sec: setup}}\label{sec: proof-setup}
\begin{proof}[Proof of \Cref{lem: U-identification}]
We first prove the identification via $\phi_{\DM}$:
\begin{align*}
    J &=\E\braces{\int Y(a)\epol(a|X)\rd\mu(a) } \\
    &=\E\braces{\int \E[Y(a)|U,X]\epol(a|X)\rd\mu(a) } \tag{Tower property}\\
    &=\E\braces{\int \E[Y(a)|U,X,A=a]\epol(a|X)\rd\mu(a) } \tag{$Y(a)\perp A|U,X$}\\
    &=\E\braces{\int \E[Y|U,X,A=a]\epol(a|X)\rd\mu(a) } = \Eb{\phi_{\DM}\prns{Y, A, U, X; k_0}} \tag{Consistency}.
\end{align*}
We next prove the identification via $\phi_{\ipw}$:
\begin{align*}
    J&=\E\braces{\int \E[Y|U,X,A=a]\epol(a|X)\rd\mu(a) }\\
    &=\E\braces{\int \E[Y|U,X,A=a]\{\epol(a|X)/f(a|X,U)\}\times f(a|X,U) \rd\mu(a) }\\
    &=\E\braces{\E[Y|U,X,A]\epol(A|X)/f(A|X,U) }\\
    &=\E\braces{Y\epol(A|X)/f(A|X,U) } = \Eb{\phi_{\ipw}\prns{Y, A, U, X; f}}.
\end{align*}
Finally, the identification via $\phi_{\DR}$ follows from the fact that
\begin{align*}
    \Eb{\phi_{\DR}\prns{Y, A, U, X; k_0, f}} = \Eb{\phi_{\DM}\prns{Y, A, U, X; k_0}} = J.
\end{align*}
\end{proof}

\begin{proof}[Proof for \Cref{lem:bridge-identification}]
We derive the equations in \Cref{lem:bridge-identification} one by one.

First,
\begin{align*}
    \Eb{\tilde\phi_\ipw(O;q_0)}
        &= \Eb{\Eb{\epol(A\mid X)q_0(Z, A, X)Y \mid A, U, X}} \\
        &= \Eb{\Eb{\epol(A\mid X)q_0(Z, A, X) \mid A, U, X}\Eb{Y \mid A, U, X}} \\
        &= \Eb{\Eb{\epol(A\mid X)/f(A\mid U, X) \mid A, U, X}\Eb{Y \mid A, U, X}} \\
        &= \Eb{\epol(A\mid X)Y/f(A\mid U, X)} \\
        &= \Eb{\int \epol\prns{a\mid X}\Eb{Y(a) \mid A = a, U, X}\rd\mu(a)}\\
        &= \Eb{\int \epol\prns{a\mid X}\Eb{Y(a) \mid  U, X}\rd\mu(a)} = J.
\end{align*}
Here the second equality follows from $Y \perp Z \mid A, U, X$ and the third equality follows from the definition of $q_0$ according to \Cref{eq: bridge-U-q}.

Second,
\begin{align*}
    \Eb{\tilde\phi_\DM(O;h_0)}
        &= \Eb{\int \epol(a\mid X)h_0(W, a, X)\rd\mu(a)} \\
        &= \Eb{\int \epol(a\mid X)\Eb{h_0(W, a, X)\mid A = a, U, X}\rd\mu(a)} \\
         &=\Eb{\int \epol(a\mid X)\Eb{Y\mid A = a, U, X}\rd\mu(a)} \\
        &= J.
\end{align*}
Here the second equality follows from $W \perp A \mid U, X$ and the third equality follows from the definition of $h_0$ according to \Cref{eq: bridge-U-h}.

Finally,
\begin{align*}
    \Eb{\tilde\phi_\dr(O;h_0,q_0)}
    &= \Eb{\tilde\phi_\DM(O;h_0,q_0)} + \Eb{{\epol(A\mid X)q_0(Z, A, X)\prns{Y - h_0(W, A, X)}}} \\
    &= J + \Eb{\Eb{{\epol(A\mid X)q_0(Z, A, X)\prns{Y - h_0(W, A, X)}}\mid A, U, X}} \\
    &= J + \Eb{\Eb{{\epol(A\mid X)q_0(Z, A, X)\mid A, U, X}}\Eb{{Y -  h_0(W, A, X)}\mid A, U, X}} \\
    &= J.
\end{align*}
Here the third equality follows from that $Z\perp Y \mid A, U, X$ and the last equality follows from the definition of $h_0$ according to \Cref{eq: bridge-U-h}.
\end{proof}

\begin{proof}[Proof for \Cref{lemma: observed-bridge}]
For any $h_0 \in \Hbbb_0$, we have that
\[
\Eb{Y-h_0(W, A, X) \mid Z, A, U, X} = \Eb{Y-h_0(W, A, X) \mid A, U, X} = 0,
\]
where the first equality holds because $\prns{W, Y} \perp Z \mid A, U, X$. Therefore,
\[
\E[Y-h_0(W, A,X) \mid Z, A, X]=\Eb{\Eb{Y-h_0(W, A, X) \mid Z, A, U, X}\mid Z, A, X} = 0.
\]

For any $q_0 \in\Qbbb_0$, we have that
\[
\Eb{\epol(A\mid X){q_0(Z, A, X) } \mid W, A, U, X} = \Eb{\epol(A\mid X){q_0(Z, A, X) } \mid A, U, X} = \frac{\epol(A\mid X)}{f(A\mid U, X)},
\]
where the first equality follows from $Z\perp W \mid A, U, X$. Therefore,
\begin{align*}
    \Eb{\epol(A\mid X){q_0(Z, A, X) } \mid W, A, X}
        &= \Eb{\Eb{\epol(A\mid X){q_0(Z, A, X) } \mid W, A, U, X}\mid W, A, X}\\
        &= \Eb{\frac{\epol(A\mid X)}{f(A\mid U, X)} \mid W, A, X}.
\end{align*}
Then the conclusion follows from the fact that
\begin{align*}
    \Eb{\frac{1}{f(A\mid U, X)} \mid W, A, X}
        &= \int \frac{1}{f(A\mid u, X)}f(u \mid W, A, X)\rd(u) \\
        &= \int \frac{f(W, A\mid u, X)f(u, X)}{f(A\mid u, X)f(W, A, X)}\rd\mu(u) \\
        &= \int \frac{f(W \mid u, X)f(u, X)}{f(W, A, X)}\rd\mu(u) \\
        &= \frac{1}{f(A\mid W, X)}.
\end{align*}
Here the third equality follows from the fact that $W \perp A \mid U, X$.

In summary, the condition that $\prns{W, Y} \perp Z \mid A, U, X$ ensures \Cref{eq: observed-bridge-h} holds, while the condition that $\prns{Z, A} \perp W \mid U, X$ ensures \Cref{eq: observed-bridge-q} holds.
\end{proof}

\begin{proof}[Proof for \Cref{lemma: relaxed-ipw-dm}]
Obviously, the conclusions in statements \ref{lemma: relaxed-ipw-dm-1} and \ref{lemma: relaxed-ipw-dm-2} directly follow from \Cref{lemma: identification-3}. We only need to prove the statement \ref{lemma: relaxed-ipw-dm-3}.

When the condition in statement \eqref{lemma: relaxed-ipw-dm-1} holds, for any $h_0 \in\Hbbb_0^{\obs}$ and $q_0  \in \Qbbb_0^{\obs}$,
\begin{align*}
    \Eb{\tilde\phi_\DR(O;h_0,q_0)}
        &= \Eb{\epol(A\mid X)q_0(Z, A, X)\prns{Y - h_0(W, A, X)}} + \Eb{(\Tcal h_0)(X,W)} \\
        &= \Eb{\epol(A\mid X)q_0(Z, A, X)\Eb{Y - h_0(W, A, X) \mid Z, A, X}} + \Eb{\tilde\phi_\DM(O;h_0)}   \\
        &= \Eb{\tilde\phi_\DM(O; h_0)}  = J.
\end{align*}
When the condition in statement \eqref{lemma: relaxed-ipw-dm-2} holds, for any $h_0 \in\Hbbb_0^{\obs}$ and $q_0  \in \Qbbb_0^{\obs}$,
\begin{align*}
\Eb{\tilde\phi_\DR(O;h_0,q_0)}
    &= \Eb{\tilde\phi_\ipw(O; q_0)} + \Eb{(\Tcal h_0)(X,W)} - \Eb{\epol(A\mid X)q_0(Z, A, X){h_0(W, A, X)}} \\
    &= \Eb{\tilde\phi_\ipw(O; q_0)} = J.
\end{align*}
Here the second equality follows from \Cref{eq: T-operator}.
\end{proof}

\subsection{Proofs for \Cref{sec:identity}}
\begin{lemma}[Bridge Functions in Linear Models]\label{lemma: bridge}
Consider the linear models in \Cref{ex: bridge-linear}, where $\alpha_Z \in \R{p_z\times p_u}, \alpha_W \in \R{p_w \times p_u}$ both have full column rank. Then bridge functions $h_0$ and $q_0$ have the following forms:
\begin{align*}
    &h_0(W, A, X) = \prns{\theta_W + \omega_Y}^\top W + \prns{\beta_Y - \theta_W^\top \beta_W}X + \gamma_Y A, ~~ \forall \theta_W \text{ s.t. } \theta^\top_W \alpha_W = \alpha_Y^\top, \\
    &q_0(Z, A, X) = \theta_Z^\top Z + \prns{\overline{\beta}_A - \underline{\beta}_A - \theta_Z^\top \beta_Z}X -\theta_Z^\top \gamma_Z A, ~~ \forall \theta_Z \text{ s.t. } \theta_Z^\top \alpha_Z = \overline{\alpha}_A^\top - \underline{\alpha}_A^\top.
\end{align*}
\end{lemma}
\begin{proof}[Proof for \Cref{lemma: bridge}]
Here we show the details of deriving bridge functions in \Cref{ex: bridge-linear}.

We first derive the function $h_0(W, A, X)$. If $\theta_w^\top\alpha_W = \alpha_Y^\top$,then
\[
\theta_W^\top W = \alpha_Y^\top U + \theta_W^\top \beta_W X + \theta_W^\top\epsilon_W \implies \alpha_Y^\top U = \theta_W^\top W - \theta_W^\top \beta_W X - \theta_W^\top\epsilon_W.
\]
Therefore,
\begin{align*}
 Y
    &= (\theta_W + \omega_Y)^\top W + \prns{\beta_Y^\top - \theta_W^\top\beta_W}X + \omega_W^\top W + \epsilon_Y - \theta_W^\top\epsilon_W \\
    &= h_w(W, A, X) + \epsilon_Y - \theta_W^\top\epsilon_W.
\end{align*}
It follows from the independence of $\epsilon_Y,\epsilon_W$ with $A, U, X$ that
\[
\Eb{Y\mid A, U, X} = \Eb{h_0(W, A, X) \mid A, U, X} + \Eb{\epsilon_Y - \theta_W^\top\epsilon_W} = \Eb{h_0(W, A, X) \mid A, U, X}.
\]

Now we derive $q_0(Z, A, X)$. First note that $1/f(A\mid U, X) = 1/\prns{\prns{\overline\alpha_A - \underline\alpha_A}U + \prns{\overline\beta_A - \underline\beta_A}X}$.
Because $\theta_Z^\top\alpha_Z = \overline\alpha_A^\top - \underline\alpha_A^\top$,
\begin{align*}
    \theta_Z^\top Z = \prns{\overline\alpha_A - \underline\alpha_A}^\top U + \theta_Z^\top \beta_Z X + \theta_Z^\top \gamma_Z A + \theta_Z^\top\epsilon_Z,
\end{align*}
which means that \begin{align*}
q_0(Z, A, X)
    &= \theta_Z^\top Z + \prns{\overline{\beta}_A - \underline{\beta}_A - \theta_Z^\top \beta_Z}X -\theta_Z^\top \gamma_Z A \\
    &= \prns{\overline\alpha_A - \underline\alpha_A}^\top U + \prns{\overline\beta_A - \underline\beta_A}^\top X + \theta_Z^\top\epsilon_Z.
\end{align*}
Therefore,
\[
\Eb{q_0(Z, A, X) \mid A, U, X} = \prns{\overline\alpha_A - \underline\alpha_A}^\top U + \prns{\overline\beta_A - \underline\beta_A}^\top X  + \Eb{\theta_Z^\top\epsilon_Z} = 1/f(A\mid U, X).
\]
\end{proof}

\subsection{Proofs for \Cref{sec: realizable}}
\subsubsection{Proofs for \Cref{seq: error-gace-nostab}}

\begin{proof}[Proof of \Cref{thm:ipw_reg_no_stab}]
We start by proving \Cref{eq:ipw_no_stab}.
We fix functions $h_0\in \Hbbb^{\obs}_0,q_0\in \Qbbb^{\obs}_0$. Recall that
\begin{align*}
    \hat q
        &\in  \argmin_{q\in \Qbbb}\max_{h\in \Hbbb'}~~\prns{\E_n{[\epol(A|X)q(Z, A, X)h(W,A,X) - (\Tcal h)(W, X)}]}^2 = \argmin_{q\in \Qbbb}\max_{h\in \Hbbb'} ~~\abs{\E_n{f_q\prns{q, h}}},
\end{align*}
where
\begin{align*}
    f_q(q,h) \coloneqq -\pi q h +\Tcal h.
\end{align*}

\textit{First Step.}
We first define $J_{\ipw}\coloneqq \E[\hat q\epol Y]$. Obviously,
\begin{align*}
    |  J_{\ipw}-\hat  J_{\ipw}|&\leq \sup_{q\in \Qbbb}|(\E-\E_n)[q\epol Y]|.
\end{align*}
In addition, we have
\begin{align*}
 |J_{\ipw}-J|&\leq |\E[(\hat q-q_0)\epol Y]|= |\E[\E[(\hat q-q_0)\epol Y|A,X,U]]|\\
 &=|\E[(\hat q-q_0)\epol h_0]| \tag{$Z\perp Y \mid A, X,U$}\\
  &\leq \sup_{h\in \Hbbb'}|\E[(\hat q-q_0)\epol h]|+\inf_{h\in \Hbbb'}|\E[(\hat q-q_0)\epol(h_0-h)]| \\
  &\le |\Eb{f_q(\hat q, \overline h)}|+\sup_{q\in \Qbbb}\inf_{h\in \Hbbb'}|\E[(q-q_0)\epol(h_0-h)]| , \qquad\qquad\qquad  \bar h \in  \argmax_{h\in \Hbbb'}|\E[(\hat q-q_0)\epol h]|\\
   &\leq |\Eb{f_q(\hat q, \overline h)}| +\inf_{h\in \Hbbb'}\sup_{q\in \Qbbb}|\E[(q-q_0)\epol(h_0-h)]|.
\end{align*}
\textit{Second Step.}
Define $q'=\argmin_{q\in \Qbbb}\sup_{h\in \Hbbb'}|\E[f_q(q,h)]|$. Then
\begin{align*}
|\Eb{f_q(\hat q, \overline h)}|
    &\le |\mathbb{E}_n\bracks{-\pi \hat q \overline h +\Tcal \overline h}| + |\prns{\mathbb{E}_n - \mathbb{E}}\bracks{-\pi \hat q \overline h +\Tcal \overline h}|   \\
    &\le \sup_{h \in \Hbbb'}|\mathbb{E}_n\bracks{-\pi \hat q  h +\Tcal  h}| + |\prns{\mathbb{E}_n - \mathbb{E}}\bracks{-\pi \hat q \overline h +\Tcal \overline h}| \\
    &\le \sup_{h \in \Hbbb'}|\mathbb{E}_n\bracks{-\pi q'  h +\Tcal  h}| + |\prns{\mathbb{E}_n - \mathbb{E}}\bracks{-\pi \hat q \overline h +\Tcal \overline h}| \tag{Definition of $\hat q$}\\
    &\le \sup_{h \in \Hbbb'}|\mathbb{E}\bracks{-\pi  q'  h +\Tcal  h}| + \sup_{h \in \Hbbb'}|\prns{\mathbb{E}_n - \mathbb{E}}\bracks{-\pi  q'  h +\Tcal  h}|  + |\prns{\mathbb{E}_n - \mathbb{E}}\bracks{-\pi \hat q \overline h +\Tcal \overline h}| \\
    &\le \sup_{h \in \Hbbb'}|\mathbb{E}\bracks{-\pi  q'  h +\Tcal  h}| + 2\sup_{q\in\Qbbb, h \in \Hbbb'}|\prns{\mathbb{E}_n - \mathbb{E}}\bracks{-\pi  q  h +\Tcal  h}| \\
    &= \inf_{q\in\Qbbb}\sup_{h \in \Hbbb'}|\mathbb{E}\bracks{-\pi  q  h +\Tcal  h}| + 2\sup_{q\in\Qbbb, h \in \Hbbb'}|\prns{\mathbb{E}_n - \mathbb{E}}\bracks{-\pi  q  h +\Tcal  h}| \tag{Definition of $q'$}\\
    &= \inf_{q\in\Qbbb}\sup_{h \in \Hbbb'}|\mathbb{E}\bracks{\pi \prns{q_0- q}  h}| + 2\sup_{q\in\Qbbb, h \in \Hbbb'}|\prns{\mathbb{E}_n - \mathbb{E}}\bracks{-\pi  q  h +\Tcal  h}|
\end{align*}

\textit{Third Step.}
Combining all results,
\begin{align*}
    |J-\hat  J_{\ipw}|&\leq  |J_{\ipw}-\hat  J_{\ipw}|+|J_{\ipw}-J|\\
    &\leq \sup_{q\in \Qbbb}|(\E-\E_n)[q\epol Y]|+|\E[f_q(\hat q,\bar h)]|+\inf_{h\in \Hbbb'} \sup_{q\in \Qbbb}|\E[(q-q_0)\epol(h_0-h)]|\\
    &\le \sup_{q\in \Qbbb}|(\E-\E_n)[q\epol Y]| + \inf_{q\in\Qbbb}\sup_{h \in \Hbbb'}|\mathbb{E}\bracks{\pi \prns{q_0- q}  h}|  \\
    &+ 2\sup_{q\in\Qbbb, h \in \Hbbb'}|\prns{\mathbb{E}_n - \mathbb{E}}\bracks{-\pi  q  h +\Tcal  h}| + \inf_{h\in \Hbbb'} \sup_{q\in \Qbbb}|\E[(q-q_0)\epol(h_0-h)]|.
\end{align*}

Next, we turn to proving \Cref{eq:reg_no_stab}.
We fix functions $h_0\in \Hbbb^{\obs}_0,q_0\in \Qbbb^{\obs}_0$. Recall that
\begin{align*}
    \hat h
    &\in \argmin_{h\in \Hbbb}\sup_{q\in \Qbbb'}~~\prns{\E_n{[q(Z, A, X)\prns{ h(W, A, X)-Y}}]}^2 = \argmin_{h\in \Hbbb}\sup_{q\in \Qbbb'}~~\abs{\E_n{f_h\prns{q, h}}},
\end{align*}
where $$f_h(q,h) \coloneqq q \prns{Y - h}.$$

\textit{First step }
We first define $J_{\DM}\coloneqq \E[\Tcal h]$.
Then,
\begin{align*}
     |J_{\DM}-\hat J_{\DM}|\leq \sup_{h\in \Hbbb}(\E-\E_n)|[\Tcal h ]|.
\end{align*}
In addition,
\begin{align*}
    |J_{\DM}-J|&\leq |\E[\Tcal (\hat h-h_0)|=|\E[q_0\epol(\hat h-h_0)|\\
    &\leq  \sup_{q\in \Qbbb'}|\E[q (\hat h-h_0)|+\inf_{q\in \Qbbb'}|\E[(q_0\epol-q) (\hat h-h_0)|\\
    &\leq  |\E[\bar q (\hat h-h_0)|+\sup_{h\in \Hbbb}\inf_{q\in \Qbbb'}|\E[(q_0\epol-q) ( h-h_0)|,\qquad\qquad \bar q \in \argmax_{q\in \Qbbb'}|\E[q (\hat h-h_0)|\\
      &\leq  |\E[\bar q (\hat h-h_0)|+\inf_{q\in \Qbbb'}\sup_{h\in \Hbbb}|\E[(q_0\epol-q) ( h-h_0)|.
\end{align*}

\textit{Second step.}
{
Define $h'=\argmin_{h\in \Hbbb}\sup_{q\in \Qbbb'}|\E[f_h(q,h)]|$.
Then,
\begin{align*}
   & |\E[\bar q (\hat h-h_0)|=  |\E[\bar q (\hat h-Y)|\\
   &\leq |(\E-\E_n)[\bar q (\hat h-Y)|+|\E_n[\bar q (\hat h-Y)]|\\
   &\leq |(\E-\E_n)[\bar q (\hat h-Y)|+\sup_{q\in \Qbbb'}|\E_n[q (\hat h-Y)]|\\
    &\leq |(\E-\E_n)[\bar q (\hat h-Y)|+\sup_{q\in \Qbbb'}|\E_n[q (h'-Y)]| \tag{Definition of $\hat h$}\\
  &\leq |(\E-\E_n)[\bar q (\hat h-Y)|+\sup_{q\in \Qbbb'}|(\E_n-\E)[q (h'-Y)]|+\sup_{q\in \Qbbb'}|\E[q (h'-Y)]|\\
   &\leq2\sup_{q\in \Qbbb',h\in \Hbbb}|(\E_n-\E)[q(Y - h)]|+\inf_{h\in \Hbbb}\sup_{q\in \Qbbb'}|\E[q(h_0 - h)]|. \tag{Definition of $h'$}
\end{align*}

}

\textit{Third step.}

Combining all results,
\begin{align*}
    |J-\hat J_{\DM}|&\leq     |J_{\DM}-\hat J_{\DM}|+ |J_{\DM}-J|\\
    &\leq \sup_{h\in \Hbbb}|(\E-\E_n)[\Tcal h ]|+|\E[\bar q (\hat h-h_0)|+\inf_{q\in \Qbbb'}\sup_{h\in \Hbbb}|\E[(q_0\epol-q) (h-h_0)|\\
     &\leq \sup_{h\in \Hbbb}|(\E-\E_n)[\Tcal h ]|+2\sup_{q\in \Qbbb',h\in \Hbbb}|(\E-\E_n)[ q(Y-h) ]  |\\
    &+\inf_{h\in \Hbbb}\sup_{q\in \Qbbb'}|\E[\{h_0-h\}q]|+\inf_{q\in \Qbbb'}\sup_{h\in \Hbbb}|\E[(q_0\epol-q) (h-h_0)|.
\end{align*}

\end{proof}

\begin{proof}[Proof of \Cref{thm:dr_no_stab}]
We fix functions $h_0\in \Hbbb^{\obs}_0,q_0\in \Qbbb^{\obs}_0$.
~

\textit{First Statement.}
First, we prove
\begin{align*}
     | \hat J_{\DR}-J| &\leq \sup_{q\in \Qbbb,h\in \Hbbb}|(\E-\E_n)[q\epol Y-q\epol h+\Tcal h]|+2\sup_{q\in \Qbbb,h\in \Hbbb'}|(\E-\E_n) [-q\epol h+\Tcal h]|+\\
      &+\inf_{q\in \Qbbb} \sup_{h\in \Hbbb'}|\E[(q_0-q)\epol h ]|+\inf_{h\in \Hbbb'}\sup_{h'\in \Hbbb} \sup_{q\in \Qbbb}|\E[(q-q_0)\epol(h_0-h'-h)]|.
\end{align*}

We define
\begin{align*}
 J_{\dr}\coloneqq \E[\hat q\epol \{Y-\hat h\}+\Tcal \hat h].
\end{align*}
Then,
\begin{align*}
    |  J_{\dr}-\hat  J_{\dr}|&\leq \sup_{q\in \Qbbb,h\in \Hbbb}|(\E-\E_n)[q\epol \{Y-h\}+\Tcal h]|.
\end{align*}
In addition,
\begin{align*}
 |J_{\dr}-J|&\leq |\E[(\hat q-q_0)\epol (h_0-\hat h)]|\\
  &\leq \sup_{h\in \Hbbb'}|\E[(\hat q-q_0)\epol h]|+\inf_{h\in \Hbbb'}|\E[(\hat q-q_0)\epol(h_0-\hat h-h)]| \\
  &=|\E[(\hat q-q_0)\pi \bar h]|+\sup_{h'\in \Hbbb}\sup_{q\in \Qbbb}\inf_{h\in \Hbbb'}|\E[(q-q_0)\epol(h_0-h'-h)]| , \qquad\qquad \bar h\coloneqq \argmax_{h\in \Hbbb'}|\E[(\hat q-q_0)\epol h]|\\
    &=|\E[(\hat q-q_0)\pi \bar h]|+\inf_{h\in \Hbbb'}\sup_{h'\in \Hbbb,q\in \Qbbb}|\E[(q-q_0)\epol(h_0-h'-h)]| \\
    &= |\E[f_q\prns{\hat q, \overline h}]|+\inf_{h\in \Hbbb'}\sup_{h'\in \Hbbb,q\in \Qbbb}|\E[(q-q_0)\epol(h_0-h'-h)]|
\end{align*}
The rest of the proof is the same as that of \Cref{eq:ipw_no_stab} in \Cref{thm:ipw_reg_no_stab}.

\textit{Second Statement.}
Second, we prove
\begin{align*}
        |\hat J_{\DR}-J|&\leq  \sup_{q\in \Qbbb,h\in \Hbbb}|(\E-\E_n)[q\epol Y-q\epol h+\Tcal h ]|+2\sup_{q\in \Qbbb',h\in \Hbbb}|(\E-\E_n)[q(Y-h) ]  |\\
    &+\inf_{h\in \Hbbb}\sup_{q\in \Qbbb'}|\E[\{h_0-h\}q]|+\inf_{q\in \Qbbb'}\sup_{h\in \Hbbb,q'\in \Qbbb}|\E[(q_0\epol-q'-q) ( h-h_0)|.
\end{align*}

We define
\begin{align*}
 J_{\dr}\coloneqq \E[\hat q\epol \{Y-\hat h\}+\Tcal \hat h].
\end{align*}
Then,
\begin{align*}
    |  J_{\dr}-\hat  J_{\dr}|&\leq \sup_{q\in \Qbbb,h\in \Hbbb}|(\E-\E_n)[q\epol \{Y-h\}+\Tcal h]|.
\end{align*}
In addition,
\begin{align*}
 |J_{\dr}-J|&\leq |\E[(\hat q-q_0)\epol (h_0-\hat h)]|\\
    &\leq  \sup_{q\in \Qbbb'}|\E[q (\hat h-h_0)|+\inf_{q\in \Qbbb'}|\E[(q_0\epol-\hat q\epol- q) \hat h-h_0)|\\
         &\leq  |\E[\bar q (\hat h-h_0)]|+\sup_{h\in \Hbbb,q'\in \Qbbb}\inf_{q\in \Qbbb'}|\E[(q_0\epol-q'\epol-q) ( h-h_0)|,\qquad\qquad \bar q\coloneqq \argmax_{q\in \Qbbb'}|\E[q (\hat h-h_0)|\\
        &\leq  |\E[\bar q (\hat h-h_0)]|+\inf_{q\in \Qbbb'}\sup_{h\in \Hbbb,q'\in \Qbbb}|\E[(q_0\epol-q'\epol-q)( h-h_0)| \\
        &= |\E[f_h({\overline q, \hat h})]|+\inf_{q\in \Qbbb'}\sup_{h\in \Hbbb,q'\in \Qbbb}|\E[(q_0\epol-q'\epol-q)( h-h_0)|
\end{align*}
The rest of the proof is the same as that of \Cref{eq:reg_no_stab} in \Cref{thm:ipw_reg_no_stab}.
\end{proof}

\subsubsection{Proofs for \Cref{sec: rate_analysis}}

\begin{proof}[Proof of \Cref{cor: vc_class}]
Given that $\Hbbb'=\Hbbb$ and $\Qbbb' = \pi\Qbbb$, \Cref{cor: ipw_reg_no_stab} states that
\begin{align*}
   &| \hat J_{\ipw}-J| \leq \sup_{q\in \Qbbb}|(\E-\E_n)[q\epol Y]|+2\sup_{q\in \Qbbb,h\in \Hbbb}|(\E-\E_n) [-q\epol h+\Tcal h ]|, \\
  &|\hat J_{\mathrm{REG}}-J|\leq  \sup_{h\in \Hbbb} |(\E-\E_n)[\Tcal h ]|+2\sup_{q\in \Qbbb,h\in \Hbbb}|(\E-\E_n)[q\pi(Y-h) ]|.
\end{align*}
We define three function classes for the analysis:
\begin{align*}
    \Abbb_1 &=  \{(y, z, a, x)\mapsto q(z, a, x)\pi\prns{a \mid x}y:q\in \Qbbb\}, \\
     \Abbb_2 &= \{(y, w, z, a, x)\mapsto  q(z, a, x)\pi\prns{a \mid x}h(w, a, x): q\in \Qbbb, h \in \Hbbb  \}, \\
     \Abbb_3 &= \{(w, a, x)\mapsto (\Tcal h)(w, x) \in \Hbbb  \}.
\end{align*}
Then, from Theorem 4.10 in \citet{WainwrightMartinJ2019HS:A} and \Cref{cor: ipw_reg_no_stab},there exists a universal positive constant $c$ such that
\begin{align*}
    \max\braces{|\hat J_{\DM}-J|, |\hat J_{\ipw}-J|}\leq c\{\Rad(\infty;\Abbb_1)+\Rad(\infty;\Abbb_2)+\Rad(\infty;\Abbb_3) +\sqrt{\log(1/\delta)/n}\}.
\end{align*}
We first note that
\begin{align*}
    \Rad(\infty;\Abbb_2) \le \|\epol(A\mid X)\|_{\infty}  \Rad(\infty;\Hbbb\Qbbb)
\end{align*}
and
\begin{align*}
    \Rad(\infty;\Hbbb\Qbbb)& \leq  \Rad(\infty;0.25\braces{ (\Hbbb+\Qbbb)^2 - (\Hbbb-\Qbbb)^2 } ) \\
    & \leq 0.25\braces{\Rad(\infty;(\Hbbb-\Qbbb)^2)+ \Rad(\infty;(\Hbbb+\Qbbb)^2)}\\
      & \leq 0.5(\|\Hbbb\|_{\infty} + \|\Qbbb\|_{\infty})\{\Rad(\infty;(\Hbbb-\Qbbb))+ \Rad(\infty;(\Hbbb+\Qbbb))\} \tag{Contraction property \citep{mendelson2002improving}}\\
  & \leq (\|\Hbbb\|_{\infty} + \|\Qbbb\|_{\infty})\prns{\Rad(\infty;\Hbbb)+ \Rad(\infty;\Qbbb)}
\end{align*}
Thus
\begin{align*}
    \Rad(\infty;\Abbb_2) \le  \|\epol(A\mid X)\|_{\infty}  \prns{\|\Hbbb\|_{\infty} + \|\Qbbb\|_{\infty}}\prns{\Rad(\infty;\Hbbb)+ \Rad(\infty;\Qbbb)}
\end{align*}
Moreover, we have
\begin{align*}
     \Rad(\infty;\Abbb_1) \le \|\epol(A\mid X)\|_{\infty}\|Y\|_{\infty} \Rad(\infty;\Qbbb).
\end{align*}

From the Dudley integral (\Cref{lem:dudley}),  the covering number of the VC-subgraph class (\Cref{lem: vc}), and the boundedness of $\|\pi(A\mid X)\|_\infty$ and $\|Y\|_{\infty}$, we have
\begin{align*}
    \Rad(\infty;\Hbbb) = O\prns{\sqrt{V\prns{\Hbbb}/n}}, \Rad(\infty;\Qbbb) = O\prns{\sqrt{V\prns{\Qbbb}/n}}.
\end{align*}
In addition,
\begin{align*}
      \Rad(\infty;\Abbb_3)= O\prns{\sqrt{V\prns{\Abbb_3}/n}}= O\prns{\sqrt{V\prns{\Hbbb}/n}}.
\end{align*}
In the second equality, we use the proof of \citep[Corollary 9]{UeharaMasatoshi2021FSAo}
Therefore, we have
\begin{align*}
     \max\braces{|\hat J_{\DM}-J|, |\hat J_{\ipw}-J|} = O\prns{\sqrt{V( \Qbbb)/n}+ \sqrt{V(\Hbbb)/n}+\sqrt{\log(1/\delta)/n}}.
\end{align*}

\end{proof}

\begin{proof}[Proof of \Cref{cor: nonparametric}]
From the proof of \Cref{cor: vc_class}, we have
\begin{align*}
  |\hat J-J|\leq c ( \Rad(\infty;\Qbbb)+ \Rad(\infty;\Hbbb)+\Rad(\infty;\Tcal \Hbbb)+\sqrt{\log(1/\delta)/n}).
\end{align*}
for some universal constant $c$. According to Lemma 9 in \citet{UeharaMasatoshi2021FSAo}, we have
\begin{align*}
    \log \Ncal(\epsilon,\Tcal\Hbbb,\|\cdot\|_{\infty})\leq     \log \Ncal(\epsilon,\Hbbb,\|\cdot\|_{\infty}).
\end{align*}
Then, from \Cref{lem:dudley}, we have
\begin{align*}
      |\hat J-J|&\leq c\inf_{\tau\geq 0}\braces{\tau+\int_{\tau}^{\sup_{f\in \Qbbb}\sqrt{\E_n[f^2]}}\sqrt{\frac{\log \Ncal(\beta,\Qbbb,\|\cdot\|_n)}{n}}d \beta}  \\
       &+ c\inf_{\tau\geq 0}\braces{\tau+\int_{\tau}^{\sup_{f\in \Hbbb}\sqrt{\E_n[f^2]}}\sqrt{\frac{\log \Ncal(\beta,\Hbbb,\|\cdot\|_n)}{n}}d \beta}+c\sqrt{\log(1/\delta)/n}.
\end{align*}
By calculating Dudley integral for each term, the statement is concluded following the proof of Corollary 2 in \citet{UeharaMasatoshi2021FSAo}.
\end{proof}

\subsection{Proofs for \Cref{sec: unstab-closed}}

\begin{proof}[Proof of \Cref{thm:slow_rates}]

We use notation in the proof of \Cref{thm:ipw_reg_no_stab}. Then,
\begin{align*}
    \E[\{P_z(\hat h-h_0)\}^2]
        &\leq \sup_{q\in \Qbbb'}\abs{\E[P_z(\hat h-h_0)q]} \tag{$P_z(\Hbbb-h_0)\subseteq \Qbbb'$}\\
        &\leq \sup_{q\in \Qbbb'}\abs{\E[(\hat h-h_0)q]} \\
        &\leq 2 \sup_{q\in \Qbbb',h\in \Hbbb} |(\E-\E_n)[\{y-h\}q] |,
\end{align*}
where the last inequality follows from the second step in the proof for $\hat J_{\DM}$ in \Cref{thm:ipw_reg_no_stab}.
\begin{align*}
      \E[\{\pi P_w(\hat q-q_0)\}^2]&\leq \sup_{h\in \Hbbb'}\abs{\E[\pi P_w(\hat q-q_0)h]} \tag{$\pi P_w(\Qbbb-q_0)\subseteq \Hbbb'$}\\
      &\leq \sup_{h\in \Hbbb'}\abs{\E[\pi(\hat q-q_0)h]}\\
      &\leq 2 \sup_{q\in \Qbbb,h\in \Hbbb'} |(\E-\E_n)[-q\epol h+\Tcal h] |,
\end{align*}
where the last inequality follows from the second step in the proof for $\hat J_{\ipw}$ in  \Cref{thm:ipw_reg_no_stab}.
\end{proof}

\begin{proof}[Proof of \Cref{thm:ipw-dm}]
~
\textit{IPW estimators}
The error is decomposed into the three terms:
\begin{align}
    &|\hat J_{\ipw}-J| \nonumber \\
    =&  |\E_n[\hat q\epol y]-J| \nonumber \\
    \leq& |\{\E_n[(\hat q-q_0)\epol y]-\E[(\hat q-q_0)\epol y]\}|+ |\{\E[\hat q\epol y]-\E[q_0\epol y] \}|+|\E_n[q_0\epol y]-J| \nonumber \\
    \leq&  |\{\E_n[\hat q\epol y]-\E[\hat q\epol y]\}|+ |\{\E[\hat q\epol y]-\E[q_0\epol y] \}| +  |\{\E_n[q_0\epol y]-\E[q_0\epol y]\}| +|\E_n[q_0\epol y]-J|  \label{eq: ipw}
\end{align}
The first term in \eqref{eq: ipw} is upper bounded by
\begin{align*}
    \sup_{q\in \Qbbb}|(\E_n-\E)[q\epol y]|.
\end{align*}
The second term in \eqref{eq: ipw} is
\begin{align*}
    \E[\hat q\epol Y]-\E[q_0\epol Y]&=\E[\{(\hat q-q_0)\epol\}Y]=\E[\{(\hat q-q_0)\epol\}\E[Y \mid Z, A, X]]\\
    &=\E[\{(\hat q-q_0)\epol\}\E[h_0 \mid Z,A,X]] \tag{Use the assumption $\Hbbb_0\neq  \emptyset$}\\
    &=\E[\{(\hat q-q_0)\epol\}h_0]\\
    &= \E[P_w\{(\hat q-q_0)\epol\}h_0].
\end{align*}
Thus, from CS inequality, this term is upper-bounded by
\begin{align*}
    \|P_w\{(\hat q-q_0)\epol\}\|_2\|h_0\|_2.
\end{align*}
The third and fourth terms in \eqref{eq: ipw} are upper-bounded by Bernstein inequality. This concludes
\begin{align}
    |\hat J_{\mathrm{IPW}}-J|\leq  c_1 \Rad(\infty;\epol\Qbbb)+    \|P_w\{(\hat q-q_0)\epol\}\|_2\|h_0\|_2+ c_1\sqrt{\log(c_2/\delta)/n}.
\end{align}

\textit{REG estimators}

The error is decomposed into the three terms:
\begin{align}
   &|\hat J_{\mathrm{REG}}-J| \nonumber  \\
   \leq& |\{\E_n[\Tcal(\hat h-h_0)]-\E[\Tcal (\hat h-h_0)]\}|+|\{\E[\Tcal \hat h]-\E[\Tcal h_0]\}|+|\E_n[\Tcal  h_0]-J|  \nonumber \\
   \leq& |\{\E_n[\Tcal\hat h]-\E[\Tcal\hat h]\}|+|\{\E[\Tcal \hat h]-\E[\Tcal h_0]\}|+|\{\E_n[\Tcal h_0]-\E[\Tcal h_0]\}| + |\E_n[\Tcal  h_0]-J|   \label{eq: dm}
\end{align}
The first term in \Cref{eq: dm} is upper bounded by
\begin{align*}
|\sup_{h\in \Hbbb}(\E_n-\E)[\Tcal h]|.
\end{align*}
The second term in \Cref{eq: dm} is upper-bounded as follows:
\begin{align*}
    |\E[\Tcal \hat h]-\E[\Tcal h_0]|&=  |\E[\epol(A \mid X)/f(A\mid W,X)\{\hat h-h_0\}]|\\
    &=  |\E[\E[\epol(A \mid X)q_0(Z, A, X)\mid W, A, X]\{\hat h-h_0\}]| \tag{Use the assumption $\Qbbb_0\neq  \emptyset$}\\
    &=  |\E[\epol(A \mid X)q_0(Z, A, X)\{\hat h-h_0\}]|\\
    &=|\E[\epol(A \mid X)q_0(Z, A, X)P_z\{\hat h-h_0\}]|\leq \|q_0\epol\|_2\|P_z\{\hat h-h_0\}\|_2.
\end{align*}
The third and fourth terms in \Cref{eq: dm} are upper-bounded by Bernstein inequality. This concludes
\begin{align}
    |\hat J_{\mathrm{REG}}-J|\leq  c_1\Rad(\infty;\Tcal\Hbbb)+\|q_0\epol\|_2\|P_z\{\hat h-h_0\}\|_2+ c_1\sqrt{\log(c_2/\delta)/n}).
\end{align}
\end{proof}

\begin{proof}[Proof of \Cref{thm:drr}]
The proof is similar to \Cref{thm:ipw-dm}.
The error is decomposed into the three terms:
\begin{align}
   |\hat J_{\mathrm{DR}}-J|&\leq \abs{(\E_n-\E)\bracks{\{\epol \hat q\{Y-\hat h\}+\Tcal \hat h\}-\{\epol q_0\{Y-h_0\}+\Tcal h_0\}}} \nonumber \\
   &+|\{\E[\{\epol \hat q\{Y-\hat h\}+\Tcal \hat h\}]-\E[\{\epol q_0\{Y-h_0\}+\Tcal h_0\}]\}| \nonumber \\
   &+|\E_n[\{\epol q_0\{Y-h_0\}+\Tcal h_0\}]-J| \nonumber
\end{align}
It follows that
\begin{align}
    |\hat J_{\mathrm{DR}}-J|&\leq \abs{(\E_n-\E)\bracks{\{\epol \hat q\{Y-\hat h\}+\Tcal \hat h\}}} \nonumber \\
   &+|\{\E[\{\epol \hat q\{Y-\hat h\}+\Tcal \hat h\}]-\E[\{\epol q_0\{Y-h_0\}+\Tcal h_0\}]\}| \nonumber \\
   &+ \abs{(\E_n-\E)\bracks{\{\epol q_0\{Y-h_0\}+\Tcal h_0\}}} \nonumber \\
   &+ |\E_n[\{\epol q_0\{Y-h_0\}+\Tcal h_0\}]-J|
\end{align}
The first term above is upper-bounded by
\begin{align*}
    |\sup_{h\in \Hbbb,q\in \Qbbb}(\E_n-\E)[\epol q\{Y-h\}+\Tcal h]|.
\end{align*}
The second term above is equal to
\[
\abs{\Eb{\epol \prns{\hat q - q_0}\prns{\hat h - h_0}}}.
\]
This term can be upper bounded by both
\begin{align*}
  \|P_w\{(\hat q-q_0)\epol\}\|_2\   \sup_{h\in \Hbbb}\|h_0-h\|_2
\end{align*}
and
\begin{align*}
  \|P_z\{\hat h-h_0\}\|_2  \sup_{q\in \Qbbb}\|\{q_0-q\}\epol\|_2.
\end{align*}
The third and fourth terms are upper-bounded by Bernstein's inequality. Then, $|\hat J_{\mathrm{DR}}-J|$ is upper-bounded as the statements.
\end{proof}

\subsection{Proofs for \Cref{sec: closedness}}
Here we prove \Cref{thm:q_function,thm:w_function}.
Conclusions in \Cref{thm:w_function_easy,thm:q_function_easy} then follow immediately.

\begin{proof}[Proof of \Cref{thm:w_function}]
Define
\begin{align*}
    \Phi(h,q) &=\E[\{Y-h(W, A, X)\}q(Z, A, X)]\\
    \Phi_n(h,q) &=\E_n[\{Y-h(W, A, X)\}q(Z, A, X)]\\
    \Phi^{\lambda}(h,q)&=\E[\{Y-h(W, A, X)\}q(Z, A, X)]-\lambda \|q\|^2_n\\
    \Phi^{\lambda}_n(h,q)&=\E_n[\{Y-h(W, A, X)\}q(Z, A, X)]-\lambda \|q\|^2_2.
\end{align*}
where $\|q\|_n=\{\E_n[q^2]\}^{1/2}$ and $\|q\|_2=\{\E[q^2]\}^{1/2}$. From Lemma \ref{lem:support1}, we have
\begin{align}\label{eq:upper22}
    \forall q\in \Qbbb',|\|q\|^2_n-\|q\|^2_2|\leq 0.5\|q\|^2_2+0.5(\eta'_{h})^2
\end{align}
for our choice of $\eta'_{h}:=\eta_{h}+\sqrt{c_0\log(c_1/\delta)/n}$, where $\eta_{h}$ upper bounds the critical radius of $\Qbbb'$.

\paragraph{Step I.}

By definition of the estimator $\hat h$ and the assumption $h'\in \Hbbb$, we have
\begin{align}\label{eq:upper33}
    \sup_{q\in \Qbbb'} \Phi^{\lambda}_n(\hat h,q)\leq \sup_{q\in \Qbbb'}\Phi^{\lambda}_n(h',q).
\end{align}
From \Cref{lem:support2}, we have that there exists a positive constant $c$ such that
\begin{align}\label{eq:upper_q}
    \forall q\in \Qbbb': | \Phi_n(h',q)- \Phi(h',q)  |\le c C_1\{\eta'_{h} \|q\|_2+(\eta'_{h})^2\}.
\end{align}
To prove this, we apply \Cref{lem:support2} to $l(a_1,a_2):=a_1a_2,a_1=q(Z, A, X),a_2=Y-h'(W, A, X)$ that is $C_1$-Lipschitz with respect to $a_1$ by  noting $Y-h'(W, A, X)$ is in $[-C_1,C_1]$ with some constants $C_1 = \|Y\|_{\infty} + \|\Hbbb\|_{\infty}$:
\begin{align*}
    |l(a_1,a_2)-l(a'_1,a_2)|\leq C_1|a_1-a'_1|.
\end{align*}
Thus,
\begin{align*}
    \sup_{q\in \Qbbb'}\Phi^{\lambda}_n(h',q)&\leq    \sup_{q\in \Qbbb'}\{\Phi_n(h',q)-\lambda \|q\|^2_n\}   \\
&\leq    \sup_{q\in \Qbbb'}\{\Phi(h',q)+  cC_1\{\eta'_{h} \|q\|_2+(\eta'_{h})^2\} -\lambda \|q\|^2_n\}  && \text{From \Cref{eq:upper_q}} \\
&\leq    \sup_{q\in \Qbbb'}\{\Phi(h',q)+  cC_1\{\eta'_{h} \|q\|_2+(\eta'_{h})^2\} -0.5\lambda \|q\|^2_2+ 0.5\lambda (\eta'_{h})^2 \}  && \text{From \Cref{eq:upper22}} \\
&\leq    \sup_{q\in \Qbbb'}\{\Phi(h',q)-0.25\lambda \|q\|^2_2+  cC_1\{\eta'_{h} \|q\|_2+(\eta'_{h})^2\} -0.25\lambda \|q\|^2_2+\lambda (\eta'_{h})^2 \}   \\
&\leq  \sup_{q\in \Qbbb'}\{\Phi(h',q)-0.25\lambda \|q\|^2_2+ (\lambda+c^2C^2_1/\lambda+cC_1)(\eta'_{h})^2\}.
\end{align*}
In the last line, we use a general inequality that for any $a,b>0$:
\begin{align}\label{eq: ab-inequality}
    \sup_{q\in \Qbbb'}(a\|q\|_2-b\|q\|^2_2)\leq a^2/4b.
\end{align}
Moreover,
\begin{align*}
    \sup_{q\in \Qbbb'}\Phi^{\lambda}_n(\hat h,q)&=  \sup_{q\in \Qbbb'}\{ \Phi_n(\hat h,q)- \Phi_n(h',q)  + \Phi_n(h',q)-\lambda \|q\|^2_n\}\\
    &\geq  \sup_{q\in \Qbbb'}\{ \Phi_n(\hat h,q)- \Phi_n(h',q)-2\lambda\|q\|^2_n\}+\inf_{q\in \Qbbb'}\{\Phi_n(h',q)+ \lambda\|q\|^2_n\}\\
   &= \sup_{q\in \Qbbb'}\{ \Phi_n(\hat h,q)- \Phi_n(h',q)-2\lambda\|q\|^2_n\}+\inf_{-q\in \Qbbb'}\{\Phi_n(h',-q)+ \lambda\|q\|^2_n\}\\
   &=  \sup_{q\in \Qbbb'}\{ \Phi_n(\hat h,q)- \Phi_n(h',q)-2\lambda\|q\|^2_n\}+\inf_{-q\in \Qbbb'}\{-\Phi_n(h',q)+ \lambda\|q\|^2_n\}\\
      &=  \sup_{q\in \Qbbb'}\{ \Phi_n(\hat h,q)- \Phi_n(h',q)-2\lambda\|q\|^2_n\}-\sup_{-q\in \Qbbb'}\{\Phi_n(h',q)- \lambda\|q\|^2_n\}\\
     &=  \sup_{q\in \Qbbb'}\{ \Phi_n(\hat h,q)- \Phi_n(h',q)-2\lambda\|q\|^2_n\}-\sup_{-q\in \Qbbb'}\{\Phi^{\lambda}_n(h',q)\}\\
          &=  \sup_{q\in \Qbbb'}\{ \Phi_n(\hat h,q)- \Phi_n(h',q)-2\lambda\|q\|^2_n\}-\sup_{q\in \Qbbb'}\{\Phi^{\lambda}_n(h',q)\},
\end{align*}
where in the last equation we use the symmetry of $\Qbbb'$.

Therefore,
\begin{align*}
 \sup_{q\in \Qbbb'}\{ \Phi_n(\hat h,q)- \Phi_n(h',q)-2\lambda\|q\|^2_n\}& \leq     \sup_{q\in \Qbbb'}\{\Phi^{\lambda}_n(\hat h,q)\}+\sup_{q\in \Qbbb'}\{\Phi^{\lambda}_n(h',q)\}\\
 &\leq 2\sup_{q\in \Qbbb'}\{\Phi^{\lambda}_n(h',q)\} \tag{From \eqref{eq:upper33}}\\
 &\leq 2\sup_{q\in \Qbbb'}\{\Phi(h',q)-0.25\lambda \|q\|^2_2+ (\lambda+c^2 C^2_1/\lambda+cC_1)(\eta'_{h})^2\} \\
  &\leq 2\sup_{q\in \Qbbb'}\{\|P_z(h_0-h')\|_2\|q\|_2-0.25\lambda \|q\|^2_2+ (\lambda+c^2C^2_1/\lambda+cC_1)(\eta'_{h})^2\} \\
   &\leq 2\{\|P_z(h_0-h')\|^2_2/\lambda+(\lambda+c^2C^2_1/\lambda+cC_1)(\eta'_{h})^2\},
\end{align*}
where the last inequality again uses  \Cref{eq: ab-inequality}.

\paragraph{Step II.}
For the fixed $h'\in \Hbbb$ and for any $h\in\Hbbb$ define $$q_{h}\coloneqq \argmin_{q\in \Qbbb'}\|q-P_z(h'- h)\|.$$
Further define
$$\epsilon_n=\sup_{h\in\Hbbb}\inf_{q\in \Qbbb'}\|q-P_z(h-h')\|_2.$$
According to the asserted assumptions, we have $\|q_h-P_z(h'- h)\| \le \epsilon_{n}$.

Suppose $\|q_{\hat h}\|_2\geq \eta'_{h}$, and let $r=\eta'_{h}/\{2\|q_{\hat h}\|_2\}\in [0,0.5]$. Then, noting $\Qbbb'$ is star-shaped and symmetric, we have $rq_{\hat h}\in \Qbbb$ and
\begin{align*}
    r^2\|q_{\hat h}\|^2_n&\le r^2\{1.5\|q_{\hat h} \|^2_2+0.5 (\eta'_{h})^2 \}   && \text{From \Cref{eq:upper22}   }\\
    &\le (\eta'_{h})^2.     && \text{From definition of $r$}     .
\end{align*}
It follows that
\begin{align*}
    \sup_{q \in \Qbbb'}\{\Phi_n(\hat h,q)-\Phi_n(h',q)-2\lambda \|q\|^2_n  \}
        &\geq r\{\Phi_n(\hat h,q_{\hat h})-\Phi_n(h',q_{\hat h})\}-2\lambda r^2\|q_{\hat h}\|^2_n \\
        &\geq r\{\Phi_n(\hat h,q_{\hat h})-\Phi_n(h',q_{\hat h}) \}-2\lambda (\eta'_{h})^2.
\end{align*}
Observe that
\begin{align*}
    \Phi_n( h,q_{\hat h})-\Phi_n(h',q_{\hat h})=\E_n[\{h'-h \}q_{\hat h}(Z,A,X)].
\end{align*}
We now invoke \Cref{lem:support2} with $l(a_1,a_2),a_1=(h-h')q$ and $a_2 = 1$, by noting that $\eta'_{h}$ upper bounds the critical radius of $\Star(\Gcal_h)$:
\begin{align*}
    & |  \Phi_n(h,q_h)- \Phi_n(h',q_h)-\{ \Phi(h,q_h)- \Phi(h',q_h)\}|\\
    &\leq c(\eta'_{h} \|\{h-h' \}q_{\hat h}  \|_2  +(\eta'_{h})^2) \le (cC_2\eta'_{h} \|q_{\hat h}   \|_2  +c(\eta'_{h})^2),
\end{align*}
where the last line uses $\|h -h'\|_{\infty}\le C_2$ for $C_2 = 2\|\Hbbb\|_{\infty}$.

Thus,
\begin{align*}
    r\{\Phi_n(\hat h,q_{\hat h})-\Phi_n(h',q_{\hat h})\} & \geq     r\{\Phi(\hat h,q_{\hat h})-\Phi(h',q_{\hat h})\}-rc(C_2\eta'_{h} \|q_{\hat h}  \|_2  +(\eta'_{h})^2)\\
    & \geq   r\{\Phi(\hat h,q_{\hat h})-\Phi(h',q_{\hat h})\}-rcC_2\eta'_{h} \|q_{\hat h}    \|_2-0.5c(\eta'_{h})^2 \tag{$r \in [0, 0.5]$} \\
    &\overset{(a)}{=} r\E[ P_z(- \hat h+h')q_{\hat h}  ]-rcC_2\eta'_{h} \|q_{\hat h}    \|_2-0.5c(\eta'_{h})^2 \\
       &=\frac{\eta'_{h}}{2\|q_{\hat h}\|_2}\{\E[ P_z(- \hat h+h')q_{\hat h}  ]-cC_2\|q_{\hat h}   \|_2\eta'_{h}\} -0.5c (\eta'_{h})^2 \tag{Definition of $r$}\\
    &\overset{(b)}{\geq} 0.5\eta'_{h}\{\|P_z(- \hat h+h')\|_2  -2\epsilon_n\}-0.5c(1+C_2) (\eta'_{h})^2,
\end{align*}
where  (a) follows from
\begin{align*}
    \Phi(\hat h,q_{\hat h})-\Phi(h',q_{\hat h})&=\E[ \{-\hat h+h' \}q_{\hat h}  ]= \E[ P_z(- \hat h+h')q_{\hat h}  ],
\end{align*}
and (b) follows from
\begin{align*}
    \frac{ \E[ P_z(- \hat h+h')q_{\hat h}  ]}{\|q_{\hat h}\|_2}&=    \frac{ \E[ \{-q_{\hat h} +q_{\hat h} + P_z(- \hat h+h')\}q_{\hat h}  ]}{\|q_{\hat h}\|_2}\\
    &\geq    \frac{ \| q_{\hat h}\|^2_2- \|\{-q_{\hat h} + P_z(- \hat h+h')\}\|_2\|q_{\hat h}\|_2}{\|q_{\hat h}\|_2}\\
    &=    \| q_{\hat h}\|_2- \|\{-q_{\hat h} + P_z(- \hat h+h')\}\|_2\\
       &\ge   \| q_{\hat h}\|_2- \epsilon_{n}\geq   \| P_z(- \hat h+h')\|_2- 2\epsilon_{n}.
\end{align*}

\paragraph{Step III: Combining All Results.}

Thus, we have either $\|q_{\hat h}\|_2\leq \eta'_{h}$ or with high probability $1-\delta$,
\begin{align*}
      0.5\eta'_{h}\{\|P_z(\hat h-h')   \|_2-2\epsilon_n\}-(0.5c+0.5cC_2+2\lambda)(\eta'_{h})^2 \le  \frac{2\|P_z(h_0-h')   \|^2_2}{\lambda}+2(\lambda+c^2C^2_1/\lambda+cC_1)(\eta'_{h})^2.
\end{align*}
Therefore, we have either
\begin{align*}
    \|P_z(\hat h-h') \|_2\le \|P_z(h_0-h')   \|_2 + \|q_{\hat h}\|_2+ \|P_z(\hat h-h')-q_{\hat h}\|_2\leq \|P_z(h_0-h')   \|_2 + \eta'_{h}+\epsilon_n
\end{align*}
or
\begin{align*}
    \|P_z(\hat h-h')   \|_2\lesssim (1+\lambda+1/\lambda)\eta'_{h}+ \frac{\|P_z(h_0-h')   \|^2_2}{\eta'_{h}\lambda}+\epsilon_n.
\end{align*}
Thus, from triangle inequality,
\begin{align*}
    \|P_z(h_0-\hat h)   \|_2\lesssim (1+\lambda+1/\lambda)\eta'_{h}+ \frac{\|P_z(h_0-h')   \|^2_2}{\eta'_{h}\lambda}+\epsilon_n+\|P_z(h_0-h')   \|_2.
\end{align*}
\end{proof}

\begin{proof}[Proof of \Cref{thm:q_function}]
Define
\begin{align*}
    \Phi(q,h)  &=\E[-q(Z,A,X)\epol(A\mid X)h(W, A, X)+\E_{a\sim \epol(A\mid X)}[h(W, A, X)] ]\\
    \Phi_n(q,h)&=\E_n[-q(Z,A,X)\epol(A\mid X)h(W, A, X)+\E_{a\sim \epol(A\mid X)}[h(W, A, X)] ]\\
    \Phi^{\lambda}(q,h)  &=\E[-q(Z,A,X)\epol(A\mid X)h(W, A, X)+\E_{a\sim \epol(A\mid X)}[h(W, A, X)]]-\lambda \E[h^2]\\
    \Phi^{\lambda}_n(q,h)&=\E_n[-q(Z,A,X)\epol(A\mid X)h(W, A, X)+\E_{a\sim \epol(A\mid X)}[h(W, A, X)]]-\lambda\E_n[h^2].
\end{align*}
From \Cref{lem:support1}, we have
\begin{align}\label{eq: auxi}
    \forall h\in \Hbbb',\, |\|h\|^2_n-\|h\|^2_2 |\leq 0.5\|h\|^2_2+0.5\prns{\eta'_{q}}^2
\end{align}
for our choice of $\eta'_{q}:=\eta_{q}+\sqrt{c_0\log(c_1/\delta)/n}$, where  $\eta_{q}$ upper bounds the critical radius of $\Hbbb'$.

\paragraph{First Part}

By definition of $\hat q$ and $q'\in \Qbbb$, we have
\begin{align}\label{eq: hatq-opt}
    \sup_{h\in \Hbbb'} \Phi^{\lambda}_n(\hat q;h) \le     \sup_{h\in \Hbbb'} \Phi^{\lambda}_n(q';h).
\end{align}
From \Cref{lem:support2}, there exists a universal constant $c>0$ and $C_3 > 0$ such that with probability $1-\delta$,
\begin{align}\label{eq:partial}
    \forall h\in \Hbbb', |\Phi_n(q',h)- \Phi(q',h)|\le   cC_{3}\{\eta'_{q}\|h\|_2+\prns{\eta'_{q}}^2\}.
\end{align}
To prove this, we first apply \Cref{lem:support2} to $ l(a_1,a_2)=a_1a_2$ with $a_1=-q'(Z, A, X)\epol(A \mid X),a_2=h(W, A, X)$. This function is $C_{3, 1}$-Lipschitz in $a_2$ where $C_{3, 1}=\|\pi\Qbbb\|_{\infty}$. This means that
\begin{align*}
   | \E_n[-q'(Z, A, X)\epol(A \mid X)h(W, A, X)]-\E[-q'(Z, A, X)\epol(A \mid X)h(W, A, X)]|\le cC_{3,1}\{\eta'_{q}\|h\|_2+\prns{\eta'_{q}}^2\}.
\end{align*}
Similarly, we have
\begin{align*}
    |\E_n[\Tcal h]-\E[\Tcal h]|\le  c\{\eta'_{q}\|\Tcal h\|_2+\prns{\eta'_{q}}^2\}\le cC_{3,2}\{\eta'_{q} \|h\|_2+\prns{\eta'_{q}}^2\}, ~~ C_{3,2} = \|\pi(A\mid X)/f(A \mid X, W)\|_2.
\end{align*}
Combining the above two equations gives \Cref{eq:partial} with $C_3 = C_{3,1 }+ C_{3,2}$.

Thus,
\begin{align*}
    \sup_{h\in \Hbbb'}\Phi^{\lambda}_n(q',h) &=  \sup_{h\in \Hbbb'}\{\Phi_n(q',h)-\lambda \|h\|^2_n\}\\
    &\leq \sup_{h\in \Hbbb'}\{\Phi(q',h)+cC_3( \eta'_{q} \|h\|_2+\prns{\eta'_{q}}^2 )-\lambda \|h\|^2_n\} \tag{Use \Cref{eq:partial}}\\
    &\leq \sup_{h\in \Hbbb'}\{\Phi(q',h)+cC_3( \eta'_{q} \|h\|_2+\prns{\eta'_{q}}^2 ) -0.5\lambda \| h\|^2_2+0.5\lambda \prns{\eta'_{q}}^2\} \tag{Use \Cref{eq: auxi}}\\
    &\leq \sup_{h\in \Hbbb'}\{\Phi(q',h)-0.25\lambda \| h\|^2_2+cC_3( \eta'_{q} \|h\|_2+\prns{\eta'_{q}}^2 ) -0.25\lambda \| h\|^2_2+\lambda \prns{\eta'_{q}}^2\} \\
    &\leq \sup_{h\in \Hbbb'}\{\Phi(q',h)-0.25\lambda \| h\|^2_2+\{\lambda + c^2 C^2_3/\lambda+cC_3\}\prns{\eta'_{q}}^2\}.
\end{align*}
In the last line, we use a general inequality that for any $a,b>0$,
\begin{align*}
  a\|h\|_2-b\|h\|^2_2\leq a^2/4b.
\end{align*}
Moreover,
\begin{align*}
&\sup_{h\in \Hbbb'}\Phi^{\lambda}_n(\hat q,h)\\
&=\sup_{h\in \Hbbb'}\{\Phi_n(\hat q,h)-\Phi_n(q',h)+\Phi_n(q',h)- \lambda\|h\|^2_n\}\\
&\geq \sup_{h\in \Hbbb'}\{\Phi_n(\hat q,h)-\Phi_n(q',h)- 2\lambda\|h\|^2_n\}+\inf_{h\in \Hbbb'}\{\Phi_n(q',h)- \lambda\|h\|^2_n\}\\
&=\sup_{h\in \Hbbb'}\{\Phi_n(\hat q,h)-\Phi_n(q',h)- 2\lambda\|h\|^2_n\}+\inf_{-h\in \Hbbb'}\{\Phi(q',-h)- \lambda\|h\|^2_n\}\\
&=\sup_{h\in \Hbbb'}\{\Phi_n(\hat q,h)-\Phi_n(q',h)- 2\lambda\|h\|^2_n\}-\sup_{-h\in \Hbbb'}\{\Phi_n(q',h)+ \lambda\|h\|^2_n\}\\
&=\sup_{h\in \Hbbb'}\{\Phi_n(\hat q,h)-\Phi_n(q',h)- 2\lambda\|h\|^2_n\}-\sup_{h\in \Hbbb'}\{\Phi^{\lambda}_n(q',h)\}.
\end{align*}
where in the last equation we use the symmetry of $\Hbbb'$.

Therefore,
\begin{align*}
    \sup_{h\in \Hbbb'}\{\Phi_n(\hat q,h)-\Phi_n(q',h)- 2\|h\|^2_n\}&\leq \sup_{h\in \Hbbb'}\{\Phi^{\lambda}_n(\hat q,h)\}+\sup_{h\in \Hbbb'}\{\Phi^{\lambda}_n(q',h)\}\\
    &\leq 2\sup_{h\in \Hbbb'}\{\Phi^{\lambda}_n(q',h)\} \tag{From \Cref{eq: hatq-opt}}\\
    &\leq 2\sup_{h\in \Hbbb'}\{\Phi(q',h)-0.25\lambda \| h\|^2_2+\{\lambda + c^2 C^2_3/\lambda+cC_3\}\prns{\eta'_{q}}^2\}\\
    &\leq\sup_{h\in \Hbbb'}2\{\|\pi P_w(q'-q_0)\|_2\|h\|_2-0.25\lambda \| h\|^2_2+c\{C^2_2/\lambda+ \lambda+C_2\}\prns{\eta'_{q}}^2\}\\
    &\leq 2\{\|\pi P_w(q'-q_0)\|^2_2/\lambda +c\{C^2_2/\lambda+ \lambda+C_2\}\prns{\eta'_{q}}^2\}
\end{align*}

\paragraph{Second Part}
For the fixed $q'$ and for any $q \in \Qbbb$ we define
$$h_q\coloneqq \argmin_{h\in \Hbbb'}\|h-P_w\{(q-q')\epol\}\|.$$
Further define
$$\tilde \epsilon_{n}=\sup_{q\in \Qbbb}\inf_{h\in\Hbbb'} \|h-\epol P_w{\prns{q - q'}}\|_2.$$
According to the asserted assumptions, we have $\|h_q-P_w\{(q-q')\epol\}\| \le \tilde\epsilon_{n}$.

Suppose $\| h_{\hat q}\|\geq \eta'_{q}$, and let $r=\eta'_{q}/\{2\| h_{\hat q}\|_2\}\in (0,0.5]$. Then, noting $\Hbbb'$ is star-shaped and symmetric, we have $rh_{\hat q}\in \Hbbb'$ and
\begin{align*}
    r^2\|h_{\hat q}\|^2_n &\leq r^2\{1.5\|h_{\hat q}\|^2_2+0.5\prns{\eta'_{q}}^2 \} \\
      & \leq \prns{\eta'_{q}}^2  \tag{Definition of $r$}.
\end{align*}
It follows that
\begin{align*}
        \sup_{h\in \Hbbb'}\{\Phi_n(\hat q,h)-\Phi_n(q',h)- 2\lambda \|h\|^2_n\}
            &\geq
            r\{\Phi_n(\hat q,h_{\hat q})-\Phi_n( q',h_{\hat q})\}-2\lambda r^2\|h_{\hat q}\|^2_n \\
            &\geq r\{\Phi_n(\hat q,h_{\hat q})-\Phi_n( q',h_{\hat q})\}-2\lambda \eta_n^2.
\end{align*}
Observe that
\begin{align*}
    \Phi_n(q,h_q)-\Phi_n(q',h_q)&=\E_n[(-q(Z, A, X)+q'(Z, A, X) )\epol(a|x) h_q(W, A, X) ].
\end{align*}
We now invoke \Cref{lem:support2} with $l(a_1,a_2)=a_1a_2$ with $a_1=(q-q')\epol h_q, a_2 = 1$, by noting that $\eta'_{q}$ upper bounds the critical radius of $\Gcal_q$:
\begin{align*}
     &|  \Phi_n(q,h_q)- \Phi_n(q',h_q)-\{\Phi(q,h_q)-\Phi(q',h_q)\}|\\
    &\leq c(\eta'_{q} \E[\{(q'(Z, A, X)-q(Z, A, X) )\epol(A \mid X) h_q(W, A, X)\}^2]^{1/2}+\prns{\eta'_{q}}^2)\\
    &\leq c(C_4\eta'_{q} \|h_q\|_2 +\prns{\eta'_{q}}^2).
\end{align*}
where $\|(q'(Z, A, X)-q(Z, A, X) )\epol(A \mid X)\|_{\infty}\leq C_4$ for $C_4 = 2\|\pi\Qbbb\|_{\infty}$.

Thus,
\begin{align*}
     &r\{\Phi_n(\hat q,h_{\hat q})-\Phi_n( q',h_{\hat q})\}\\
    &\geq r\{\Phi(\hat q,h_{\hat q})-\Phi( q',h_{\hat q})\}-rc(C_4\eta'_{q} \|h_{\hat q}\|_2+\prns{\eta'_{q}}^2)\\
    &\geq r\{\Phi(\hat q,h_{\hat q})-\Phi( q',h_{\hat q})\}-rc(C_4\eta'_{q} \|h_{\hat q}\|_2)-0.5c\prns{\eta'_{q}}^2 \tag{$r \in [0, 0.5]$}\\
      &\overset{(a)}{=}r \E[P_w\bracks{(-q'+\hat q)\pi h_{\hat q}}]-rc(C_4\eta'_{q} \|h_{\hat q}\|_2)-0.5c\prns{\eta'_{q}}^2\\
        &=\frac{\eta'_{q}}{2\{\| h_{\hat q}\|_2\}}(\E[P_w\bracks{(-q'+\hat q)\pi h_{\hat q}}]-cC_4\eta'_{q} \| h_{\hat q}\|_2)-0.5c\prns{\eta'_{q}}^2 \tag{Definition of $r$}\\
    &\overset{(b)}{\geq}0.5\eta'_{q}(\|P_w\{(-q'+\hat q)\epol\} \|_2-2\tilde\epsilon_{n})-0.5 c\{1+C_4\}\prns{\eta'_{q}}^2,
\end{align*}
where (a) follows from
\begin{align*}
    \Phi( q,h_{ q})-\Phi( q',h_{q})&=\E[\{ -q(Z, A, X)+q'(Z, A, X)\}\epol(A \mid X)h_q(W, A, X)] \\
    &=\E[\{\E[-q(Z, A, X)+q'(Z, A, X)\mid W, A, X]\}\epol(A \mid X)h_q(W, A, X)]\\
    &= \E[P_w\{q'-q\}(W, A, X)\pi(A \mid X) h_q(W, A, X)],
\end{align*}
and (b) follows from
\begin{align*}
    \frac{\E[\epol P_w\{-q'+\hat q\} h_q]}{\| h_{\hat q}\|_2}&=   \frac{\E[[-h_{\hat q}+h_{\hat q}+\epol P_w\{-q'+\hat q\} ]h_q]}{\| h_{\hat q}\|_2}\\
    &=   \frac{\| h_{\hat q}\|^2_2-\| h_{\hat q}\|_2\|-h_{\hat q}+\epol P_w\{-q'+\hat q\}  \|_2 }{\| h_{\hat q}\|_2}\\
    &\geq \| h_{\hat q}\|_2-\tilde\epsilon_{n}\geq \|\epol P_w\{-q'+\hat q\} \|_2-2\tilde\epsilon_{n}.
\end{align*}

\paragraph{Combining All Results}
Thus, $\|h_{\hat q} \|\leq \eta'_{q}$ or
\begin{align*}
0.5\eta'_{q}(\|P_w\{(-q'+\hat q)\epol\} \|_2-2\tilde\epsilon_{n})&-\{0.5 c+0.5 cC_4 + 2\lambda\}\prns{\eta'_{q}}^2 \\
    & \le 2\{\|\pi P_w(q'-q_0)\|^2_2/\lambda +c\{C^2_2/\lambda+ \lambda+C_2\}\prns{\eta'_{q}}^2\}.
\end{align*}
Therefore, we have either
\begin{align*}
    \|P_w(\epol\hat q-\epol q')\|_2\le  \|P_w(\epol q_0-\epol q')\|_2+  \|h_{\hat q}-\epol P_w(\epol\hat q-\epol q')\|_2+\|h_{\hat q}\|_2\leq  \|P_w(\epol q_0-\epol q')\|_2 + \eta'_{q}+\tilde\epsilon_{n}.
\end{align*}
or
\begin{align*}
    \|P_w(\epol\hat q-\epol q')\|_2\lesssim   \|\epol P_w(q'-q_0)\|^2_2/\lambda\eta'_{q}+\{1+1/\lambda+ \lambda\}\eta'_{q}+\tilde\epsilon_{n}.
\end{align*}
Finally, from triangle inequality,
\begin{align*}
      \|P_w(\epol\hat q-\epol q_0)\|_2 \lesssim  \|\epol P_w(q'-q_0)\|^2_2/\lambda\eta'_{q}+\{1+1/\lambda+ \lambda\}\eta'_{q}+\tilde\epsilon_{n}+    \|P_w(\epol q_0-\epol q')\|_2.
\end{align*}
\end{proof}

\begin{proof}[Proof of \Cref{cor: vc_nonpara}]
For a function $f(O)$, we denote its empirical $L^2$-norm as $\|f\|_{n,2}=\{1/n\sum_{i=1}^n f(O_i)^2\}^{1/2}$ and $L^{\infty}$-norm as $\|f(O)\|_{n,\infty}=\max_{1\leq i\leq n} |f(O_i)|$.

Note that
\begin{align*}
    & \log \Ncal(t, (\Hbbb-h_0)\Qbbb', \|\cdot\|_{n,\infty}) \\
    &\leq \log \Ncal(t, \Hbbb\Qbbb', \|\cdot\|_{n,\infty})+\log \Ncal(t, h_0\Qbbb', \|\cdot\|_{n,\infty})\\
    &\leq \log \Ncal(t/\{2\|\Hbbb\|_{\infty}\}, \Qbbb', \|\cdot\|_{n,\infty})+\log \Ncal(t/\{2\|\Qbbb'\|_{\infty}\}, \Hbbb, \|\cdot\|_{n,\infty})+\log \Ncal(t/\|\Hbbb\|_{\infty}, \Qbbb', \|\cdot\|_{n,\infty}).
\end{align*}
It follows that
\begin{align*}
    & \int_{0}^{\eta}\sqrt{\frac{\log \Ncal(t, (\Hbbb-h_0)\Qbbb', \|\cdot\|_{n,2}) }{n}}\rd(t)\\
    &\leq \int_{0}^{\eta}\sqrt{\frac{\log \Ncal(t, (\Hbbb-h_0)\Qbbb', \|\cdot\|_{n,\infty}) }{n}}\rd(t) \tag{$\|f\|_{n,2}\leq \|f\|_{n,\infty}$ for any $f$}\\
    &\leq \int_{0}^{\eta}\sqrt{\frac{2\log \Ncal(t/2\|\Hbbb\|_{\infty}, \Qbbb', \|\cdot\|_{n,\infty}) }{n}}+\sqrt{\frac{\log \Ncal(t/2\|\Qbbb'\|_{\infty}, \Hbbb, \|\cdot\|_{n,\infty}) }{n}} \rd(t) \\
    &\leq \int_{0}^{\eta}\sqrt{\frac{2\log \Ncal(t/2\|\Hbbb\|_{\infty}, \Qbbb', \sqrt{n}\|\cdot\|_{n,2}) }{n}}+\sqrt{\frac{\log \Ncal(t/2\|\Qbbb'\|_{\infty}, \Hbbb, \sqrt{n}\|\cdot\|_{n,2}) }{n}} \rd(t). \tag{$\|f\|_{n,\infty}\leq \|f\|_{n,2}\sqrt{n}$ for any $f$  } \\
    &=O\left(\frac{\log n}{\sqrt{n}}\int_{0}^{\eta}\sqrt{\max(V(\Hbbb),V(\Qbbb')\log(1/t)}\,\rd(t) \right) \tag{\Cref{lem: vc}}\\
    &=O\left(\sqrt{\max(V(\Hbbb),V(\Qbbb')}\eta \log(1/\eta)\frac{\log n}{\sqrt{n}}\right).
\end{align*}
Then, the critical inequality in \Cref{lem:critical_basic} becomes
\begin{align*}
    O\left(\sqrt{\max(V(\Hbbb),V(\Qbbb)}\eta \log(1/\eta)\frac{\log n}{\sqrt{n}}\right)\leq \eta^2.
\end{align*}
This is satisfied with $$\eta_h=O\left(\sqrt{\max(V(\Hbbb),V(\Qbbb)}\frac{\log n}{\sqrt{n}}\log\left(\sqrt{\max(V(\Hbbb),V(\Qbbb)}\frac{\log n}{\sqrt{n}}\right)\right) . $$

According to \Cref{lem:critical_basic}, the $\eta_h$ above upper bounds the critical radius of $\mathcal{G}_h$. By a similar calculation, it can be shown that this $\eta_h$ can also upper bound the critical radius of $\Qbbb'$. Plugging it  into \Cref{thm:w_function_easy} then proves the conclusion of \Cref{cor: vc_nonpara}.
\end{proof}

\begin{proof}[Proof for \Cref{cor: nonpara}]
Again, we have
\begin{align*}
    & \log \Ncal(t, (\Hbbb-h_0)\Qbbb', \|\cdot\|_{n,\infty}) \\
    &\leq \log \Ncal(t, \Hbbb\Qbbb', \|\cdot\|_{n,\infty})+\log \Ncal(t, h_0\Qbbb', \|\cdot\|_{n,\infty})\\
    &\leq 2\log \Ncal(t/\{2\|\Hbbb\|_{\infty}\}, \Qbbb', \|\cdot\|_{n,\infty})+\log \Ncal(t/\{2\|\Qbbb'\|_{\infty}\}, \Hbbb, \|\cdot\|_{n,\infty}).
\end{align*}
and
\begin{align*}
   & \int_{0}^{\eta}\sqrt{\frac{\log \Ncal(t, (\Hbbb-h_0)\Qbbb', \|\cdot\|_{n,2}) }{n}}\, \rd(t)\\
   &\leq \int_{0}^{\eta}\sqrt{\frac{\log \Ncal(t, (\Hbbb-h_0)\Qbbb', \|\cdot\|_{\infty}) }{n}}\, \rd(t)\\
    &\leq \int_{0}^{\eta}\sqrt{\frac{2\log \Ncal(t/2\|\Hbbb\|_{\infty}, \Qbbb', \|\cdot\|_{\infty}) }{n}}+\sqrt{\frac{\log \Ncal(t/2\|\Qbbb'\|_{\infty}, \Hbbb, \|\cdot\|_{\infty}) }{n}}\, \rd(t). \\
    &= \begin{cases}
    O(n^{-1/2}\eta^{1-\frac{1}{2}\beta }),\,(\beta \ge 2) \\
    O(n^{-1/2}\log(\eta)), (\beta =2)\\
    O(n^{-1/2}\eta^{-\beta +2}),(\beta \le 2).
    \end{cases}
\end{align*}
Then we solve the following inequalities:
\begin{align*}
    \begin{cases}
    n^{-1/2}\eta^{1-\frac{1}{2}\beta }&\leq \eta^2,\,(\beta \ge 2) \\
       n^{-1/2}\log(\eta)&\leq \eta^2, (\beta =2)\\
       n^{-1/2}\eta^{-\beta +2}&\leq \eta^2,(\beta \le 2).
    \end{cases}
\end{align*}
which gives
\begin{align*}
\eta_h = \begin{cases}  O(n^{-1/(2+\beta)})&\quad\beta<2 \\  O(n^{-1/4}\log n)&\quad\beta=2 \\  O(n^{-1/(2\beta)})&\quad\beta>2 \end{cases}
\end{align*}
According to \Cref{lem:critical_basic}, the $\eta_h$ above upper bounds the critical radius of $\mathcal{G}_h$. By a similar calculation, it can be shown that this $\eta_h$ can also upper bound the critical radius of $\Qbbb'$. Plugging it  into \Cref{thm:w_function_easy} then proves the conclusion of \Cref{cor: vc_nonpara}.
\end{proof}

\subsection{Proofs for \Cref{sec: both}}\label{sec: proof-sec-stabilizer-dr}
\begin{proof}[Proof of \Cref{thm:dr2}]
We denote
\begin{align*}
    \eta'_{h} = \eta_{h}+\sqrt{1+\log(1/\delta)/n},\,\eta'_{q} = \eta_{q}+\sqrt{1+\log(1/\delta)/n}.
\end{align*}
Recall that we use sample splitting with two data subsamples $\Dcal_1$ and $\Dcal_0$, and the estimator is
\begin{align*}
    \hat J_{\DR} = \frac{1}{2}\E_{n_1}[\tilde\phi_{\DR}(O; \hat q^{(0)},\hat h^{(0)})]+  \frac{1}{2}\E_{n_0}[\tilde\phi_{\DR}(O; \hat q^{(1)},\hat h^{(1)})].
\end{align*}
where $\E_{n_1}[\cdot]$ is the empirical average over $\Dcal_1$, $\E_{n_0}[\cdot]$ is the empirical average over $\Dcal_0$, and
\begin{align*}
    \tilde\phi_{\DR}(O; h,q)=\epol(A\mid X)q(Z,A, X)\{Y-h(W, A, X)\}+(\Tcal h)(W, X).
\end{align*}
We take arbitrary elements $h_0,q_0$ s.t. $h_0\in \Hbbb^{\obs}_0,q_0\in \Qbbb^{\obs}_0$ and $J = \Eb{\tilde\phi_{\DR}(O; q_0,h_0)}$ and denote $\hat h = \hat h^{\prns{0}}$, $\hat q = \hat q^{\prns{0}}$. Then, the bias is decomposed into the three terms:
\begin{align}
    \E_{n_1}[\tilde\phi_{\DR}(O; \hat q,\hat h)] -J
        &=(\E_{n_1}-\E)[ \tilde\phi_{\DR}(O; \hat q,\hat h)- \tilde\phi_{\DR}(O; q_0,h_0)]+\E[\tilde\phi_{\DR}(O; \hat q,\hat h)- \tilde\phi_{\DR}(O;  q_0,h_0)]\nonumber \\
        &+\E_{n_1}[ \tilde\phi_{\DR}(O; q_0,h_0)]-J. \label{eq: three3}
\end{align}
According to \Cref{thm:w_function_easy,thm:q_function_easy}, we have that with probability at least $1-\delta$,
\begin{align*}
\|P_z(\hat h-h_0)\|_2\leq O(\eta'_{h}), ~~ \|P_w(\hat q-q_0)\|_2\leq O(\eta'_{q}).
\end{align*}
Given this event, \Cref{eq:dr_property1} implies that
\begin{align*}
&\|\E[\hat h(W, A, X)-Y \mid A,U,X]\|_2= O(\tau^{\Hbbb}_{1,n}\eta'_{h}), \\
&\|\E[\epol(A \mid X)\{\hat q(Z, A, X)-1/f(A \mid X,W)\}\mid A, U, X]\|_2= O(\tau^{\Qbbb}_{1,n}\eta'_{q}).
\end{align*}
In the rest of the proof, we always condition on this event. Now we bound each term in \eqref{eq: three3} respectively.
\paragraph{First Term} Note that $\| \tilde\phi_{\DR}\prns{O; {h, q}}\|_{\infty} \le \|\pi\Qbbb\|_{\infty}\prns{\|Y\|_{\infty} + \|\Hbbb\|_{\infty}} + \|\Hbbb\|_{\infty}$. So by Bernstein inequality, the first term in \eqref{eq: three3} is $O(\sqrt{\log(1/\delta) /n_1})$ with probability at least $1-\delta$.

\paragraph{Second Term}
To bound the second term in \eqref{eq: three3}, we note that
\begin{align*}
     &\E[ \tilde\phi_{\DR}(O; \hat q,\hat h)- \tilde\phi_{\DR}(O; q_0,h_0)|\Dcal_0]  \\
     &=\E[\epol (\hat q-q_0)(Y-h_0)|\Dcal_0]+\E[\epol q_0\{-\hat h+h_0\}+\Tcal (\hat h-h_0) |\Dcal_0]+\E[\epol \{\hat q-q_0\}\{h_0-\hat h\} |\Dcal_0] \\
    &=\E[\epol \{\hat q-q_0\}\{h_0-\hat h\} |\Dcal_0]  \\
       &=\E[\E[\epol\{\hat q-q_0\} \mid A, U, X]\E[\{h_0-\hat h\} \mid A, U, X] |\Dcal_0]  \tag{$Z\perp W \mid A,U, X$ }\\
     &\leq \|\E[\hat h(W, A, X)-Y \mid A,U,X]\|_2\|\E[\epol(A\mid X)\{\hat q(Z, A,X)-1/f(A \mid X, W)\} \mid A,U,X]\|_2 \tag{CS inequality }\\
     &=O( \tau^{\Hbbb}_{1,n}\tau^{\Qbbb}_{1,n}\eta'_{h}\eta'_{q} ).
\end{align*}

\paragraph{Third Term} Again, we can use Bernstein inequality to bound the third term by $O(\sqrt{\log(1/\delta) /n_1})$ with probability at least $1-\delta$.

\paragraph{Combining All Terms}
The above proves that with probability at least $1-\delta$, we have
\begin{align*}
|\E_{n_1}[\tilde\phi_{\DR}(O; \hat q^{(0)},\hat h^{(0)})] -J|  =O\prns{\tau^\Hbbb_{1}\tau^\Qbbb_{1}\eta_{h}\eta_{q}+\sqrt{\{1+\log(1/\delta)\}/n}}.
\end{align*}
Similarly, we can prove the same bound for $|\E_{n_0}[\tilde\phi_{\DR}(O; \hat q^{(1)},\hat h^{(1)})] -J|$. Combining these two proves \Cref{eq: rate-dr-stabilizer}.
\end{proof}

\begin{proof}[Proof of \Cref{thm:dr}]
Again, according to \Cref{thm:w_function_easy,thm:q_function_easy}, we have that with probability at least $1-2\delta$,
\begin{align*}
\|P_z(\hat h-h_0)\|_2\leq O(\eta'_{h}), ~~ \|P_w(\hat q-q_0)\|_2\leq O(\eta'_{q}).
\end{align*}
Given this event, the assumption \eqref{eq:dr_property2} implies that
\begin{align*}
 \|\hat h-h_0\|_2\leq O(\tau^{\Hbbb}_{2,n}\eta'_{h}) ,\, \|\hat q-q_0\|_2\leq O(\tau^{\Qbbb}_{2,n}\eta'_{q}) .
\end{align*}
We again condition on this event in the rest of the proof, and analyze each term in \Cref{eq: three3} respectively.

\paragraph{First Term.} To bound the first term in \Cref{eq: three3}, we first note that
\begin{align}
       & (\E_{n_1}-\E)[ \phi_{\DR}(\hat q,\hat h)- \phi_{\DR}(q_0,h_0)|\Dcal_0] \label{eq:stochastic-equicont} \\
       &=(\E_{n_1}-\E)[\epol (\hat q-q_0)(Y-h_0)|\Dcal_0]  \label{eq:first1}\\
       &+(\E_{n_1}-\E)[\epol q_0\{-\hat h+h_0\}+\Tcal (\hat h-h_0) |\Dcal_0] \label{eq:first2} \\
       &+(\E_{n_1}-\E)[\epol \{\hat q-q_0\}\{h_0-\hat h\} |\Dcal_0].\label{eq:first3}
\end{align}
From Bernstein's inequality, with probability $1-\delta$, \Cref{eq:first1} is
\begin{align*}
   &| (\E_{n_1}-\E)[\epol(\hat q-q_0)(Y-h_0)|\Dcal_0]  |\\
   &\lesssim \sqrt{\frac{2\E[\{\epol(\hat q-q_0)\}^2\{Y-h_0\}^2|\Dcal_0]\log(1/\delta) }{n_1}}+\frac{2\|\epol(\hat q-q_0)(Y-h_0) \|_{\infty}\log(1/\delta)}{3n_1}\\
   & = O\prns{\sqrt{\frac{(\tau^{\Qbbb}_{2,n}\eta'_{q})^2\log(1/\delta)}{n_1}} + \frac{\log \prns{1/\delta}}{n_1}},
\end{align*}
and similarly
\Cref{eq:first2,eq:first3} are
\begin{align*}
&|(\E_{n_1}-\E)[q_0\{-\hat h+h_0\}+\Tcal (\hat h-h_0) |\Dcal_0] | = O\prns{\sqrt{\frac{(\tau^{\Hbbb}_{2,n}\eta'_{h})^2\log(1/\delta)}{n_1}} + \frac{\log \prns{1/\delta}}{n_1}} \\
&(\E_{n_1}-\E)[\pi\epol \{\hat q-q_0\}\{h_0-\hat h\} |\Dcal_0] = O\prns{\sqrt{\frac{(\tau^{\Hbbb}_{2,n}\eta'_{h})^2\log(1/\delta)}{n_1}} + \frac{\log \prns{1/\delta}}{n_1}}.
\end{align*}
It follows that with probability at least $1-3\delta$,
\begin{align*}
| (\E_{n_1}-\E)[\epol(\hat q-q_0)(Y-h_0)|\Dcal_0]  | = O\prns{\sqrt{\frac{\log\prns{1/\delta}}{n}}\prns{\tau^{\Hbbb}_{2,n}\eta'_{h} + \tau^{\Qbbb}_{2,n}\eta'_{q}} + \frac{\log \prns{1/\delta}}{n}}
\end{align*}
\paragraph{Second Term.} To bound the second term in \eqref{eq: three3}, we note that
\begin{align*}
     &\E[ \phi_{\DR}(\hat q,\hat h)- \phi_{\DR}(q_0,h_0)|\Dcal_0]\\
     &=\E[\epol (\hat q-q_0)(y-h_0)|\Dcal_0] +\E[\epol q_0\{-\hat h+h_0\}+\Tcal (\hat h-h_0) |\Dcal_0]+\E[\epol \{\hat q-q_0\}\{h_0-\hat h\} |\Dcal_0] \\
    &=\E[\epol \{\hat q-q_0\}\{h_0-\hat h\} |\Dcal_0].
\end{align*}
Thus we have
\begin{align*}
|\E[ \phi_{\DR}(\hat q,\hat h)- \phi_{\DR}(q_0,h_0)|\Dcal_0]| = |\E[\epol \{\hat q-q_0\}\E[\{h_0-\hat h\} \mid Z, A, X] |\Dcal_0]| \le \|\{q_0-\hat q\}\epol\|_2\|P_z(h_0-\hat h)\|_2,
\end{align*}
and
\begin{align*}
|\E[ \phi_{\DR}(\hat q,\hat h)- \phi_{\DR}(q_0,h_0)|\Dcal_0]| = |\E[\Eb{\epol \{\hat q-q_0\}\mid W, A, X}\{h_0-\hat h\} |\Dcal_0]| \le \|\epol P_w\{q_0-\hat q\}\|_2\|h_0-\hat h\|_2.
\end{align*}
It follows that
\begin{align*}
\E[ \phi_{\DR}(\hat q,\hat h)- \phi_{\DR}(q_0,h_0)|\Dcal_0] = O\prns{\min(\tau^{\Hbbb}_{2,n},\tau^{\Qbbb}_{2,n})\eta'_{h}\eta'_{q}}.
\end{align*}

\paragraph{Combining All Terms.}
The above proves that with probability at least $1-\tilde c\delta$ for a universal positive constant $\tilde c$, we have
\begin{align*}
&|\E_{n_1}[\tilde\phi_{\DR}(O; \hat q^{(0)},\hat h^{(0)})] -\E_{n_1}[ \tilde\phi_{\DR}(O; q_0,h_0)]|  \\
&\qquad\qquad =O\prns{\sqrt{\frac{\log\prns{1/\delta}}{n}}\prns{\tau^{\Hbbb}_{2,n}\eta'_{h} + \tau^{\Qbbb}_{2,n}\eta'_{q}} + \frac{\log \prns{1/\delta}}{n}+\min(\tau^{\Hbbb}_{2,n},\tau^{\Qbbb}_{2,n})\eta'_{h}\eta'_{q}}.
\end{align*}
Similarly, we can show
\begin{align*}
&|\E_{n_0}[\tilde\phi_{\DR}(O; \hat q^{(1)},\hat h^{(1)})] -\E_{n_1}[ \tilde\phi_{\DR}(O; q_0,h_0)]|  \\
&\qquad\qquad=O\prns{\sqrt{\frac{\log\prns{1/\delta}}{n}}\prns{\tau^{\Hbbb}_{2,n}\eta'_{h} + \tau^{\Qbbb}_{2,n}\eta'_{q}} + \frac{\log \prns{1/\delta}}{n}+\min(\tau^{\Hbbb}_{2,n},\tau^{\Qbbb}_{2,n})\eta'_{h}\eta'_{q}}.
\end{align*}
Given that $\min(\tau^\Hbbb_{2},\tau^\Qbbb_{2})\eta_{q}\eta_{h} = o(n^{-1/2})$,  $\tau^\Hbbb_{2}\eta_{h} = o(1)$ and $\tau^\Qbbb_{2}\eta_{q} = o(1)$, we have
\begin{align*}
\hat J - J = \E_{n}[\tilde\phi_{\DR}(O;  q_0, h_0)] - J + o_p\prns{n^{-1/2}}.
\end{align*}
By central limiting theorem, we have
\begin{align*}
\hat J_{\DR} - J \overset{d}{\to} \mathcal N\prns{0, \Eb{\tilde\phi_{\DR}^2(O;h_0,q_0)}}.
\end{align*}

\end{proof}

\subsection{Proofs for \Cref{sec: comparison}}
\begin{proof}[Proof for \Cref{lemma: valid-bridge}]
According to \Cref{lemma: observed-bridge}, we have that  $\Hbbb_0 \subseteq \Hbbb_0^{\obs}$. For any $h_0 \in \Hbbb_0^{\obs} \subseteq \Hbbb$,
\begin{align*}
    \Eb{Y-h_0(W, A,X) \mid Z, A, X} = \Eb{\Eb{Y-h_0(W, A,X)\mid A, U, X} \mid Z, A, X} =0.
\end{align*}
By \Cref{assump: full-completeness} condition \eqref{assump: completeness-2}, we have that $\Eb{Y-h_0(W, A,X)\mid A, U, X} = 0$, \ie, $h_0\in \Hbbb_0$.  It follows that $\Hbbb_0^{\obs} = \Hbbb_0$. Similarly, we can prove that $\Qbbb_0^{\obs} = \Qbbb_0$ under \Cref{assump: full-completeness} condition \eqref{assump: completeness-1}.
\end{proof}

\subsection{Proofs for \Cref{sec: completeness}}
\begin{proof}[Proof for \Cref{lemma: unique-bridge}]
Consider two bridge functions $h_0, h_0' \in \Hbbb_0$:
\[
\Eb{Y - h_0(W, A, X) \mid U, A, X} = \Eb{Y - h_0'(W, A, X) \mid U, A, X} = 0.
\]
Thus we have
\[
\Eb{h_0(W, A, X) - h_0'(W, A, X) \mid U, A, X} = 0.
\]
Then \Cref{assump: full-completeness-2} condition \eqref{full-completeness-1} implies that $h_0(W, A, X) = h_0'(W, A, X)$. Therefore, $\Hbbb_0$ is at most a singleton. We can similarly prove that $\Qbbb_0$ is at most a singleton.
\end{proof}

\begin{proof}[Proof for \Cref{lemma:completeness-relation}]
Under \Cref{asm:whole_assm}, we have that \begin{align*}
    &\Eb{g(W, A, X) \mid Z, A, X}  = \Eb{\Eb{g(W, A, X) \mid U, A, X} \mid Z, A, X}, \\
    &\Eb{g(Z, A, X) \mid W, A, X}  = \Eb{\Eb{g(Z, A, X) \mid U, A, X} \mid W, A, X}.
\end{align*}
We first prove statement \eqref{completeness-relation-1}. According to \Cref{assump: full-completeness} condition \eqref{full-completeness-2}, $\Eb{g(W, A, X) \mid Z, A, X} = 0$ if and only if $\Eb{g(W, A, X) \mid U, A, X} = 0$. \Cref{assump: full-completeness-2} condition \eqref{full-completeness-1} further ensures that this holds if and only if $g(W, A, X) = 0$. In other words, \Cref{assump: full-completeness} conditions \eqref{full-completeness-2} and \Cref{assump: full-completeness-2} condition \eqref{full-completeness-1} are sufficient for \Cref{assump: obs-completeness} condition \eqref{obs-completeness-1}. Similarly, we can show that \Cref{assump: full-completeness} condition \eqref{full-completeness-1} and \Cref{assump: full-completeness-2} condition \eqref{full-completeness-2} are sufficient for \Cref{assump: obs-completeness} condition \eqref{obs-completeness-2}.

Next, we prove statement \eqref{completeness-relation-2}. If $\Eb{g(W, A, X) \mid U, A, X} = 0$, then $\Eb{g(W, A, X) \mid Z, A, X} = 0$ as well. By \Cref{assump: obs-completeness} condition \eqref{obs-completeness-1}, this holds if and only if $g(W, A, X) = 0$. Therefore, \Cref{assump: obs-completeness} condition \eqref{obs-completeness-1} is sufficient for \Cref{assump: full-completeness-2} condition \eqref{full-completeness-1}. Similarly, we can prove that \Cref{assump: obs-completeness} condition \eqref{obs-completeness-2} is sufficient for \Cref{assump: full-completeness-2} condition \eqref{full-completeness-2}.
\end{proof}

\begin{proof}[Proof for \Cref{lemma: completeness-existence-W}]
We need to prove that the following equation of the first kind is solvable:
\[
[K_{W \mid a, x}h](a, u, x) = k_0(a, u,  x), ~~ \text{a.e. } u, a, x \text{ w.r.t } \mathbb{P}.
\]
Thus we only need to verify the assumptions in  Picard's theorem in \Cref{thm: Picard} with $K = K_{W\mid a, x}$ and $\phi = k_0$. Note that condition $2$ in  \Cref{thm: Picard} is satisfied by our asserted assumptions. Thus we only need to show $k_0 \in \mathcal{N}\prns{K^*_{W \mid a, x}}^\perp$.

Since $(K^*_{W \mid a, x} g)(w, a, x) = \Eb{g(U, a, x) \mid W = w, A = a, X = x}$. By \Cref{assump: full-completeness} condition 1,  $\mathcal{N}\prns{K^*_{W \mid a, x}}= \braces{0}$, which means that $\mathcal{N}\prns{K^*_{W \mid a, x}}^\perp = \operatorname{dom} K^*_{W \mid a, x} = L_2(U \mid A = a, X = x)$. Therefore, $k_0 \in \mathcal{N}\prns{K^*_{W \mid a, x}}^\perp$.
\end{proof}

\begin{proof}[Proof for \Cref{lemma: completeness-existence-Z}]
The proof is completely analogous to the proof for \Cref{lemma: completeness-existence-W}.
\end{proof}

\subsection{Proofs for \Cref{sec: efficiency-bound}}
Before proving \Cref{thm: efficiency}, we first introduce a generalized implicit function theorem below.
\begin{lemma}[Implicit Function Theorem, Theorem 7.13-1 in \cite{ciarlet2013linear}]\label{lemma: implicit}
Let $\mathcal{X}$ and $\mathcal{Y}$ be Banach spaces, $O \subset \mathcal{X} \times \mathcal{Y}$ be an open neighborhood containing $(\bar{x}, \bar{y})$. Consider a mapping $F$ satisfies the following conditions:
\begin{enumerate}
\item $F(\bar{x}, \bar{y})=0$.
\item For any $(x, y) \in O$, $F$ is continuous at $(x, y)$, and $\frac{\partial F}{\partial y}(x, y)$ exists and  is a continuous linear mapping over $O$.
\item The linear mapping $\frac{\partial F}{\partial y}(\bar{x}, \bar{y})$ is a bijection.
\item The mapping $F$ is differentiable at $(\bar{x}, \bar{y})$.
\end{enumerate}
Then there exists an open neighborhood $U$ of $\bar{x}$ in $\mathcal{X}$, an open neighborhood $V$ of $\bar{y}$ in $\mathcal{Y}$ and a continuous implicit function $f: U \mapsto V$ such that $U \times V \subseteq O$,
\begin{align*}
F(x, f(x))= 0, ~~ \text{for any } x \in U,
\end{align*}
and $f$ is differentiable at $\bar{x}$ with derivative
\begin{align*}
f'(\bar{x}) = -\prns{\frac{\partial F}{\partial y}(\bar{x}, \bar{y})}^{-1}\frac{\partial F}{\partial x}(\bar{x}, \bar{y}).
\end{align*}
\end{lemma}

\begin{proof}[Proof for \Cref{thm: efficiency}]
\textbf{Step I: deriving the tangent space.}
First, consider regular parametric submodel $\mathcal P_{t} = \braces{f_t(y, w, z, a, x): t \in \R s}$ with $f_0(y, w, z, a, x)$ equals the true density $f(y, w, z, a, x)$ (with respect to an appropriate dominating measure). The associated score function is denoted as $S(y, w, z, a, x) = \partial_t \log f_t(y, w, z, a, x) \vert_{t = 0}$. The expectation w.r.t the distribution $f_t(y, w, z, a, x)$ is denoted by $\E_t$.
We can similarly denote the score functions for any component of this parametric submodel. For example, the score function for $f_t(y, w \mid z, a, x)$ is denoted as $S(y, w \mid z, a, x) = \partial_t \log f_t(y, w \mid z, a, x) \vert_{t = 0}$.
It is easy to show that
\[
S(Y, W, Z, A, X) = S\prns{Z, A, X} + S\prns{Y, W \mid Z, A, X}, \Eb{S(Z, A, X)} = 0, \Eb{S(Y, W\mid Z, A, X)\mid Z, A, X} = 0.
\]

Let $h_t$ be a curve such that $h_t \vert_{t = 0} = h_0$ and $\partial_t h_t\vert_{t = 0}$ exists, and
\begin{align}
    \E_t\bracks{Y - h_t(W, A, X) \mid Z, A, X} = 0 \label{eq: restrict1}.
\end{align}

\Cref{eq: restrict1} implies that
\begin{align*}
    \partial_t \E_t\bracks{Y - h_t(W, A, X) \mid Z, A, X}\vert_{t = 0} = 0,
\end{align*}
which in turn implies that
\begin{align}\label{eq: h-derivative}
  \Eb{\prns{Y - h_0(W, A, X)}S(W, Y\mid Z, A, X) \mid Z, A, X} = \Eb{\partial_t h_t(W, A, X)\vert_{t = 0} \mid Z, A, X}.
\end{align}
This means that $S(W, Y \mid Z, A, X)$ must satisfy that $\Eb{\prns{Y - h_0(W, A, X)}S(W, Y\mid Z, A, X) \mid Z, A, X} \in \text{Range}\prns{P_z}$.
Therefore, all score vectors under $\mathcal M_{np}$ must lie in the  set $\text{Closure}\prns{\mathcal S}$ where:
\begin{align}
    \mathcal S =\bigg \{
    &S(Y, W, Z, A, X) = S\prns{Z, A, X} + S\prns{Y, W \mid Z, A, X}: \nonumber \\
    &S\prns{Z, A, X} \in L_2(Z, A, X), S\prns{Y, W \mid Z, A, X} \in L_2(Y, W \mid Z, A, X), \nonumber \\
    &\Eb{S(Z, A, X)} = 0, \Eb{S(Y, W\mid Z, A, X)\mid Z, A, X} = 0, \nonumber \\
    &\Eb{\prns{Y- h_0(W, A, X)}S(Y, W\mid Z,A,X) \mid Z, A, X} \in \text{Range}\prns{P_z}\bigg\}. \label{eq: tangent-set}
\end{align}
Now we show that for any $S(Y, W, Z, A, X) \in \mathcal{S}$, we can find a parametric submodel $f_t$ whose score function is $S(Y, W, Z, A, X) = S\prns{Z, A, X} + S\prns{Y, W \mid Z, A, X}$ such that \Cref{eq: restrict1} holds for a certain $h_t$ for $t$ near $0$ and $\partial_t h_t(W, A, X)\vert_{t}$ exists.
We consider the parametric submodel of the following form:
\begin{align*}
f_t(z, a, x) = f(z, a, x)\prns{1 + tS\prns{z, a, x}}, ~~ f_t(y, w \mid z, a, x) = f(y, w \mid z, a, x)\prns{1 + S(y, w \mid z, a, x)},
\end{align*}
where $t \in \R{}$ is small enough such that ${1 + tS\prns{z, a, x}} \ge 0$ and ${1 + S(y, w \mid z, a, x)} \ge 0$ for any $y, w, z, a, x$ so that $f_t$ defined above is a valid density function. It is easy to check that $\partial_t \log f_t(z, a, x)\vert_{t = 0} = S(z, a, x)$ and $\partial_t \log f_t(y, w \mid z, a, x)\vert_{t = 0} = S(y, w \mid z, a, x)$.

Consider the mapping $F(t, h) = P_{z, t} h - \E_t\bracks{Y \mid Z, A, X}$ for $t \in \R{}$ and $h \in {L}_2(W, A, X)$, where $P_{z, t}: {L}_2(W, A, X) \mapsto {L}_2(Z, A, X)$ is the conditional expectation operator defined as $$[P_{z, t} h](Z, A, X) = \E_t\bracks{h(W, A, X) \mid Z, A, X}.$$
Note that $F(0, h_0) = 0$, $\partial_h F(0, h_0) = P_z$ and $\partial_t F(0, h_0) = -\Eb{(Y - h_0(W, A, X))S(Y,W,Z,A,X)\mid Z,A,X}$. Since $P_z$ is a bijective linear operator, \Cref{lemma: implicit} implies that for all $t$ close to $0$, there exists $h_t \in L_2(W, A, X)$ such that $F(t, h_t) = 0$ and $\partial_t h_t\vert_{t = 0}$ exists with
\begin{align*}
\partial_t h_t(W, A, X)\vert_{t = 0} = - \prns{\partial_h F(0, h_0)}^{-1}\partial_t F(0, h_0).
\end{align*}
This is equivalent to $\partial_t h_t\vert_{t = 0}$ satisfying \Cref{eq: h-derivative}.

Therefore, we have shown that the tangent space for the model $\mathcal{M}_{np}$ is $\text{Closure}\prns{\mathcal{S}}$.

\textbf{Step II: deriving a preliminary influence function.}
Denote the target parameter under distribution $f_t(y, w, z, a, x)$ as $J_t = \E_t\bracks{\mathcal T h_t(W, X)}$. We need to derive
\begin{align*}
    \partial_t J_t \vert_{t = 0} = \Eb{\mathcal T h_0(W, X)S(W, X)} + \Eb{\partial_t \mathcal T h_t(W, X)\vert_{t = 0}}
\end{align*}

Note that
\begin{align*}
    \Eb{\partial_t \mathcal T h_t(W, X)\vert_{t = 0}}
        &= \Eb{\epol(A\mid X)q_0(Z, A, X)\Eb{\partial_t h_t(W, A, X)\vert_{t = 0} \mid Z, A, X}}  \\
        &= \Eb{\epol(A\mid X)q_0(Z, A, X)\prns{Y - h_0(W, A, X)}S(W, Y\mid Z, A, X)}  \\
        &= \Eb{\epol(A\mid X)q_0(Z, A, X)\prns{Y - h_0(W, A, X)}S(W, Y\mid Z, A, X)}  \\
        &+ \Eb{\epol(A\mid X)q_0(Z, A, X)\prns{Y - h_0(W, A, X)}S(Z, A, X)}  \\
        &= \Eb{\epol(A\mid X)q_0(Z, A, X)\prns{Y - h_0(W, A, X)}S(W, Y, Z, A, X)}.
\end{align*}
Here the second equality follows from \Cref{eq: h-derivative}, and the third equality follows from the fact that $$\Eb{Y - h_0(W, A, X) \mid Z, A, X} = 0.$$

Moreover,
\begin{align*}
    \Eb{\mathcal T h_0(W, X)S(W, X)}
        &= \Eb{\mathcal T h_0(W, X)S(W, X)} + \Eb{\mathcal T h_0(W, X)S(Y, Z, A \mid W, X)} \\
        &= \Eb{\mathcal T h_0(W, X)S(Y, W, Z, A, X)} \\
        &= \Eb{\prns{\mathcal T h_0(W, X) - J}S(Y, W, Z, A, X)}.
\end{align*}
Here the first equality follows from the fact that $\Eb{S(Y, Z, A \mid W, X) \mid W, X} = 0$, and the last equality follows from $\Eb{S(Y, Z, A, W, X) \mid W, X} = 0$.

Therefore, we have that
\begin{align*}
    \partial_t J_t \vert_{t = 0} = \Eb{\prns{\epol(A\mid X) q_0(Z, A, X)[Y-h_0(W, A, X)]+\mathcal Th_0(W, X) - J}S(Y, W, Z, A, X)}.
\end{align*}
This means that $\epol(A\mid X) q_0(Z, A, X)[Y-h_0(W, A, X)]+\mathcal Th_0(W, X) - J$ is a valid influence function for $J$ under the model $\mathcal M_{np}$.

\textbf{Step III: verifying efficient influence function.}
Now we verify that $\epol(A\mid X) q_0(Z, A, X)[Y-h_0(W, A, X)]+\mathcal Th_0(W, X) - J$ also belongs to $\text{Closure}\prns{\mathcal S}$ so that it is also the efficient influence function for $J$ relative to the tangent space $\text{Closure}\prns{\mathcal S}$.

First, note that we can decompose this influence function in the following way:
\begin{align*}
    \tilde{S}(Y, W, Z, A, X)
        &\coloneqq \epol(A\mid X) q_0(Z, A, X)[Y-h_0(W, A, X)]+\mathcal Th_0(W, X) - J \\
        &= \tilde{S}(Z, A, X) + \tilde{S}(Y, W \mid Z, A, X),
\end{align*}
where
\begin{align*}
    \tilde{S}(Z, A, X)
        &= \Eb{\mathcal Th_0(W, X) - J \mid Z, A, X}, \\
    \tilde{S}(Y, W \mid Z, A, X)
        &= \mathcal Th_0(W, X) - \Eb{\mathcal Th_0(W, X)\mid Z, A, X} + \epol(A\mid X) q_0(Z, A, X)[Y-h_0(W, A, X)].
\end{align*}
It is easy to show that  $\Eb{\tilde{S}(Z, A, X)} = 0$ and $\Eb{\tilde{S}(Y, W \mid Z, A, X) \mid Z, A, X} = 0$. Since $P_z$ is bijective, we have that
\begin{align*}
    \Eb{\prns{Y - h_0(W, A, X)}\tilde{S}(Y, W \mid Z, A, X)\mid Z, A, X} \in \text{Closure}\prns{\text{Range}\prns{P_z}}.
\end{align*}

Therefore,
$$
\mathrm{EIF}(J) = \epol(A\mid X) q_0(Z, A, X)[Y-h_0(W, A, X)]+\mathcal Th_0(W, X) - J \in \text{Closure}\prns{\mathcal S},
$$
which means that
$\mathrm{EIF}(J)$
 is the efficient influence function of $J$  and $\mathbb{E}\left[\mathrm{EIF}^{2}(J)\right]$ is
the corresponding semiparametric efficiency bound.
\end{proof}

\subsection{Proofs for \Cref{sec: minimax-stab}}
\begin{proof}[Proof for \Cref{lemma: linear-est-2}]
The conclusion follows easily by simple algebra.
\end{proof}

\begin{proof}[Proof of \Cref{lem:kernel3}]
From the representer theorem, an  solution of the inner maximization in \Cref{eq: est-rkhs-2-h-minimax}  should be $q^{*}(\cdot)=\sum_i \alpha_i k_z((Z_i,A_i,X_i),\cdot )$. Thus, this inner maximization problem can be reduced to  solving
\begin{align*}
    &\max_{\alpha\in \mathbb{R}^n}~\psi^{\top}_nK_{z,n}\alpha-\alpha^{\top}(\lambda K_{z,n}+\gamma_1 I)K_{z,n}\alpha \\
    =& \max_{\tilde\alpha\in \mathbb{R}^n} \psi^{\top}_nK^{1/2}_{z,n}\tilde\alpha - \tilde\alpha^\top (\lambda K_{z,n}+\gamma_1 I)\tilde\alpha
     \tag{$\tilde \alpha = K_{z,n}^{1/2}\alpha$} \\
     =& \frac{1}{4}\psi^{\top}_nK_{z,n}^{1/2}(\lambda K_{z,n}+\gamma_1 I)^{-1}K_{z,n}^{1/2}\psi_n,
\end{align*}
where the last maximum is achieved by
\begin{align*}
    \tilde \alpha^*_q =\frac{1}{2}(\lambda K_{z,n}+\gamma_1 I)^{-1}\psi_n.
\end{align*}

From the representer theorem, a solution of the inner maximization problem in \Cref{eq: est-rkhs-2-q-minimax} should be $h^{*}(\cdot)=\sum_i \alpha_i k_w((X_i,A_i,W_i),\cdot )$. Thus, this inner maximization problem can be reduced to solving
\begin{align*}
  \max_{\alpha \in \mathbb{R}^n}\phi^{\top}_n K_{w1,n}\alpha-\mathbf{1}_n^\top K_{w2,n}\alpha-\alpha^{\top}(\lambda K_{w1,n}+\gamma_2 I)K_{w1,n}\alpha.
\end{align*}
Assuming $K_{w1,n}$ is a positive definite matrix, the optimization problem above is solved by
\begin{align*}
   \tilde \alpha^*_{h}=\frac{1}{2} (\lambda K_{w1,n}+\gamma I)^{-1}(\phi_n-K^{-1}_{w1,n}K_{w2,n}\mathbf{1}_n).
\end{align*}
The corresponding optimal value is
\begin{align*}
    \frac{1}{4} \phi^{\top}_n K^{1/2}_{w1,n}(\lambda K_{w1,n}+\gamma I)^{-1}K^{1/2}_{w1,n}\phi_n-\frac{1}{2}\{\phi_n^{\top}(\lambda K_{w1,n}+\gamma I)^{-1} K_{w2,n}\mathbf{1}_n \}.
\end{align*}
\end{proof}

\subsection{Proofs for \Cref{sec: single-realizability}}
We prove the second statement. The first statement is similarly proved.

We only need to prove
\begin{align*}
    \inf_{h \in \Hcal'}\sup_{q\in \Qcal}|\E[(q-q_0)\pi(h-h_0)]|=0.
\end{align*}
Here, letting $\langle \alpha_1 , \phi(\cdot) \rangle =(q-q_0)$ and $h(\cdot)=\langle \beta_h,\psi \rangle$,
\begin{align*}
    \E[(q-q_0)\pi(h-h_0)]=\alpha^{\top}_1\E[\tilde \phi(Z,A,X)\{\psi^{\top}(W,X,A)\beta_h-h_0\} ].
\end{align*}
Hence, there exists $\beta_h$ such that
\begin{align*}
    0=\alpha^{\top}_1\E[\tilde \phi(Z,A,X)\{\psi^{\top}(W,X,A)\beta_h-h_0\} ]
\end{align*}
when $ \E[\tilde \phi(Z,A,X)\psi^{\top}(W,X,A)]$ is full column rank. This shows
\begin{align*}
    \inf_{h \in \Hcal'}\sup_{q\in \Qcal}|\E[(q-q_0)\pi(h-h_0)]|=0.
\end{align*}

\subsection{Proofs for \Cref{sec: sieve-nn}}

\begin{proof}[Proof of \Cref{cor:linear_sieves1}]
We prove the result for $\hat J_{\ipw}$. The result for $\hat J_{\dm}$ can be proved analogously so is omitted here.

Fix some element in $h_0\in \Hbbb^{\obs}_0 , \pi q_0\in \epol\Qbbb^{\obs}_0$. Then, from \Cref{thm:ipw_reg_no_stab}, with $1-\delta$, we have the following for a universal constant $c > 0$:
\begin{align*}
    |\hat J_{\ipw}-J|\leq c\bigg\{
        &\Rad(\infty;\Qbbb)+ \Rad(\infty;\Hbbb)+\Rad(\infty;\Tcal \Hbbb)+ \\
        &\inf_{h\in \Hbbb'} \sup_{q\in \Qbbb}|\E[(q-q_0)\epol(h_0-h)]|+\inf_{q\in \Qbbb} \sup_{h\in \Hbbb'}|\E[(q_0-q)\epol h ]|  +\sqrt{\frac{\log(1/\delta)}{n}}\bigg\}.
\end{align*}
Thee first three terms are upper bounded by $O(\sqrt{k_n/n})$ since the VC dimension of $\Qbbb = \mathcal{S}_{2,n}$ and $\Hbbb' = \mathcal{S}_{1,n}$ are both  $k_n$. Besides, from the assumptions \eqref{eq: sieve-error1} and \eqref{eq: sieve-error2},
\begin{align*}
  & \inf_{h\in \Hbbb'} \sup_{q\in \Qbbb}|\E[(q-q_0)\epol(h_0-h)]|\leq   2\|\mathcal{S}_{2, n}\|_{\infty}\|\pi\prns{A\mid X}\|_{\infty}\inf_{h\in \Hbbb'}\E[|h_0-h|]=O(k^{-\alpha/d}_n), \\
    &\inf_{q\in \Qbbb} \sup_{h\in \Hbbb'}|\E[\epol(q_0-q) h ]| \leq 2\|\mathcal{S}_{1, n}\|_{\infty}\|\pi\prns{A\mid X}\|_{\infty}\inf_{q\in \Qbbb'}\E[|q-q_0|]=O(k^{-\alpha/d}_n).
\end{align*}
In the end, the final error becomes
\begin{align*}
    O(k^{-\alpha/d}_n+\sqrt{k_n/n}+\sqrt{\log(1/\delta)/n}),
\end{align*}
where the second and third term are statistical error terms, which is derived in the proof of \Cref{cor: vc_class}. By setting $k_n = O(n^{\frac{d}{d + 2\alpha}})$, we obtain the optimal rate $O(n^{-\alpha/(2\alpha+d)} + \sqrt{\log(1/\delta)/n})$.
\end{proof}

\begin{proof}[Proof of \Cref{cor:linear_sieves2}]
According to \Cref{thm:w_function}, for some $h'\in \Hbbb$  to be specified later and $h_0\in \Hbbb^{\obs}_0\cap \Lambda^{\alpha}([0,1]^d)$, we have that with probability $1- \delta$,
\begin{align*}
    \|P_z(\hat h-h_0)\|_2=O\left(\eta'_{h}+\epsilon_n+\frac{\|P_z(h'-h_0)\|^2_2}{\eta'_{h}}+    \|P_z(h'-h_0)\|_2\right),
\end{align*}
where $\eta'_{h}=\eta_{h}+c_0\sqrt{\log(c_1/\delta)/n}$, and $\eta_{h}$ is the maximum of critical radii of $\Qbbb' = \mathcal{S}_{2, n}$ and $\Gcal_{h}$, and $\epsilon_n=\sup_{h\in \Hbbb}\inf_{q\in \Qbbb'}\|q-P_z(h-h')\|_2.$

First, from \Cref{cor: vc_nonpara}, we have
$$\eta'_{h}=O\prns{\sqrt{k_n}\frac{\log n}{\sqrt{n}}\log(\sqrt{k_n}\frac{\log n}{\sqrt{n}})+\sqrt{ (1+\log(1/\delta))/n} }.$$
Next, according to \Cref{eq: sieve-error1}, we can find $h'$ such that $\|h_0-h'\|=O(k^{-\alpha/d}_n)$. It follows that
\begin{align*}
    \|P_z(h'-h_0)\|\leq  \|h'-h_0\|=O(k^{-\alpha/d}_n).
\end{align*}
Besides, $P_z(\Lambda^{\alpha}([0,1]^d)-h_0)\subset \Lambda^{\alpha}([0,1]^d)$ implies that $P_z(h-h_0) \in \Lambda^{\alpha}([0,1]^d)$ for any $h \in \Hbbb$. According to \Cref{eq: sieve-error2}, for any $h \in \Hbbb$, we can find $q\in \mathbb{Q}'$ such that $\|q- P_z(h-h_0)\| = O(k^{-\alpha/d}_n)$ and thus
\begin{align*}
   \|q- P_z(h-h')\|&\leq   \|q- P_z(h-h_0)\|+    \|P_z(h'-h_0)\|= O(k^{-\alpha/d}_n).
\end{align*}
This implies that $\epsilon_n = O(k^{-\alpha/d}_n)$.

Combining all terns above, we have
\begin{align*}
      \|P_w(h_0-\hat h)\|_2=\tilde O\prns{\sqrt{k_n/n}+k^{-\alpha/d}_n+\frac{ k^{-2\alpha/d}_n}{\sqrt{k_n/n}} + \sqrt{ (1+\log(1/\delta))/n}  }.
\end{align*}
By setting $k_n = O(n^{\frac{d}{d + 2\alpha}})$, we can obtain the optimal rate $\tilde O(n^{-\alpha/(2\alpha+d)})$.
\end{proof}

\begin{proof}[Proof of \Cref{cor:neural_realiza}]
We prove the result for $\hat J_{\ipw}$. The result for $\hat J_{\dm}$ can be proved analogously so is omitted here.

Fix some element in $h_0\in \Hbbb^{\obs}_0 , \pi q_0\in \epol\Qbbb^{\obs}_0$. Then, from \Cref{thm:ipw_reg_no_stab}, with $1-\delta$, we have the following for a universal constant $c > 0$:
\begin{align*}
    |\hat J_{\ipw}-J|\leq c\bigg\{
        &\Rad(\infty;\Qbbb)+ \Rad(\infty;\Hbbb)+\Rad(\infty;\Tcal \Hbbb)+ \\
        &\inf_{h\in \Hbbb'} \sup_{q\in \Qbbb}|\E[(q-q_0)\epol(h_0-h)]|+\inf_{q\in \Qbbb} \sup_{h\in \Hbbb'}|\E[(q_0-q)\epol h ]|  +\sqrt{\frac{\log(1/\delta)}{n}}\bigg\}.
\end{align*}

By \Cref{cor: nonpara,lem:covering_neural}, we can compute the covering number of neural networks and upper bound the first three terms by $\tilde O(\sqrt{\Omega L/n})$.

According to \Cref{lem:yaro},
\begin{align*}
  & \inf_{h\in \Hbbb'} \sup_{q\in \Qbbb}|\E[(q-q_0)\epol(h_0-h)]|\leq   2\|\mathcal{F}_{2, n}\|_{\infty}\|\pi\prns{A\mid X}\|_{\infty}\inf_{h\in \Hbbb'}\E[|h_0-h|]=\tilde O(\Omega^{-\alpha/d}), \\
    &\inf_{q\in \Qbbb} \sup_{h\in \Hbbb'}|\E[\epol(q_0-q) h ]| \leq 2\|\mathcal{F}_{1, n}\|_{\infty}\|\pi\prns{A\mid X}\|_{\infty}\inf_{q\in \Qbbb'}\E[|q-q_0|]=\tilde O(\Omega^{-\alpha/d}).
\end{align*}
Combining the results above, the total error can be bounded by
\begin{align*}
    \tilde O(  \sqrt{\Omega L/n} +\Omega^{-\alpha/d}+\sqrt{\log(1/\delta)/n}).
\end{align*}
By setting $L=\Theta(\log(n)),\Omega=\Theta(n^{d/(2\alpha+d)})$, we can obtain the optimal rate as $\tilde O(n^{-\alpha/(2\alpha+d})$.
\end{proof}

\begin{proof}[Proof of \Cref{cor:linear_sieves3}]
According to \Cref{thm:w_function}, for some $h'\in \Hbbb$  to be specified later and $h_0\in \Hbbb^{\obs}_0\cap \Lambda^{\alpha}([0,1]^d)$, we have that with probability $1- \delta$,
\begin{align*}
    \|P_z(\hat h-h_0)\|_2=O\left(\eta'_{h}+\epsilon_n+\frac{\|P_z(h'-h_0)\|^2_2}{\eta'_{h}}+    \|P_z(h'-h_0)\|_2\right),
\end{align*}
where $\eta'_{h}=\eta_{h}+c_0\sqrt{\log(c_1/\delta)/n}$, and $\eta_{h}$ is the maximum of critical radii of $\Qbbb' = \mathcal{F}_{2, n}$ and $\Gcal_{h}$, and $\epsilon_n=\sup_{h\in \Hbbb}\inf_{q\in \Qbbb'}\|q-P_z(h-h')\|_2.$

First, from \Cref{cor: vc_nonpara,lem:covering_neural}, we have $$\eta_{h}'=\tilde O(\sqrt{\Omega L/n} + c_0\sqrt{\log(c_1/\delta)/n}).$$

Next, according to \Cref{lem:yaro}, we can find $h'$ such that $\|h_0-h'\|_{\infty}=\tilde O(\Omega^{-\alpha/d})$. It follows that
\begin{align*}
    \|P_z(h'-h_0)\|\leq  \|h'-h_0\|=\tilde O(\Omega^{-\alpha/d}).
\end{align*}
Besides, $P_z(\mathcal{S}^{\alpha}([0,1]^d)-h_0)\subset \mathcal{S}^{\alpha}([0,1]^d)$ implies that $P_z(h-h_0) \in \mathcal{S}^{\alpha}[0,1]^d)$ for any $h \in \Hbbb$. According to \Cref{lem:yaro}, for any $h \in \Hbbb$, we can find $q\in \mathbb{Q}'$ such that $\|q- P_z(h-h_0)\| = \tilde O(\Omega^{-\alpha/d})$ and thus
\begin{align*}
   \|q- P_z(h-h')\|&\leq   \|q- P_z(h-h_0)\|+    \|P_z(h'-h_0)\|= \tilde O(\Omega^{-\alpha/d}).
\end{align*}
This implies that $\epsilon_n = \tilde O(\Omega^{-\alpha/d})$.

Combining all terns above, we have
\begin{align*}
      \|P_w(h_0-\hat h)\|_2=\tilde O\prns{\sqrt{\Omega L/n}+ c_0\sqrt{\log(c_1/\delta)/n}+\Omega^{-\alpha/d}+\frac{ \Omega^{-2\alpha/d}}{\sqrt{\Omega L/n}}}.
\end{align*}
By setting $L=\Theta(\log(n)),\Omega=\Theta(n^{d/(2\alpha+d)})$, we can obtain the optimal rate as $\tilde O(n^{-\alpha/(2\alpha+d})$.

\end{proof}

\end{document}